\documentclass[11pt]{article}

\usepackage[protrusion=false]{microtype} 

\usepackage{setspace}
\usepackage{comment}
\usepackage{caption}
\usepackage{subcaption}
\usepackage{float} 
\usepackage{authblk}
\usepackage[left=1in,right=1in,top=.9in,bottom=1in]{geometry}
\usepackage{amsmath}
\usepackage{graphicx,psfrag,epsf}
\usepackage[dvipsnames]{xcolor}
\definecolor{Bleu}{RGB}{0,0,204}
\usepackage[hypertexnames=false]{hyperref}
\hypersetup{
colorlinks,
  citecolor=Bleu,
  linkcolor=Bleu,
  urlcolor=Bleu}

\newcommand\smallO{
  \mathchoice
    {{\scriptstyle\mathcal{O}}}
    {{\scriptstyle\mathcal{O}}}
    {{\scriptscriptstyle\mathcal{O}}}
    {\scalebox{.7}{$\scriptscriptstyle\mathcal{O}$}}
  }

\usepackage{enumerate}
\usepackage{pifont}

\usepackage{url} 
\usepackage{amssymb}
\newcommand{\norm}[1]{\left\lVert#1\right\rVert}
\DeclareMathOperator*{\argmin}{argmin}

\usepackage{tikz}

\usepackage{helvet}
\usepackage{url}
\usepackage{amsmath,amssymb, bm}

\usepackage{algorithm}
\usepackage{algorithmic}
\usepackage{enumitem}

\usepackage{amsthm}

\usepackage{tikz}

\usepackage{amsthm}
\usepackage{thmtools}

\makeatletter
\declaretheoremstyle[
  headfont=\bfseries,
  headindent=-0.25em,
  postheadspace=0.5em,
  notefont=\bfseries, 
  notebraces={}{},
  headformat=\NOTE\thmt@space\NUMBER,
  bodyfont=\mdseries,
  spaceabove=12pt,spacebelow=12pt,
]{namedthmstyle}
\makeatother

\usepackage{mathrsfs}

\usepackage{natbib}
\setcitestyle{authoryear,open={(},close={)}}

\usepackage{amsthm}{

   \newtheorem{theorem}{Theorem}
 \newtheorem{lemma}{Lemma}

}

\newcommand{\miid}{\,|\,}

\makeatletter
\declaretheoremstyle[
  headfont=\bfseries,
  headindent=-0.25em,
  postheadspace=0.5em,
  notefont=\bfseries, 
  notebraces={}{},
   bodyfont=\normalfont\itshape,
  headformat=\NAME\NUMBER\NOTE\thmt,
  spaceabove=12pt,spacebelow=12pt,
]{nospacetheorem}
\makeatother

\declaretheoremstyle[
  bodyfont=\normalfont\itshape,
  headformat=\NAME\NUMBER  
]{tmp}

\declaretheorem[style=definition]{example}

\usepackage{thmtools} 
\renewcommand\thmcontinues[1]{continued}
\usepackage{chngcntr}
\newcounter{parentnumber}


\usepackage{cleveref}
\crefname{example}{Example}{Examples} 

\usepackage{thmtools} 
\renewcommand\thmcontinues[1]{continued}

\allowdisplaybreaks

\begin{document}




\title{Combining T-learner and DR-learning:\\
a framework for oracle-efficient estimation of causal contrasts
}
 
\date{\today}
\author[1]{Lars van der Laan, Marco Carone, Alex Luedtke}


\maketitle

\begin{abstract}
 
\onehalfspacing
We introduce efficient plug-in (EP) learning, a novel framework for the estimation of heterogeneous causal contrasts, such as the conditional average treatment effect and conditional relative risk. The EP-learning framework enjoys the same oracle efficiency as Neyman-orthogonal learning strategies, such as DR-learning and R-learning, while addressing some of their primary drawbacks: (i) their practical applicability can be hindered by non-convex loss functions; and (ii) they may suffer from poor performance and instability due to inverse probability weighting and pseudo-outcomes that violate bounds. To overcome these issues, the EP-learner leverages an efficient plug-in estimator of the population risk function for the causal contrast. In doing so, it inherits the stability of plug-in strategies such as T-learning, while improving on their efficiency. Under reasonable conditions, EP-learners based on empirical risk minimization are oracle-efficient, exhibiting asymptotic equivalence to the minimizer of an oracle-efficient one-step debiased estimator of the population risk function. In simulation experiments, we show that EP-learners of the conditional average treatment effect and conditional relative risk outperform state-of-the-art competitors, including the T-learner, R-learner, and DR-learner. Open-source implementations of the proposed methods are available in our \texttt{R} package \texttt{hte3}.


\end{abstract}

\section{Introduction}

\subsection{Background}
In recent years, there has been a surge in interest in the development of tools designed for the flexible estimation of causally interpretable functions through the use of machine learning techniques. Many of these functions constitute subgroup-specific contrasts possibly describing heterogeneous treatment effects. Notable examples of such functions include the conditional average treatment effect (CATE) \citep{Robins2004NestedModels, KennedyCATERates}, conditional relative risk (CRR) \citep{van2007estimation, richardson2017modeling, LuedtkeEntropyDiscussion}, and causal conditional log-odds ratio \citep{vansteelandt2003causal, robins2004estimation, TchetgenOddsRatio}. Under causal conditions, these functions can be expressed in terms of quantities estimable from randomly sampled data; often, these quantities involve outcome regressions, that is, conditional means of the outcome given certain baseline covariates and treatment levels.

Plug-in estimation, also known as T-learning \citep{Tlearner}, offers a straightforward approach for estimating heterogeneous causal contrasts. In this approach, outcome regressions are estimated and then directly substituted into the contrast to compute the causal quantity of interest. However, a limitation of T-learners is that their performance is heavily dependent on the quality of the outcome regression estimators. This sensitivity is undesirable as the form of the outcome regressions can be significantly more complex---and therefore more difficult to estimate---than that of the causal contrast of interest. For example, the outcome regressions can be highly non-smooth, even if the conditional average treatment effect is constant. Consequently, there is a need for developing estimation methodologies that leverage the parsimony of the contrast function when feasible, while maintaining sufficient flexibility to accommodate complex functional forms \citep{superLearnOptRule,KennedyCATERates,QuasiOracleWager}.

 As a means of estimating a causal contrast, it is natural to first identify a risk function for this particular contrast. When an estimator for such risk function is available, contrasts can be estimated using flexible supervised learning tools such as regularized empirical risk minimizers \citep{van2000empirical}, random forests \citep{breiman2001random}, and gradient boosting \citep{BoostingFruend}, among others. However, estimating a risk function can be challenging, particularly when the loss function depends on unknown nuisance functions. For instance, the DR-learner loss \citep{TMLEOptTrt, superLearnOptRule, KennedyCATERates} and R-learner loss \citep{QuasiOracleWager, atheywager2014causalforest} for CATE estimation as well as the E-learner loss \citep{JiangEntropy} for CRR estimation depend on the outcome regression and propensity score. 
 When a nonparametric efficient estimator of the resulting risk is used, the empirical risk minimizer based on estimated nuisance functions often estimates the causal contrast as accurately as the (oracle) empirical risk minimizer using the true nuisance functions \citep{atheywager2014causalforest, orthogonalLearning,KennedyCATERates, van2023causal, yang2023forster}. In parametric settings, an empirical risk minimizer based on an efficient risk estimator is typically efficient for the causal contrast under certain regularity conditions \citep{VaartEffFunctionEff, mcclean2022nonparametric, van2023adaptive}---this motivates the use of an efficient risk estimator in such problems. In nonparametric settings, to our knowledge, such a general statement is not currently available. However, the advantages of using an efficient risk estimator have been theoretically established in the case of certain causal functions (Theorem 3 of \citealp{conttrtKennedy}; Theorems 1 and 5 of \citealp{atheyWagerPolicyLearning}) and validated numerically for estimating the CATE \citep{KennedyCATERates, van2023causal}.

Many efficient estimators of commonly used causal contrast risks are based on loss functions that are Neyman-orthogonal \citep{RobinsRotnitzkyZhao1995, VanderLaanRobins2003, DoubleML,  orthogonalLearning, curth2020estimating, yang2023forster}, including those leveraged by DR-learner and R-learner \citep{orthogonalLearning}. In most instances in the literature, a Neyman-orthogonal loss is derived, either explicitly or implicitly, by debiasing an initial plug-in risk estimator using the one-step estimation methodology \citep{pfanzagl1985contributions, bickel1993efficient}. However, there are two substantial drawbacks associated with efficient risk estimators built upon Neyman-orthogonal loss functions. First, Neyman-orthogonal loss functions can be nonconvex, even when they are derived by orthogonalizing a convex loss function. For instance, the Neyman-orthogonal losses associated with the E-learning risk for the CRR \citep{JiangEntropy, LuedtkeEntropyDiscussion} and the bound-enforcing risk for the CATE \citep{superLearnOptRule} are nonconvex. Specifically, they can be expressed as weighted logistic regression loss functions with weights that may take negative values. This is especially problematic as numerous supervised learning implementations, including widely used software for fitting generalized additive models \citep{mgcv}, random forests \citep{ranger}, and gradient-boosted regression trees \citep{xgboost}, necessitate nonnegative weights. Second, in certain scenarios, empirical risk minimizers based on Neyman-orthogonal loss functions may yield extreme and unrealistic values for the causal contrast of interest. For example, when the outcome is binary, DR-learners \citep{TMLEOptTrt, KennedyCATERates} may output treatment effect estimates residing outside of $[-1,1]$, which is undesirable since the true CATE is confined within this interval. This failure to respect known bounds arises because fitting DR-learner involves regressing a pseudo-outcome on the covariates of interest, and this pseudo-outcome can take extreme values since it involves inverse-weighting by the estimated propensity score.

The inherent limitations of efficient risk estimators based on Neyman-orthogonal losses have been recognized in the literature, and several strategies have been proposed to address them. As these efficient risk estimators can yield nonconvex loss functions, researchers may opt for risk estimators derived from convex loss functions, even despite their inefficiency. Common examples include inverse-probability-weighted \citep{luedtke2016super,JiangEntropy, chen2017general} or plug-in \citep{curth2021nonparametric} risk estimators, which often fail to achieve parametric-rate consistency when the nuisance functions are estimated flexibly. \textcolor{black}{To avoid instability from inverse propensity weighting, overlap weighting can be used to down-weight observations in the tails of the propensity-score distribution \citep{overlapweights, vansteelandt2020assumption, morzywolek2023general}. For example, IV-DR-learner uses a DR-learner loss with estimated overlap weights \citep{fisher2023connection}, and R-learner uses an orthogonal loss for the overlap weighted risk (Corollary~9.1 of \citealp{robins2004optimal}; Section~5.2 of \citealp{robins2008higher};  \citealp{QuasiOracleWager}). The key to the improved stability of these methods is that they more heavily weight regions of stronger overlap and underweight others. When the CATE is sufficiently smooth to allow reliable extrapolation from regions of greater overlap, such weighting can improve unweighted risk performance \citep{KennedyCATERates, KennedyMiximax, fisher2023connection}. 
}

\subsection{Our contributions}

We provide a method for estimating causal contrasts based on a novel efficient plug-in risk estimator. 
Rather than focusing on a single risk function for a particular contrast, such as the CATE, we study a broad class of risk functions that includes this well-studied case as well as others, such as the CRR function. 
\textcolor{black}{Like T-learners, EP-learners rely on the plug-in principle. However, unlike T-learners,  they do this by computing a constrained minimizer of an efficient plug-in risk estimator. As such, EP-learners achieve the fast rates and oracle efficiency enjoyed by Neyman-orthogonal learning strategies.} Our main contributions are as follows:
\begin{enumerate}[label=(\roman*)]
    \item we introduce the EP-learner for estimating causal contrasts, which is based on a novel efficient plug-in (EP) risk estimator;

    \item we establish that, under reasonable conditions, the EP-learner is oracle-efficient, being asymptotically equivalent to the minimizer of an oracle-efficient one-step risk estimator;
 
    \item we introduce EP-learners for the CATE and CRR functions, both of which are doubly robust.
\end{enumerate}

\section{Problem setup}
\label{section::setup3} 

\label{section::setup}

\subsection{Data structure and notation}

Suppose that we have at our disposal a sample of $n$ independent and identically distributed observations, $O_1,O_2, \dots, O_n$, of the data structure $O = (W,A,Y)$ drawn from a probability distribution $P_0$ belonging to a nonparametric statistical model $\mathcal{M}$. In this data structure, $W \in \mathcal{W} \subset \mathbb{R}^d$ is a vector of baseline covariates, $A \in \mathcal{A} \subset \mathbb{R}$ is a possibly continuous treatment assignment, and $Y \in \mathbb{R}$ is a bounded real-valued outcome. These observations may arise from an observational study or a randomized controlled trial. We let $\{Y^a: a \in \mathcal{A}\}$ be the set of potential outcomes associated with the observed data structure $(W,A,Y)$ \citep{Rubin2005}, where $Y^a$ is the outcome that would have been observed if, possibly contrary to fact, treatment $a \in \mathcal{A}$ had been administered. 

 For a given distribution $P \in \mathcal{M}$ and a generic realization $(a,w)$ of $(A,W)$, we denote the outcome regression by $\mu_P(a,w):= E_P(Y \miid A = a , W = w)$ and the propensity score by $\pi_P(a\miid w) := P(A=a\miid W =w)$. Further, we denote $\mu^a(w) :=\mathbb{E}(Y^a \miid W=w)$, where $\mathbb{E}$ denotes the expectation under the joint distribution of the covariates $W$ and the potential outcomes. Throughout, we let $E_0^n$ denote the expectation with respect to the product measure $P_0^n$ from which $(O_1,O_2,\ldots,O_n)$ is drawn. Further, we denote by $P_{0,W}$ the marginal distribution of $W$ under $P_0$, and by $\norm{f}$ and $\norm{f}_{\infty}$ the $L^2(P_0)$ and $P_0$--essential supremum norm, respectively, of a given function $f \in L^2(P_0)$. For notational convenience, we denote the set $[m] := \{1,2, \dots, m\}$ for $m \in \mathbb{N}$ and 
 write $\mathcal{S}_0$ to denote any summary $\mathcal{S}_{P_0}$ of the true distribution $P_0$.

\subsection{Statistical goal}
\label{section::riskfunctions}

\textcolor{black}{Our goal is to estimate a causal summary, such as the CATE, or its projection onto a convex action space \(\mathcal{F}\). Specifically, we consider \(\theta_0 = \argmin_{\theta \in \overline{\mathcal{F}}} R_{0}(\theta)\), the minimizer of a population risk function \(R_0\) over the \(L^2(P_{0,W})\)–closure \(\overline{\mathcal{F}}\) of \(\mathcal{F}\). The function class \(\mathcal{F}\) is user-specified and may be infinite-dimensional, as is the case for a reproducing kernel Hilbert space or H\"older class. If the causal summary lies in \(\mathcal{F}\), then \(\theta_0\) coincides with the summary itself, with \(\mathcal{F}\) encoding smoothness assumptions such as H\"older continuity. Otherwise, \(\mathcal{F}\) serves as a working model that provides a useful approximation.}

In this work, we restrict our attention to risk functions of the form $R_0 := R_{P_0}$, where $R_{P}(\theta) := E_P[L_{\mu_P}(W,\theta)]$ with loss function 
 \begin{equation}
L_{\mu_P}(w, \theta) := h_1 (\theta(w)) \sum_{s \in \mathcal{A}} c_{s,1} g_1 (\mu_P(s,w)) + h_2(\theta(w)) \sum_{s \in \mathcal{A}} c_{s,2} g_2(\mu_P(s,w))\label{eqn::GeneralClassRisk}
 \end{equation}
for arbitrary functions $g_1,g_2: \mathbb{R} \rightarrow \mathbb{R}$ with a Lipschitz derivative, twice-differentiable Lipschitz functions $h_1,h_2: \mathbb{R} \rightarrow \mathbb{R}$ with a continuous second derivative, and known constants $c_{a,1},c_{a,2}$ for $a\in\mathcal{A}$. \textcolor{black}{We introduce the general loss structure in \eqref{eqn::GeneralClassRisk} to encompass a wide range of causal summaries. Notably, any summary of the form $\theta_0 : w\mapsto \sum_{a\in\mathcal{A}} c_a f(\mu_a(w))$ with constants $c_a\in\mathbb{R}$  and fixed function $f:\mathbb{R}\rightarrow\mathbb{R}$ can be expressed in this manner. The particular form of the loss is not central to our method---beyond mild smoothness conditions, EP-learners can be constructed much more generally. In many applications, the loss simplifies considerably; for instance, in common cases $g_1$ and $g_2$ are simply the identity.}
The CATE and CRR functions are notable examples of such summaries.

 \begin{example}[label=ex1, name=conditional average treatment effect] 
 The CATE $\theta_0^{-}:w \mapsto \mu_0(1,w) - \mu_0(0,w)$ minimizes the risk $\theta\mapsto R_0^{-}(\theta)$ over $L^2(P_{0,W})$, where we define
 \begin{equation}
R_P^{-}(\theta) := E_P\left[\theta(W)^2 - 2\theta(W)\left\{\mu_P(1,W) - \mu_P(0,W) \right\} \right].\label{eqn::CATERisk}
\end{equation}
This risk function is obtained by taking $\mathcal{A} = \{0,1\}$, $h_1(\theta) = \theta^2$, $h_2(\theta) = -2\theta $,  $g_1 \circ \mu = 1$, $g_2 \circ \mu = \mu$, $c_{1,1} = 1, \, c_{0, 1} = 0$, $c_{1,2} = 1$ and $c_{0,2} = -1$. We note that this same risk function can be used to learn any $V$-specific CATE function $v\mapsto E_0[\mu_0(1,W) - \mu_0(0,W)\,|\,V=v]$ for a coarsened covariate vector $V := f(W)$ with  $f:\mathcal{W} \rightarrow \mathcal{V}\subseteq \mathbb{R}^{d_0}$ and $d_0\leq d$. For example, $V$ may represent a subset of components of $W$. This is achieved by minimizing $\theta\mapsto R_0^-(\theta)$ over $L^2(P_{0,V})$, where $P_{0,V}$ denotes the distribution of $f(W)$ under sampling from $P_0$, rather than $L^2(P_{0,W})$ \citep{morzywolek2023general}.
\label{Example::CATE}
\end{example}

 \begin{example}[label=ex2, name=conditional relative risk]
When the outcome $Y$ is binary or nonnegative, the log-CRR function $\theta_0^\div : w\mapsto \log \mu_0(1,w)-\log \mu_0(0,w)$  minimizes the risk $\theta\mapsto R_0^{\div}(\theta)$ over $L^2(P_{0,W})$ \citep{JiangEntropy, LuedtkeEntropyDiscussion}, where we define
\begin{equation}
    R_P^{\div}(\theta) := E_P\left[\left\{\mu_P(1,W) + \mu_P(0,W)\right\}  \log(1 + \exp \{\theta(W)\}) - \mu_P(1,W)\theta(W)  \right]. \label{eqn::popriskRR}
\end{equation}
This risk function is obtained by taking $\mathcal{A} = \{0,1\}$, $h_1(\theta) = \log (1 + \exp (\theta))$, $h_2(\theta) = -\theta$, $g_1(\mu) = g_2(\mu) =\mu$, $c_{0,1} = c_{1,1}= 1$, $c_{0,2} = 0$ and $c_{1,2} = 1$. Similarly as in Example \ref{Example::CATE}, for a coarsened covariate vector $V := f(W)$, this risk function can also be used to learn the $V$-specific CRR function $v\mapsto \log E_0[\mu_0(1,W) \,|\, V=v] - \log E_0[\mu_0(0,W) \,|\, V=v]$. A related exponential loss was also proposed by \cite{chen2017general}.

\label{Example::CRR}
\end{example}

 We assume that the loss function giving rise to $R_0$ is $\gamma$--strongly convex \citep[Equation 14.42 of][]{wainwright_2019}---in Lemma~\ref{lemma::uniquePopMinimizer2}, we show that this condition suffices for the existence and uniqueness of $\theta_0$. In Lemma~\ref{lemma::uniquePopMinimizer1}, we show that this strong convexity holds whenever there exists a constant $\gamma>0$ such that $\sum_{a \in \mathcal{A}} c_{a,1} g_1(\mu_0(a,W))\not=0$, $\ddot{h}_2(\theta(W))\not=0$ and
$$\frac{\ddot{h}_1(\theta(W))}{\ddot{h}_2(\theta(W))} > - \frac{ \sum_{a \in \mathcal{A}} c_{a,2} g_2(\mu_0(a,W))    }{ \sum_{a \in \mathcal{A}} c_{a,1}  g_1(\mu_0(a,W))  } +   \frac{\gamma}{\ddot{h}_2(\theta(W)) \sum_{a \in \mathcal{A}} c_{a,1} g_1(\mu_0(a,W)) }$$ both hold $P_0$--almost surely.
Here, for $m \in \{1,2\}$, we denote the first derivatives of $g_m$ and $h_m$ as $\dot{g}_m$ and $\dot{h}_m$, respectively, and the second derivative of $h_m$ by $\ddot{h}_m$. These conditions apply to both the CATE and CRR examples introduced above ---see Examples 14.16 and 14.18 in \cite{wainwright_2019} for details.

\subsection{Our proposed approach: EP-learning}
\label{section:eplrnrproposal}

Given a population risk function $R_0$, it is natural to estimate the causal summary $\theta_0$ using empirical risk minimization techniques based on an efficient estimator of $R_0$. In the existing literature, an `orthogonal learning' strategy is often employed, wherein an efficient risk estimator is derived from a Neyman-orthogonal loss function \citep{orthogonalLearning}. Such a strategy benefits from relative insensitivity to the accuracy of involved nuisance estimators. In this section, we introduce, at a high level, our proposed EP-learning framework---an alternative approach to orthogonal learning---based on a novel, efficient plug-in estimator for the class of population risk functions under consideration.

To derive an efficient estimator of $R_0$, in the context of either orthogonal or EP-learning, we exploit the fact that the population risk parameter $P \mapsto R_{P}(\theta)$ at a specified $\theta \in \overline{\mathcal{F}}$ is a pathwise differentiable parameter under $\mathcal{M}$ and therefore amenable to nonparametric efficient estimation using standard techniques \citep{diaz2013targeted}. Pathwise differentiability implies the existence of a nonparametric efficient influence function of the $\theta$--specific population risk parameter, the variance of which provides the generalized Cramér-Rao lower bound for estimating $R_0(\theta)$. Specifically, under regularity conditions, for given $\theta \in \overline{\mathcal{F}}$, this efficient influence function has the form
\[D_{P,\theta} : (w,a,y) \mapsto L_{\mu_P}(\theta,w)  + \Delta_{\pi_P, \mu_P}(w,a,y; \theta) - R_{P}(\theta)\ ,\]
where, letting $H_{m,\mu_P}(a,w) :=   c_{a,m}\cdot\dot{g}_m(\mu_P(a,w))$ for $m\in \{1,2\}$, we write
\begin{align*}
    &\Delta_{\pi_P, \mu_P}(w,a,y; \theta) := \frac{ 1}{\pi_P(a,w)} \left\{\sum_{m \in \{1,2\}}  H_{ m,\mu_P}(a,w) h_m(\theta(w))  \right\}\left\{y - \mu_P(a,w) \right\}.
\end{align*}
 In the theorem below, we formalize this fact. To do so, we require the following overlap condition:

  \begin{enumerate}[label=\bf{A\arabic*)}, ref = A\arabic*]
    \item $P(\pi_P(a \miid W) > \delta) = 1$ for all $a \in \mathcal{A}$ and $P \in \mathcal{M}$. \label{cond::pos}
\end{enumerate}
\begin{theorem}
Suppose Condition \ref{cond::pos} holds. Then, for an arbitrary element $\theta \in \overline{\mathcal{F}}$, the nonparametric efficient influence function of $P' \mapsto R_{P'}(\theta)$ at $P \in \mathcal{M}$ is given by $D_{P,\theta}$. 
\label{theorem::generalEIF}
\end{theorem}
Knowledge of this efficient influence function is important because it encodes, in first order, the sensitivity of the population risk to perturbations in its nuisance parameters and facilitates the debiasing of plug-in estimators. We note that $D_{P,\theta}$ can be composed into a sum of three terms: the plug-in loss function $L_{\mu_P}$, a weighted residual $\Delta_{\pi_P,\mu_P}$ of the outcome regression $\mu_P$, and the negative population risk $-R_{P}(\theta)$. Using that $E_P[D_{P, \theta}(O)] = 0$, this decomposition allows us to deduce that $L_{\mu_P} + \Delta_{\pi_P,\mu_P}$ is a Neyman-orthogonal loss function for $R_P$.

Let $\pi_n$ and $\mu_n$ be estimators of $\pi_0$ and $\mu_0$, respectively, which we assume for the time being are obtained using an independent dataset; later, we will describe approaches to use cross-fitting when such an independent dataset is not available. Given these nuisance estimators, a one-step debiased estimator of the $\theta$--specific risk $R_0(\theta)$ is given by
\begin{equation}
    R_{n,\pi_n,\mu_n}(\theta) := \frac{1}{n}\sum_{i=1}^n  L_{\mu_n}(\theta,W_i) + \frac{1}{n}\sum_{i=1}^n \Delta_{\pi_n, \mu_n}(O_i; \theta)\ .\label{eqn::onestepRisk}
\end{equation}
This estimator can be seen as a debiased version of the plug-in estimator $\frac{1}{n}\sum_{i=1}^n  L_{\mu_n}(\theta,W_i)$ with debiasing term $\frac{1}{n}\sum_{i=1}^n \Delta_{\pi_n, \mu_n}(O_i; \theta)$. Under appropriate conditions, $R_{n,\pi_n,\mu_n}(\theta)$ is an asymptotically efficient estimator of $R_0(\theta)$. 

The decomposition in \eqref{eqn::onestepRisk} suggests two approaches for learning $\theta_0$ based on an efficient estimator of $R_0$. One approach, orthogonal learning, uses loss-based learning techniques based on the orthogonalized loss $L_{\mu_n} + \Delta_{\pi_n, \mu_n}$, that is, it relies on using the debiased risk estimator $R_{n,\pi_n,\mu_n}$. The alternative approach we propose, EP-learning, instead relies on the \emph{plug-in} risk estimator $R^*_{n}: \theta \mapsto \frac{1}{n}\sum_{i=1}^{n}L_{\mu_n^*}(\theta,W_i)$, where $\mu_n^*$ is a carefully constructed estimator of $\mu_0$ that negates the need for the debiasing term $\frac{1}{n}\sum_{i=1}^n \Delta_{\pi_n, \mu_n^*}(O_i; \theta)$. \textcolor{black}{When empirical risk minimization is used, the EP-learner for the causal contrast is the plug-in estimator $\argmin_{\theta \in \mathcal{F}} R_n^*(\theta)$ of the constrained risk minimizer $\theta_0 = \argmin_{\theta \in \mathcal{F}} R_0(\theta)$.} In this sense, EP-learning follows the spirit of targeted minimum loss-based estimation (TMLE) \citep{vanderLaanRose2011,seqDRLuedtke}.

\begin{algorithm}[htb!]
\textcolor{black}{
\caption{Meta-algorithm for EP-learning (informal)}
\label{alg::EPlearnerMeta}
\begin{algorithmic}[1]
\STATE \textbf{Input:} Data $\{O_i\}_{i=1}^n$; nuisance estimators $\mu_n$, $\pi_n$; function class $\mathcal{F}$; learning algorithm.
\STATE \textbf{Step 1: Debias nuisance.} Construct $\mu_n^*$ from $\mu_n$ and $\pi_n$ so that the debiasing term
\[
\tfrac{1}{n}\sum_{i=1}^n \Delta_{\pi_n,\mu_n^*}(O_i;\theta) 
= \tfrac{1}{n}\sum_{i=1}^n \!\big\{L_{\pi_n,\mu_n^*}(\theta, O_i) - L_{\mu_n^*}(\theta, W_i)\big\}
\]
is small for all $\theta \in \mathcal{F}$.
\STATE \textbf{Step 2: Estimate risk by plug-in.} Define $R_n^*(\theta) := \tfrac{1}{n}\sum_{i=1}^n L_{\mu_n^*}(W_i; \theta)$ for $\theta \in \mathcal{F}$.
\STATE \textbf{Step 3: Estimate causal contrast.} Apply the learning algorithm to $R_n^*$ to obtain an estimate $\theta_n^*$ of $\theta_0$, e.g., $\theta_n^* := \argmin_{\theta \in \mathcal{F}} R_n^*(\theta)$.
\STATE \textbf{Output:} EP-learner $\theta_n^*$
\end{algorithmic}
}
\end{algorithm}

The EP-learner algorithm is designed to attain, on one hand, the oracle-efficiency of orthogonal learning strategies based on the loss $L_{\mu_n}+\Delta_{\pi_n,\mu_n}$, and on the other hand, the desirable properties---stability and loss convexity---enjoyed by plug-in estimation strategies based on the `naive' loss $L_{\mu_n}$. Rather than substituting any `good' estimator of $\pi_0$ and $\mu_0$ into the orthogonal loss function, in EP-learning, an outcome regression estimator $\mu_n^*$ is constructed from $\mu_n$ through a sieve-based adjustment to ensure that the debiasing term $\frac{1}{n}\sum_{i=1}^n \Delta_{\pi_n, \mu_n^*}(O_i; \theta)$ is negligible across values of $\theta$, rendering explicit debiasing unnecessary. Since the difference between $R_n^*$ and $R_{\pi_n, \mu_n^*}$ is negligible, the EP-learner risk estimator benefits both from the plug-in property of the loss $L_{\mu_n^*}$ and the orthogonality property of the orthogonalized loss $L_{\mu_n^*}+\Delta_{\pi_n,\mu_n^*}$. \textcolor{black}{We present a high-level meta-algorithm for EP-learning in Algorithm \ref{alg::EPlearnerMeta}.} For the class of losses we consider, the precise procedure for building the estimator $\mu_n^*$ is detailed later in Algorithm \ref{alg::debiasing}, and our EP-learner for $\theta_0$ is presented in Section \ref{section::algo}. In the next section, we illustrate how the plug-in property of EP-learning resolves the issues with Neyman-orthogonal learning referred to in the Introduction.

{\color{black}It is worth elaborating on the plug-in nature of EP-learning, particularly in contrast to T-learning. We say that an estimator $\widehat{R}_n(\theta)$ of the population risk $R_0(\theta)$ is a uniform plug-in estimator over $\theta \in \mathcal{F}$ if there exists a distribution $\widehat{P}_n \in \mathcal{M}$ such that $\widehat{R}_n(\theta) = R_{\widehat{P}_n}(\theta)$ for all $\theta \in \mathcal{F}$. The EP-learner risk estimator $R_{n}^*$ is such a uniform plug-in estimator, with $\widehat{P}_n$ chosen so that its outcome regression satisfies $\mu_{\widehat{P}_n} = \mu_n^*$ and its marginal distribution of $(W,A)$ equals the empirical distribution $P_n$. When paired with empirical risk minimization, EP-learner likewise returns a plug-in estimator, $\theta_n^*\in \argmin_{\theta\in\mathcal{F}} R_{\widehat{P}_n}(\theta)$, of the constrained minimizer $\theta_0\in \argmin_{\theta\in\mathcal{F}} R_{P_0}(\theta)$. Similarly, T-learner returns a plug-in estimator of the unconstrained risk minimizer $\argmin_{\theta \in L^2(P_{0,W})} R_P(\theta)$, where the $\widehat{P}_n$ that is plugged in is any distribution whose outcome regressions equal $\mu_n$; in the special case of CATE estimation, this unconstrained risk minimizer is just $\mu_n(1,\cdot)-\mu_n(0,\cdot)$ \citep{kunzel2019metalearners}. Compared to the unconstrained risk minimizer, EP-learner  exploits the restriction to $\mathcal{F}$ to inherit the favorable robustness properties of orthogonal learning. Indeed, when the initial estimator $\mu_n$ is misspecified, the EP-learner risk estimator can still be consistent uniformly over $\mathcal{F}$ provided the propensity is well-specified. As a consequence, EP-learner can still return a best approximation to the CATE in $\mathcal{F}$ in such cases.

While restricting to $\mathcal{F}$ may seem restrictive, this is standard in statistical learning learning theory \citep{vapnik1999overview, orthogonalLearning, wainwright_2019}. In related problems such as regression, this choice is made to reflect structural or smoothness assumptions on the estimand. Such smoothness conditions are necessary because, without them, regression functions are not nonparametrically learnable \citep{vapnik1971chervonenkis, stone1982optimal, gyorfi2002distribution}. For the CATE, the same limitation is reflected in minimax rates that become arbitrarily slow as smoothness decreases \citep{KennedyMiximax}.
}

An efficient TMLE-based plug-in estimator for the population risk $R_0(\theta)$ at a specified candidate function $\theta$ for the causal dose–response curve was proposed in \cite{diaz2013targeted}. However, their estimator does not provide a single plug-in estimator of the full population risk function $R_0$ and is therefore not amenable to empirical risk minimization over an infinite-dimensional function class. \textcolor{black}{In contrast, the EP-learner risk estimator $R_{n}^*(\theta)$ provides a uniform plug-in estimator, meaning that the same outcome regression estimator $\mu_n^*$ is used to compute $R_{n}^*(\theta)$ for every $\theta \in \mathcal{F}$. This uniform plug-in property ensures that if $\theta \mapsto R_P(\theta)$ is a convex risk function for any $P$, then $\theta \mapsto R_{n}^*(\theta)$ will also be convex.} The concurrent work of \cite{vansteelandt2023orthogonal} similarly uses sieves to construct a debiased plug-in risk estimator, although they restrict their attention to estimation of a covariate-adjusted conditional mean. Their risk is a special case of \eqref{eqn::GeneralClassRisk} with $\mathcal{A} = \{1\}$, $h_1(\theta) = \theta^2$, $h_2(\theta) = -2\theta $,  $g_1 \circ \mu = 1$, $g_2 \circ \mu = \mu$, $c_{1,1} = 1$, and $ c_{0, 1} = 0$. 
Besides considering a general class of risks, our approach uses a different form of cross-fitting and sieve-based adjustment to ensure the asymptotic equivalence of our EP-learner risk estimator with an oracle-efficient one-step risk estimator. This difference complicates our theoretical analysis, as it prohibits us from invoking statistical learning bounds based on sample-split nuisance estimators \citep{orthogonalLearning}. \textcolor{black}{In particular, ensuring that the debiasing term $\tfrac{1}{n}\sum_{i=1}^n \Delta_{\pi_n, \mu_n}(O_i; \theta)$ is negligible (i.e., $o_p(n^{-1/2})$) requires the regression estimator $\mu_n^*$ to satisfy certain score equations on the full dataset (see Section~\ref{section::algo} for details). This requirement is not unique to the EP-learner but is shared more broadly by sieve-based plug-in estimation \citep{shen1997methods, chen2007large, CaroneDataAdaptSieve}. Consequently, the empirical process remainders in our analysis must be handled directly rather than through sample splitting, which necessitates sharper maximal inequalities and entropy bounds.
}

\section{Limitations of existing Neyman-orthogonal learning strategies}

\label{section::introexamp}

\subsection{Sensitivity of CATE DR-learner to large propensity weights}
 \label{section::introCATEexample}

For estimation of the CATE, the DR-learner is a popular orthogonal learning approach based on the least-squares empirical risk function
$$\theta \mapsto \frac{1}{n} \sum_{i=1}^n \left\{\theta(W_i)^2 - 2\theta(W_i)\chi_n(W_i, A_i, Y_i)\right\},$$
where $\chi_n(w,a,y) := \mu_n(1,w) - \mu_n(0,w) + \frac{2a-1}{\pi_n(a \mid w)}\{y - \mu_n(a,w)\}$ is an estimated pseudo-outcome. This empirical risk coincides with the debiased risk estimator $R_{n,\pi_n,\mu_n}$ of the population risk $R_0^-$ introduced in Example \ref{Example::CATE}. \textcolor{black}{A DR-learner can be obtained either by minimizing this empirical risk over $\mathcal{F}$ or by using standard regression methods to regress the pseudo-outcomes $\{\chi_n(W_i,A_i,Y_i)\}_{i=1}^n$ on the covariates $\{W_i\}_{i=1}^n$.} As discussed in the Introduction, the pseudo-outcome used by the DR-learner can take extreme values when the propensity score is close to zero or one, which can lead to poor behavior of the resulting CATE estimator.

In contrast to the DR-learner, our EP-learner for the CATE---formally defined in Section \ref{section::algo}---is based \textcolor{black}{on the squared-error loss with} the estimated pseudo-outcome $\mu_n^*(1,W) - \mu_n^*(0,W)$, where $\mu_n^*$ is an outcome regression estimator carefully constructed from the initial estimates $\mu_n$ and $\pi_n$. Unlike a typical T-learner, the estimate $\mu_n^*$  is constructed to ensure that the resulting plug-in risk estimator 
\begin{align*}
    R_{n}^{*-}(\theta):= \frac{1}{n}\sum_{i=1}^n \left[\theta(W_i)^2  -2 \theta(W_i)  \left\{ \mu_n^*(1,W_i) - \mu_n^*(0,W_i) \right\}\right]
\end{align*}
is efficient under reasonable conditions. \textcolor{black}{EP-learning then estimates the CATE by a second-stage regression of $\{\mu_n^*(1,W_i) - \mu_n^*(0,W_i)\}_{i=1}^n$ on $\{W_i\}_{i=1}^n$. The combination of the specially constructed $\mu_n^*$ with the second-stage regression is what allows EP-learning to inherit the robustness properties of orthogonal learning strategies---such as fast rates and oracle efficiency---that the T-learner $\mu_n^*(1,\cdot) - \mu_n^*(0,\cdot)$ alone typically does not have.} The construction of $\mu_n^*$ incorporates an inverse-propensity-weighted regression adjustment of $\mu_n$, which renders the corresponding risk estimator doubly robust in the sense that, for each $\theta$, $R_n^{*-}(\theta)$ is consistent if either the outcome regression or the propensity score is estimated consistently. \textcolor{black}{This double robustness of the risk generally carries over to the corresponding EP-learner, as we show formally for empirical risk minimization in Section \ref{section::theory}. }

We now provide a heuristic explanation of why EP-learner can outperform DR-learner, particularly when paired with a local regression estimator, such as a Nadaraya-Watson \citep{nadaraya1964estimating,watson1964smooth}, $K$-nearest neighbors \citep{fix1989discriminatory}, or random forest \citep{Breiman1984} estimator. When used in a DR-learner or EP-learner, local methods estimate a regression function at a point by averaging the pseudo-outcomes of `nearby' observations. When the number of points averaged is small, the estimates returned by DR-learner can be erratic due to the potentially large values taken by its estimated pseudo-outcomes. In contrast, the EP-learner is expected to be more stable since it is simply a local average of an initial CATE estimator. In fact, if this initial estimator is uniformly consistent, then the corresponding EP-learner remains consistent even when the number of local observations averaged is held fixed with sample size---a formal argument in the case of $K$-nearest neighbors estimation is provided in Section \ref{section::knn}. 
Moreover, even if the initial estimator fails to be uniformly consistent, then EP-learner can still benefit from the desirable properties that it shares with DR-learner, namely the efficiency and double robustness of its risk estimator---details are provided in Section \ref{section::EffEPLearnerRisk}.

\begin{figure}[htb]
    \centering
    \begin{subfigure}{1\textwidth}\centering
    \hspace{1.1in}\includegraphics[width=.55\textwidth]{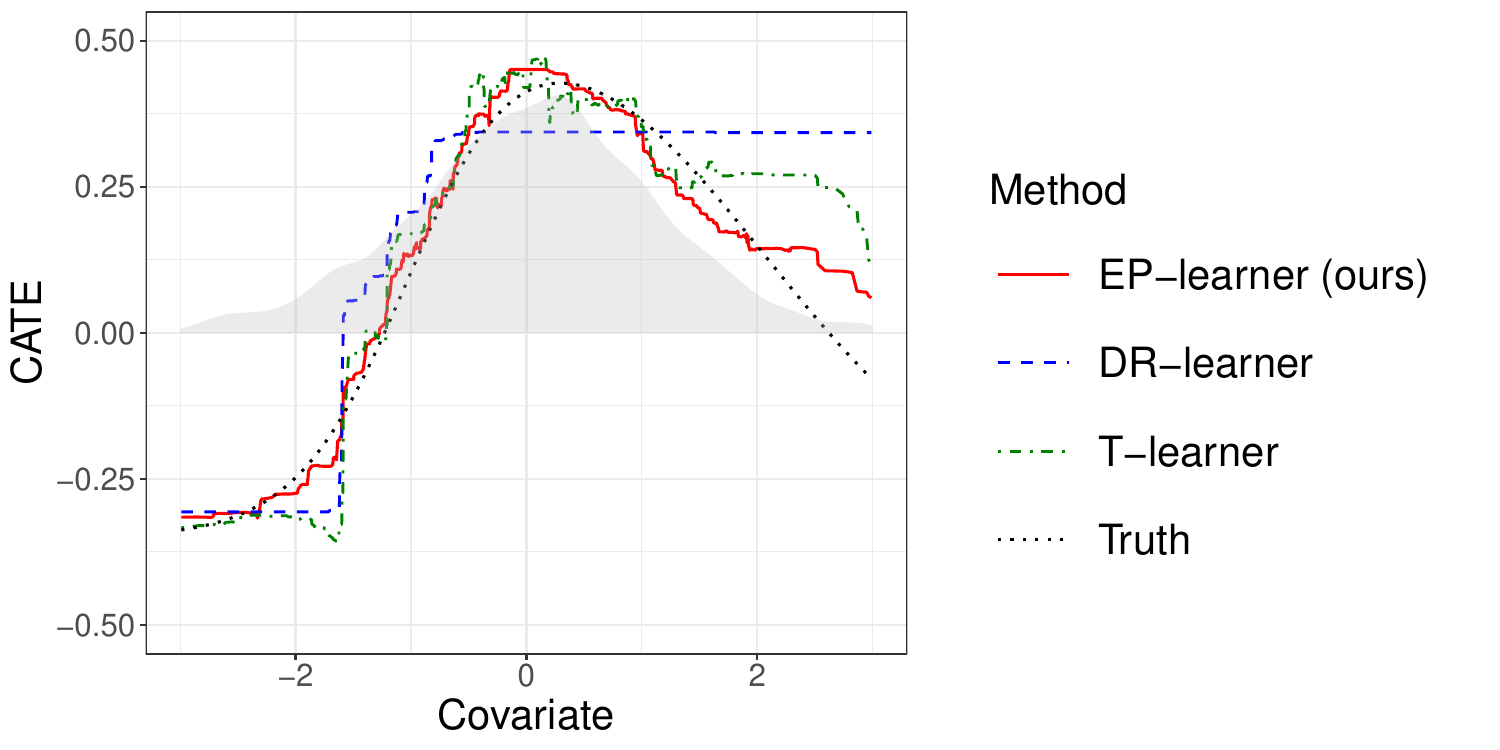}
    \subcaption[]{Cross-validated maximum tree depth}
          \label{fig:DRlearnerBoundViol_CV}
    \end{subfigure}
    \begin{subfigure}{1\textwidth}\centering
    \includegraphics[width=1\textwidth]{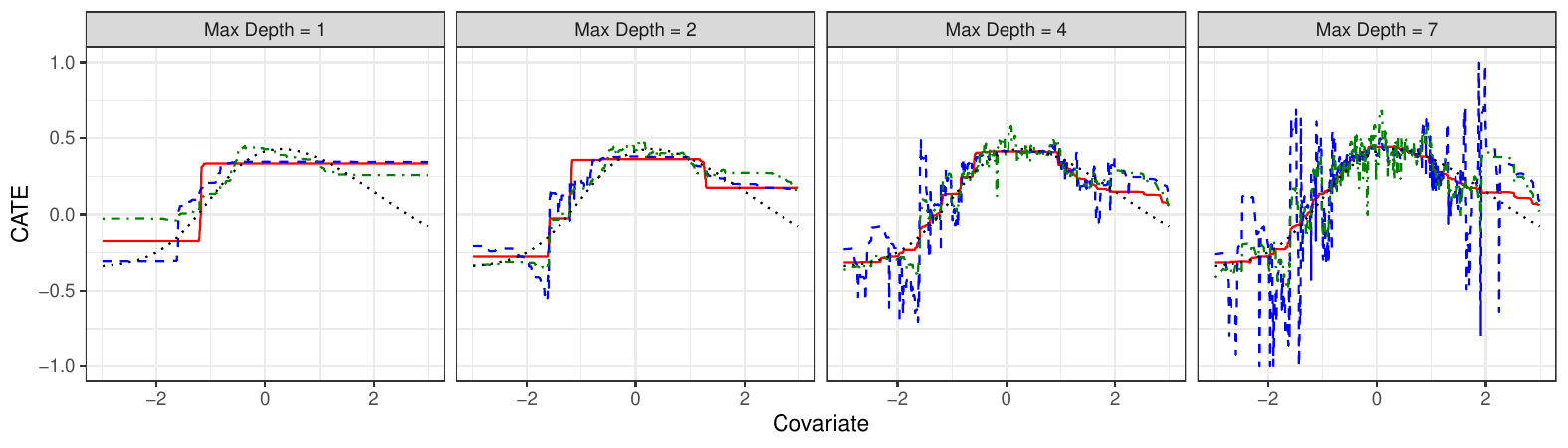}
      \subcaption[]{Fixed maximum tree depth}
 \label{fig:DRlearnerBoundViol_depth}
    \end{subfigure}
    \caption{CATE estimates based on EP-learner, DR-learner, and T-learner with random forests and various maximum tree depths computed on a single dataset. (Top) EP-learner, DR-Learner, and T-learner CATE estimates with 10-fold cross-validated maximum tree depth; mean squared prediction errors  are 0.0021, 0.012, and 0.005, respectively. Observed covariate distribution is depicted in gray. (Bottom) EP-learner, DR-learner,and T-learner CATE estimate for maximum tree depths of 1, 2, 4 and 7.}
    \label{fig:DRlearnerBoundViol}
\end{figure}

 The benefits of EP-learner over DR-learner can be illustrated through a simple numerical experiment. 
To do so, we simulated a dataset consisting of 1,500 observations and used the \texttt{ranger} implementation of random forests \citep{ranger} to estimate the CATE function based on both the DR-learner and EP-learner risks---the code used to run our experiment is provided in Appendix \ref{appendix::introFigs}. In our illustration, the outcome is binary, and we use the known CATE range, $[-1,1]$, to truncate DR-learner predictions. Figure \ref{fig:DRlearnerBoundViol_CV} displays the estimated CATE curve for DR-learner and EP-learner with maximum tree depth in random forests selected within $\{1,2,\ldots,7\}$ via 10-fold cross-validation. This figure reveals that the CATE curve estimated using EP-learner more accurately reflects the bell shape of the true CATE curve than using the DR-learner. This improvement is evident quantitatively, with a six-fold reduction in mean squared prediction error when using EP-learner (0.002) compared to DR-learner (0.012). To delve deeper into this phenomenon, we assessed the fits obtained by using four of the seven maximum tree depth values considered, namely 1, 2, 4 and 7 (Figure~\ref{fig:DRlearnerBoundViol_depth}). DR-learner fits become unstable at depths as low as 2, stemming from the limited observation numbers in the regression tree nodes and some pseudo-outcomes reaching magnitudes of 6---this can be seen in Figure \ref{fig:PseudoOutcomeDensityPlot} of Appendix \ref{appendix::introFigs}. As a result, cross-validation chooses a depth of 1 for DR-learner, which fails to accurately capture the true CATE shape. In contrast, EP-learner fits tend to improve as the maximum depth increases from 1 to 7, with a selected maximum tree depth of 7.
 \textcolor{black}{Regression trees, as used in random forests, are outcome-adaptive, often learned using CART \citep{breiman2017classification}. Thus, the difference in pseudo-outcomes between EP-learner and DR-learner affects not only the outcomes averaged within each tree node but also the learned tree structure itself. In particular, EP-learner uses the T-learner $\mu_n^*(1,\cdot) - \mu_n^*(0,\cdot)$ to build the tree structure, whereas DR-learner relies on a noisier pseudo-outcome. Thus, EP-learner can yield superior fits by leveraging less noisy pseudo-outcomes to both stabilize node averages and guide more accurate tree splits.}




\subsection{Nonconvexity of Neyman-orthogonal CRR loss function}\label{section::introRRexample}

We introduced in Example \ref{Example::CATE} the population risk function $R_0^{\div}$ for the CRR $\theta_0^\div : w\mapsto \log \mu_0(1,w)-\log \mu_0(0,w)$. In light of \eqref{eqn::onestepRisk}, a doubly-robust and efficient one-step debiased risk estimator for the population risk function is given by
\begin{align}
R_{n,DR}^{\div}(\theta) :=  \frac{1}{n}\sum_{i=1}^n \left[   \left(\widehat{\mu}_{0,i} + \widehat{\mu}_{1,i} \right) \log \left(1 + \exp  \{\theta(W_i) \} \right) - \widehat{\mu}_{1,i} \theta(W_i) \right],
   \label{eqn::DRriskRR} 
\end{align}
where, for $s \in \{0,1\}$ and $i \in [n]$, we define $\widehat{\mu}_{s,i} := \mu_n(s,W_i) + \frac{1(A_i=s)}{\pi_n(s \miid W_i)}\left\{ Y_i - \mu_n(s,W_i)\right\}$. The corresponding doubly-robust orthogonal learner of the log-CRR function is computed by performing a weighted logistic regression of the pseudo-outcomes $\left\{ \widehat{\mu}_{1,i}/(\widehat{\mu}_{0,i} + \widehat{\mu}_{1,i}) : i \in [n]\right\}$ on the covariate values $\{W_i:i\in[n]\}$ with weights $\left\{  \widehat{\mu}_{0,i} + \widehat{\mu}_{1,i} : i \in [n]\right\}$.

Because the weights may take negative values, the Neyman-orthogonal loss associated with the one-step debiased risk estimator need not be convex. \textcolor{black}{To illustrate, consider a simple example with a binary covariate \(X \in \{0,1\}\), an unrestricted parameter space \(\Theta = \mathbb{R}^2\), and two observations (\(n=2\)), one for each value of \(X\). In this setting, the orthogonal loss takes the form
\[
L_{\pi_n,\mu_n}^{\div}(\theta_i, O_i)
= \bigl(\widehat{\mu}_{0,i}+\widehat{\mu}_{1,i}\bigr)
\Big[\log\{1+\exp(\theta_i)\}
- \tfrac{\widehat{\mu}_{1,i}}{\widehat{\mu}_{0,i}+\widehat{\mu}_{1,i}}\,\theta_i\Big],
\]
so the second derivative with respect to \(\theta_i\) is
\[
\nabla^2 L_{\pi_n,\mu_n}^{\div}(\theta_i, O_i)
= \bigl(\widehat{\mu}_{0,i}+\widehat{\mu}_{1,i}\bigr)
\,\sigma(\theta_i)\bigl(1-\sigma(\theta_i)\bigr),
\qquad
\sigma(t)=\tfrac{e^t}{1+e^t}.
\]
Since the weight \(\widehat{\mu}_{0,i}+\widehat{\mu}_{1,i}\) may be negative, the Hessian can also be negative, implying that the loss is not guaranteed to be convex---even in this simple case. The same issue persists in higher-dimensional settings, where negative weights can lead to indefinite Hessians.} Furthermore, the pseudo-outcomes may take values outside of $[0,1]$, which renders this a nonstandard use case for logistic regression. Figure \ref{fig:exampDataLRR} lists several common logistic regression implementations in \texttt{R} and indicates whether they allow the use of negative weights or outcomes that fall outside of $[0,1]$. As is revealed in this figure, most do not, which demonstrates the limited applicability of orthogonal learning in this setting. Figure \ref{fig:exampDataLRR} shows the first five rows of a synthetic input dataset containing the covariate values, pseudo-weights and pseudo-outcomes that would be entered into logistic regression software to estimate the log conditional relative risk.  Notably, the mock dataset contains negative pseudo-weight values and pseudo-outcomes that fall outside of $[0,1]$. \textcolor{black}{These issues are also encountered in the orthogonalization of the IPW loss for the CRR proposed in Section~2.3 of \cite{chen2017general}.} Code to reproduce the mock dataset can be found in Appendix \ref{appendix::introFigs}.

\begin{figure}[htb]
    \centering
    \includegraphics[width=0.5\textwidth]{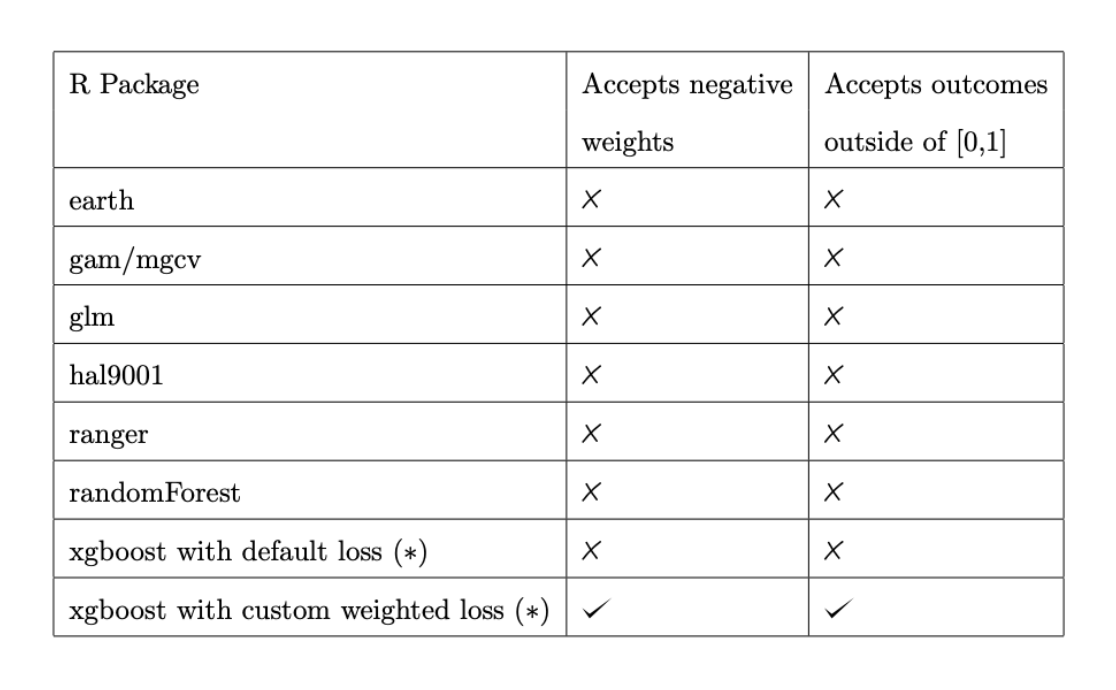}\includegraphics[width=0.5\textwidth]{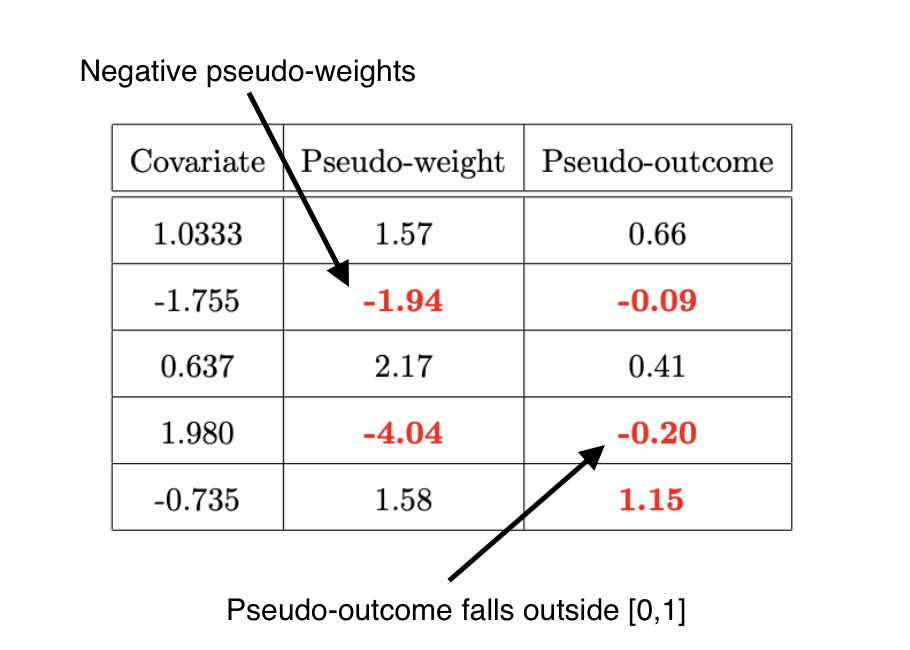}
    \caption{(Left) Common R packages for logistic regression may or may not allow negative weights and outcomes outside of [0,1]; a checkmark indicates that they do. ($\ast$) The xgboost package has many built-in loss functions, but they do not accept negative weights. However, these negative weights can be absorbed  into a custom loss function. (Right) Example dataset for estimating the CRR using the one-step efficient risk estimator.    }
    \label{fig:exampDataLRR}
\end{figure}

For a post-hoc constructed outcome regression estimator $\mu_n^*$, our EP-learner of the CRR is based on the following EP risk estimator
$$R_{n,EP}^{\div}(\theta) := \frac{1}{n}\sum_{i=1}^n \left\{\mu_n^*(1,W_i) + \mu_n^*(0,W_i)\right\}\left[  \log(1 + \exp \{\theta(W_i)\}) - \frac{\mu_n^*(1,W_i)}{\mu_n^*(1,W_i) + \mu_n^*(0,W_i) } \theta(W_i) \right] . $$
An estimator of the log-CRR function based on the EP-learner risk estimator is computed by performing the weighted logistic regression of the pseudo-outcome $\mu_n^*(1,W)/[\mu_n^*(1,W) + \mu_n^*(0,W) ] $ onto the covariate vector $W$ with weight $\mu_n^*(1,W) + \mu_n^*(0,W) $. Unlike for the efficient one-step risk estimator, the pseudo-weight is always nonnegative and the pseudo-outcome always falls in $[0,1]$. As a consequence, the EP-learner risk estimator corresponds with a convex loss function, and EP-learners of the log-CRR function can be implemented using most standard logistic regression software. In Section~\ref{section::algo}, we will provide an algorithm to construct $\mu_n^*$ so that $R_{n,EP}^\div$ is an efficient, doubly robust plug-in risk estimator.

\section{Proposed approach: EP-learning algorithm}\label{section::algo}

The EP-learner algorithm is designed to attain both the efficiency of the one-step debiased risk estimator in \eqref{eqn::onestepRisk} and the desirable properties, such as stability and convexity, that are enjoyed by plug-in risk estimators. This learner is based on a risk function of the form $\theta\mapsto \frac{1}{n}\sum_{i=1}^n  L_{\mu_n}(\theta,W_i)$ for an outcome regression estimator $\mu_n$. Naturally, the fact that EP-learner does not explicitly appear to debias this plug-in estimator may lead to concerns that it might be based on a suboptimal---in particular, inefficient---estimator of the risk. However, the EP-learner outcome regression estimator is carefully constructed to perform implicit debiasing. Specifically, it is designed such that the seemingly-critical debiasing term $\frac{1}{n}\sum_{i=1}^n \Delta_{\pi_n, \mu_n}(O_i; \theta)$ is negligible in an appropriate sense across values of $\theta$, so that the plug-in risk estimator is asymptotically equivalent to a one-step debiased estimator. Before describing the general EP-learner approach, we illustrate its construction in the context of a simple, finite-dimensional example.

\begin{example}[continues = ex1, , name=conditional average treatment effect]
Consider the problem of CATE estimation from Example \ref{Example::CATE} in the special case where the action space $\mathcal{F}$ is the linear class including each function $w \mapsto \alpha^\top w$ with $\alpha \in \mathbb{R}^d$. In this finite-dimensional case, we write $R_0^-(\alpha)$ as shorthand for $R_0^-(w\mapsto \alpha^\top w)$. Given initial estimators $\pi_n$  of $\pi_0$ and $\mu_n$ of $\mu_0$, EP-learner uses the refined estimator $\mu_n^*:(a,w)\mapsto\mu_n(a,w) + (2a-1)\beta_n^\top w$, where
$$\beta_n := \argmin_{\beta \in \mathbb{R}^d}\frac{1}{n} \sum_{i=1}^n \frac{1}{\pi_n(A_i \miid W_i)} \left\{Y_i - \mu_n(A_i, W_i) - (2A_i-1)\beta^\top W_i \right\}^2.$$
The first-order conditions satisfied by $\beta_n$ ensure that the debiasing term for estimating $R_0^-(\alpha)$, notably $\frac{1}{n} \sum_{i=1}^n \frac{2A_i - 1}{\pi_n(A_i \miid W_i)} \alpha^\top W_i \left\{Y_i - \mu_n^*(A_i, W_i) \right\}$, is exactly zero for $\alpha \in \mathbb{R}^d$. \textcolor{black}{In this special case, the parametric EP-learner $w \mapsto (\alpha_n^*)^\top w$ based on $\mu_n^*$, where $\alpha_n^* := \argmin_{\alpha \in \mathbb{R}^d} \sum_{i=1}^n \{\mu_n^*(1,W_i) - \mu_n^*(0,W_i) - \alpha^\top W_i\}^2$, corresponds to a targeted minimum loss-based estimator \citep{vanderLaanRose2011} of the best linear predictor of the CATE, $w \mapsto \alpha_0^\top w$, with $\alpha_0 := \argmin_{\alpha \in \mathbb{R}^d} E_0\{\mu_0(1,W) - \mu_0(0,W) - \alpha^\top W\}^2$.}


 \label{example::toyCATE}
\end{example}

When the dimension of the action space $\mathcal{F}$ is large relative to sample size (e.g., infinite action space), it is generally not possible to construct an outcome regression estimator $\mu_n$ that both has good predictive performance and makes the debiasing term $\frac{1}{n}\sum_{i=1}^n \Delta_{\pi_n, \mu_n}(O_i; \theta)$ exactly zero for each $\theta\in\mathcal{F}$. 
 Indeed, there is an inherent trade-off between the size of the debiasing term and the mean squared error of the outcome regression estimator. The EP-learner algorithm is designed to carefully balance these two terms. It does so by separating the estimation of the outcome regression into two stages. In the first stage, statistical learning tools are used to obtain an initial estimate of the outcome regression with good predictive performance. In the second stage, this initial estimator is refined to make the debiasing term as small as possible without harming the performance of EP-learner.

 The debiasing approach described in Example \ref{example::toyCATE} can be generalized to infinite-dimensional function spaces using the idea of sieves. Specifically, if the linear span of a finite but growing set of basis functions approximates $\mathcal{F}$ increasingly well, then the refinement procedure described in Example \ref{example::toyCATE} could be performed within this span. This procedure would ensure that the debiasing term is small for each element in the linear space spanned by these basis functions. Critically, the number of basis functions used would need to grow with sample size so that the worst-case approximation error tends to zero. \textcolor{black}{The separation of the outcome regression estimation and debiasing steps in the EP-learner ensures that the sieve does not need to approximate the true outcome regression $\mu_0$ itself, but only needs to provide sufficient flexibility to reduce the debiasing term.} \textcolor{black}{For the CATE, given an approximating sequence of linear spaces $\mathcal{H}_k$ for $\mathcal{F}$ with dimension $k = k(n)$ typically chosen to grow with $n$, we can construct $\mu_n^*$ as in Example \ref{example::toyCATE} so that the debiasing term $\tfrac{1}{n}\sum_{i=1}^n \Delta_{\pi_n,\mu_n^*}(O_i;\theta)$ is zero for each $\theta \in \mathcal{H}_k$. Assuming that $\mathcal{H}_k$ approximates $\mathcal{F}$ with vanishing error, we formally show in Section \ref{section::theory} that the empirical risk minimizer $\argmin_{\theta \in \mathcal{F}} \sum_{i=1}^n \{\mu_n^*(1,W_i) - \mu_n^*(0,W_i) - \theta(W_i)\}^2$ over $\mathcal{F}$ will inherit this debiasing property for the constrained minimizer $\argmin_{\theta \in \mathcal{H}_k} R_0(\theta)$.}

We now formalize our general EP-learner algorithm, which is an instance of the meta-algorithm in Algorithm~\ref{alg::EPlearnerMeta}. Similar to competing procedures such as the DR-learner and R-learner, EP-learner can incorporate cross-fitting of the initial outcome regression and propensity score estimators to improve performance. Below, we present a general cross-fitted implementation of EP-learner. Let $J \in \mathbb{N}$ denote a fixed number of cross-fitting splits. Let $\mathcal{D}_n^1,\mathcal{D}_n^2, \ldots, \mathcal{D}_n^J$ be a partition of the available data $\mathcal{D}_n$ into $J$ datasets of approximately equal size, corresponding to index sets $\mathcal{I}^1_n,\mathcal{I}^2_n,\ldots,\mathcal{I}^J_n$. For each $i \in [n]$, let $j(i) \in [J]$ be the index of the data fold containing observation $i$, so that $i \in \mathcal{I}_{n}^{j(i)}$. Suppose that, for each $j\in[J]$, we have constructed estimators $\pi_{n,j}$ of $\pi_0$ and $\mu_{n,j}$ of $\mu_0$ based on $\mathcal{D}_n \backslash \mathcal{D}_n^j$. Let $\varphi_{k}:\mathbb{R}^d \mapsto \mathbb{R}^{k}$ be a feature mapping such that elements of the transformed action space ${\mathcal{H}}:=\{h_j \circ \theta : \theta \in \mathcal{F},j=1,2\}$ are approximated well by the linear space $\mathcal{H}_{k}:=\{w \mapsto \beta^\top\varphi_{k}(w) : \beta \in \mathbb{R}^{k} \}$, where $k=k(n)$ is typically chosen to grow with $n$. Algorithm~\ref{alg::debiasing} is \textcolor{black}{a particular choice of Step~1 of Algorithm~\ref{alg::EPlearnerMeta}}. For each $j$, it transforms the initial outcome regression estimator $\mu_{n,j}$ into a refined estimator $\mu_{n,j}^{*}$ such that
\begin{equation}
   \frac{1}{n}\sum_{i=1}^n \frac{1}{\pi_{n,j(i)}(A_i \miid W_i)} \left\{\sum_{m \in \{1,2\}}H_{m,n,j(i)}(A_i,W_i) \psi_m(W_i) \right\}\left\{Y_i - \mu_{n,j(i)}^{*}(A_i , W_i) \right\} = 0 \label{eqn::scoreEqnSolvedByAdjust}
\end{equation}for each $\psi_1,\psi_2 \in \mathcal{H}_{k(n)}$, where 
    $H_{m,n,j}$ is defined as $(a,w)\mapsto c_{a,m}\dot{g}_m(\mu_{n,j}(a,w))$ with $c_{a,m}$ and $g_m$ introduced in \eqref{eqn::GeneralClassRisk}.  \textcolor{black}{
We note that the EP-learner still relies on inverse propensity weight estimates through this step and, like the DR-learner, may suffer from instability when treatment overlap is limited.}

 Modifications of Algorithm~\ref{alg::debiasing} are possible. 
 For certain population risk functions, including the CATE risk function used in Example \ref{Example::CATE}, Algorithm \ref{alg::debiasing} can be simplified by performing the adjustment over a lower-dimensional space; details are provided in Appendix \ref{appendix:algoDebiasing1}.  It is also possible to implement an alternative debiasing step that preserves bounds on the initial outcome regression estimator, even in cases in which the outcome itself may be unbounded; details are provided in Appendix \ref{appendix:algoDebiasing2}. Such modification may improve performance by ensuring that the refinement step does not exceed bounds possibly specified a priori and enforced in the initial learning step. 

\begin{algorithm}[htb]
\caption{\textcolor{black}{Debiasing procedure for implementing Step~1 of Algorithm~\ref{alg::EPlearnerMeta}}}
\label{alg::debiasing}
\vspace{.1in}
   \begin{algorithmic}
     \REQUIRE
    \STATE-- dataset $\mathcal{D}_n$ consisting of observations $(W_1,A_1,Y_1),(W_2,A_2,Y_2),\ldots,(W_n,A_n,Y_n)$;
    \STATE
     -- cross-fitted estimates $\pi_{n,1},\pi_{n,2},\ldots,\pi_{n,J}$ of $\pi_{0}$ and $\mu_{n,1},\mu_{n,2},\ldots,\mu_{n,J}$ of $\mu_0$;
    \STATE  -- feature mapping $\varphi_{k}:\mathbb{R}^d \mapsto \mathbb{R}^{k}$ of output dimension $k \in \mathbb{N}$;
    \end{algorithmic}
    \vspace{.05in}
\begin{algorithmic}[1]
\STATE for each $i\in[n]$, construct $\widehat \varphi_{k, j(i)}: \mathbb{R}^{d+1} \mapsto \mathbb{R}^{2k}$ as $( a, w)\mapsto (H_{1,n,i}(a, w), H_{2,n,i}(a, w))\varphi_{k}(w) $ with $H_{1,n,i}(a,w):= c_{a,1}\dot{g}_1(\mu_{n,j(i)}(a,w))$ and $H_{2,n,i}(a,w):= c_{a,2}\dot{g}_2(\mu_{n,j(i)}(a,w))$;
\STATE choose link function $g$ and obtain regression coefficient estimate $\beta_n$ as follows:
 \begin{itemize}[label={--},leftmargin=*]
     \item \textit{method 1: outcomes in $[0,1]$.} set $g:x\mapsto\text{logit}(x)$ and $g^{-1}:x\mapsto \text{expit}(x)$, and compute coefficient vector estimate $\beta_n$ obtained from logistic regression of outcome $Y_i$ on feature vector $\widehat\varphi_{k, j(i)}(A_i,W_i)$ with offset $\text{logit}\,\mu_{n,j(i)}(A_i,W_i)$ and weight $1/\pi_{n,j(i)}(A_i \miid W_i)$ for $i\in[n]$;
     \item \textit{method 2: general outcomes.}  set $g:x\mapsto x$ and $g^{-1}:x\mapsto x$, and compute coefficient vector estimate $\beta_n$ obtained from linear regression of outcome $Y_i$,  on feature vector $\widehat \varphi_{k, j(i)}(A_i,W_i)$ with offset $\mu_{n,j(i)}(A_i,W_i)$ and weight $1/\pi_{n,j(i)}(A_i \miid W_i)$ for $i\in[n]$;
 \end{itemize}
  \STATE for each $j\in[J]$, return debiased out-of-fold  outcome regression estimator 
\begin{equation}
    \mu_{n,j}^{*}:(a,w)\mapsto g^{-1}\left(g(\mu_{n,j}(a,w)) + \widehat \varphi_{k,j}(a,w)^{\top}\beta_n\right).
\end{equation}  

\end{algorithmic}
\end{algorithm}
 
  \textcolor{black}{We now present our EP-learner algorithm for feature mappings of fixed dimension $k$ in Algorithm~\ref{alg::EPlearner}. This algorithm is a cross-fitted variant of Algorithm~\ref{alg::EPlearnerMeta} that incorporates Algorithm~\ref{alg::debiasing} in Step~1.} Line \ref{step::constructEPLearner} of  this algorithm encompasses many estimation strategies, including penalized empirical risk minimization \citep{van2000empirical}, random forests \citep{breiman2001random}, and gradient boosting \citep{BoostingFruend}. 
Algorithm \ref{alg::CVEPlearner} instead describes a cross-validated EP-learner implementation incorporating data-driven selection of the feature mapping dimension. 
  This selection is made based on a one-step debiased estimate of the risk function.  
 The potential nonconvexity of this risk does not cause computational issues in this case since optimization is performed over a finite set of candidate values of $k$.

\begin{algorithm}[htb]
\caption{EP-learner algorithm (fixed sieve dimension)}
\label{alg::EPlearner}
\vspace{.1in}
\begin{algorithmic}
    \REQUIRE    
    \STATE-- dataset $\mathcal{D}_n$ consisting of observations $(W_1,A_1,Y_1),(W_2,A_2,Y_2),\ldots,(W_n,A_n,Y_n)$;
    \STATE --  cross-fitted estimates $\pi_{n,1},\pi_{n,2},\ldots,\pi_{n,J}$ of $\pi_{0}$ and $\mu_{n,1},\mu_{n,2},\ldots,\mu_{n,J}$ of $\mu_0$;
    \STATE -- feature mapping $\varphi_{k}:\mathbb{R}^d \mapsto \mathbb{R}^{k}$ of output dimension $k \in \mathbb{N}$;
    \STATE -- supervised learning algorithm for causal contrast;
\end{algorithmic}
    \vspace{.05in}
\begin{algorithmic}[1]

\STATE  obtain debiased outcome regression estimates $\mu_{n,j}^*$ from $\pi_{n,j}$ and $\mu_{n,j}$ using Algorithm \ref{alg::debiasing}; 
 \vspace{.1in}
\STATE return minimizer $\theta_{n,k}^*$ of $\theta\mapsto \frac{1}{n}\sum_{i =1}^n L_{\mu^*_{n,j(i)}}(\theta, W_i)$. \label{step::constructEPLearner}
\end{algorithmic}
\end{algorithm}

\begin{algorithm}[htb]
   \caption{EP-learner algorithm (cross-validated sieve dimension)}
   \label{alg::CVEPlearner}
   \vspace{.1in}
\begin{algorithmic}
    \REQUIRE    
    \STATE-- dataset $\mathcal{D}_n$ consisting of observations $(W_1,A_1,Y_1),(W_2,A_2,Y_2),\ldots,(W_n,A_n,Y_n)$;
    \STATE -- number of cross-fitting splits $J$;
    \STATE -- growing sequence $\varphi_1,\varphi_2,\ldots,\varphi_{K(n)}$ of feature mappings, with $\varphi_k:\mathbb{R}^d \mapsto \mathbb{R}^{k}$ and $K(n)\in \mathbb{N}$;
    \STATE -- supervised learning algorithm for causal contrast;
\end{algorithmic}
\vspace{.05in}
 \begin{algorithmic}[1]
       \STATE partition $\mathcal{D}_n$ into $J$ datasets $\mathcal{D}_n^1,\mathcal{D}_n^2, \ldots, \mathcal{D}_n^J$ of approximately equal size, say with corresponding index sets $\mathcal{I}^1_n,\mathcal{I}^2_n,\ldots,\mathcal{I}^J_n$;
        \STATE for $j \in [J]$, construct estimates $\pi_{n,j}$ and $\mu_{n,j}$ of $\pi_0$ and $\mu_0$ using $\mathcal{D}_n \,\backslash \, \mathcal{D}_n^j$; 
        
    \FOR{$k=1,2,\dots,K(n)$}
        \FOR{$j=1,2,\dots, J$}
            \STATE obtain debiased estimates $\mu_{n,j,k}^*$ from $\pi_{n,j}$ and $\mu_{n,j}$ using Algorithm \ref{alg::debiasing} with dataset $\mathcal{D}_n \backslash \mathcal{D}_n^j$, cross-fitted estimates $\pi_{n,1},\pi_{n,2},\ldots,\pi_{n,J}$ and $\mu_{n,1},\mu_{n,2},\ldots,\mu_{n,J}$, and feature-mapping $\varphi_k$; 
            \STATE obtain minimizer $\theta_{n,j,k}^*$ of $\theta\mapsto \sum_{i \in [n] \backslash \mathcal{I}_n^j} L_{\mu^*_{n,j(i)}}(\theta, W_i)$.
        \ENDFOR
    \ENDFOR
    \vspace{.1in}
    \STATE compute optimal dimension $k_{cv}(n) := \argmin_{k \in \mathbb{N}} \frac{1}{n} \sum_{i=1}^n
 L_{\pi_{n,j(i)}, \mu_{n,j(i)}}(O_i, \theta_{n,j(i),k}^{*})$;
    \STATE return $\theta_{n,k_{cv}(n)}^*$ obtained using Algorithm \ref{alg::EPlearner} with dataset $\mathcal{D}_n$ and feature mapping $\varphi_{k_{cv}(n)}$.
\end{algorithmic}
\end{algorithm}

 As presented in Algorithm \ref{alg::CVEPlearner}, the cross-validated EP-learner only requires the construction of cross-fitted outcome regression and propensity score estimators once. Nuisance estimators do not need to be recomputed within each training fold. Since the cross-fitted nuisance estimators are constructed using the entire dataset, there is some data leakage across the training folds used for cross-validation. To avoid such data leakage, the nuisance functions could alternatively be re-estimated within each training fold, although this would come at a possibly significant computational cost.

The sieve $\mathcal{H}_{1},\mathcal{H}_{2},\ldots$ used in the EP-learner implementation is an important ingredient that must be specified. In some cases, the transformed action space ${\mathcal{H}}$ suggests a natural choice of sieve. For instance, if ${\mathcal{H}}$ is a reproducing kernel Hilbert space, the linear span of the first $k$ elements of the leading eigenfunctions of the integral kernel operator \citep{mendelson2010regularization} is a canonical choice for $\mathcal{H}_{k}$. More generally, sieves based on the trigonometric polynomials \citep{JacksonOnAB}, B-splines \citep{Gordon1974BSPLINECA}, and wavelets \citep{Antoniadis_2007} are natural candidates with strong approximation guarantees under smoothness assumptions on the action space. \textcolor{black}{In our experiments, we used the tensor-product trigonometric polynomial sieve proposed by \citet{zhang2022regression} and implemented in the associated R package \texttt{Sieve} \citep{Sieve}.} In principle, the cross-validated EP-learner could be modified to select from different basis functions used to define the sieve. For example, Algorithm \ref{alg::EPlearner} can be used to construct EP-learners based on different subsets of trigonometric polynomial and B-spline basis functions.

 \section{Theoretical guarantees for empirical risk minimization}
\label{section::theory}

 \subsection{Efficiency of the EP risk estimator} \label{section::EffEPLearnerRisk}

 In this section, we study the EP risk estimator presented in Algorithm \ref{alg::EPlearner} with a deterministic feature mapping dimension $k=k(n)$ growing with sample size. Specifically, we establish conditions under which this EP risk estimator is equivalent to an oracle-efficient one-step risk estimator. To do so, we show that the debiasing term for the EP-learner risk estimator has a second-order dependence on the estimation error of the outcome regression estimator and the sieve approximation error of the transformed action space ${\mathcal{H}}$. Thus, if this debiasing term converges to zero rapidly enough, the EP-learner risk estimator is $n^{\frac{1}{2}}$--consistent, asymptotically linear, and nonparametric efficient for the population risk uniformly over the action space $\mathcal{F}$, as we formalize below.

\textcolor{black}{In what follows, we say `the loss is linear in \(\mu\)' if both \(g_1\) and
\(g_2\) are the identity map in \eqref{eqn::GeneralClassRisk}, and is not linear in $\mu$ otherwise.} For  $\theta \in L^2(P_{0,W})$, we denote by $\ \|\theta\|_{\infty}$ the $P_0$--essential supremum norm and by ${\Pi_{k}\,}\theta \in \argmin_{\phi \in \mathcal{H}_{k}} \norm{\theta - \phi}$ the  $L^2(P_0)$--projection of $\theta \in \mathcal{H}$ onto the sieve $\mathcal{H}_k$. Below, we make use of the following conditions:

  \begin{enumerate}[label=\bf{C\arabic*)}, ref=C\arabic*,series=cond]
    \item  \label{cond::boundPos}   there exist constants $\eta\in (0,1)$ and $M\in(0,\infty)$ such that, for all $j \in [J]$, for $P_0$--almost every realization $(w,a)$ of $(W,A)$, and with $P_0$--probability tending to one:
    \begin{enumerate}
        \item \textit{strong positivity:} \label{cond::positivity} $\pi_0(a\miid w)\in [\eta,1-\eta]$ and $\pi_{n,j}(a\miid w)\in [\eta,1-\eta]$;
        \item \textit{boundedness:}  \label{cond::boundedEStimators} $|\mu_{n,j}(a,w)| < M$ and $|\mu_{n,j}^*(a,w)| < M$;
    \end{enumerate} 
    \item \textit{polynomial nuisance rates:} \label{cond::A2Nuisance} there exist $\gamma > 0$ and $\beta > \frac{1}{4\gamma}$ such that, for all $j \in [J]$:
    \begin{enumerate}
     \item \textit{propensity score:}  \label{cond::outcomerateFirst} $\norm{\pi_{n,j} - \pi_0} = \mathcal{O}_p\big{(}n^{-\gamma/(2\gamma+1)}\big{)}$;
     \item \textit{outcome regression:}   \label{cond::outcomerateSecond} $\norm{\mu_{n,j} - \mu_0} = \mathcal{O}_p\big{(} n^{-\beta/(2\beta+1)}\big{)}$. \textcolor{black}{Moreover, if the loss is not linear in $\mu$, then $\beta>1/2$, so that $\norm{\mu_{n,j} - \mu_0}=o_p(n^{-1/4})$.}
     
     
     \item \textit{debiased outcome regression:} $\|\mu_{n,j}^* - \mu_0\| = \mathcal{O}_p\big{(}n^{-\beta/(2\beta+1)} + \sqrt{k(n) \log n / n}\big{)}$;\label{cond::outcomerateDebiased}
    \end{enumerate}
     \item \textit{sufficiently small sieve approximation error:}\label{cond::sieveApproxthrm} there exist $\rho > \frac{1}{2}$ and $\nu \geq 0$ such that:
     \begin{enumerate}
            \item \textit{slowly growing Lebesgue constant:}\label{cond::sieveApproxthrm0}  $\sup_{\phi \in \mathcal{H}, v \in \mathcal{H}_{k(n)}}\frac{\|\Pi_{k(n)} (\phi + v)\|_{\infty}}{\|\phi + v\|_{\infty}} = \mathcal{O}\big{(}\{\log k(n)\}^{\nu}\big{)}$;
         \item \textit{polynomial approximation rate:} \label{cond::sieveApproxthrm1}  $\sup_{\theta \in \mathcal{H}} \inf_{\phi \in \mathcal{H}_{k(n)}}\norm{\theta -  \phi}_{\infty} = \mathcal{O}\big{(}k(n)^{-\rho}\big{)}$;
         \item \textit{large enough exponent:}\label{cond::sieveApproxthrm2} if $\min\{\beta, \gamma\} > \frac{1}{2}$, then $\rho \geq \frac{1}{2\beta + 1}$, and otherwise, $\rho \geq \frac{1}{4\min\{\beta, \gamma\}}$. 
     \end{enumerate}

\end{enumerate}

 Conditions \ref{cond::boundPos} and \ref{cond::A2Nuisance} of Theorem~\ref{theorem::EPriskEff} commonly appear in the nonparametric inference literature \citep{bickel1993efficient, shen1997methods,VanderLaanRobins2003, vanderLaanRose2011, DoubleML}. Condition \ref{cond::positivity} strengthens \ref{cond::pos} so that strong positivity also holds for the propensity score estimators. \textcolor{black}{Conditions \ref{cond::outcomerateFirst} and \ref{cond::outcomerateSecond} are satisfied for some $\gamma > 1/2$ and $\beta > 1/2$ if the cross-fitted nuisance estimators converge at a polynomial rate faster than $n^{-1/4}$.} Under suitable smoothness assumptions, such rates can be achieved by several statistical learning methods, including generalized additive models \citep{GAMhastie1987generalized}, reproducing kernel Hilbert space estimators \citep{QuasiOracleWager}, the highly adaptive lasso \citep{vanderlaanGenerlaTMLE}, and certain neural network architectures \citep{Farrell2018DeepNN}. For the CATE and log-CRR risks in Section \ref{section::introexamp}, \textcolor{black}{the loss is linear in $\mu$}, Condition \ref{cond::A2Nuisance} requires that the rate exponent of the outcome regression estimators satisfy $\beta > 1/(4\gamma)$. \textcolor{black}{Condition \ref{cond::outcomerateFirst} and \ref{cond::outcomerateSecond} together represent a rate double robustness condition \citep{bangrobins, rotnitzky2021characterization}.} \textcolor{black}{Condition \ref{cond::outcomerateDebiased} imposes a high-level rate requirement on the sieve-adjusted estimator $\mu_{n,j}^*$. In particular, under \ref{cond::outcomerateSecond} it holds if $\max_{j \in [J]} \|\mu_{n,j}^* - \mu_0\| = \mathcal{O}_p(\max_{j \in [J]} \|\mu_{n,j} - \mu_0\| + \sqrt{k(n)\log n / n})$, where the first and second terms on the right capture the sieve approximation error for $\mu_0$ and the statistical estimation error a $k(n)$-dimensional parameter, respectively \citep{chen2007large}. In Appendix~\ref{appendix::debiasedrate}, we verify this condition explicitly for the case where $(\mu_j^*: j \in [J])$ is obtained from Method 2 in Algorithm~\ref{alg::debiasing}.} Condition \ref{cond::regularityOnActionSpace} controls the complexity of the action space $\mathcal{F}$ in supremum norm. This condition holds with exponent $\alpha := s/d$ when $\mathcal{F}$ falls in a $d$--variate H\"{o}lder or Sobolev smoothness class with smoothness exponent $s > d/2$ (Theorem 2.7.2. of \citealp{vanderVaartWellner}; Corollary 4 of \citealp{nickl2007bracketing}).

By controlling the Lebesgue constant \citep{BelloniSieveApprox} in \ref{cond::sieveApproxthrm0}, approximation theory can be used to bound $\sup_{\theta \in \mathcal{H}}\|\theta - \Pi_{k(n)} \theta\|_{\infty}$ by the sup-norm approximation error of \ref{cond::sieveApproxthrm1} up to logarithmic dependence on $k(n)$ \citep{huang2003asymptotics, chen2013optimal, BelloniSieveApprox}. Only requiring elements of $\mathcal{H}$ to be continuous, Section 3.2 of \cite{BelloniSieveApprox} provides an overview of several sieve choices that satisfy this condition, including those based on trigonometric series, splines, wavelets, and local polynomials. Condition \ref{cond::outcomerateDebiased} is the convergence rate expected for an empirical risk minimizer over a sieve space of dimension $k(n)$ based on a loss indexed by cross-fitted nuisances \citep{orthogonalLearning}. This condition is satisfied under regularity conditions when each $\mu_{n,j}^*$ is obtained using Algorithm \ref{alg::debiasing}. Together, conditions \ref{cond::A2Nuisance} and \ref{cond::sieveApproxthrm} ensure the existence of a sequence of sieve dimensions $k(n)$ such that, for each $j \in [J]$, remainder terms of the order $\|\mu_{n,j}^* - \mu_0 \|\norm{\pi_{n,j} - \pi_0}$ and $\sup_{\phi \in {\mathcal{H}}} \|\phi- \Pi_{k(n)}\phi\| \|\mu_{n,j}^* - \mu_0 \|$ are $\smallO_p(n^{-\frac{1}{2}})$. The precise choice of $k(n)$ depends on $\beta$, $\gamma$ and $\rho$, which are typically unknown. However, under mild conditions, using arguments similar to those in \cite{van2003unified}, it can be shown that the cross-validated EP-learner is equivalent to the infeasible EP-learner that uses an oracle selector of $k(n)$ minimizing the population risk.

Compared to cross-fitted empirical risk estimators \citep{KennedyCATERates, orthogonalLearning}, the nuisance estimator $\mu_{n,j}^{*}$ depends on the entire dataset, and so, the EP-learner loss function $L_{\mu_{n,j}}^{*}$ is not independent of observations in the $j$th fold. Consequently, to obtain tighter control of empirical process remainders \citep{van2014uniform}, we require that the action space $\mathcal{F}$ not be too large with respect to the supremum norm metric. To quantify the complexity of $\mathcal{F}$, we consider the $\varepsilon$--covering number $N_{\infty}(\varepsilon, \mathcal{F})$ with respect to the $P_{0,W}$--essential supremum metric $(\theta_1,\theta_2)\mapsto d_{\infty}(\theta_1,\theta_2):=\|\theta_1-\theta_2\|_{\infty}$ on $\mathcal{F}$ as the smallest number of balls of $d_{\infty}$--radius $\varepsilon$ required to cover $\mathcal{F}$ \citep[Chapter 2 of][]{vanderVaartWellner}, and the corresponding metric entropy integral $\mathcal{J}_{\infty}(\delta, \mathcal{F}) := \int_0^{\delta} \{\log N_{\infty}(\varepsilon, \mathcal{F})\}^{\frac{1}{2}}d\varepsilon$. To establish our results, we require the following regularity conditions on the class of functions $\mathcal{F}$:

  \begin{enumerate}[label=\bf{C\arabic*)}, ref = C\arabic*, resume=cond]
\item \textit{nonparametric, convex action space that is not too large:} \label{cond::regularityOnActionSpace}
   $\mathcal{F}$ is uniformly bounded, convex, and $\mathcal{J}_{\infty}(\delta, \mathcal{F}) \leq C \delta^{1-1/(2\alpha)}$ for some $\alpha > 1/2$ and $C>0$, and for every $\delta > 0$. 
\end{enumerate}

 In the following theorem, $l^\infty(\mathcal{F})$ denotes the Banach space of bounded functionals $b : \mathcal{F}\rightarrow\mathbb{R}$ equipped with the norm $b\mapsto\|b\|_{\ell^\infty(\mathcal{F})}:=\sup_{\theta\in\mathcal{F}}|b(\theta)|$. Finally, for $\theta \in \mathcal{F}$, we denote the EP-learner risk estimator obtained from Algorithm \ref{alg::EPlearner} by $R_{n,k(n)}(\theta) := \frac{1}{n}\sum_{i=1}^n L_{\mu_{n,j(i)}^{*}}(\theta, W_i)$. 

\begin{theorem}[Oracle efficiency of EP-learner risk]
  Suppose that conditions \ref{cond::boundPos}--\ref{cond::regularityOnActionSpace} hold and that $k(n)$ is a deterministic sequence satisfying the rate conditions \[
n^{\tfrac{1}{2\rho(2\beta+1)}} \cdot \frac{(\log n)^{\nu/\rho}}{k(n)}
\;\longrightarrow\; 0, 
\qquad
\frac{k(n)\,\log n}{n^{\tfrac{2c(\beta,\gamma)}{2c(\beta,\gamma)+1}}}
\;\longrightarrow\; 0,
\]
with $c(\beta,\gamma) := \min\{\beta, \gamma, 1/2\}$.
 Then, it holds that \[n^{\frac{1}{2}}\sup_{\theta \in \mathcal{F}}\left| \frac{1}{n}\sum_{i=1}^n \Delta_{\pi_{n,j(i)},\mu_{n,j(i)}^{*}}(O_i; \theta) \right| \stackrel{p}{\longrightarrow}0\ .\] Furthermore, it follows that $\sup_{\theta \in \mathcal{F}} |  n^{\frac{1}{2}} \{ R_{n,k(n)}(\theta)- R_{n,0}(\theta)\}|\stackrel{p}{\longrightarrow}0$, and thus, the stochastic process $(n^{\frac{1}{2}}\{R_{n,k(n)}(\theta) - R_0(\theta) \}: \theta \in \mathcal{F})$ converges weakly in $\ell^\infty(\mathcal{F})$ to the tight Gaussian process $\mathbb{G}$ with mean zero and covariance function $(\theta_1, \theta_2) \mapsto P_0(D_{0,\theta_1} D_{0,\theta_2})$.
  \label{theorem::EPriskEff}
\end{theorem}
We note that the constraints on $\gamma$, $\beta$ and $\rho$ imposed by \ref{cond::A2Nuisance} and \ref{cond::sieveApproxthrm} ensure the existence of a sieve dimension $k(n)$ satisfying the growth rate bounds of Theorem \ref{theorem::EPriskEff}. Since $D_{0,\theta}$ is the (nonparametric) efficient influence function of $P\mapsto R_P(\theta)$, it holds that $R_{n,k(n)}(\theta)$ is a nonparametric efficient estimator of $R_0(\theta)$ for each $\theta\in\mathcal{F}$ \citep{bickel1993efficient}. In addition to being efficient, the estimator $R_{n,k(n)}$ exhibits a form of double robustness whenever \textcolor{black}{the loss is linear in $\mu$;} 
this is the case for both the CATE and log-CRR loss functions introduced in Sections \ref{section::introCATEexample} and \ref{section::introRRexample}. In particular, even when the outcome regression estimators are inconsistent, the EP-learner risk estimator $R_{n,k(n)}(\theta)$ is still consistent for $R_0(\theta)$ if $\|\pi_{n,j} - \pi_0\|$ and $\sup_{\phi \in {\mathcal{H}}} \|\phi- \Pi_{k(n)}\phi\|$ tend to zero in probability---this is studied in Lemmas \ref{lemma::P0boundDR} and \ref{lemma::P0boundSieve2} of the Supplement.

Theorem \ref{theorem::EPriskEff} indicates that no additional correction is needed to ensure the asymptotic linearity and efficiency of the plug-in risk function estimator whenever it is based on the debiased outcome regression estimators $\mu_{n,1}^*,\mu_{n,2}^*,\ldots,\mu_{n,j}^*$. However, this property comes at the cost of condition \ref{cond::sieveApproxthrm}, which requires that the product of the outcome regression estimator rate $\|\mu_{n,j}^* - \mu_0\|$ and sieve approximation error $\sup_{\phi \in \mathcal{H}} \|\phi - \Pi_{k(n)} \phi \|_{\infty}$ be $\smallO_p(n^{-\frac{1}{2}})$. Fixing nuisance rate exponents $\beta$ and $\gamma$ of \ref{cond::A2Nuisance}, in order to obtain a sieve approximation rate satisfying \ref{cond::sieveApproxthrm}, elements in the transformed action space ${\mathcal{H}}$ must generally be sufficiently smooth with respect to the sieve basis. If $\mathcal{H}$ is a subset of the class of $d$--variate H\"{o}lder smooth functions with smoothness exponent $s > d/2$, this condition is known to hold with $\rho:=s/d$ for appropriate sieves based on trigonometric, spline, wavelets, and local polynomial basis functions (Section 2.3 of \citealp{chen2007large}; Section 3 of \citealp{BelloniSieveApprox}). For appropriate sieves based on trigonometric and wavelet basis functions, condition \ref{cond::sieveApproxthrm} also holds with $\rho := s/d$ when $\mathcal{H}$ is contained in the $d$--variate Sobolev class of smoothness exponent $s > d/2$ \citep{kon2002sup, cobos2016optimal}.

\subsection{Convergence rate guarantees for ERM-based EP-learners}
 \label{Section::ERMEPLearner}

  We now establish properties of an EP-learner of the population risk minimizer $\theta_0 := \argmin_{\theta \in \mathcal{F}} R_0(\theta)$ constructed using empirical risk minimization. We refer to this special case of Algorithm \ref{alg::EPlearner} as the ERM-based EP-learner $\theta_{n,k(n)}^*$ defined as any element of $\argmin_{\theta \in \mathcal{F}} R_{n, k(n)}(\theta)$. To simplify our results, we assume that $\theta_{n,k(n)}^*$ always exists. However, when that is not the case, it suffices to identify a near-minimizer in the sense of \cite{orthogonalLearning}.

To theoretically evaluate the performance of the ERM-based EP-learner, we compare it with the oracle learner $\theta_{n,0} \in \argmin_{\theta \in \mathcal{F}} R_{n, \pi_0, \mu_0}(\theta)$, which minimizes the oracle efficient one-step risk estimator over $\mathcal{F}$. Before studying the ERM-based EP-learner, we present an upper bound on the rate of the oracle learner. For the following theorem and the remainder of this section, the exponent $\alpha > 1/2$ is the entropy integral exponent from \ref{cond::regularityOnActionSpace}.

 \begin{theorem}[Oracle learner rate]
  \label{theorem::oracleRate}
 Under \ref{cond::positivity} and \ref{cond::regularityOnActionSpace}, it holds that
 $\|\theta_{n,0} - \theta_0\| = \mathcal{O}_p(n^{-\alpha/(2\alpha +1)})$.
 \end{theorem}

If $\mathcal{F}$ consists of $d$--variate, compactly-supported, H\"{o}lder-smooth functions with H\"{o}lder exponent $s > d/2$, the entropy integral satisfies \ref{cond::regularityOnActionSpace} with $\alpha = s/d$ \citep[Theorem 2.7.2. of][]{vanderVaartWellner}, which gives an oracle rate of $n^{-\alpha/(2\alpha + 1)}$. When the target of interest is the value of the CATE function at a point, \cite{KennedyMiximax} shows that this rate is minimax for an oracle learner based on observing counterfactual outcomes. The results in \cite{YangBarron} can similarly be used to show that this oracle rate is minimax when the goal is instead estimation of the CATE function in an $L^2(P_{0,W})$--sense. Hence, at least in the CATE estimation setting, this suggests that the rate for H\"{o}lder smooth functions provided by Theorem~\ref{theorem::oracleRate} cannot be improved. In light of this, throughout, we will refer to $n^{-\alpha/(2\alpha+1)}$ as the oracle rate.

The following theorem provides an upper bound on the convergence rate of the ERM-based EP-learner. To obtain fast rates for the ERM-based EP-learner, we impose the following coupling between the supremum and $L^2(P_0)$ norms:

\begin{enumerate}[label=\bf{C\arabic*)}, ref = C\arabic*, resume=cond]
\item \label{cond::theorem2::couplings}  \textit{sup-norm coupling:} there exists $C>0$ such that $ \norm{ \theta }_{\infty} \leq C \norm{ \theta}^{1-1/(2\alpha)}$ for all $\theta \in \mathcal{F} - \mathcal{F}$.
\end{enumerate}
Condition \ref{cond::theorem2::couplings} is known to hold with exponent $\alpha := 1/d$ when $\mathcal{F}$ is a subset of the $d$--variate H\"{o}lder smoothness class of order $s =1$ \citep[Lemma 4 of][]{bibaut2021sequential}. Moreover, this condition holds for appropriate reproducing kernel Hilbert spaces \citep[Lemma 5.1 of ][]{mendelson2010regularization}; in particular, it holds with exponent $\alpha := s/d$ when $\mathcal{F}$ is a subset of the $d$--variate Sobolev smoothness class of order $s > d/2$. This condition also holds when $\mathcal{F}$ equals the signed convex hull of appropriate basis functions \citep[Lemma 2 of][]{van2014uniform}. \textcolor{black}{For empirical risk minimization within an RKHS, such sup-norm couplings have been used to establish fast rates under $L^2$ convergence of the nuisance estimators \citep{QuasiOracleWager, orthogonalLearning}.}

Our next result describes how close the ERM-based EP-learner and its oracle counterpart are, and involves the sequence $\varepsilon_n^2:=\min\{a_n,b_n\}$ with \begin{align*}
    a_n&:=\frac{(\log n)^{2\nu}}{k(n)^\rho}\left[\frac{1}{n^{\beta/(2\beta+1)}}  + \left\{\frac{k(n) \log n}{n}\right\}^{1/2} + \left\{\frac{k(n)^{(\rho/\alpha)}}{n}\right\}^{1/2}\right];\\
    b_n&:=\frac{1}{n^{2\beta/(2\beta+1)}} +  \frac{k(n) \log n}{n} +  \frac{(\log n)^{2\nu}}{n^{1/2}k(n)^{\rho\{1- 1/(2\alpha)\}}}\ .\label{eqn::rateEPlearnerEpsn}
\end{align*}

  \begin{theorem}[EP-learner convergence rate for a deterministic sieve growth rate]
  Suppose that conditions \ref{cond::boundPos}, \ref{cond::A2Nuisance}, \ref{cond::sieveApproxthrm0} , \ref{cond::sieveApproxthrm1}, \ref{cond::regularityOnActionSpace} and \ref{cond::theorem2::couplings} hold. Then, it holds that $$\|\theta_{n,k(n)}^* - \theta_{n,0}\| = \mathcal{O}_p\left(\varepsilon_n\right)+\smallO_p\left(n^{-\alpha/(2\alpha+1)}\right)$$
  for any sieve growth rate satisfying $(\log n)k(n)/n^{2c(\beta,\gamma)/\{2c(\beta,\gamma)+1\}}\rightarrow 0$.

  \label{theorem::EpLearnerRate}
 \end{theorem}

 Theorem \ref{theorem::EpLearnerRate} establishes that, under certain conditions, the rate of the root-mean-squared error between the EP-learner and the oracle learner is dominated by the oracle rate up to a remainder term $\varepsilon_n$ that depends on nuisance estimation rates and the sieve approximation error. In view of the $b_n$ term, the EP-learner rate is no worse than the debiased outcome regression estimator rate $n^{-\beta/(2\beta+1)} + \{k(n) \log n /n\}^{1/2}$ as long as (i) $\theta_0$ is at least as smooth as $\mu_0$ so that $\beta \leq \alpha$; and (ii) the sieve growth rate is fast enough so that $ n^{\beta/\{\rho(2\beta+1)\}}(\log n)^{2\nu\beta/\rho}/k(n)\rightarrow 0$. \textcolor{black}{The rate in Theorem~\ref{theorem::EpLearnerRate} does not appear to depend directly on $\gamma$ or $\beta$, except through the sieve approximation error term $\varepsilon_n$. This is because \ref{cond::A2Nuisance} ensures that the errors of the initial nuisance estimators can be absorbed into the $\smallO_p(n^{-\alpha/(2\alpha+1)})$ term.}

 The next result demonstrates that, under appropriate conditions, the EP-learner is oracle-efficient in the sense of \cite{KennedyCATERates} so long as the oracle rate is tight in expectation and the sieve approximation error and nuisance estimation rates tend to zero quickly enough. To achieve oracle efficiency, we require $k(n)$, $\beta$ and $\rho$ to be such that $n^{\alpha/(2\alpha+1)}\varepsilon_n\rightarrow 0$, which mandates that $k(n)$ grow at a faster rate than imposed by Theorem \ref{theorem::EPriskEff}. We also require the following additional condition:

  \begin{enumerate}[label=\bf{C\arabic*)}, ref = C\arabic*, resume=cond]
 \item  \textit{tight oracle learner rate:} $n^{-\alpha/(2\alpha+1)} / E_0^n\norm{\theta_{n,0} - \theta_0} = \mathcal{O}(1)$.\label{cond::theorem2::oracleRateTight}
\end{enumerate}

\begin{theorem} [Oracle efficiency of ERM-based EP-learner]   \label{theorem::EpLearnerOracleEff}
\textcolor{black}{Suppose that conditions \ref{cond::theorem2::oracleRateTight}, \ref{cond::A2Nuisance} with $\gamma >  1/2$ and $\beta >  \max\{\tfrac{1}{4\gamma}, \frac{2\alpha}{4\alpha^2+1}\}$,} and  \ref{cond::sieveApproxthrm1} with $\rho = \alpha$ hold in addition to the conditions of Theorem \ref{theorem::EpLearnerRate}. Then,  \[\|\theta_{n, k(n)}^* - \theta_{n,0}\|=\smallO_p\left(E_0^n\|\theta_{n,0} - \theta_0\|\right)\] for any sieve growth rate satisfying $\frac{(\log n)\,k(n)}{n^{2c(\beta,\gamma)/(2c(\beta,\gamma)+1)}}\to 0$   and $a_n n^{\tfrac{2\alpha}{2\alpha+1}} \rightarrow 0 $.


 \end{theorem}

By the reverse triangle inequality, Theorem \ref{theorem::EpLearnerOracleEff} implies that $\|\theta_{n,k(n)}^* - \theta_0\| - \|\theta_{n,0} - \theta_0\|$ tends to zero in probability faster than $E_0^n\|\theta_{n,0} - \theta_0\|$. If this statement can be strengthened to convergence in expectation rather than in probability, it would follow that $E_0^n\|\theta_{n, k(n)}^* - \theta_{0}\|/E_0^n\|\theta_{n,0} - \theta_{0}\| \rightarrow 1$. In such case, $\theta_{n, k(n)}^*$ would converge in expected mean squared error at the same rate and constant as the oracle learner.

In view of Lemma \label{lem:exist-k} in Appendix \ref{appendix::proofsoracle}, there exists a sieve growth rate $k(n)$ satisfying the growth bounds of Theorem \ref{theorem::EpLearnerOracleEff}. It is not immediately obvious that it is possible to obtain a sieve approximation rate exponent $\rho$ equal to the metric entropy and a supremum norm coupling exponent $\alpha$ satisfying \ref{cond::regularityOnActionSpace} and \ref{cond::theorem2::couplings}, as is required in Theorem~\ref{theorem::EpLearnerRate}. However, this is indeed guaranteed if, for example, $\mathcal{F}$ and $\mathcal{H}$ consist of $d$--variate functions that are H\"{o}lder smooth with order $s=1$ or Sobolev smooth with order $s > d/2$.

\subsection{Discussion and related literature}

\textcolor{black}{Under the conditions of Theorem \ref{theorem::EpLearnerOracleEff}, EP-learner is oracle efficient if both nuisance estimators converge in $L^2$ at polynomial rates faster than $n^{-1/4}$ ($\gamma > 1/2$ and $\beta > 1/2$). These rate conditions are analogous to those for orthogonal learning when no additional structure is assumed beyond \ref{cond::regularityOnActionSpace} and \ref{cond::theorem2::couplings} (Theorem~3 of \citealp{QuasiOracleWager}; Example~1 of \citealp{orthogonalLearning}). We expect that alternative conditions can be obtained by working with $L^4$ rates \citep[cf.][Proposition 1]{orthogonalLearning}. Here, we instead follow \citet{QuasiOracleWager} in formulating
conditions in terms of $L^2$ rates, as these are more widely established.}

{\color{black}In doubly robust settings,  DR and other orthogonal learners can achieve
oracle efficiency when the outcome regression
converges more slowly than $n^{-1/4}$, provided the
propensity is estimated sufficiently quickly so that the product of the $L^2$ errors is $o_p(n^{-1/2})$. In this regime, EP-learner's outcome regression can similarly converge slower than $n^{-1/4}$, but the rate requirement is dictated by the more stringent of the usual product rate condition ($\beta>1/[4\gamma]$) and a condition dictated by the smoothness of the causal summary ($\beta>2\alpha/[4\alpha^2+1]$). The latter requirement is imposed to ensure the sieve approximation error is negligible. As $\alpha \to 1/2$, this condition is at its most restrictive, requiring an $o_p(n^{-1/4})$ rate of convergence for the outcome regression. However, once $\alpha$ is large enough (more than $2\gamma - 1/[8\gamma]$), the usual product rate condition is the more restrictive, and so the requirement on the outcome regression rate is the same as for orthogonal learners.}

\textcolor{black}{The rates achieved by many Neyman-orthogonal learning strategies---including the DR-learner, R-learner, and EP-learner---may be suboptimal when additional structure is imposed on the causal contrast and nuisance functions. In particular, for CATE estimation under H\"older smoothness of the estimand, the propensity score, and the outcome regression, these learners fail to attain the minimax rate and oracle efficiency in low-smoothness regimes of the nuisances. By contrast, specialized procedures that exploit H\"older regularity and higher-order bias corrections can achieve minimax-optimal rates in such settings \citep{KennedyCATERates, KennedyMiximax, fisher2023connection}. However, even for marginal treatment effect estimation, such rate improvements are known to not be achievable in structure-agnostic settings, where only high-level $L^2$ rate conditions are imposed on the nuisances \citep{balakrishnan2023fundamental,
jin2024structure}.}


\section{Example of when EP-learners are more robust than DR-learners: $K$-nearest neighbors}

\label{section::knn}

\textcolor{black}{In this section, we provide theoretical support for how EP-learners can outperform orthogonal learners, such as DR-learners, by being less sensitive to tuning parameter choices, as discussed in Section \ref{section::introCATEexample}. Focusing on estimation of the CATE using $K$-nearest neighbors, we show that if the adjusted T-learner $\mu_n^*(1,\cdot) - \mu_n^*(0,\cdot)$, used to construct pseudo-outcomes for EP-learning, is uniformly consistent, then EP-learner remains consistent even with a fixed number of neighbors as $n$ grows.} 

Consider estimation of the CATE based on the population risk function in Section \ref{section::introCATEexample}. Let $\mu_n^*$ be constructed from initial estimators $\pi_n$ and $\mu_n$ of $\pi_0$ and $\mu_0$ using Algorithm \ref{alg::debiasing}, where for each $j \in [J]$ we set $\pi_{n,j} = \pi_n$ and $\mu_{n,j} = \mu_n$. Let $\theta_{n,EP}$ denote the $K$-nearest neighbors EP-learner, defined pointwise for each $w \in (0,1)^d$ as
\[
\theta_{n,EP}(w) = \frac{1}{K}\sum_{i=1}^n 1\{i \in \mathcal{I}_{K,n}(w)\}\left[ \mu_n^*(1,W_i) - \mu_n^*(0,W_i) \right],
\]
where $\mathcal{I}_{K,n}(w)$ denotes the indices of the $K$ nearest observations to $w$ under the Euclidean norm, with ties broken arbitrarily. Our main result is the following theorem,.

\begin{theorem}\label{theorem:knn}
Suppose $W$ is uniformly distributed on $[0,1]^d$ and that $\mu_0$ is Lipschitz-continuous on $\{0,1\}\times [0,1]^d$. \textcolor{black}{At each sample size $n$, let the number of neighbors $K_n$ be an arbitrary element of $[n]$.} Then, 
$$\sup_{w \in (0,1)^d} \left|\theta_{n,EP}(w) - \theta_0^-(w) \right| \lesssim \sup_{w \in (0,1)^d} \left|\mu_n^*(1, w) - \mu_n^*(0,w) - \theta_0^-(w)\right| + \mathcal{O}_p\left( (K_n \log n/n)^{1/d}\right). $$ 
\end{theorem}

Importantly, the above establishes consistency provided $K_n=o(n/\log(n))$ and $\mu_n^*(1,\cdot)-\mu_n^*(0,\cdot)$ is uniformly consistent; special case of this is that $K_n=1$ for all $n$.  Hence, for $K$-nearest neighbors, EP-learner is more robust to suboptimal tuning than the DR-learner, as illustrated in Figure~\ref{fig:DRlearnerBoundViol}, where it remains stable across a wider range of tuning parameters.

\section{Numerical experiments} \label{Section::sims}

\subsection{Evaluating EP-learner}

We evaluated the empirical performance of the implementation of our cross-validated EP-learner described in Algorithm \ref{alg::CVEPlearner} through three experiments based on the CATE and log-CRR risk functions introduced in Sections \ref{section::introCATEexample} and \ref{section::introRRexample}.

In the first two experiments, we considered estimation of the CATE function. In the first experiment, the covariate vector $W$ was taken to have dimension $d=3$, whereas in the second, it was taken to have dimension $d=20$ but with only $5$ active components. \textcolor{black}{The goal of the second experiment is to assess how EP-learner performs under sparsity of the causal summary. Because the sieve uses all covariates without penalization, it is natural to ask whether including basis functions for inactive variables degrades performance.} In the third experiment, we considered estimation of a log-CRR with a covariate vector of dimension $d=3$. In each experiment, we further considered four settings characterized by estimation based on a simple versus moderately complex functional form, and on limited versus moderate propensity overlap across treatment arms. Additional details on the design of these simulation experiments are provided in Appendix \ref{appendix::experiments}.

In all experiments conducted, the data-generating mechanism implied a causal contrast additive in the covariates, and the EP-learner sieve was constructed using a univariate trigonometric cosine polynomial basis, as considered in \cite{zhang2022regression}. The truncation level for the trigonometric series was selected from the first six leading frequencies of the trigonometric series using cross-validation. To obtain initial cross-fitted propensity score and outcome regression estimates, ensemble learning based on cross-validation and a library of candidate algorithms including gradient-boosted trees (xgboost) with tree depths ranging from 1 to 8 and generalized additive models (GAMs) was used. The performance of EP-learner was compared to that of DR-learner, R-learner \citep{QuasiOracleWager}, T-learner, and a causal forest estimator \citep{atheywager2014causalforest} for the CATE in the first two experiments. We considered various EP-learners, DR-learners, T-learners, and R-learners, each differing in the supervised learning algorithm used in Step 2 of Algorithm \ref{alg::EPlearner}. The supervised learning algorithms considered were GAMs, multivariate adaptive regression splines (MARS), random forests (RF), and xgboost, and are implemented using the \texttt{R} packages \texttt{gam} \citep{hastie2015package}, \texttt{earth} \citep{milborrow2019package}, \texttt{ranger} \citep{ranger}, and \texttt{xgboost} \citep{xgboost}. \textcolor{black}{We note that only GAM (for a fixed tuning parameter) is an example of an empirical risk minimizer, with $\mathcal{F}$ taken as the cubic smoothing spline class \citep{HastieTibshirani1990}, whereas all other methods are greedy estimators that select tree splits or basis functions in an adaptive manner.} For the R-learner, we omitted the use of MARS  because the weight implementation of \texttt{earth} is computationally prohibitive. The causal forest estimator used is implemented using the R package \texttt{grf} and was internally tuned using the tune.parameters=``all" argument of the \texttt{causal\_forest()} function \citep{tibshirani2018package}. For the DR-learner, R-learner, T-learner, and EP-learners, tuning parameters were selected using 10-fold cross-validation based on the DR-learner loss.

In the third experiment, we compare the performance of EP-learner for log-CRR estimation with that of the T-learner and of the IPW-based E-learner introduced in \cite{JiangEntropy}. We did not consider learners based on the one-step debiased risk estimator due to the nonconvexity of its corresponding loss function. The base-learner MARS was omitted for all causal learners as the \texttt{R} package \texttt{earth} does not support non-binary outcomes for logistic regression. For similar reasons, we used the \texttt{xgboost} (rather than the \texttt{ranger}) implementation of random forests. Apart from these changes, the base statistical learning algorithms used for the causal learners were the same as those used for CATE estimation.

For each experimental setup and sample size $n\in\{500,1000,2000, 3000, 4000,5000\}$, we generated 1000 simulated datasets. The Monte Carlo empirical mean squared error was computed for each EP-learner and its competitors. Results for the CATE experiments with a complex functional form and moderate treatment overlap are displayed in Figure \ref{figSim:CATELowDimMain}, whereas results for the log-CRR experiment with limited and moderate treatment overlap are shown in Figure \ref{figSim:CRRMain}. Results for simple causal contrasts and settings with limited treatment overlap are qualitatively similar and summarized in Appendix~\ref{appendix::experiments}.

\begin{figure}[!htbp]
 \centering

  \begin{subfigure}[b]{0.45\textwidth}
   \includegraphics[width=0.5\textwidth]{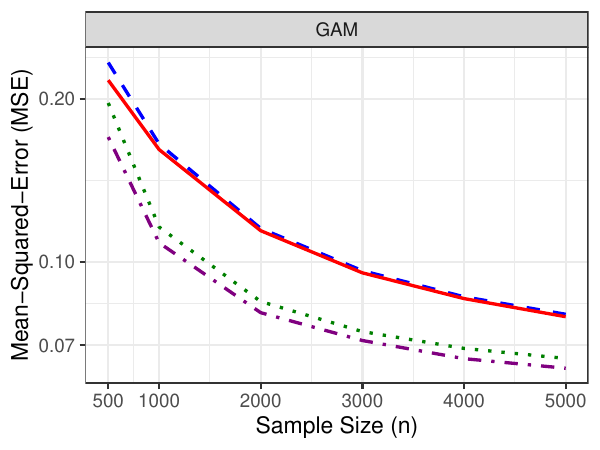}\includegraphics[width=0.5\textwidth]{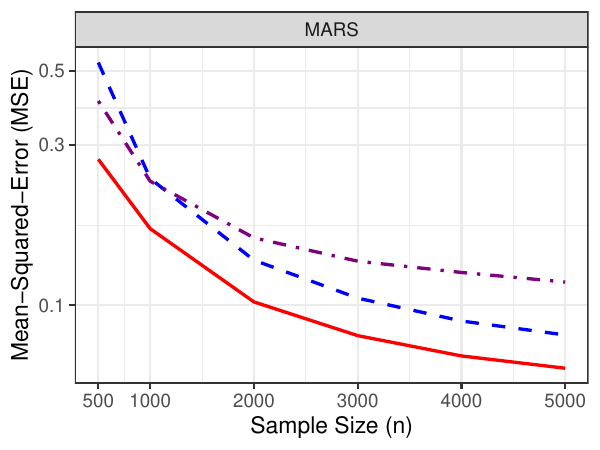}   
    \includegraphics[width=0.5\textwidth]{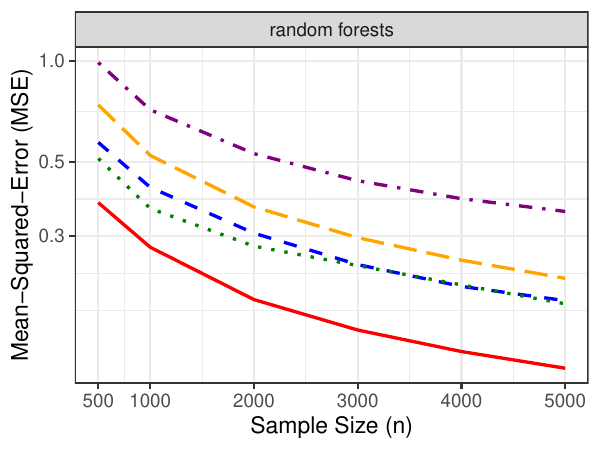}\includegraphics[width=0.5\textwidth]{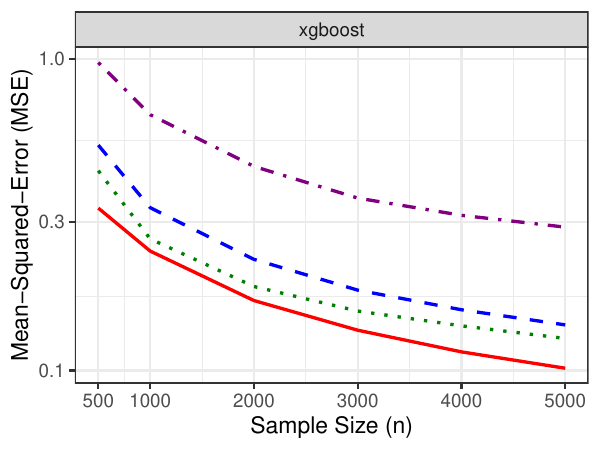}
    \subcaption{$d=3$}
    \end{subfigure}\hfill\begin{subfigure}[b]{0.45\textwidth}
\includegraphics[width=0.5\textwidth]{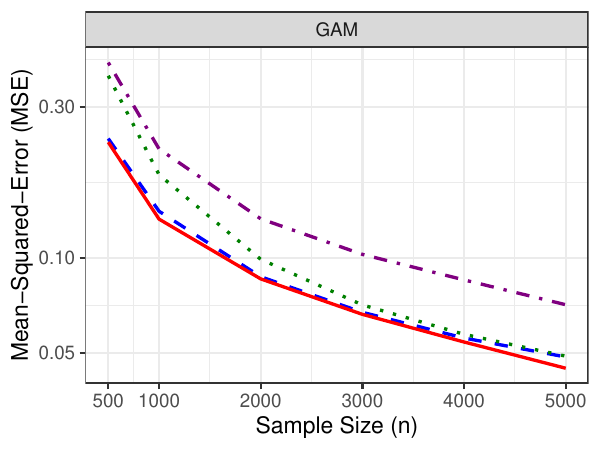}\includegraphics[width=0.5\textwidth]{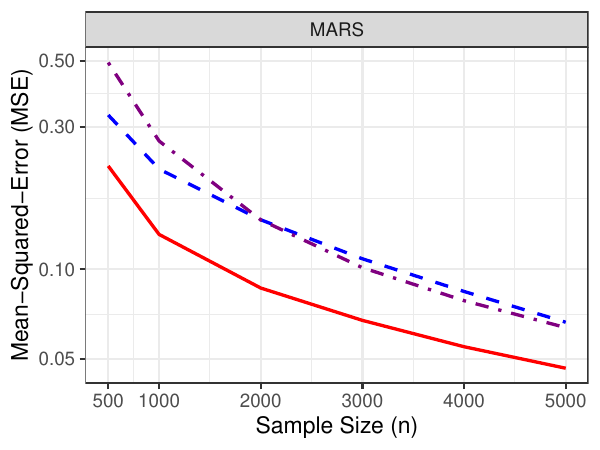}   
   \includegraphics[width=0.5\textwidth]{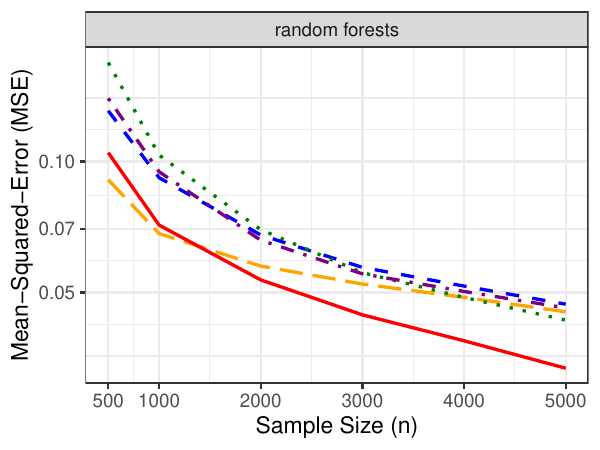}\includegraphics[width=0.5\textwidth]{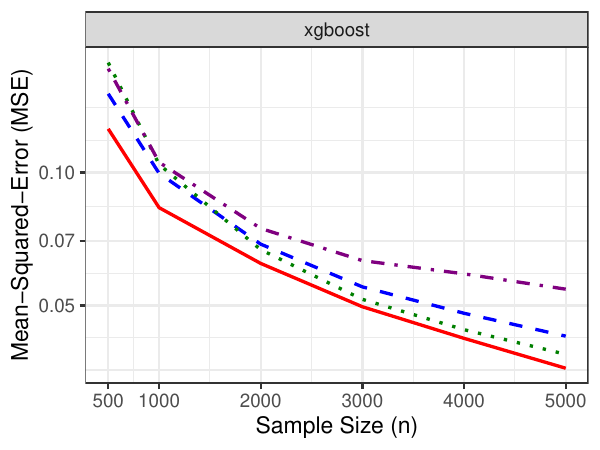}
   \subcaption{$d=20$ with $5$ active}
    \end{subfigure} 
 \includegraphics[width=\textwidth]{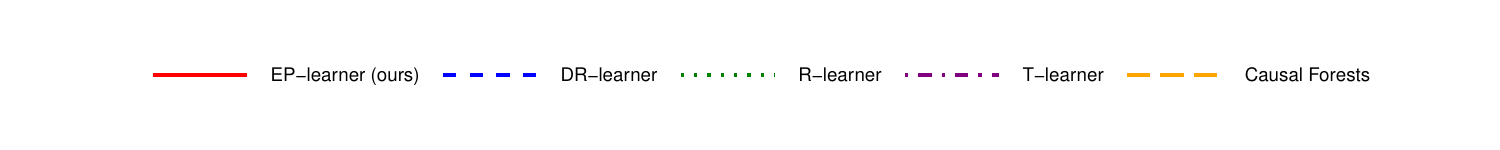}  
    \caption{
    CATE experiments with a complex CATE and moderate treatment overlap: Mean squared error for DR-learner, R-learner, T-learner, and cross-validated EP-learner with supervised learning algorithm GAM, MARS, \texttt{ranger}, and \texttt{xgboost}. For the plots reporting the results of \texttt{ranger}, we also display the results of causal forests for comparison. 
    }  
 
      \label{figSim:CATELowDimMain}
\end{figure}

 \begin{figure}[!htbp]
    \centering
    \begin{subfigure}[b]{0.45\textwidth}
    \includegraphics[width=0.5\textwidth]{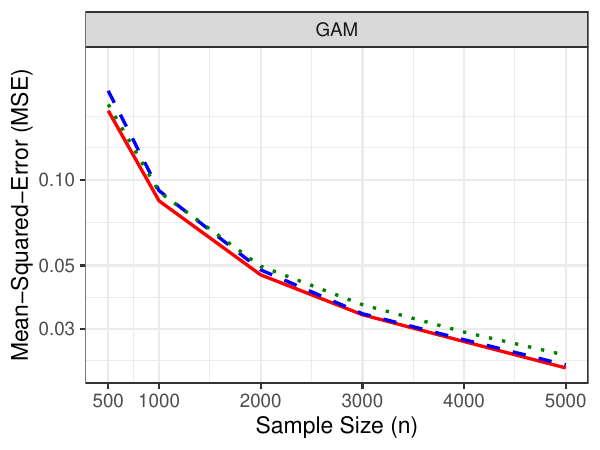}\includegraphics[width=0.5\textwidth]{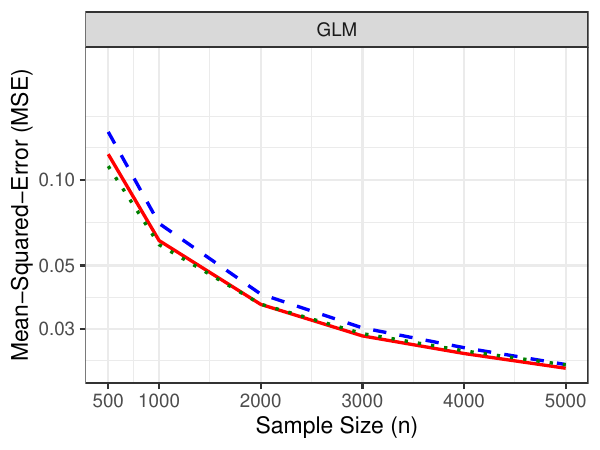}  
    \includegraphics[width=0.5\textwidth]{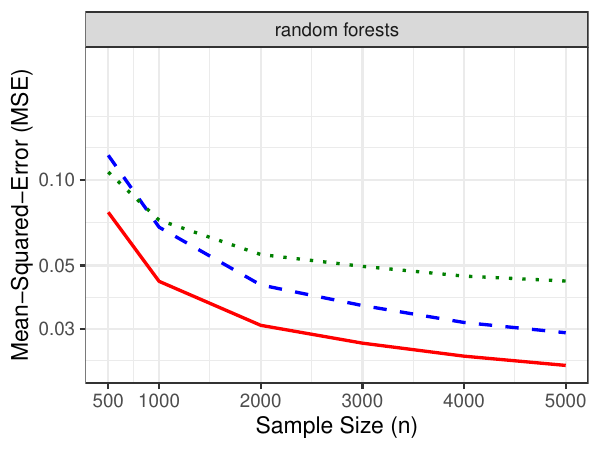}\includegraphics[width=0.5\textwidth]{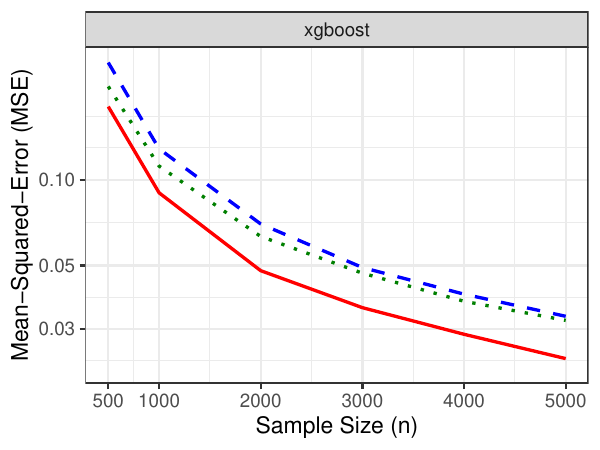}
    \subcaption{
    Moderate treatment overlap}
    \end{subfigure}\hfill\begin{subfigure}[b]{0.45\textwidth}
    \includegraphics[width=0.5\textwidth]{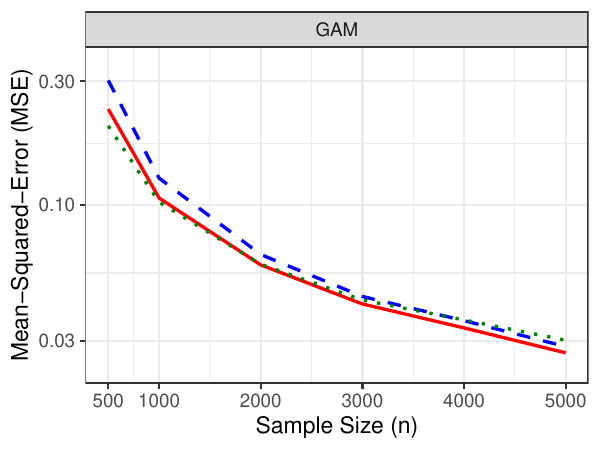}\includegraphics[width=0.5\textwidth]{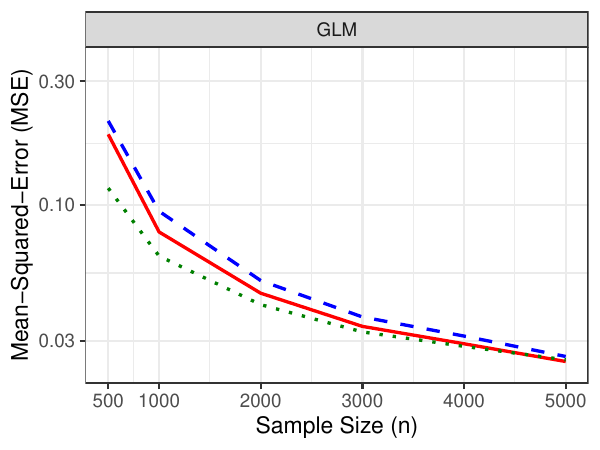}  
    \includegraphics[width=0.5\textwidth]{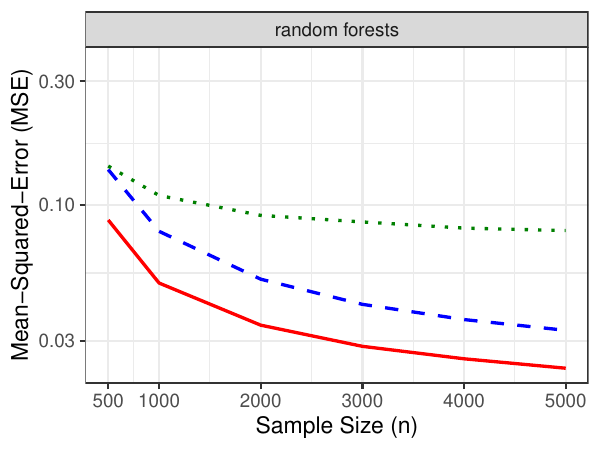}\includegraphics[width=0.5\textwidth]{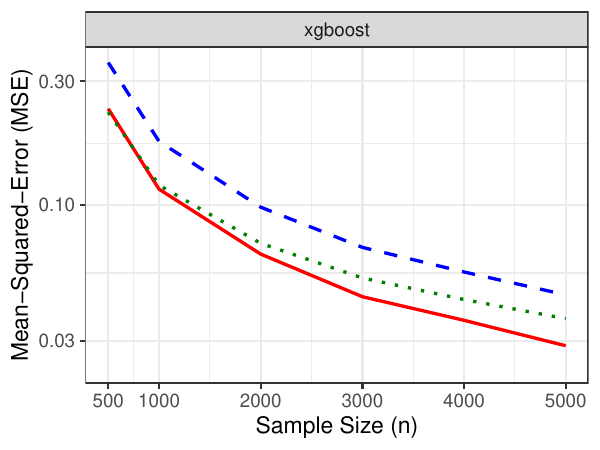}
    \subcaption{
    Limited treatment overlap}
    \end{subfigure}
     
 \includegraphics[width=\textwidth]{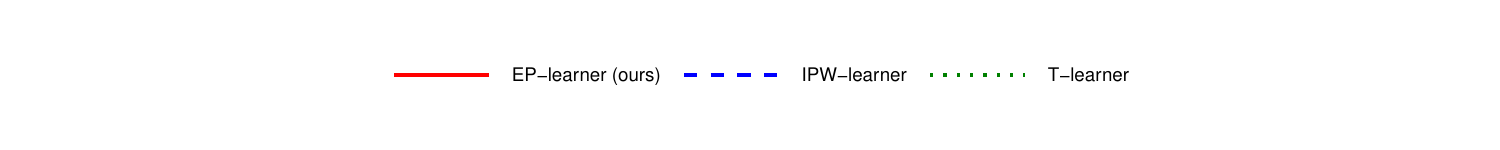}  
 \caption{
  CRR experiments with a complex CRR and moderate (left) and limited (right) treatment overlap: Mean-squared error for IPW-learner, T-learner, and CV-EP-learner with supervised learning algorithm GAM, random forests, and xgboost. The DR-learner algorithm for the CRR is not implemented due to nonconvexity of the loss.
    }

      \label{figSim:CRRMain}
\end{figure}

\textcolor{black}{In most experimental settings and simulation scenarios, EP-learner performed at least as well as---and in some cases better than---all other methods considered, particularly when using tree-based methods. In the case of CATE learners based on GAMs, EP-learner and DR-learner had comparable performance across all sample sizes. However, the R-learner and T-learner outperformed EP-learner and DR-learner for GAM with $d=3$, suggesting that plug-in estimation and overlap weighting were beneficial for improving efficiency in this case.} For more flexible algorithms, such as tree-based learners and MARS, EP-learner significantly outperformed DR-learner in mean squared error across all settings, even for large sample sizes. For instance, in the CATE experiments with \texttt{xgboost}, EP-learner reduced the mean squared error approximately three-fold compared to DR-learner at sample size $n=5000$. In all CATE experiments, R-learner and causal forests performed well. In the CATE experiment with covariate dimension $d=3$, the \texttt{ranger} and \texttt{xgboost} EP-learners exhibited better mean squared error than R-learner and causal forests across all settings and sample sizes. In the CATE experiment with covariate dimension $d=20$, causal forests slightly outperformed the \texttt{ranger} and \texttt{xgboost} EP-learners for smaller sample sizes, but EP-learners perform best for larger sample sizes. In the log-CRR experiments, we found that EP-learners generally performed as well or better than the T-learner, depending on the degree of treatment overlap. The IPW-learner exhibited significantly worse performance than EP-learner in most settings.

\subsection{Comparing with R-learner}

\textcolor{black}{We reproduce the four benchmark settings (A–D) from \citet{QuasiOracleWager}, where covariates $X$ are drawn from a distribution $P_d$, treatment is assigned via a propensity $e^\star(x)$, and outcomes follow $Y=b^\star(X)+(W-\tfrac12)\tau^\star(X)+\sigma\varepsilon$ with $W\mid X\sim \mathrm{Bernoulli}(e^\star(X))$ and $\varepsilon\sim\mathcal{N}(0,1)$. Setup~A features complex nuisances but a simple treatment effect; Setup~B mimics a randomized trial with no confounding; Setup~C imposes strong confounding with an easy propensity, a difficult baseline, and a constant treatment effect; and Setup~D generates uncorrelated treated and control outcomes so joint learning offers no benefit. We fix $d=5$ and $\sigma=1$, vary $n\in\{500,1000,2000\}$, and use the exact data-generating functions from \citet{QuasiOracleWager}.}

\textcolor{black}{
Figure~\ref{fig:wager_replication} reports the mean squared error (MSE) of CATE estimators across sample sizes $n=500,1000,2000$ under settings A--D, averaged over 100 independent data draws. The EP-learner consistently outperformed or matched the performance of the DR-learner and R-learner across all settings when the base learner was XGBoost or Random Forest, while all methods performed similarly with GAM. The EP-learner also outperformed the T-learner except in Setup~D, where the strong performance of the T-learner is expected since the treated and control models are unrelated and separate fitting naturally aligns with the data-generating process. With random forests, EP-learner was competitive with causal forests (the
random-forest–based R-learner from \texttt{grf}): it outperformed them in
Setups~B and~D and matched their performance in larger samples for Setup~A. Setup~C features a constant CATE and overlap weighting used by the R-learner and causal forests is expected to be beneficial since extrapolation from regions of high overlap to those of lower overlap is straightforward. This is reflected in R-learner and causal forests generally outperforming both EP-learner and DR-learner. Overall, EP-learner
either outperforms or matches the DR-learner and R-learner with gradient
boosting or random forests across diverse regimes, while remaining competitive
with state-of-the-art causal forests.
}

\begin{figure}
    \centering
    \includegraphics[width=1\linewidth]{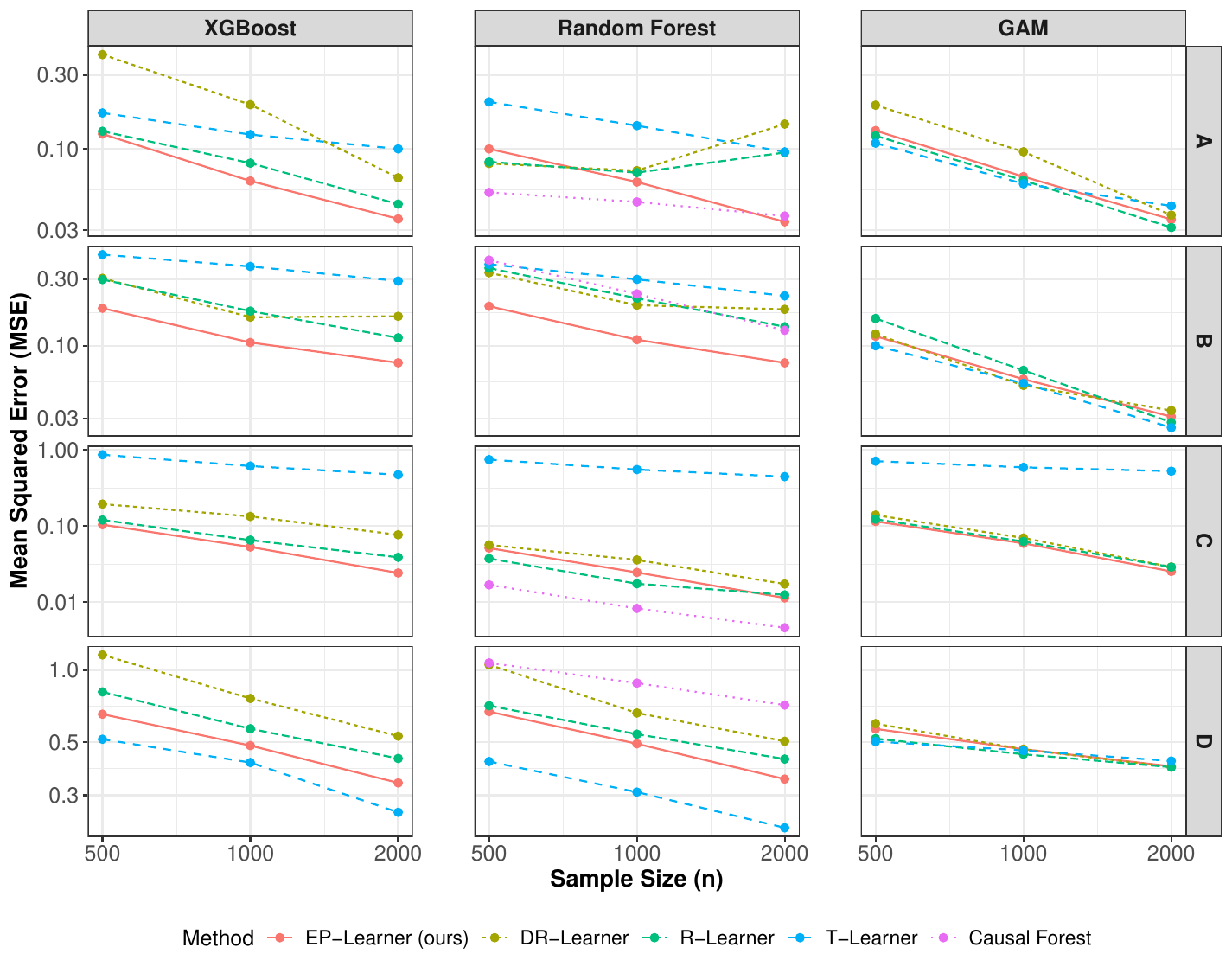}
    \caption{Mean squared error (MSE) of CATE estimates across sample sizes $n \in \{500,1000,2000\}$ under the four simulation settings (A--D) of \citet{QuasiOracleWager}, comparing EP-learner, R-learner, DR-learner, and T-learner across different base learners (XGBoost, Random Forest, GAM). For comparison, the causal forest implementation from \texttt{grf} is also displayed under the Random Forest base learner.}
    \label{fig:wager_replication}
\end{figure}

\section{Conclusion}

In this work, we have introduced a framework for estimating a heterogeneous causal contrast based on a particular construction of an efficient plug-in risk function estimator. To the best of our knowledge, our construction gives the first doubly-robust, efficient estimator of a population risk function for the log-CRR function that corresponds to a convex loss function. While our primary focus has been on EP-learners for use in settings with a single time-point, discrete treatment, our framework can be extended to contexts involving continuous or longitudinal treatments. In future research, it would be of interest to explore the development of EP-learners for broader causal summaries, such as conditional treatment effects under longitudinal interventions \citep{seqDRLuedtke,rotnitzky2017multiply}. In contrast to the plug-in methods presented in these earlier works, EP-learner would make it possible to leverage parsimony found in contrasts between counterfactual means under longitudinal interventions but otherwise not found in the counterfactual means themselves, for example. 

To determine sieve specifications for EP-learner, including basis function types and dimensions, we proposed using the cross-validated one-step risk estimator. In high-dimensional settings, defining a candidate basis for constructing a sieve can be challenging. In our simulation experiments, additive sieves based on univariate trigonometric series performed well, even with complex action spaces like random forests and boosted trees. However, incorporating tensor-product sieves with higher interaction degrees could potentially improve the performance of EP-learner. Alternatively, it may be fruitful to data-adaptively learn a subset of variables that drive treatment effect heterogeneity and then use basis functions of those variables. In the case of estimating the covariate-adjusted conditional mean, \cite{vansteelandt2023orthogonal} proposes a penalized sieve method for constructing debiased plug-in risk estimators. Adapting our approach to incorporate penalization would be an interesting area for future work.  

\textcolor{black}{While in this work we assumed access to an estimator of the propensity, from which estimates of the inverse propensity were obtained, we note that EP-learner can also be modified to use inverse-propensity weights estimated directly---for example, using stable balancing weights \citep{zubizarreta2015stable} or autoDML \citep{chernozhukov2022automatic}. Moreover, EP-learner may benefit from stabilizing propensity score estimates through calibration techniques \citep{van2024stabilized, klaassen2025calibration}.}

An unexplored application of EP-learner is its potential to improve the calibration of predictors for the CATE through causal isotonic calibration \citep{van2023causalcal}. Isotonic calibration often exhibits poor calibration at the boundary of uncalibrated predictions, primarily because of overfitting of isotonic regression in that region \citep{groeneboom1993isotonic}. However, this issue could be mitigated by utilizing an EP-learner risk estimator for the CATE, which would ensure that the pseudo-outcome employed for calibration itself serves as a CATE estimator. 

\hspace{0.5cm}

\noindent\textbf{Acknowledgements.} Research reported in this publication was supported by NIH grants DP2-LM013340 and  R01-HL137808, and NSF grant DMS-2210216. The content is solely the responsibility of the authors and does not necessarily represent the official views of the funding agencies.

{\singlespacing

\bibliography{ref}}

\newpage 
\doublespacing
\appendix

\section{Code}

An \texttt{R} package, \textit{hte3}, implementing EP-learner can be found on GitHub here: \url{https://github.com/Larsvanderlaan/hte3}. Code to reproduce our experiments can be found here: \url{https://github.com/Larsvanderlaan/hte3/tree/main/paper_EPlearner_experiments}.

\section{Supplemental information for experiments}

\subsection{Small-scale experiments}
\label{appendix::introFigs}

 \begin{figure}[htb]
    \centering
    \includegraphics[width=0.5\textwidth]{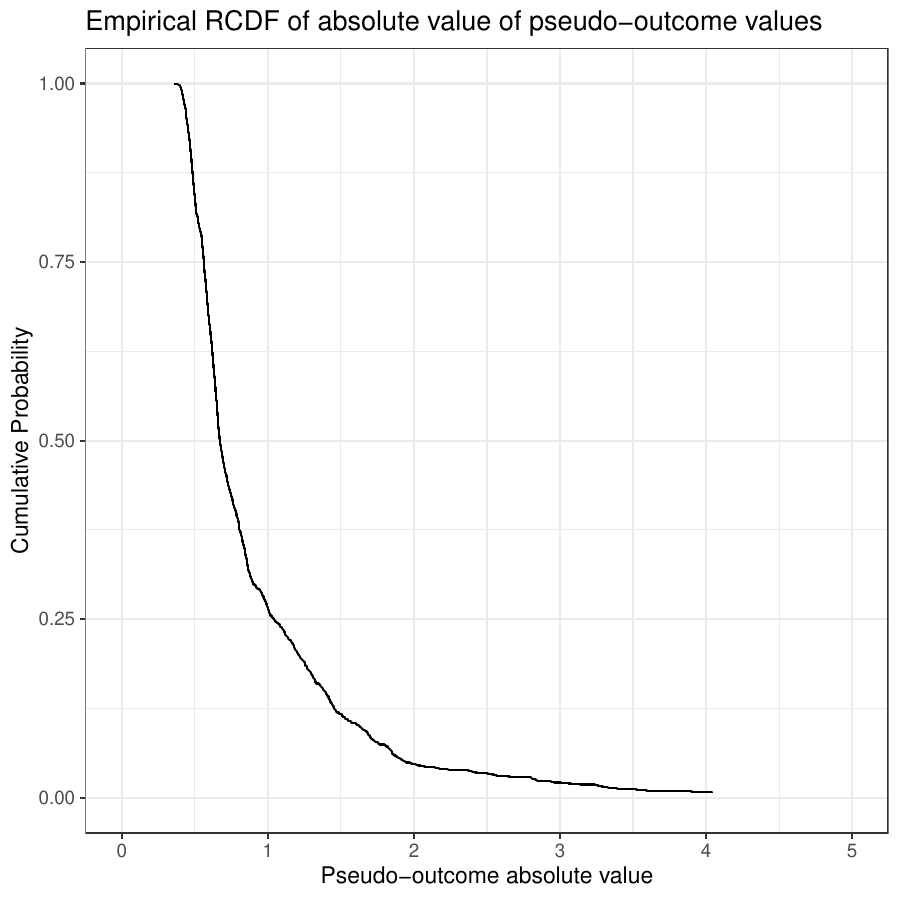}\includegraphics[width=0.5\textwidth]{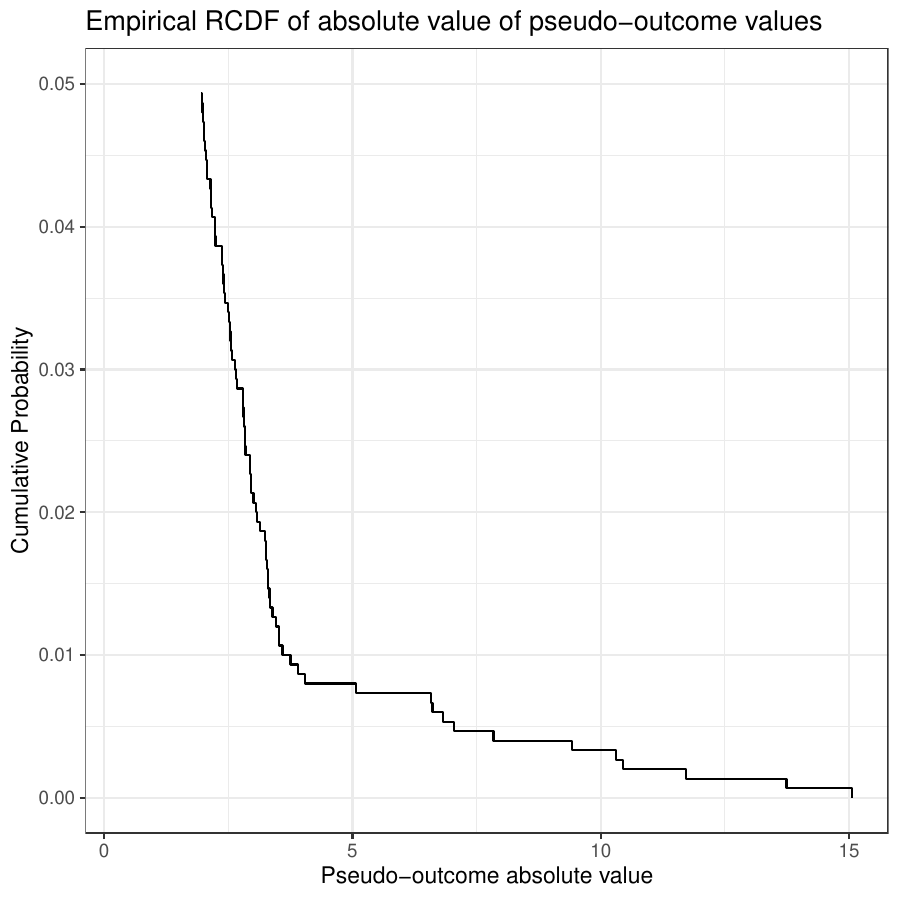}
    \caption{(Left) Empirical reverse cumulative distribution function (RCDF) of estimated DR pseudo-outcome values for the simulated dataset. (Right) A zoomed in empirical RCDF. The 1\% and 99\% empirical quantiles of the pseudo-outcome are -3.2 and 3.1 while the maximum and minimum values are -16 and 14.}
    \label{fig:PseudoOutcomeDensityPlot}
\end{figure}

 The following \texttt{R} code was used to generate the dataset corresponding to Figure \ref{fig:DRlearnerBoundViol}. The code for running the analysis and generating the figures can be found at \url{github/Larsvanderlaan/npcausalML/reproduceSims/introCATEFigure.R}.

\singlespacing
 \begin{verbatim}
set.seed(12345)
n = 1500
W1 <- rt(n, df = 5)
pi0 <- plogis(W1)
A <- rbinom(n, 1, pi0)
mu0 =  0.35 + 0.65*plogis(W1-2 )
mu1 <- mu0 + plogis(2*W1+2) - plogis(W1-2) - 0.349
mu <- ifelse(A==1, mu1, mu0)
Y <- rbinom(n, 1, mu)
 \end{verbatim}

\doublespacing
\noindent The following \texttt{R} code was used to generate the example dataset of Figure \ref{fig:exampDataLRR}.

\singlespacing
\begin{verbatim}
set.seed(123456)
n = 5
W1 <- rt(n, df = 5)
pi0 <- plogis(W1)
pi <- ifelse(A==1, pi0, 1-pi0)
A <- rbinom(n, 1, pi0)
mu0 =  0.35 + 0.65*plogis( W1-2 ) 
mu1 <- mu0 + plogis(2*W1+2) - plogis(W1-2) - 0.349 
mu <- ifelse(A==1, mu1, mu0)
Y <- rbinom(n, 1, mu)
theta <- W1
weights <- round( (mu1 + mu0 + 1/pi*(Y - mu)),2)
outcome <- round((mu1  + A/pi0*(Y - mu1)) / weights,2)
dat <- data.frame(Covariate = round(W1,3), Weight  = weights, Outcome = outcome)
dat
\end{verbatim}
\doublespacing

\subsection{Simulation design}
 
\label{appendix::experiments}

For the low-dimensional CATE experiments, we consider $O = (W_1, W_2, W_3, A, Y)$ for independent covariates $W_1, W_2, W_3$ each drawn from the uniform distribution on $(-1,+1)$. Given $W=w:=(w_1,w_2,w_3)$, the treatment assignment $A$ was generated from a Bernoulli distribution with conditional mean defined, for the moderate overlap setting, as $\pi_0(1\miid w) :=  \text{expit}\left\{(w_1 + w_2 + w_3)/3 \right\}$ and, for the limited overlap setting, as $\pi_0(1\miid w) :=  \text{expit}\left\{w_1 + w_2 + w_3 \right\}$. Given $(W,A)=(w,a)$, the outcome variable $Y$ was generated from a normal distribution with conditional mean $\mu_0(0, w) + a\cdot \theta_0^{-}(w)$ and variance $4$ where $\mu_0(0, w) := \sum_{k=1}^3 \{w_k/2+ \sin(5w_k) + 1/(w_k + 1.2)\}$ and, in the simple setting, $\theta_0^{-}(w) :=  1 +\sum_{k=1}^3 w_k$ and, in the complex setting, $\theta_0^{-}(w) :=  1 +\sum_{k=1}^3 \{w_k + \sin(5w_k)\}$.

For the high-dimensional CATE experiments, we generate $O = (W, A, Y)$ as follows. The covariate $W$ is drawn such that $W/2$ is distributed as a mean-zero multivariate truncated normal random variable with support $[-2,2]$, variances $1$, and covariances $0.4$. Given $W=w$, the treatment assignment $A$ was generated from a Bernoulli distribution with conditional mean $\pi_0(1\miid w)$ defined by $\text{logit}\{\pi_0(w)\} := (1/1.3)(w_1 + w_5 + w_9 + w_{11} + w_{19})$.  Given $(W,A)=(w,a)$, the outcome variable $Y$ was generated from a normal distribution with conditional mean $\mu_0(0, w) + A\cdot \theta_0^{-}(w)$ and variance $4$ where $\mu_0(0,w) = (\cos(4w_1) + \cos(4w_5) + \sin(4w_9) + 1/(1.5+w_{15}) + 1/(1.5+ w_{10}))/5$ and, in the simple setting, $\theta_0^{-}(w) := 1 + (w_1 + w_5 + w_9 + w_{15} + w_{10})/5$ and, in the complex setting, $\theta_0^{-}(w) := 1 + (\sin(4w_1) + \sin(4w_5) + \cos(4w_9) + 1.5 (w_{15}^2 - w_{10}^2))/5 $.

For the CRR experiments, we consider $O = (W_1, W_2, W_3, A, Y)$ for independent covariates $W_1, W_2, W_3$ each drawn from the uniform distribution on $(-1,+1)$. Given $W =w$, the treatment assignment $A$ was generated from a Bernoulli distribution with conditional mean $\pi_0(1\miid w)$ where, in the moderate overlap setting, $\pi_0(1\miid w) :=  \text{expit}\left\{(w_1 + w_2 + w_3)/3 \right\}$ and, in the limited overlap setting, $\pi_0(1\miid w) :=  \text{expit}\left\{w_1 + w_2 + w_3 \right\}$. Given $(W,A) = (w,a)$, the outcome variable $Y$ was generated from Bernoulli distribution with conditional mean $\mu_0(0, w)\exp\{a\theta_0^{\div}(w)\} $ where $\text{logit}\,\mu_0(0, w) := -1 + 0.3 \sum_{k=1}^3 w_k + \sin (4 w_k)$ and, in the simple setting, $\theta_0^{\div}(w) :=  -0.1 + 0.1\sum_{k=1}^3 w_k$ and, in the complex setting, $\theta_0^{\div}(w) := -0.1 + 0.1\sum_{k=1}^3 \{w_k + \sin (4 w_k) \}$.

\subsection{Additional figures}
\begin{figure}[H]
 
    \centering

  \begin{subfigure}[b]{0.45\textwidth}
   \includegraphics[width=0.5\textwidth]{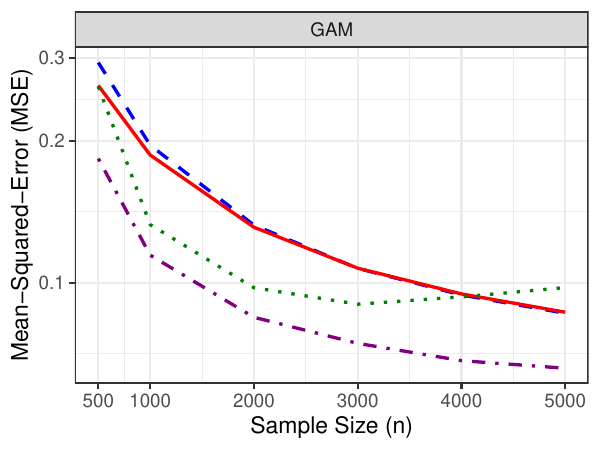}\includegraphics[width=0.5\textwidth]{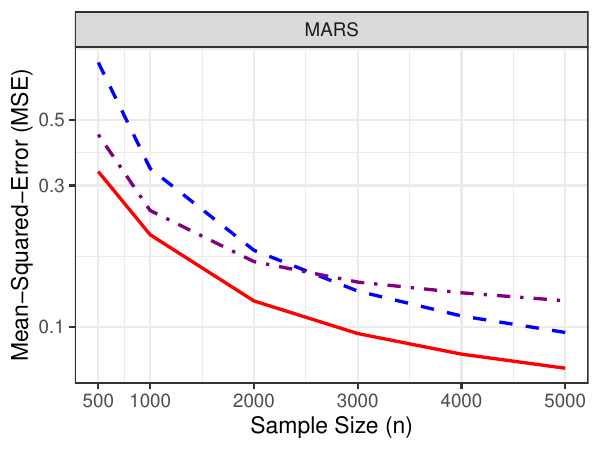}   
    \includegraphics[width=0.5\textwidth]{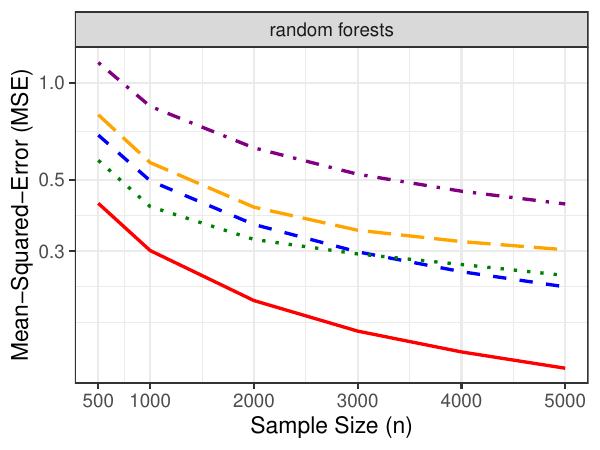}\includegraphics[width=0.5\textwidth]{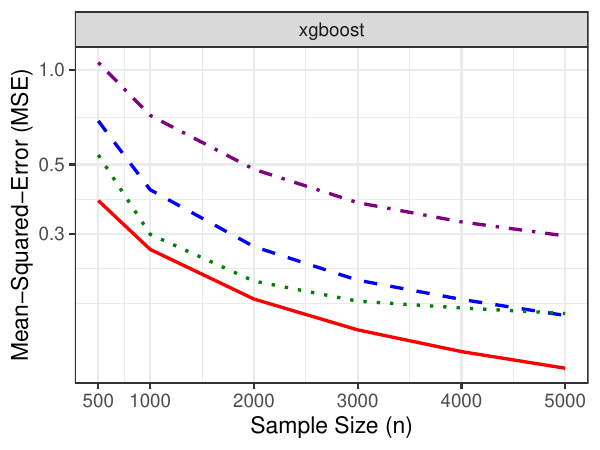}
    \subcaption{$d=3$, limited overlap}
    \end{subfigure}\hfill\begin{subfigure}[b]{0.45\textwidth}
\includegraphics[width=0.5\textwidth]{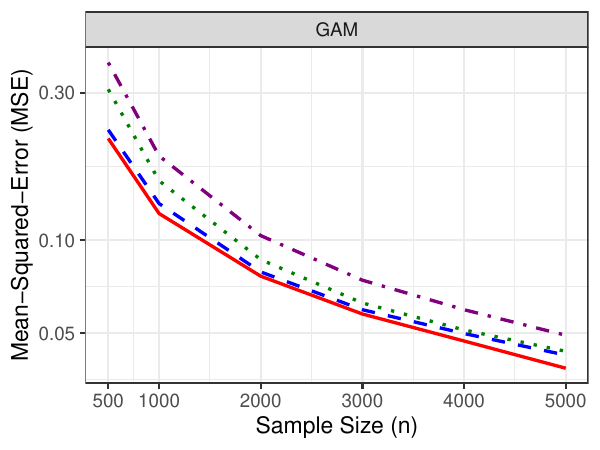}\includegraphics[width=0.5\textwidth]{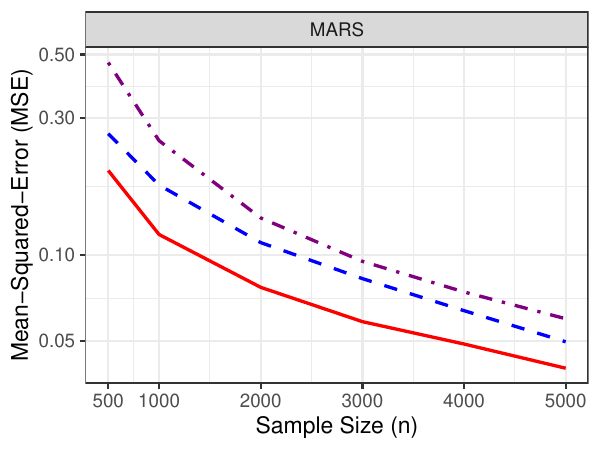}   
   \includegraphics[width=0.5\textwidth]{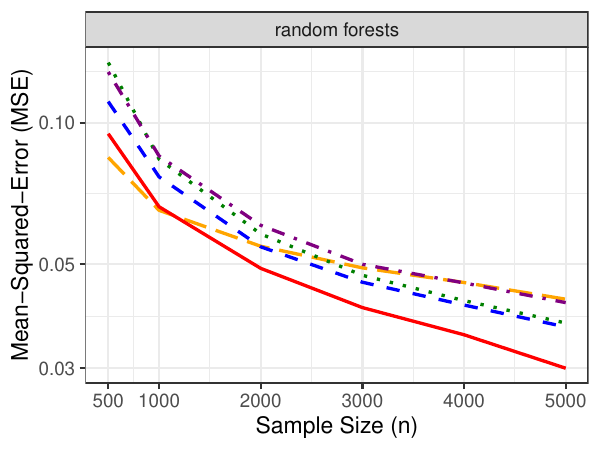}\includegraphics[width=0.5\textwidth]{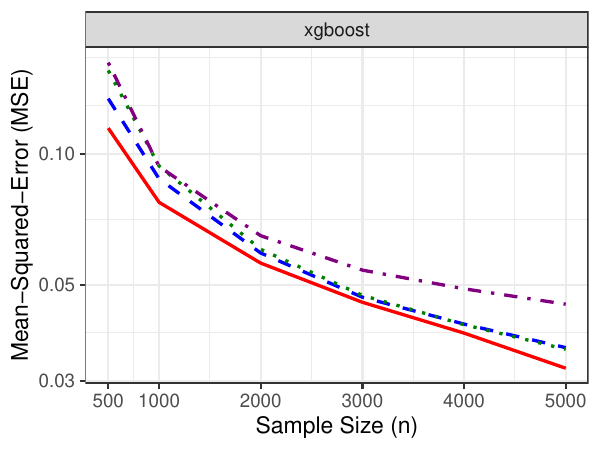}
   \subcaption{$d=20$ with $5$ active, limited overlap}
    \end{subfigure}

\includegraphics[width=.8\textwidth]{newplots/legend_CATElow.pdf}  
      
      \caption{
        CATE experiments with a complex CATE and limited treatment overlap: Mean-squared error for DR-learner and CV-EP-learner with supervised learning algorithm GAM, MARS, ranger, and xgboost. For the plots reporting the results of tree-based algorithms (ranger and xgboost), we also display the results of causal forests for comparison. As a common benchmark across all learners, we display a cross-validated T-learner obtained from an ensemble library including xgboost and GAM learners. 
    }  
 
\end{figure}

\begin{figure}[H]
 
    \centering

  \begin{subfigure}[b]{0.45\textwidth}
   \includegraphics[width=0.5\textwidth]{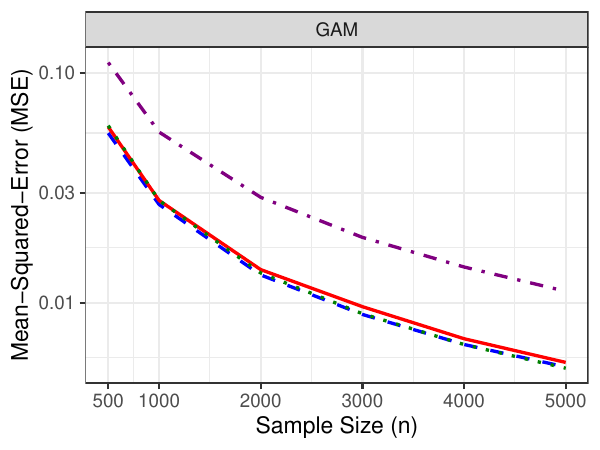}\includegraphics[width=0.5\textwidth]{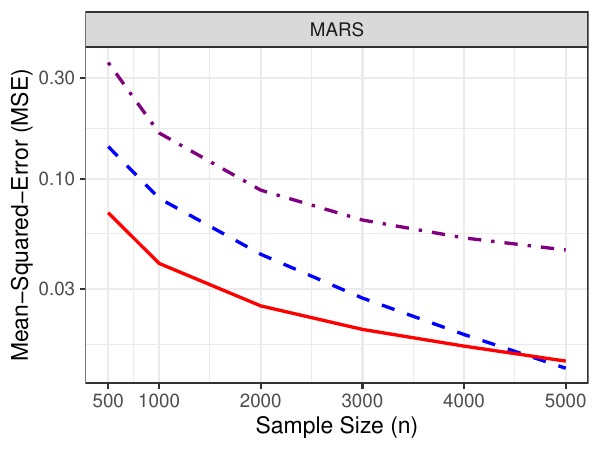}   
    \includegraphics[width=0.5\textwidth]{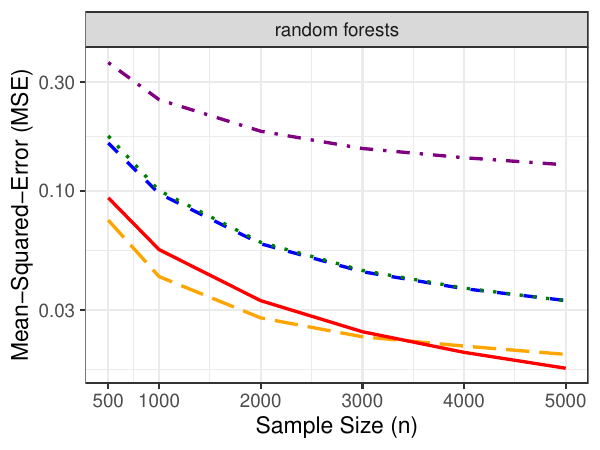}\includegraphics[width=0.5\textwidth]{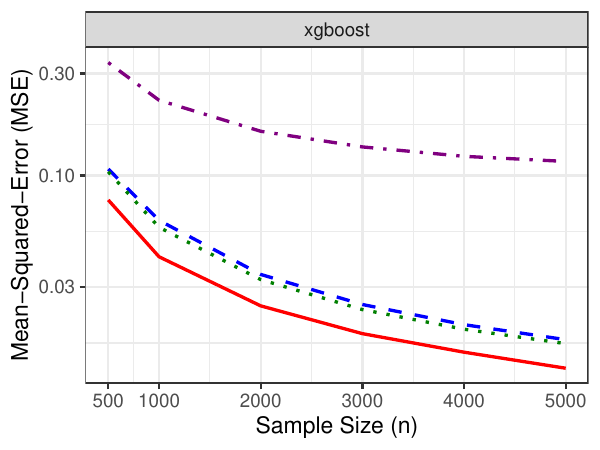}
    \subcaption{$d=3$, moderate overlap}
    \end{subfigure}\hfill\begin{subfigure}[b]{0.45\textwidth}
\includegraphics[width=0.5\textwidth]{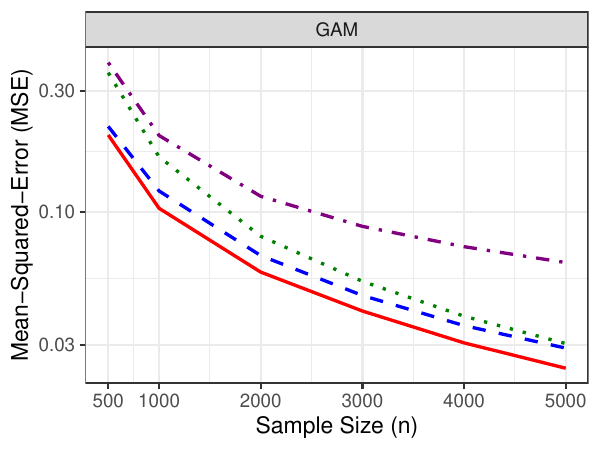}\includegraphics[width=0.5\textwidth]{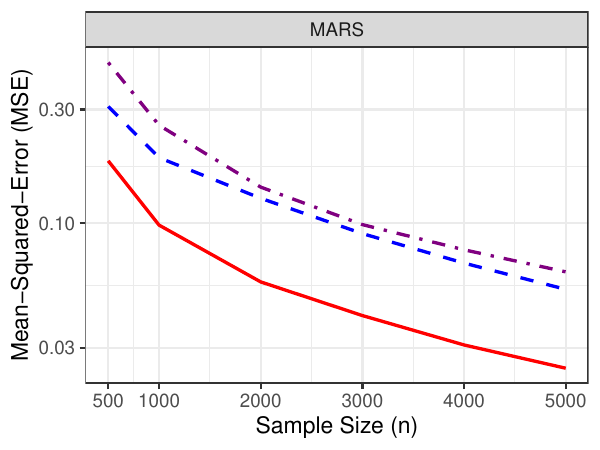}   
   \includegraphics[width=0.5\textwidth]{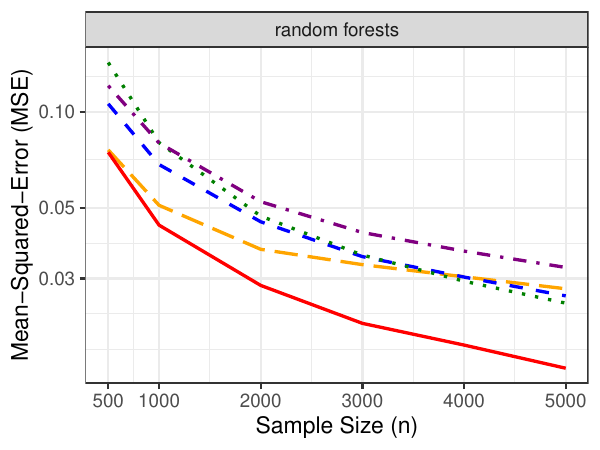}\includegraphics[width=0.5\textwidth]{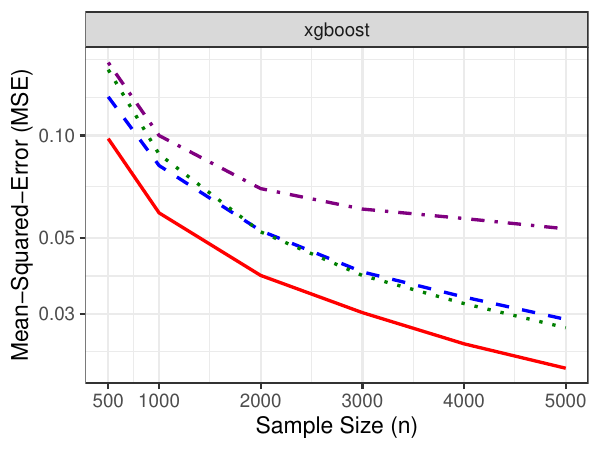}
   \subcaption{$d=20$ with $5$ active, moderate overlap}
    \end{subfigure} 

 \begin{subfigure}[b]{0.45\textwidth}
   \includegraphics[width=0.5\textwidth]{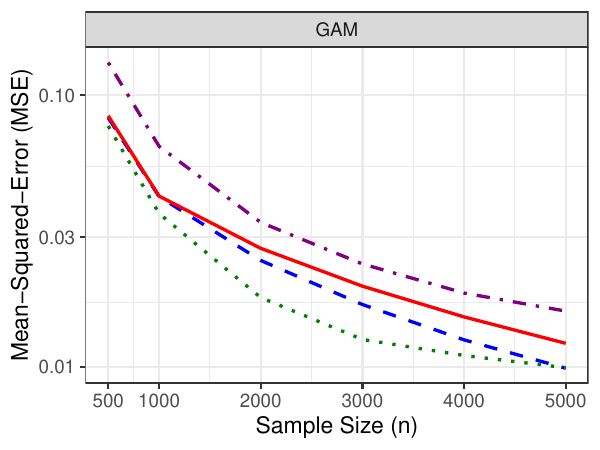}\includegraphics[width=0.5\textwidth]{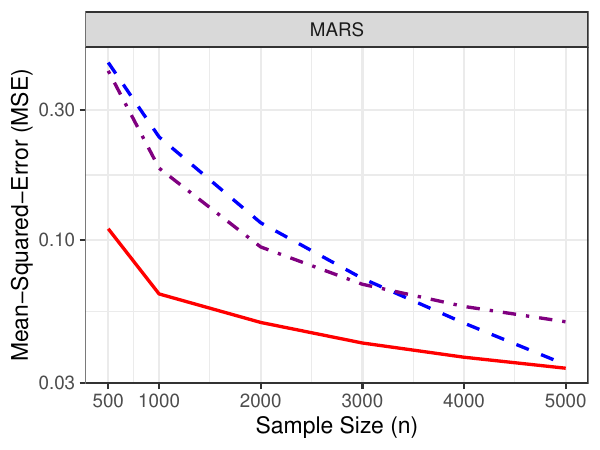}   
    \includegraphics[width=0.5\textwidth]{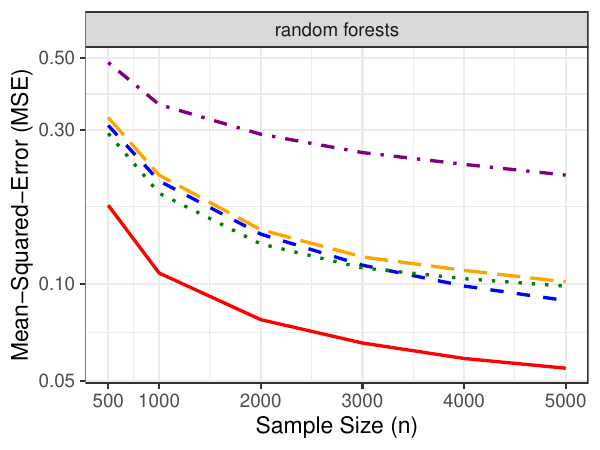}\includegraphics[width=0.5\textwidth]{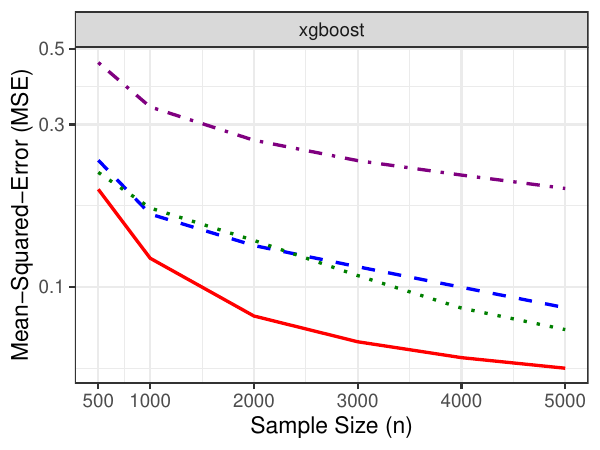}
    \subcaption{$d=3$, limited overlap}
    \end{subfigure}\hfill\begin{subfigure}[b]{0.45\textwidth}
\includegraphics[width=0.5\textwidth]{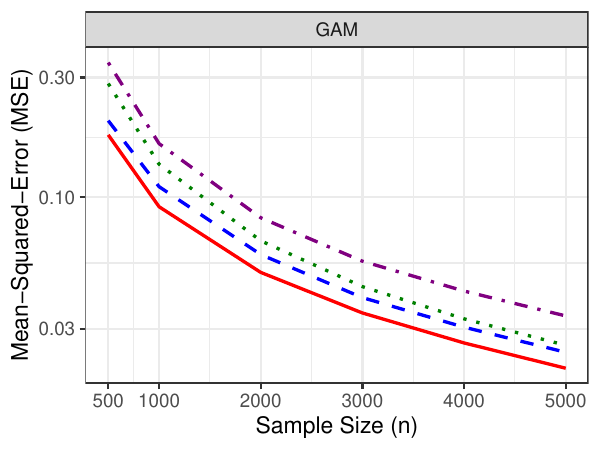}\includegraphics[width=0.5\textwidth]{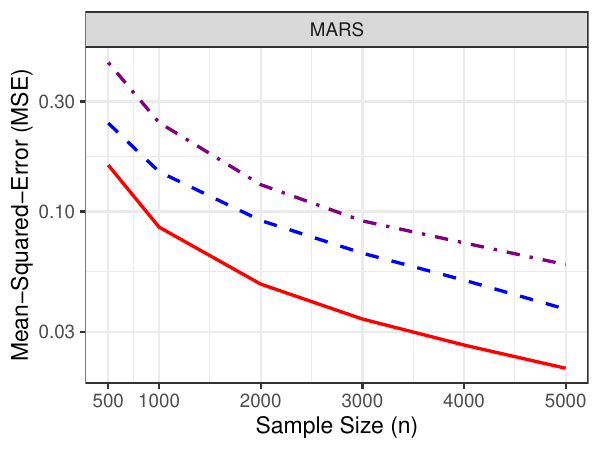}   
   \includegraphics[width=0.5\textwidth]{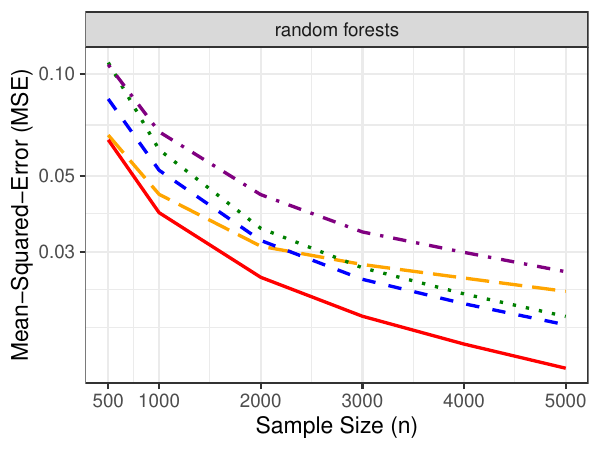}\includegraphics[width=0.5\textwidth]{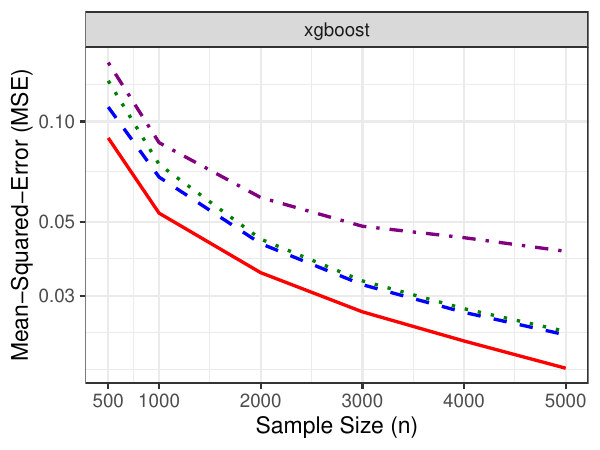}
   \subcaption{$d=20$ with $5$ active, limited overlap}
    \end{subfigure} 

\includegraphics[width=.8\textwidth]{newplots/legend_CATElow.pdf}  
      
      \caption{
   CATE experiments with a simple CATE and moderate treatment overlap: Mean-squared error for DR-learner and CV-EP-learner with supervised learning algorithm GAM, MARS, ranger, and xgboost. For the plots reporting the results of tree-based algorithms (ranger and xgboost), we also display the results of causal forests for comparison. As a common benchmark across all learners, we display a cross-validated T-learner obtained from an ensemble library including xgboost and GAM learners. 
    }  
 
\end{figure}

 \begin{figure}[htb]
    \centering
 \begin{subfigure}[b]{0.45\textwidth}
    \includegraphics[width=0.5\textwidth]{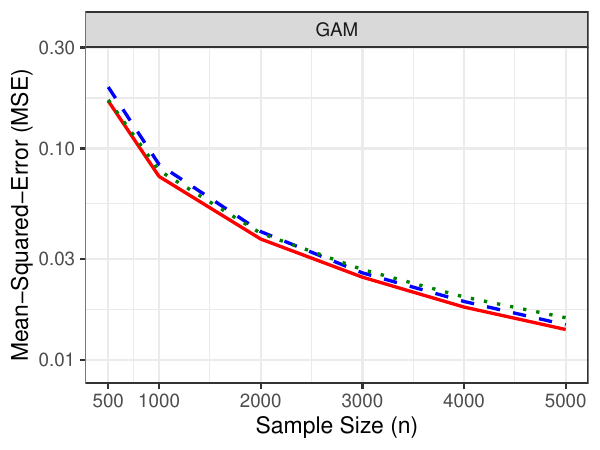}\includegraphics[width=0.5\textwidth]{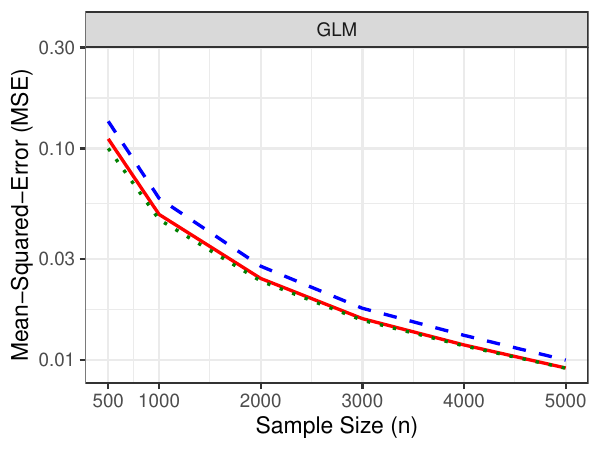}  
    \includegraphics[width=0.5\textwidth]{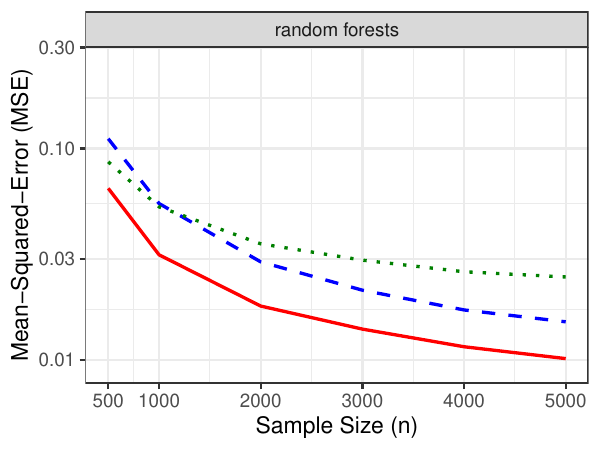}\includegraphics[width=0.5\textwidth]{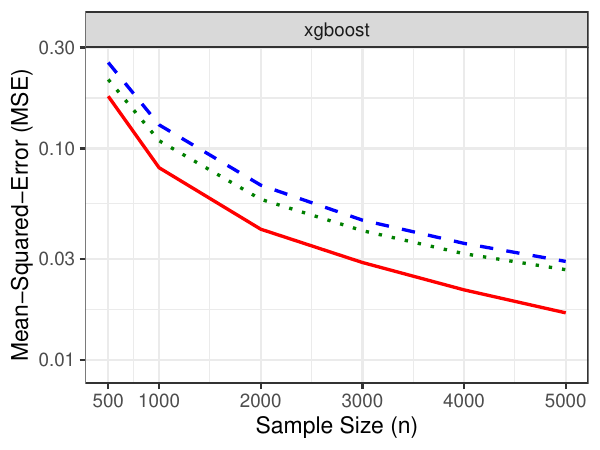}
    \subcaption{Moderate treatment overlap}
    \end{subfigure}\hfill\begin{subfigure}[b]{0.45\textwidth}
    \includegraphics[width=0.5\textwidth]{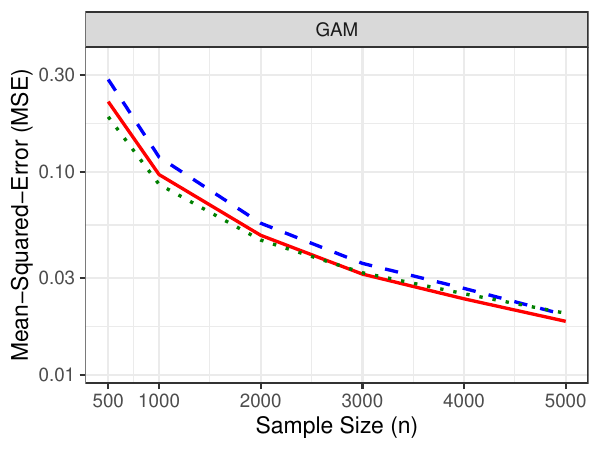}\includegraphics[width=0.5\textwidth]{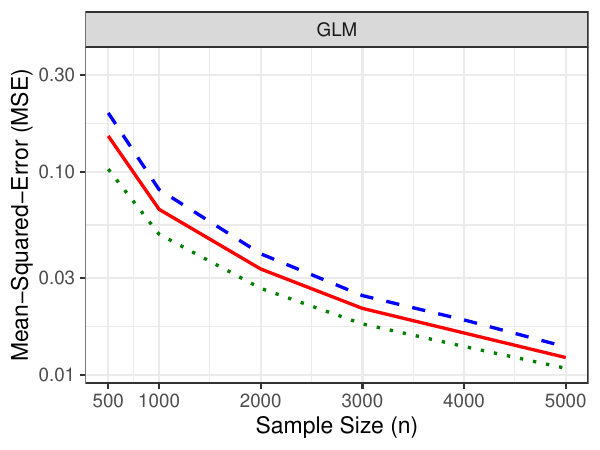}  
    \includegraphics[width=0.5\textwidth]{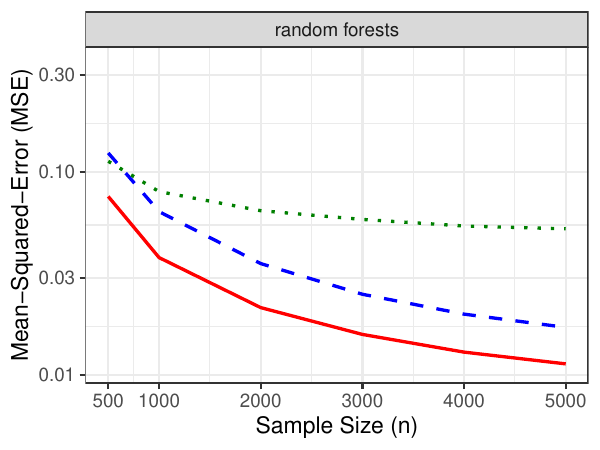}\includegraphics[width=0.5\textwidth]{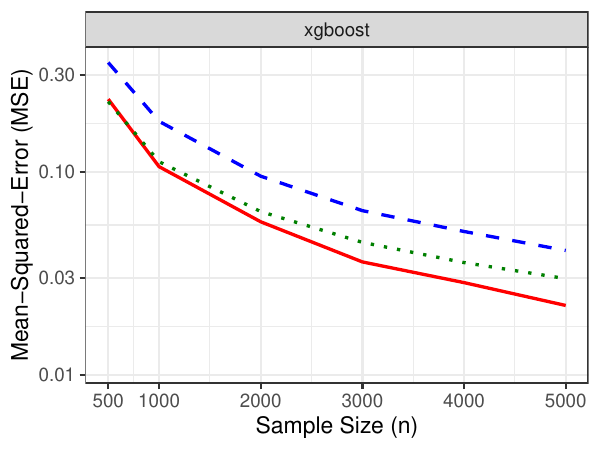}
    \subcaption{limited treatment overlap}
    \end{subfigure}
 \includegraphics[width=\textwidth]{newplots/legend_CRR.pdf}  
 \caption{
  CRR experiments with a simple CRR and moderate (left) and limited (right) treatment overlap: Mean-squared error for IPW-learner, T-learner, and CV-EP-learner with supervised learning algorithm GAM, random forests, and xgboost. The DR-learner algorithm for the CRR is not implemented due to nonconvexity of the loss.
    }

\end{figure}

 \section{Additional details on EP-learner algorithm}
 \label{appendix:algoDebiasing}

 \subsection{Simplification of algorithm 1 for the CATE}
\label{appendix:algoDebiasing1}
When $h\circ \theta := h_1 \circ \theta = h_2 \circ \theta$ for all $\theta \in \mathcal{F}$, the debiasing term $\frac{1}{n}\sum_{i=1}^n \Delta_{\pi_n, \mu_n}(O_i; \theta)$ in the general case simplifies to $$ \frac{1}{n}\sum_{i=1}^n \frac{ 1}{\pi_{n,j(i)}(A_i\miid W_i)} \left\{\sum_{m \in \{1,2\}}    H_{m,\mu_{n,j(i)}}(A_i,W_i)  \right\} (h \circ \theta)(W_i) \left\{Y_i - \mu_{n,j(i)}(A_i,W_i) \right\} .$$
In this setting, \eqref{eqn::scoreEqnSolvedByAdjust} is more than sufficient to make the debiasing term small. We can instead take the data-dependent feature vector in Algorithm \ref{alg::debiasing} to be 
$$\widehat \varphi_{k, j(i)}(a,w) = \left\{\sum_{m \in \{1,2\}}    H_{m,\mu_{n,j(i)}}(A_i,W_i)  \right\} \varphi_k(w). $$
With this modification, the resulting estimator $\mu_{n,k(n)}^*$ will, for all $\psi \in \mathcal{H}_{k(n)}$, satisfy
$$ \frac{1}{n}\sum_{i=1}^n \frac{ 1}{\pi_{n,j(i)}(A_i\miid W_i)} \left\{\sum_{m \in \{1,2\}}    H_{m,\mu_{n,j(i)}}(A_i,W_i)  \right\} \psi(W_i)\left\{Y_i - \mu_{n,j(i)}^*(A_i,W_i) \right\}  = 0.$$
This modification halves the dimension of the sieve regressions of Algorithm 1 without any cost in the size of the debiasing term, and may lead to better finite-sample performance. In particular, for the CATE example, we recommend using $\widehat \varphi_{k, j(i)}(a,w) = (2a-1) \varphi_k(w) $, noting that $h_1 \circ \theta = h_2 \circ \theta = \theta$ and $ \sum_{m \in \{1,2\}}    H_{m,\mu_{n,j(i)}}(a,w)  = 2a-1$. 

  \subsection{Debiasing method for general outcomes}
 \label{appendix:algoDebiasing2}
 \textbf{Method 3: general outcomes (bound-preserving).}             $g(x) = \text{logit}\left((x - \widehat a)/(\widehat b - \widehat a) \right)$ where $\widehat a = \min\left\{Y_i, \mu_{n,j(i)}(A_i,X_i): i=1,\dots,n \right\}$ and $\widehat b = \max\left\{Y_i, \mu_{n,j(i)}(A_i,X_i): i=1,\dots,n \right\}$, and $\beta_n$ is the coefficient vector obtained from the
           the logistic regression of the $[0,1]$-transformed outcomes $((Y_i-\widehat a)/(\widehat b - \widehat a):i\in [n])$ on the feature vectors $(\widehat \varphi_{k, j(i)}(A_i,W_i):i\in [n])$ with offsets $\left(\text{logit}\left\{ (\mu_{n,j(i)}(A_i,W_i)-\widehat a)/(\widehat b - \widehat a)\right\} :i\in [n]\right)$ and weights $\left(1/\pi_{n,j(i)}(A_i \miid W_i):i\in [n])\right)$. By construction, the sieve-adjusted estimates $\mu_{n,j(i)}(A_i,W_i)$ lie in $[\widehat a, \widehat b]$.

\section{EP-learner uniform consistency with fixed $K$-nearest neighbors} \label{appendix::knn}

\begin{proof}[Proof of Theorem \ref{theorem:knn}]
Fix $w \in (0,1)^d$ and let $B_{K_n,n}(w)$ be the smallest closed ball centered at $w$ containing $\{W_i: i \in \mathcal{I}_{K_n,n}(w)\}$. Work on the probability-1 event that the set $\mathcal{I}_{K_n,n}(w)$ is uniquely defined (i.e., no ties). On this event, $W_i \in B_{K_n,n}(w)$ if and only if $i \in \mathcal{I}_{K_n,n}(w)$. Then
\begin{align*}
\big|\theta_{n,EP}(w) - \theta_0^-(w)\big|
&= \Bigg|\frac{1}{K_n}\sum_{i=1}^n 1\{W_i \in B_{K_n,n}(w)\}\Big[\mu_n^*(1,W_i) - \mu_n^*(0,W_i) - \{\mu_0(1,w) - \mu_0(0,w)\}\Big]\Bigg| \\
&\le \Bigg|\frac{1}{K_n}\sum_{i=1}^n 1\{W_i \in B_{K_n,n}(w)\}\Big[\mu_n^*(1,W_i) - \mu_n^*(0,W_i) - \{\mu_0(1,W_i) - \mu_0(0,W_i)\}\Big]\Bigg| \\
&\quad+ \Bigg|\frac{1}{K_n}\sum_{i=1}^n 1\{W_i \in B_{K_n,n}(w)\}\Big[\mu_0(1,W_i) - \mu_0(0,W_i) - \{\mu_0(1,w) - \mu_0(0,w)\}\Big]\Bigg| \\
&\lesssim \|\theta_{n,T}^* - \theta_0^-\|_\infty + \sup_{w' \in (0,1)^d} \sup_{x \in B_{K_n,n}(w')} L\|x - w'\|_2,
\end{align*}
where $L$ is the Lipschitz constant of $\mu_0$. Taking a supremum over $w \in (0,1)^d$ yields
\[
\sup_{w \in (0,1)^d} \big|\theta_{n,EP}(w) - \theta_0^-(w)\big|
\lesssim \|\theta_{n,T}^* - \theta_0^-\|_\infty + \sup_{w' \in (0,1)^d} \sup_{x \in B_{K_n,n}(w')} L\|x - w'\|_2.
\]
By Corollary 1.5 of \citet{chenavier2022limit}, the maximal volume of $B_{K_n,n}(w')$ over $w' \in (0,1)^d$ is $O_p(K_n \log n / n)$. This implies that the maximal radius of $B_{K_n,n}(w')$ is of order $O_p((K_n \log n / n)^{1/d})$ up to a constant depending on $d$. Therefore the bound simplifies to
\[
\sup_{w \in (0,1)^d} \big|\theta_{n,EP}(w) - \theta_0^-(w)\big|
\lesssim \|\theta_{n,T}^* - \theta_0^-\|_\infty + \mathcal{O}_p\big((K_n \log n / n)^{1/d}\big),
\]
which proves the theorem.
\end{proof}

\section{Preliminaries for proofs}

\subsection{Notation and conventions for proofs}

For two quantities $x$ and $y$, we use the expression  $x \lesssim y$ to mean that $x$ is upper bounded by $y$ times a universal constant that may only depend on global constants, including $\eta$, $M$, and $C$, that appear the conditions of Theorems \ref{theorem::EPriskEff}-\ref{theorem::EpLearnerOracleEff}.

 For a function class $\mathcal{F}$, we denote its norms by $\|\mathcal{F}\| := \sup_{f \in \mathcal{F}} \|f\|$ and $\| \mathcal{F}\|_{\infty} := \sup_{f \in \mathcal{G}} \|f\|_{\infty}$. For a uniformly bounded function class $\mathcal{F}$ and distribution $P$, let $N(\epsilon,\mathcal{F},L_2(P))$ denote the $\varepsilon-$covering number \citep{vanderVaartWellner} of $\mathcal{F}$ with respect to the $L_2(P)$-metric. We define the uniform entropy integral of $\mathcal{F}$ by  
\begin{equation*}
\mathcal{J}(\delta,\mathcal{F}):= \int_{0}^{\delta} \sup_{Q}\sqrt{\log N(\varepsilon,\mathcal{F},L_2(Q))}\,d\varepsilon\ ,
\end{equation*}
where the supremum is taken over all discrete probability distributions $Q$.

 In our proofs, we use the following notation. We let $\mathbb{I}_{g_1, g_2}$ be the indicator that takes the value $0$ if both $g_1$ and $g_0$ are the identity function. We denote the nuisance rates:
 \begin{align*}
     r_n^* &:=  n^{-\beta/(2\beta+1)}+ \sqrt{k(n) \log n / n};\\
     s_n ^* &:=    r_n^* + n^{-\gamma/(2\gamma+1)}. 
 \end{align*}
 We denote, for each $P \in \mathcal{M}$, the Neyman orthogonal loss function corresponding with $R_P$ by $L_{\pi_P, \mu_P} := L_{\mu_P} + \Delta_{\pi_P,\mu_P}$. For $\phi \in L^2(P_{0,W})$, $j \in [J]$, and $m \in \{1,2\}$, we define pointwise, for $o \in \mathcal{O}$, the following variants of the debiasing term:
\begin{align*}
&\overline{\Delta}^{(m)}_{\pi_{n,j}, \mu_{n,j}^*}  (o\,, \phi):=   \frac{ 1}{\pi_{n,j}(a\miid w)} \left\{   H_{ m,\mu_{n,j}^*}(a,w) \cdot \phi(w)  \right\}\left\{y - \mu_{n,j}^{*}(a,w) \right\};\\
&\overline{\Delta}^{*(m)}_{\pi_{n,j}, \mu_{n,j}}  (o\,, \phi) := \frac{ 1}{\pi_{n,j}(a\miid w)} \left\{   H_{ m,\mu_{n,j}}(a,w) \cdot \phi(w)  \right\}\left\{y - \mu_{n,j}^{*}(a,w) \right\}.
\end{align*}
The overline encodes dependence on the transformed action-space $\mathcal{H}$ as opposed to the original action-space $\mathcal{F}$. The superscript by $*$ encodes that the residual in the definition corresponds to the debiased outcome regression estimator.

Recall $\{\mathcal{D}_n^j : j \in [J]\}$ is the partitioning of $\mathcal{D}_n := \{O_i: i \in [n]\}$ used for cross-fitting the initial nuisance estimators, $\{\mu_{n,j}, \pi_{n,j}: j \in [J]\}$, as in Algorithm \ref{alg::CVEPlearner}. We define the $j$-th training set $\mathcal{T}_n^j :=  \mathcal{D}_n \backslash \mathcal{D}_{n}^j$. We let $P_{n,j}$ denote the empirical distribution of the $j$th data-fold $\mathcal{D}_n^j$ and define the fold-specific empirical process operator $f \mapsto P_{n,j} f := \int f(o) dP_{n,j}(o)$. For ease of presentation, we will use the following cross-fitting notation. For a collection of fold-specific functions $\nu_{n,\diamond} := \{\nu_{n,j}: j \in [J]\}$ that may depend on $n$, we define the fold-averaged empirical process operators: 
\begin{align*}
    \overline{P}_0\nu_{n,\diamond} &:= \frac{1}{J} \sum_{j=1}^J P_0 \nu_{n,j}; \\
\overline{P}_n \nu_{n,\diamond}  &:= \frac{1}{J} \sum_{j=1}^J P_{n,j} \nu_{n,j}.
\end{align*}
We define the fold-averaged norm of a collection of fold-specific functions $\nu_{\diamond}:= \{\nu_j: j \in [J]\}$ as $\|\nu_{\diamond} \|_{\overline{P}_0} := \sqrt{\frac{1}{J}\sum_{j=1}^J \|\nu_j\|^2_{P_0}}$. For a function $h: \mathbb{R} \rightarrow \mathbb{R}$, we also define the composition $h \circ \nu_{\diamond}:= \{h \circ \nu_j: j \in [J]\} $.
The above notation lets us treat the collection of debiased cross-fitted estimators $\mu_{n,\diamond}^* := \{\mu_{n,j}^*: j \in [J]\}$ as extended functions defined on $\mathcal{A} \times \mathcal{W} \times [J]$. For example, in our proofs, $\overline{P}_n \overline{\Delta}^{(m)}_{\pi_{n, \diamond}, \mu_{n, \diamond}^*}(\cdot, \phi)$ is notation for $ J^{-1} \sum_{j \in [J]} \int \overline{\Delta}^{(m)}_{\pi_{n,j}, \mu_{n,j}^*}(o,\phi) dP_{n,j}(o)$.

In our proofs, for some sufficiently large $M >0$, we will work on the following events:

\begin{enumerate}[label=\bf{(E\arabic*)}, ref = E\arabic*]
    \item $ \max_{j \in [J]} \norm{\mu_{n,j} - \mu_0} \leq M n^{-\beta/(2\beta+1)}$ \label{event::mu_n}
    \item $ \max_{j \in [J]}\norm{\pi_{n,j} - \pi_0} \leq M n^{-\gamma/(2\gamma+1)}$ \label{event::pi_n}
    \item  $ \max_{j \in [J]} \norm{\mu_{n,j}^* - \mu_0} \leq M n^{-\beta/(2\beta+1)} + M \sqrt{k(n) \log n / n}$. \label{event::debias_mu_n_star}
    \item $\norm{\beta_n}_{\infty} \leq M$. \label{event::boundedCoef}
\end{enumerate}
Let $A_n$ be the union of events \ref{event::mu_n}-\ref{event::boundedCoef}, and let $\mathbb{I}_{A_n}$ denote the indicator that event $A_n$ occurs. In our proofs, we will choose $M > 0$ so that the event $A_n$ occurs asymptotically with probability at least $1 - \varepsilon$. We claim, under \ref{cond::A2Nuisance} and \ref{cond::boundedEStimators}, that we can always choose such an $M > 0$. To see this, note, under condition \ref{cond::A2Nuisance}, we can take $M$ large enough so that events \ref{event::mu_n}-\ref{event::debias_mu_n_star} occur asymptotically with probability arbitrarily close to one. Under Condition \ref{cond::boundedEStimators}, we claim that Event \ref{event::boundedCoef} occurs with probability tending to one. To show this, recall by definition that $\mu_{n,j(i)}^* = g^{-1}(g(\mu_{n,j(i)})+ \widehat{\varphi}_{k(n),i}^T \beta_n )$. By \ref{cond::boundedEStimators}, we have that $\mu_{n,j(i)}$ and $\mu_{n,j(i)}^*$ are bounded with probability tending to one and, since the feature mapping $\varphi_{k(n)}$ is bounded by definition, $\widehat{\varphi}_{k(n),i}$ is bounded with probability tending to one. Since $g$ is a continuous invertible function, we have that $\beta_n$ must also be bounded with probability tending to one.  The claim then follows taking $M >0$ large enough so that $\norm{\beta_n}_{\infty} \leq M$ with probability tending to one.

 \subsection{Supporting Lemmas and empirical process bounds}
\label{appendix::supportinglemmas}

 In this section, we present some technical lemmas used to bound empirical process remainders and second-order bias terms that appear in our proofs.

The following lemma due to \cite{BelloniSieveApprox} bounds the sup-norm sieve approximation error of the $L^2(P_0)$-projection $\Pi_{k(n)}$. For the following lemma, we define the rate $\rho_{n,\infty} := \{\log k(n)\}^{\nu} k(n)^{-\rho}$.

\begin{lemma}
    Under \ref{cond::sieveApproxthrm0} and \ref{cond::sieveApproxthrm1}, it holds that
    \begin{align*}
        \sup_{\phi \in \mathcal{H}} \|\phi - \Pi_{k(n)} \phi \|_{\infty} & \lesssim   \{\log k(n)\}^{\nu} \sup_{\phi \in \mathcal{H}} \inf_{f \in \mathcal{H}_{k(n)}} \| \phi - f \|_{\infty} \lesssim  \rho_{n,\infty}.
    \end{align*} 
        \label{lemma::sieveRateSup}
\end{lemma}
\begin{proof}
    Condition \ref{cond::sieveApproxthrm0} and Proposition 3.2. of \cite{BelloniSieveApprox} imply that $\sup_{\phi \in \mathcal{H}} \|\phi - \Pi_{k(n)} \phi \|_{\infty}  \lesssim  (\log k(n))^\nu \sup_{\phi \in \mathcal{H}} \inf_{f \in \mathcal{H}_{k(n)}} \| \phi - f \|_{\infty} $. The second bound then follows from \ref{cond::sieveApproxthrm1}, noting that $\sup_{\phi \in \mathcal{H}} \inf_{f \in \mathcal{H}_{k(n)}} \| \phi - f \|_{\infty} \lesssim  k(n)^{-\rho}$.  
\end{proof}

 \begin{lemma}
 Let $\mathcal{F}_1, \dots, \mathcal{F}_k$ be given uniformly bounded function classes with $\mathcal{J}_{\infty}(1, \mathcal{F}_j) < \infty$ for each $j \in [k]$. Let $\varphi: \mathbb{R}^k \rightarrow \mathbb{R}$ by a Lipschitz-continuous map. Then, the function class $\mathcal{G} := \left\{\varphi(f_1, \dots, f_k): (f_1, \dots, f_k) \in \mathcal{F}_1 \times \dots \times \mathcal{F}_k \right\}$ satisfies $\log N_{\infty}(\varepsilon, \mathcal{G}) \lesssim \sum_{j=1}^k \log N_{\infty}(\varepsilon, \mathcal{F}_j)$ and, hence,
 $\mathcal{J}_{\infty}(\delta, \mathcal{G}) \lesssim \sum_{j=1}^k  \mathcal{J}_{\infty}(\delta, \mathcal{F}_j)$.
     \label{lemma::lipschitzPreservation}
 \end{lemma}
 \begin{proof}
 This result follows from the proof of Theorem 2.10.20 in \cite{vanderVaartWellner}.
 \end{proof}

  The following lemma relates the sup-norm entropy integral of $\mathcal{H}$ to those of the function classes $\mathcal{H} - \Pi_k \mathcal{H} := \{\phi - \Pi_k \phi: \phi \in \mathcal{H}\}$ and  $\Pi_k \mathcal{H} := \{ \Pi_k \phi: \phi \in \mathcal{H}\}$.

 \begin{lemma}
 Under \ref{cond::sieveApproxthrm} and \ref{cond::regularityOnActionSpace}, we have, for all $\delta > 0$, that $\mathcal{J}_{\infty}( \delta,  \mathcal{H})  \lesssim \mathcal{J}_{\infty}( \delta,  \mathcal{F})$, $\mathcal{J}_{\infty}( \delta,  \Pi_k \mathcal{H}) \lesssim  \mathcal{J}_{\infty}( \{\log k(n)\}^{\nu} \delta,  \mathcal{F} )$, and $\mathcal{J}_{\infty}( \delta,  \mathcal{H} - \Pi_k \mathcal{H}) \lesssim  \mathcal{J}_{\infty}( \{\log k(n)\}^{\nu} \delta,  \mathcal{F} )$
    \label{lemma::metricentropybounds}
 \end{lemma}
 \begin{proof}
Recall, by definition, that $\mathcal{H} := \{h_i \circ \theta: \theta \in \mathcal{F}, i \in \{1,2\}\}$ where $h_1$ and $h_2$ are Lipschitz-continuous functions. Thus, by Lemma \ref{lemma::lipschitzPreservation}, we have $\mathcal{J}_{\infty}(\delta, \mathcal{H}) \lesssim \mathcal{J}_{\infty}(\delta, \mathcal{F})$. Now, for $i \in \{1,2\}$, let $(\phi_i - \Pi_k \phi_i)$ be an element of $\mathcal{H} - \Pi_k \mathcal{H}$.
  Using linearity of the $L^2(P_0)$-projection and the Lebesgue constant bound of \ref{cond::sieveApproxthrm0}, we have 
  \begin{align*}
    \|(\Pi_k \phi_1) -   \Pi_k \phi_2\|_{\infty} &=  \|\Pi_k(\phi_1 -  \phi_2)\|_{\infty}\\
    & \lesssim    \{\log k(n)\}^{\nu} \| \phi_1 -  \phi_2\|_{\infty}.
\end{align*} 
Thus, $\Pi_k \mathcal{H} =  \{ \Pi_k \phi: \phi \in \mathcal{H}\}$ is, relative to the sup-norm, a Lipschitz transformation of the function class $\mathcal{H}$ with Lipschitz constant equal to $C\{\log k(n)\}^{\nu}$ for some constant $C > 0$. Hence, we have $\log N_{\infty}(\varepsilon, \mathcal{H}_k) \lesssim \log N_{\infty}( \{\log k(n)\}^{\nu} \varepsilon, \mathcal{H})$ and, therefore, by Lemma \ref{lemma::lipschitzPreservation}, $\log N_{\infty}(\varepsilon, \mathcal{H}_k)  \lesssim \log N_{\infty}( \{\log k(n)\}^{\nu} \varepsilon, \mathcal{F})$. It then follows, from a change of variables, that
\begin{align*}
    \mathcal{J}_{\infty}(\delta, \Pi_k \mathcal{H} ) & = \int_0^{\delta} \sqrt{\log N_{\infty}(\varepsilon, \mathcal{H}_k) } d\varepsilon\\
    & \lesssim \int_0^{\delta} \sqrt{\log N_{\infty}( \{\log k(n)\}^{\nu}  \varepsilon, \mathcal{F}) } d\varepsilon \\
      & =  \{\log k(n)\}^{-\nu}\int_0^{ \{\log k(n)\}^{\nu}  \delta} \sqrt{\log N_{\infty}( \varepsilon, \mathcal{F}) } d\varepsilon\\
       & = \{\log k(n)\}^{-\nu}  \mathcal{J}_{\infty}( \{\log k(n)\}^{\nu} \delta,  \mathcal{F} ).
\end{align*}
Since $\{\log k(n)\}^{-\nu} = O(1)$ as $\nu \geq 0$, we have $  \mathcal{J}_{\infty}(\delta, \Pi_k \mathcal{H} ) \lesssim \mathcal{J}_{\infty}( \{\log k(n)\}^{\nu} \delta,  \mathcal{F} )$, as desired. Finally, observe that $\mathcal{H} - \mathcal{H}_k$ is a Lipschitz transformation of $\mathcal{H}$ and $\mathcal{H}_k$. Thus, by Lemma \ref{lemma::lipschitzPreservation}, we have $\mathcal{J}_{\infty}(\delta, \mathcal{H} - \Pi_k \mathcal{H} ) \lesssim \mathcal{J}_{\infty}(\delta, \mathcal{H} ) + \mathcal{J}_{\infty}(\delta, \Pi_k \mathcal{H} ) \lesssim \mathcal{J}_{\infty}( \{\log k(n)\}^{\nu} \delta,  \mathcal{F} )$, where the final inequality follows from our earlier claims.
 
\end{proof}

 The next lemmas provide local maximal inequalities for empirical processes in terms of the sup-norm metric entropy integral. The following lemma is a restatement of Theorem 2.1. in \cite{van2014uniform}.

\begin{lemma}
    
Suppose $\mathcal{J}_{\infty}(\infty, \mathcal{F}) < \infty$. Then, $\sqrt{E\norm{\mathcal{F}}_{P_n}^2} \lesssim \norm{\mathcal{F}} + \frac{\mathcal{J}_\infty\left(\norm{\mathcal{F}}_{\infty} , \mathcal{F}\right)}{\sqrt{n}} $ and, as such, 
 \begin{align*}
      E \left[ \sup_{f \in \mathcal{F}} \left|(P_n - P)f\right|\right] 
       \lesssim    n^{-1/2} \mathcal{J}_{\infty}\left(\delta, \mathcal{F}\right)  ,
 \end{align*}
 for any $\delta \geq \norm{\mathcal{F}} + n^{-1/2} $ .
  \label{lemma::maximalineq::supremumEntropy::oneclass}
\end{lemma}

The following lemmas provide local maximal inequalities for function classes obtained by taking the pointwise product of elements of two function classes. The first lemma follows from a modification of the proof of Theorem 2.1 in \cite{van2011local} where, in Equation (2.2) of their proof, we argue as in the proof of Theorem 3.1 of \cite{van2014uniform}. In the following, let $\mathcal{H}$ and $\mathcal{G}$ be two uniformly bounded function classes. Denote the class obtained by taking pointwise products of their respective elements as $\mathcal{H}\mathcal{G}:= \{hg: h \in \mathcal{H}, g \in \mathcal{G}\}$. 

 \begin{lemma}
\label{lemma::maximalineq::supremumEntropy}
Suppose $\mathcal{J}_{\infty}(1,\mathcal{H}) < \infty$ and $\mathcal{J}(1,\mathcal{G}) < \infty$. Then
\[
E\bigl\|\mathbb{G}_n\bigr\|_{\mathcal{G}\mathcal{H}}
\ \lesssim\
\sqrt{E\bigl\|\mathcal{G}\bigr\|_{P_n}^2}\,
\mathcal{J}_{\infty}\!\left(
\frac{\sqrt{E\bigl\|\mathcal{H}\mathcal{G}\bigr\|_{P_n}^2}}{\sqrt{E\bigl\|\mathcal{G}\bigr\|_{P_n}^2}},
\ \mathcal{H}
\right)
\;+\;
\|\mathcal{H}\|_{\infty}\,
\mathcal{J}\!\left(
\frac{\sqrt{E\bigl\|\mathcal{H}\mathcal{G}\bigr\|_{P_n}^2}}{\|\mathcal{H}\|_{\infty}},
\ \mathcal{G}
\right).
\]
Furthermore,
\[
\sqrt{E\bigl\|\mathcal{H}\mathcal{G}\bigr\|_{P_n}^2}
\ \lesssim\
\|\mathcal{G}\|_{\infty}\,
\max\!\left(
\|\mathcal{H}\|_{P_0},\
\frac{\mathcal{J}_{\infty}\!\bigl(\|\mathcal{H}\|_{\infty},\mathcal{H}\bigr)}{\sqrt{n}}
\right),
\]
and, for any $\delta_{\mathrm{crit},\mathcal{G}} > 0$ satisfying
\[
\sqrt{n}\,\delta_{\mathrm{crit},\mathcal{G}}^2
\ \ge\
\|\mathcal{G}\|_{\infty}\,
\mathcal{J}\!\left(\frac{\delta_{\mathrm{crit},\mathcal{G}}}{\|\mathcal{G}\|_{\infty}},\ \mathcal{G}\right),
\]
it holds that
\[
\sqrt{E\bigl\|\mathcal{H}\mathcal{G}\bigr\|_{P_n}^2}
\ \lesssim\
\|\mathcal{H}\|_{\infty}\,
\max\!\left(\|\mathcal{G}\|_{P_0},\ \delta_{\mathrm{crit},\mathcal{G}}\right).
\]
\end{lemma}

\begin{proof}
For a function class $\mathcal{F}$, define
\[
\mathcal{J}_n(\delta,\mathcal{F})
:=
\int_0^{\delta}
\sqrt{\log N\!\left(\varepsilon,\mathcal{F},\|\cdot\|_{P_n}\right)}\,d\varepsilon.
\]
By Dudley's inequality (Equation 2.1 of Theorem 2.1 in \cite{van2011local}), with
$\delta_n := \|\mathcal{G}\mathcal{H}\|_{P_n}$, we have
\[
E\bigl\|\mathbb{G}_n\bigr\|_{\mathcal{G}\mathcal{H}}
\ \lesssim\
E\!\left[\mathcal{J}_n(\delta_n,\mathcal{G}\mathcal{H})\right].
\]
Moreover, for $g_1,g_2\in\mathcal{G}$ and $h_1,h_2\in\mathcal{H}$,
\[
\|g_1h_1-g_2h_2\|_{P_n}
\ \le\
\|h_1-h_2\|_{\infty}\,\|g_1\|_{P_n}
\;+\;
\|g_1-g_2\|_{P_n}\,\|h_2\|_{\infty}.
\]
It follows that
\[
\log N\!\left(\varepsilon,\mathcal{G}\mathcal{H},\|\cdot\|_{P_n}\right)
\ \lesssim\
\log N\!\left(\frac{\varepsilon}{\|\mathcal{H}\|_{\infty}},\mathcal{G},\|\cdot\|_{P_n}\right)
\;+\;
\log N_{\infty}\!\left(\frac{\varepsilon}{\|\mathcal{G}\|_{P_n}},\mathcal{H}\right),
\]
and hence
\begin{align*}
\mathcal{J}_n(\delta,\mathcal{G}\mathcal{H})
&=
\int_0^{\delta}
\sqrt{\log N\!\left(\varepsilon,\mathcal{G}\mathcal{H},\|\cdot\|_{P_n}\right)}\,d\varepsilon \\
&\lesssim
\int_0^{\delta}
\sqrt{\log N\!\left(\frac{\varepsilon}{\|\mathcal{H}\|_{\infty}},\mathcal{G},\|\cdot\|_{P_n}\right)}\,d\varepsilon
\;+\;
\int_0^{\delta}
\sqrt{\log N_{\infty}\!\left(\frac{\varepsilon}{\|\mathcal{G}\|_{P_n}},\mathcal{H}\right)}\,d\varepsilon \\
&\lesssim
\|\mathcal{H}\|_{\infty}\,
\mathcal{J}\!\left(\frac{\delta}{\|\mathcal{H}\|_{\infty}},\mathcal{G}\right)
\;+\;
\|\mathcal{G}\|_{P_n}\,
\mathcal{J}_{\infty}\!\left(\frac{\delta}{\|\mathcal{G}\|_{P_n}},\mathcal{H}\right).
\end{align*}
Therefore,
\[
E\bigl\|\mathbb{G}_n\bigr\|_{\mathcal{G}\mathcal{H}}
\ \lesssim\
E\!\left[
\|\mathcal{H}\|_{\infty}\,
\mathcal{J}\!\left(\frac{\delta_n}{\|\mathcal{H}\|_{\infty}},\mathcal{G}\right)
\right]
\;+\;
E\!\left[
\|\mathcal{G}\|_{P_n}\,
\mathcal{J}_{\infty}\!\left(\frac{\delta_n}{\|\mathcal{G}\|_{P_n}},\mathcal{H}\right)
\right].
\]
As in the proof of Theorem 2.1 of \cite{van2011local}, concavity of the map
$(x,y)\mapsto \sqrt{y}\,\mathcal{J}_{\infty}(\sqrt{x/y},\mathcal{H})$ and Jensen's inequality yield
\[
E\bigl\|\mathbb{G}_n\bigr\|_{\mathcal{G}\mathcal{H}}
\ \lesssim\
\|\mathcal{H}\|_{\infty}\,
\mathcal{J}\!\left(\frac{\sqrt{E\delta_n^2}}{\|\mathcal{H}\|_{\infty}},\mathcal{G}\right)
\;+\;
\sqrt{E\|\mathcal{G}\|_{P_n}^2}\,
\mathcal{J}_{\infty}\!\left(
\frac{\sqrt{E\delta_n^2}}{\sqrt{E\|\mathcal{G}\|_{P_n}^2}},
\mathcal{H}
\right).
\]
The first display follows from the definition of $\delta_n$. The remaining norm bounds follow from
Theorems 2.1 and 2.2 of \cite{van2014uniform}.
\end{proof}

The following lemma is an immediate corollary of the above lemma and is used directly in our proofs.

\begin{lemma}
Let ${\mathcal{H}}$ be a uniformly bounded function class satisfying $\mathcal{J}_{\infty}(\infty,{\mathcal{H}}) < \infty$ and let $\mathcal{G}$ be a function class with $\mathcal{J}(\delta,\mathcal{G}) \lesssim \delta \sqrt{k(n) \log(1/\delta)}   $ where $\log(1/\norm{\mathcal{G}}) + \log(1/\norm{\mathcal{H}}) = O(1/\log n)$. Then,
\begin{align*}
E\norm{\mathbb{G}_n}_{{\mathcal{H}}\mathcal{G}} \lesssim   &\norm{\mathcal{G}}_{P_0}\mathcal{J}_{\infty}\left(\max\left\{\norm{\mathcal{H}}_{P_0}, n^{-1/2} \right\}/\norm{\mathcal{G}}_{P_0}, \mathcal{H}\right) \\
& \quad + \norm{\mathcal{H}}_{\infty} \sqrt{k(n) \log n}  \cdot \max\left\{   \norm{\mathcal{G}}_{P_0}, \sqrt{k(n) \log n/ n} \right\}. 
\end{align*}
\label{lemma::maximalineq::supremumEntropy::cor}
\end{lemma}

\section{Statement and proofs of technical lemmas}
\label{appendix::techLemmas}
 \subsection{Lemmas bounding key excess risk remainder terms}

In this section, we use the supporting lemmas of Appendix \ref{appendix::supportinglemmas} to obtain bounds for various remainder terms that appear in our proofs.

The following lemma demonstrates that Lipschitz transformations of the debiased outcome regression estimator $\mu_{n,j}^*$, where $j \in [J]$, falls in function class that is, deterministic conditional on the training set $\mathcal{D}_n \backslash \mathcal{D}_n^j$, and has uniform entropy integral at $\delta > 0$ scaling as $\delta  \sqrt{k(n)  \log (1/\delta)}$.

\begin{lemma}
Assume Condition \ref{cond::boundPos} and work on the \textit{good} events \ref{event::mu_n}, \ref{event::debias_mu_n_star}, and \ref{event::boundedCoef}. Let $T_{n,j}: \mathbb{R} \rightarrow \mathbb{R}$ be a uniformly bounded $L$-Lipschitz continuous function that is deterministic conditional on the data-fold $\mathcal{D}_n \backslash \mathcal{D}_{n}^j$. There exists a uniformly bounded random function class $\mathcal{E}_{n,j}$ with the following properties on the good events. (i) $T_{n,j}(\mu_{n,j^*}) - T_{n,j}(\mu_0)$ and $T_{n,j}(\mu_{n,j^*}) - T_{n,j}(\mu_{n,j})$ fall in $\mathcal{E}_{n,j}$ almost surely; (ii) $\mathcal{E}_{n,j}$ is deterministic conditional on $\mathcal{D}_n \backslash \mathcal{D}_{n}^j$; (iii) $\norm{\mathcal{E}_{n,j}} \lesssim LM   n^{- \beta/(2\beta+1)} + LM \sqrt{k(n) \log n / n}$; and (iv) $\mathcal{J}(\delta, \mathcal{E}_{n,j}) \lesssim \delta \sqrt{k(n)\log (1/\delta)}$.

\label{lemma::HighProbSetmunstar}
\label{lemma::mustarVCDim}
\end{lemma}
\begin{proof}[Proof of Lemma \ref{lemma::HighProbSetmunstar}]
We work on Events \ref{event::mu_n}, \ref{event::debias_mu_n_star}, and \ref{event::boundedCoef}. By construction, as shown in Algorithm \ref{alg::debiasing}, we have the expression $\mu_{n,j}^* = g^{-1}(g \circ \mu_{n,j} + \widehat{\varphi}_{k(n), j}^{\top} \beta_n)$, where $\mu_{n,j}^*$ depends randomly only through $\beta_n \in \mathbb{R}^{2k(n)}$, conditional on $\mathcal{D}_n \backslash \mathcal{D}_{n}^{j}$. Assuming \ref{cond::boundPos} and on Event \ref{event::boundedCoef}, the mapping $(a,w) \mapsto \widehat{\varphi}_{k(n), j}^{\top}(a,w)\beta_n$ belongs to a uniformly bounded subset of a random linear space of dimension $2k(n)$. This space has a VC subgraph dimension of $O(k(n))$, as proven in Lemma 2.6.15 of \cite{vanderVaartWellner}. Additionally, this function class is deterministic conditional on $\mathcal{D}_n \backslash \mathcal{D}_{n}^{j}$.

Considering that $T_{n,j}$ and $g$ are Lipschitz-continuous functions, and $\mu_{n,j}$ and $\mu_0$ are uniformly bounded by \ref{cond::boundedEStimators}, the existence of the above function class and Theorem 2.10.20 of \cite{vanderVaartWellner} imply that $T_{n,j}(\mu_{n,j}^*) - T_{n,j}(\mu_0)$ and $T_{n,j}(\mu_{n,j}^*) - T_{n,j}(\mu_{n,j})$ fall with probability 1 in a single uniformly bounded function space $\mathcal{W}_{n,j}$ with uniform entropy integral $\mathcal{J}(\delta, \mathcal{W}_{n,j}) \lesssim \delta \sqrt{k(n)\log (1/\delta)}$. Again, this function class is deterministic conditional on $\mathcal{D}_n \backslash \mathcal{D}_{n}^j$.

Let $\mathcal{E}_{n,j}$ be the subset of $\mathcal{W}_{n,j}$ consisting of elements with $L^2(P_0)$-norm less than $2LM n^{-\beta/(2\beta+1)} + 2LM \sqrt{k(n) \log n / n}$, where $L$ is the Lipschitz constant of $T_{n,j}$. In view of Event \ref{event::debias_mu_n_star}, using the Lipschitz-continuity of $T_{n,j}$, we can deduce that the event ${T_{n,j}(\mu_{n,j}^*) - T_{n,j}(\mu_0) \in \mathcal{E}_{n,j}}$ occurs almost surely on the good events. Moreover, on the good events, we also have ${T_{n,j}(\mu_{n,j}^*) - T_{n,j}(\mu_{n,j}) \in \mathcal{E}_{n,j}}$, noting that
\begin{align*}
    \|T_{n,j}(\mu_{n,j}) - T_{n,j}(\mu_{n,j}^*)\|& \leq \|T_{n,j}(\mu_{n,j}) - T_{n,j}(\mu_{0})\| + \|T_{n,j}(\mu_{0}) - T_{n,j}(\mu_{n,j}^*)\|\\
    &\leq  2LM    n^{- \beta/(2\beta+1)}  + 2LM \sqrt{k(n) \log n / n}.
\end{align*}

\end{proof}

The next few lemmas bound the localized suprema of various empirical process remainders that appear in our proofs. For this, we recall the sup-norm coupling exponent $\alpha > 0$ of \ref{cond::theorem2::couplings} and the notation $r_n^* :=  n^{-\beta/(2\beta+1)}+ \sqrt{k(n) \log n / n}$ and $s_n ^* :=  n^{-\gamma/(2\gamma+1)} + r_n^* $. We also recall that $\mathbb{I}_{A_n}$ is the indicator that events \ref{event::mu_n}- \ref{event::boundedCoef} occur.

\begin{lemma}
Let $\mathcal{G}$ be any uniformly bounded function class with $\mathcal{J}_{\infty}(\infty, \mathcal{G}) < \infty$ and $m \in \{1,2\}$. Under the conditions of Theorem \ref{theorem::EpLearnerRate}, it holds, for all $\delta \geq \norm{\mathcal{G}} + n^{-1/2}$ and $\delta_{\infty} \geq  \norm{\mathcal{G}}_{\infty}$, that
\begin{align*}
 &E_0^n \left[\mathbb{I}_{A_n} \sup_{\phi \in \mathcal{G}} \left|  \left(\overline{P}_n - \overline{P}_0 \right) \left\{ \overline{\Delta}_{\pi_{n, \diamond} \mu_{n,\diamond}^*}^{(m)}(\cdot\,, \phi) - \overline{\Delta}_{\pi_{n, \diamond} \mu_{n,\diamond}}^{*(m)} (\cdot\,, \phi)\right\} \right| \right] \\
 &\quad\lesssim  \, \frac{r_n^* \mathcal{J}_{\infty} \left( \delta / r_n^*  , {\mathcal{G}}\right)}{\sqrt{n}} + \delta_{\infty}  r_n^*   \sqrt{k(n) \log n / n} .
 \end{align*}
\label{lemma::empProcDeltastar}
\end{lemma}
\begin{proof}[Proof of Lemma \ref{lemma::empProcDeltastar}]
 By our notation, we have, for $\phi \in \mathcal{G}$, that
  \begin{align}
   &  \hspace{-2cm} (\overline{P}_n - \overline{P}_0) \left\{ \overline{\Delta}_{\pi_{n, \diamond} \mu_{n,\diamond}^*}^{(m)}(\cdot\,, \phi) - \overline{\Delta}_{\pi_{n, \diamond} \mu_{n,\diamond}}^{*(m)} (\cdot\,, \phi)\right\} \nonumber \\
     & = \frac{1}{J} \sum_{j \in [J]} ({P}_{n,j} - {P}_0)\left\{ \overline{\Delta}_{\pi_{n,j}, \mu_{n,j}^*}^{(m)}(\cdot\,, \phi) - \overline{\Delta}_{\pi_{n,j}, \mu_{n,j}}^{*(m)} (\cdot\,, \phi)\right\} .  \label{append::eqn::emp1} 
  \end{align}
By the triangle inequality for a given $j \in [J]$, it suffices to bound 
$$E_0^n \left[\mathbb{I}_{A_n}\sup_{\phi \in \mathcal{G}}\left|({P}_{n,j} - {P}_0)\left\{ \overline{\Delta}_{\pi_{n,j}, \mu_{n,j}^*}^{(m)}(\cdot\,, \phi) - \overline{\Delta}_{\pi_{n,j}, \mu_{n,j}}^{*(m)} (\cdot\,, \phi)\right\} \right| \right].$$
To do so, we apply, conditionally on $\mathcal{D}_n \backslash \mathcal{D}_{n}^j$, the maximal inequality of Lemma \ref{lemma::maximalineq::supremumEntropy::cor}, where we use Lemma \ref{lemma::mustarVCDim} to control the randomness of $\mu_{n,j}^*$. To this end, for $\phi \in \mathcal{G}$, $j \in [J]$, and $o \in \mathcal{O}$, we have, by definition, that
 $$\overline{\Delta}_{\pi_{n,j}, \mu_{n,j}^*}^{(m)} (o\,, \phi) - \overline{\Delta}_{\pi_{n,j}, \mu_{n,j}}^{*(m)} (o\,, \phi)   = \frac{ y - \mu_{n,j}^{*}(a,w)  }{\pi_{n,j}(a\miid w)}   \left\{  H_{m,\mu_{n,j}^*}(a,w) - H_{m,\mu_{n,j}}(a,w) \right\}   \left\{\phi(w)\right\} ,$$
 where we recall that $H_{m,\mu}(a,w) :=   c_{a,m}(\dot{g}_m\circ \mu)(a,w)$.
 By Lipschitz continuity of the functions $\dot{g}_1$ and $\dot{g}_2$ in the definition of $H_{m, \mu}$, we have, for some constant $L > 0$, each $o \in \mathcal{O}$, and  any map $(a,w) \mapsto \mu(a,w)$, that
 $$\left|H_{m,\mu}(a,w) - H_{m,\mu_{n,j}}(a,w)\right| \lesssim L |\mu(a,w) - \mu_{n,j}(a,w)|.$$ 
As a consequence, defining the Lipschitz continuous transformation $\mu \mapsto T_{n,j}(\mu) := H_{m, \mu}$, we can write $H_{ m,\mu}  - H_{ m,\mu_{n,j}} = T_{n,j}(\mu) - T_{n,j}(\mu_{n,j})$ where $T_{n,j}$ is deterministic given $\mathcal{D}_n \backslash \mathcal{D}_{n}^j$. Hence, by Lemma \ref{lemma::HighProbSetmunstar}, on the event $A_n$, there exists a function class $\mathcal{E}_{n,j}$ that is deterministic conditional on $\mathcal{D}_n \backslash \mathcal{D}_{n}^j$ with $H_{m,\mu_{n,j}^*}  - H_{m,\mu_{n,j}}  \in \mathcal{E}_{n,j}$ almost surely. Moreover, on the event $A_n$, this function class satisfies $\mathcal{J}(\delta, \mathcal{E}_{n,j}) \lesssim \delta \sqrt{k(n) \log(1/\delta)}   $ and $\norm{\mathcal{E}_{n,j}} \lesssim r_n^*$.

In the following, we work on the event $A_n$. Define the function class,
$$\mathcal{W}_{n,j} := \left\{o \mapsto (y-\mu(a,w))/\pi_{n,j}(a\miid w) \cdot   \left\{  H_{m,\mu}(a,w) - H_{m,\mu_{n,j}}(a,w) \right\} : \mu \in \mathcal{E}_{n,j}  \right\},$$
which is deterministic conditional on $\mathcal{D}_n \backslash \mathcal{D}_{n}^j$. Note, by \ref{cond::positivity} and boundedness of $\mathcal{E}_{n,j}$, that $\|\mathcal{W}_{n,j}\| \lesssim \|\mathcal{E}_{n,j}\| \lesssim r_n^*$. Moreover, for each $\phi \in \mathcal{G}$, the map $o \mapsto \Delta_{\pi_{n,j}, \mu_{n,j}^*}^{(m)}(o\,, \phi) - \overline{\Delta}^{*(m)}_{\pi_{n,j}, \mu_{n,j}}(o\,, \phi)$ falls almost surely in the function class ${\mathcal{G}}\mathcal{W}_{n,j}:=\{\phi g: \phi \in {\mathcal{G}}, g \in \mathcal{W}_{n,j}\}$.
 Conditions \ref{cond::positivity} and \ref{cond::boundedEStimators} along with Theorem 2.10.20. of \cite{vanderVaartWellner} imply, conditional on $\mathcal{D}_n \backslash \mathcal{D}_{n}^j$, that $\mathcal{J}(\delta, \mathcal{W}_{n,j}) \lesssim \mathcal{J}(\delta, \mathcal{E}_{n,j}) \lesssim  \delta \sqrt{k(n) \log n} $ for any $\delta > 1/\sqrt{n}$.

 We are now in a position to bound the empirical process term of the right-hand side of \eqref{append::eqn::emp1}. By the above construction, we have that
 \begin{align*}
 & \hspace{-1cm} E_0^n \left[\mathbb{I}_{A_n}  \sup_{\phi \in \mathcal{G}} \left|  \left({P}_{n,j} - {P}_0 \right) \left\{ \overline{\Delta}_{\pi_{n,j}, \mu_{n,j}^*}^{(m)} (\cdot\,, \phi) - \overline{\Delta}_{\pi_{n,j}, \mu_{n,j}}^{*(m)}  (\cdot\,, \phi)  \right\} \right|    \miid\mathcal{T}_n^j   \right] \\
 & \leq E_0^n \left[\mathbb{I}_{A_n}  \sup_{hg \in \mathcal{G} \mathcal{W}_{n,j}}  \left|  \left({P}_{n,j} - {P}_0 \right)  hg   \right| \miid\mathcal{T}_n^j   \right]
 \end{align*}
Applying Lemma \ref{lemma::maximalineq::supremumEntropy::cor} conditional on $\mathcal{D}_n \backslash \mathcal{D}_{n}^j$ and on the event $A_n$ with $\mathcal{G} := \mathcal{G}$ and $\mathcal{W} := \mathcal{W}_{n,j}$, we find
\begin{align*}
& \hspace{-1 cm} E_0^n\left[\mathbb{I}_{A_n} \sup_{hg \in \mathcal{G} \mathcal{W}_{n,j}}  \left|  \left({P}_{n,j} - {P}_0 \right)  hg   \right|\miid\mathcal{T}_n^j  \right] \\
&\lesssim E_0^n\left[\mathbb{I}_{A_n} \left\{ n^{-1/2}\norm{\mathcal{W}_{n,j}} \mathcal{J}_{\infty}\left(\max\left\{\norm{\mathcal{G}} , n^{-1/2} \right\}/\norm{\mathcal{W}_{n,j}} , \mathcal{G}\right)  \miid\mathcal{T}_n^j \right\}\right] \\
& \quad + E_0^n\left[\mathbb{I}_{A_n} \left\{ \norm{{\mathcal{G}}}_{\infty} \sqrt{k(n) \log n/ n}  \cdot \max\left\{   \norm{\mathcal{W}_{n,j}}_{P_0}, \sqrt{k(n) \log n/ n} \right\}  \right\}  \miid\mathcal{T}_n^j  \right]. 
\end{align*}
Next, using that $\|\mathcal{W}_{n,j}\| \lesssim r_n^*$ on the event $A_n$ and taking expectations of both sides of the previous inequality, we find 
\begin{align*}
E_0^n\left[\mathbb{I}_{A_n}\sup_{hg \in \mathcal{G} \mathcal{W}_{n,j}}  \left|   \left({P}_{n,j} - {P}_0 \right)  hg   \right| \right]  &\lesssim n^{-1/2} r_n^* \mathcal{J}_{\infty} \left( \max\{\norm{\mathcal{G}} , n^{-1/2}\}  / r_n^*  , \mathcal{G}\right) \\
&\quad+ r_n^* \norm{{\mathcal{G}}}_{\infty} \sqrt{k(n) \log n / n}  .
 \end{align*}
 Summing the above over $j \in [J]$ and applying the reverse triangle inequality, we obtain the desired result.

\end{proof}

\begin{lemma}
Let $\mathcal{G}$ be any uniformly bounded function class with $\mathcal{J}_{\infty}(\infty, \mathcal{G}) < \infty$ and $m \in \{1,2\}$. Assume the conditions of Theorem \ref{theorem::EpLearnerRate} and Events \ref{event::mu_n}- \ref{event::boundedCoef}. For all $\delta \geq \norm{\mathcal{G}} + n^{-1/2}$ and $\delta_{\infty} \geq  \norm{\mathcal{G}}_{\infty}$, it holds that
\begin{align*}
 &E_0^n  \left[ \mathbb{I}_{A_n} \sup_{\phi \in \mathcal{G} } \left|(\overline{P}_n - \overline{P}_0)  \left\{ \overline{\Delta}^{(m)}_{\pi_{n,\diamond}, \mu_{n,\diamond}^*}(\cdot\,;  \phi  ) -   \overline{\Delta}_{\pi_{n,\diamond}, \mu_0}^{(m)}(\cdot\,;  \phi) \right\}\right| \right] \\
 &  \lesssim   \frac{r_n^* \mathcal{J}_{\infty} \left( \delta / r_n^*  , {\mathcal{G}}\right)}{\sqrt{n}}  +  \delta_{\infty} r_n^*  \sqrt{k(n) \log n / n}.
\end{align*}\label{lemma::empProcDeltastar2}
  \end{lemma}
  \begin{proof}[Proof of Lemma \ref{lemma::empProcDeltastar2}]
  This proof follows from minor modifications to the proof of Lemma \ref{lemma::empProcDeltastar}. As before, by the triangle inequality, it suffices to bound, for $j \in [J]$, the term 
  $$ E_0^n  \left[ \mathbb{I}_{A_n}\sup_{\phi \in \mathcal{G} } \left|(P_{n,j} - P_0)  \left\{ \overline{\Delta}^{(m)}_{\pi_{n,j}, \mu_{n,j}^*}(\cdot\,;  \phi  ) -   \overline{\Delta}_{\pi_{n,j}, \mu_0}^{(m)}(\cdot\,;  \phi) \right\}\right|\right].$$
  Arguing as in the proof of Lemma \ref{lemma::empProcDeltastar} and applying \ref{cond::positivity} and \ref{cond::boundedEStimators}, we can show, for each $o \in \mathcal{O}$, that
\begin{align*}
  & \left| \overline{\Delta}^{(m)}_{\pi_{n,j}, \mu_{n,j}^*}(o\,;  \phi  ) -   \overline{\Delta}^{(m)}_{\pi_{n,j}, \mu_0}(o\,, \phi)\right| \lesssim \left|\{\mu_{n,j}^*(a,w) - \mu_0(a,w)\}\phi(w)  \right|.
\end{align*}
On Event \ref{event::debias_mu_n_star} and a similar argument as in the proof of Lemma \ref{lemma::empProcDeltastar}, we can find a product function class $\mathcal{G} \mathcal{W}_{n,j} = \{hw: \phi \in \mathcal{G} , w \in \mathcal{W}_{n,j}\}$ that contains $\overline{\Delta}_{\pi_{n,j}, \mu_{n,j}^*}(\cdot\,;  \phi  ) -   \overline{\Delta}_{\pi_{n,j}, \mu_0}(\cdot\,;  \phi  )$ with probability 1. Moreover, on event $A_n$ and by Lemma \ref{lemma::HighProbSetmunstar}, the function class $\mathcal{W}_{n,j}$ can be chosen deterministic conditional on $\mathcal{D}_n \backslash \mathcal{D}_{n}^j$ and to satisfy $\norm{\mathcal{W}_{n,j}} \lesssim  r_n^*$ and $\mathcal{J}(\delta, \mathcal{W}_{n,j}) \lesssim \delta \sqrt{k(n) \log (1/\delta) }  \lesssim \delta \sqrt{k(n) \log n }   $ for any $\delta > n^{-1/2}$. Arguing as in the proof of Lemma \ref{lemma::empProcDeltastar}, we find
\begin{align*}
 \hspace{-1cm} E_0^n  \left[ \mathbb{I}_{A_n}\sup_{\phi \in \mathcal{G} } \left|(P_{n,j} - P_0)  \left\{ \overline{\Delta}^{(m)}_{\pi_{n,j}, \mu_{n,j}^*}(\cdot\,;  \phi  )  -   \overline{\Delta}_{\pi_{n,j}, \mu_0}^{(m)}(\cdot\,;  \phi) \right\}\right| \right]&\leq  E_0^n \left[ \mathbb{I}_{A_n} \sup_{hg \in  \mathcal{G}  \mathcal{W}_{n,j}} \left|   \left({P}_{n,j} - {P}_0 \right)  hg   \right| \right]\\
&\lesssim   \frac{r_n^* \mathcal{J}_{\infty} \left( \delta / r_n^*  , {\mathcal{G}}\right)}{\sqrt{n}}  +  \delta_{\infty} r_n^*  \sqrt{k(n) \log n / n}
 \end{align*}
The result follows, noting that the final bound of the above display does not depend on $j$.

  \end{proof}

  \begin{lemma}
Let $\mathcal{G}$ be any uniformly bounded function class with $\mathcal{J}_{\infty}(\infty, \mathcal{G}) < \infty$ and $m \in \{1,2\}$. Assume the conditions of Theorem \ref{theorem::EpLearnerRate} and Events \ref{event::mu_n}- \ref{event::boundedCoef}. For any $\delta \geq \norm{\mathcal{G}} + n^{-1/2}$, it holds that
      \begin{align*}
 E_0^n \left[ \mathbb{I}_{A_n} \sup_{\phi \in \mathcal{G} } \left| (\overline{P}_n - \overline{P}_0)  \left\{  \overline{\Delta}_{\pi_{n,\diamond}, \mu_0}^{(m)}(\cdot\,;  \phi  ) \right\} \right| \right] \lesssim n^{-1/2} \mathcal{J}_{\infty}\left( \delta, \mathcal{G}\right)
\end{align*}\label{lemma::empProcDeltastar3}
  \end{lemma} 
  \begin{proof}[Proof of Lemma \ref{lemma::empProcDeltastar3}]
   As in the proof of Lemma \ref{lemma::empProcDeltastar2}, it suffices to bound 
   $$E_0^n \left[ \mathbb{I}_{A_n}  \sup_{\phi \in \mathcal{G} }\left| (P_{n,j} - P_0)  \left\{  \overline{\Delta}_{\pi_{n,j}, \mu_0}^{(m)}(\cdot\,;  \phi  ) \right\}\right|\right] .$$ From the proof of Lemma \ref{lemma::empProcDeltastar2}, we know that $\phi(\,\cdot\,) \mapsto \left\{  \overline{\Delta}_{\pi_{n,j}, \mu_0}^{(m)}(\cdot\,;  \phi  ) \right\}$ is a Lipschitz transformation of $\phi(\,\cdot\,)$. Moreover, this Lipschitz map is fixed conditional on the training set $\mathcal{D}_n \backslash \mathcal{D}_{n}^j$. Also, on event $A_n$, the Lipschitz map and its constant are uniformly bounded since $\mu_0$ and $\pi_{n,j}^{-1}$ are uniformly bounded by \ref{cond::boundedEStimators} and \ref{cond::positivity}. Thus, on $A_n$, we have that $ \left\{  \overline{\Delta}_{\pi_{n,j}, \mu_0}^{(m)}(\cdot\,;  \phi  ): \phi \in \mathcal{H} \right\}$ falls in a random bounded function class that is deterministic conditional on $\mathcal{D}_n \backslash \mathcal{D}_{n}^j$ and, by Lemma \ref{lemma::lipschitzPreservation}, has sup-norm entropy integral bounded by $C\mathcal{J}_{\infty}(\delta, \mathcal{G})$, up to a constant. The result then follows from a direct application of Lemma \ref{lemma::maximalineq::supremumEntropy::oneclass}, applied on event $A_n$ and conditionally on $\mathcal{D}_n \backslash \mathcal{D}_{n}^j$ as in Lemma \ref{lemma::empProcDeltastar}.

   \end{proof}

In the following lemmas, we denote as shorthand $\left[F - \widetilde{F} \right]  \left[ \theta_1 - \theta_2\right] := \left[F(\theta_1) - \widetilde{F}(\theta_1)\right] - \left[F(\theta_2) - \widetilde{F}(\theta_2)\right]$ for any two functionals $F, \widetilde{F} \in \ell^{\infty}(\mathcal{F})$ and $\theta_1, \theta_2 \in \mathcal{F}$. We recall that $s_n ^* := n^{-\gamma/(2\gamma+1)} + n^{-\beta/(2\beta+1)} + \sqrt{k(n) \log n / n} $.
\begin{lemma}
Assume the conditions of Theorem \ref{theorem::EpLearnerRate} and Events \ref{event::mu_n}- \ref{event::boundedCoef}.  For $\delta >  n^{-1/2}$, we have
\begin{align*}
  \hspace{-3cm} E_0^n   \Big[ \mathbb{I}_{A_n} \sup_{\theta_1, \theta_2 \in \mathcal{F}: \norm{\theta_1 - \theta_2} \leq \delta } &\left|   \left[ (\overline{P}_n - \overline{P}_0)\left(L_{\pi_{n,\diamond}, \mu_{n,\diamond}^*} -  L_{\pi_0, \mu_0} \right)\right]\left[\theta_1 - \theta_2 \right] \right| \Big] \\
  & \lesssim   n^{-1/2} s_n^* \mathcal{J}_{\infty} \left(  \delta  / s_n^*  , \mathcal{F}\right) + r_n^*\sup_{\theta_1, \theta_2  \in \mathcal{F}: \norm{\theta_1 - \theta_2 } \leq \delta} \norm{\theta }_{\infty} \sqrt{k(n) \log n / n}. 
\end{align*}
Similarly, we have
\begin{align*}
  \hspace{-3cm} E_0^n \Big[ \mathbb{I}_{A_n}  \sup_{\theta \in \mathcal{F}: \norm{\theta } \leq \delta } &\left|  (\overline{P}_n - \overline{P}_0) \left\{L_{\pi_{n,\diamond}, \mu_{n,\diamond}^*}(\theta) -  L_{\pi_0, \mu_0}(\cdot\,,\theta) \right\} \right| \Big]\\
  & \lesssim   n^{-1/2} s_n^* \mathcal{J}_{\infty} \left(  \delta  / s_n^*  , \mathcal{F}\right) + r_n^* \sup_{\theta  \in \mathcal{F}: \norm{\theta } \leq \delta} \norm{\theta }_{\infty} \sqrt{k(n) \log n / n}. 
\end{align*}
 \label{lemma::empProcDeltastar4}
\end{lemma}
 
\begin{proof}[Proof of Lemma \ref{lemma::empProcDeltastar4}]
For $\phi \in \mathcal{H}$ and $m \in \{1,2\}$, we denote
\begin{align*}
    \overline{L}_{\mu}^{(m)}(\cdot, \phi) &:= \phi(\cdot)  \cdot \sum_{a \in \mathcal{A}} c_{a,m} (g_m \circ \mu)(a, \cdot)\\
    \overline{L}_{\pi, \mu}^{(m)}(\cdot, \phi)  &:=  \overline{L}_{\mu}^{(m)}(\cdot, \phi) + \overline{\Delta}^{(m)}_{\pi, \mu}(\cdot, \phi). 
\end{align*} 
Observe that the function $h_m$ in the definition of the loss $L_{\pi ,\mu}$ and $\mathcal{H}$ are $L$-Lipschitz continuous for some $L >0$. Hence, by the triangle inequality and linearity of $L_{\pi, \mu}(\cdot, \theta)$ viewed as a mapping in $h_1 \circ \theta$ and $h_2 \circ \theta$, we find
\begin{align*}
   &  E_0^n \left[ \mathbb{I}_{A_n}  \sup_{\theta_1, \theta_2 \in \mathcal{F}: \norm{\theta_1 - \theta_2} \leq \delta } \left|   \left[ (\overline{P}_n - \overline{P}_0)\left(L_{\pi_{n,\diamond}, \mu_{n,\diamond}^*} -  L_{\pi_0, \mu_0} \right)\right]\left[\theta_1 - \theta_2 \right] \right| \right] \\
   & \hspace{1cm} \lesssim \max_{j \in [J], m \in \{1,2\}}   E_0^n\left[ \mathbb{I}_{A_n} \sup_{\phi_1, \phi_2 \in \mathcal{H}: \norm{\phi_1 - \phi_2} \leq L\delta } \left|   \left[ (P_{n,j} - P_0) \left(\overline{L}^{(m)}_{\pi_{n,j}, \mu_{n,j}^*} -  \overline{L}_{\pi_0, \mu_0}^{(m)} \right)(\cdot, \phi_1 - \phi_2)\right]  \right| \right].
\end{align*}
To bound the right-hand side, it suffices to bound the following for a given $j \in [J]$ and $m \in \{1,2\}$:
$$E_0^n \left[ \mathbb{I}_{A_n}\sup_{\phi_1, \phi_2 \in \mathcal{H}: \norm{\phi_1 - \phi_2} \leq L\delta } \left|    (P_{n,j} - P_0) \left(\overline{L}^{(m)}_{\pi_{n,j}, \mu_{n,j}^*} -  \overline{L}_{\pi_0, \mu_0}^{(m)} \right)(\cdot, \phi_1 - \phi_2)  \right| \right].$$
Towards this goal, we have the expansion:
\begin{align*}
    & \hspace{-2cm} (P_{n,j} - P_0)  \left\{ \overline{L}_{\pi_{n,j}, \mu_{n,j}^*}^{(m)}(\cdot\,,\phi_1 - \phi_2)-  \overline{L}_{\pi_0, \mu_0}^{(m)}(\cdot\,,\phi_1 - \phi_2) \right\}  \\
     &=    (P_{n,j} - P_0)  \left\{\overline{L}^{(m)}_{\pi_{n,j}, \mu_{n,j}^*}(\cdot\,,\phi_1 - \phi_2)-  \overline{L}^{(m)}_{\pi_{n,j}, \mu_0}(\cdot\,, \phi_1 - \phi_2) \right\} \\
     & \quad   +   (P_{n,j}-P_0)  \left\{\overline{L}_{\pi_{n,j}, \mu_0}^{(m)}(\cdot\,, \phi_1 - \phi_2)   - \overline{L}_{\pi_0, \mu_0}^{(m)}(\cdot\,,\phi_1 - \phi_2) \right\} \\
     &\leq  \text{(I)} +  \text{(II)} +  \text{(III)},
\end{align*}   
where, noting that $\overline{L}_{\pi, \mu}^{(m)}(\cdot, \phi)  =  \overline{L}_{\mu}^{(m)}(\cdot, \phi) + \overline{\Delta}^{(m)}_{\pi, \mu}(\cdot, \phi)$, we define:
\begin{align*}
  &   \text{(I)}  := \sup_{\phi_1, \phi_2 \in \mathcal{H}: \|\phi_1 - \phi_2\| \leq L \delta} \left| (P_{n,j} - P_0) \left\{\overline{L}_{\mu_{n,j}^*}^{(m)}(\cdot\,,\phi_1 - \phi_2) - \overline{L}_{\mu_0}^{(m)}(\cdot\,,\phi_1 - \phi_2) \right\} \right|;\\
  & \text{(II)}  := \sup_{\phi_1, \phi_2 \in \mathcal{H}: \|\phi_1 - \phi_2\| \leq L \delta} \left|  (P_{n,j} - P_0) \left\{\overline{\Delta}_{\pi_{n,j}, \mu_{n,j}^*}^{(m)}(\phi_1 - \phi_2) - \overline{\Delta}_{\pi_{n,j}, \mu_0}^{(m)}(\phi_1 - \phi_2) \right\} \right|;\\
  &  \text{(III)}   :=  \sup_{\phi_1, \phi_2 \in \mathcal{H}: \|\phi_1 - \phi_2\| \leq L \delta} \left| (P_{n,j} - P_0) \left\{\overline{L}^{(m)}_{\pi_{n,j}, \mu_0}(\cdot\,,\phi_1 - \phi_2)   - \overline{L}_{\pi_0, \mu_0}^{(m)}(\cdot\,,\phi_1 - \phi_2) \right\} \right|.
\end{align*}
We proceed by bounding in expectation each of the above terms on event $A_n$. To bound (I), take $\phi_1, \phi_2 \in \mathcal{H}$ with $ \| \phi_1 - \phi_2\| \leq L \delta$ and observe that
\begin{align*}
    \overline{L}_{\mu_{n,j}^*}^{(m)}(\cdot\,, \phi_1 - \phi_2) - \overline{L}_{\mu_0}^{(m)}(\cdot\,,\phi_1 - \phi_2) =  (\phi_1 - \phi_2) \sum_{a \in \mathcal{A}} c_{a,m} \left\{g_m \circ \mu_{n,j}^*(a, \cdot)  - g_m \circ \mu_{0}(a, \cdot) \right\}.
\end{align*}
Arguing as in the proof of Lemma \ref{lemma::empProcDeltastar}, using \ref{cond::boundPos}, and applying Lemma \ref{lemma::HighProbSetmunstar}, on event $A_n$, we can almost surely find a function class $\mathcal{E}_{n,j}$ containing $\{g_m \circ \mu_{n,j}^*(a, \cdot)  - g_m \circ \mu_{0}(a, \cdot):  a \in \mathcal{A}\}$ that is deterministic conditional on $\mathcal{D}_n \backslash \mathcal{D}_{n}^j$. Moreover, the class $\mathcal{E}_{n,j}$ can be chosen such that $\mathcal{J}(\delta, \mathcal{E}_{n,j}) \lesssim  \delta \sqrt{\log(1/\delta) k(n)} \lesssim \delta \sqrt{k(n) \log n}$ and $\norm{\mathcal{E}_{n,j}} \lesssim r_n^*$. Hence, on the event $A_n$, $ \overline{L}_{\mu_{n,j}^*}^{(m)}(\cdot\,, \phi_1 - \phi_2) - \overline{L}_{\mu_0}^{(m)}(\cdot\,,\phi_1 - \phi_2)$ is contained in the product class $\mathcal{G}\mathcal{E}_{n,j} := \{gh: g\in \mathcal{G}, h \in \mathcal{E}_{n,j}\}$ where $\mathcal{G} := \{\phi_1 - \phi_2: \phi_1, \phi_2 \in \mathcal{H}, \| \phi_1 - \phi_2\| \leq L \delta \}$.  Arguing as in Lemma \ref{lemma::empProcDeltastar2} and working on event $A_n$, we can apply Theorem 2.10.20 in \cite{vanderVaartWellner} and Lemma \ref{lemma::maximalineq::supremumEntropy::cor} to bound \text{(I)} in expectation as
\begin{align}
   E_0^n \left[ \mathbb{I}_{A_n}\cdot \text{(I)} \right]& \lesssim   n^{-1/2} r_n^* \mathcal{J}_{\infty} \left( \delta / r_n^*  ,  \mathcal{F} \right) \nonumber  + r_n^* \delta_{\infty} \sqrt{k(n) \log n / n}.  \label{eqn::lemma::eqnA}
\end{align}
Next, to bound (II), we have, as a direct consequence of Lemma \ref{lemma::empProcDeltastar2} applied to the function class $\{\phi_1, \phi_2 \in \mathcal{H}: \|\phi_1 - \phi_2\| \leq L \delta\}$, that
\begin{align*}
 E_0^n \left[ \mathbb{I}_{A_n}\cdot \text{(II)}\right] \lesssim    n^{-1/2} r_n^* \mathcal{J}_{\infty} \left(  \delta  / r_n^*  , \mathcal{F}\right)   + r_n^* \delta_{\infty} \sqrt{k(n) \log n / n}.   
\end{align*} 
Lastly, to bound (III), we use that
\begin{align*}
  & \hspace{-1cm} \sup_{\phi_1, \phi_2 \in \mathcal{H}: \|\phi_1 - \phi_2\| \leq L \delta}  \left| (P_{n,j} - P_0)  \left\{\overline{L}^{(m)}_{\pi_{n,j}, \mu_0}(\cdot\,,\phi_1- \phi_2)   - \overline{L}^{(m)}_{\pi_0, \mu_0}(\cdot\,,\phi_1 - \phi_2) \right\}\right|\\
   & = \sup_{\phi_1, \phi_2 \in \mathcal{H}: \|\phi_1 - \phi_2\| \leq L \delta}  \left|(P_{n,j} - P_0) \left\{\overline{\Delta}^{(m)}_{\pi_{n,j}, \mu_0}(\cdot\,,\phi_1 - \phi_2) - \overline{\Delta}^{(m)}_{\pi_0, \mu_0}(\cdot\,,\phi_1 - \phi_2)\right\} \right|.
\end{align*}
We will argue along similar lines as the proof of Lemma \ref{lemma::empProcDeltastar2}. We have $\overline{\Delta}_{\pi_{n,j}, \mu_0}^{(m)}(\cdot\,, \phi_1 - \phi_2) - \overline{\Delta}_{\pi_0, \mu_0}^{(m)}(\cdot\,,\phi_1 - \phi_2)$ is fixed conditional on $\mathcal{D}_n \backslash \mathcal{D}_{n}^j$. Moreover, by \ref{cond::positivity} and \ref{cond::boundedEStimators}, this function can be expressed as a Lipschitz transformation of the function class $\mathcal{W}_{n,j} := \{(\pi_{n,j} - \pi_0)(a \miid \cdot)(\phi_1 - \phi_2): a \in \mathcal{A}: \phi_1, \phi_2 \in \mathcal{H}\}$. Denote the rate of \ref{cond::outcomerateFirst} for $\|\pi_{n,j} - \pi_0\|$ as $\varepsilon_{n,j} := n^{-\gamma/(2\gamma+1)}$. Working on Event \ref{event::pi_n} and applying Lemma \ref{lemma::maximalineq::supremumEntropy::oneclass} conditionally on $\mathcal{D}_n \backslash \mathcal{D}_{n}^j$ to $\mathcal{W}_{n,j}$, we obtain the bound:
\begin{align*}
    E_0^n\left[ \mathbb{I}_{A_n} \cdot \text{(III)} \right] &\lesssim  n^{-1/2} \int_0^{\varepsilon_{n,j}} \sqrt{\log N_{\infty}(\varepsilon,\mathcal{W}_{n,j} )} d\varepsilon \nonumber \\
    & \lesssim n^{-1/2}\varepsilon_{n,j} \mathcal{J}_{\infty} \left(\delta/\varepsilon_{n,j}, \mathcal{F}\right)\\
     & \lesssim n^{-1/2}s_n^* \mathcal{J}_{\infty} \left(\delta/s_n^*, \mathcal{F}\right),
\end{align*}  
where the second inequality follows from $N_{\infty}(\varepsilon,\mathcal{W}_{n,j} ) \lesssim N_{\infty}(\varepsilon / \varepsilon_{n,j},\mathcal{F})$ and a change of variables -- see the proof of Lemma \ref{lemma::maximalineq::supremumEntropy::cor} for a related argument.

Now, we note that $n^{-1/2} r_n^* \mathcal{J}_{\infty} \left(\delta / r_n^*  , \mathcal{F}\right) \lesssim n^{-1/2} s_n^* \mathcal{J}_{\infty} \left(\delta / s_n^*  , \mathcal{F}\right) $ since $r_n^* \leq s_n^*$. Using this, combining the above bounds for (I), (II), and (III), and applying the triangle inequality, we finally obtain:
\begin{align*}
    E_0^n \left[ \mathbb{I}_{A_n}\cdot \text{(I)} \right] +  E_0^n \left[ \mathbb{I}_{A_n}\cdot \text{(II)} \right] +  E_0^n  \left[ \mathbb{I}_{A_n}\cdot \text{(III)} \right] & \lesssim  n^{-1/2}s_n^* \mathcal{J}_{\infty} \left(\delta/s_n^*, \mathcal{F}\right)  + r_n^* \delta_{\infty} \sqrt{k(n) \log n / n},
\end{align*}
as desired.

\end{proof}

 The next lemmas bound several second-order remainders that appear in our proofs. Recall that $\mathbb{I}_{g_1, g_2}$ is the indicator that takes the value $0$ if both $g_1$ and $g_0$ are the identity function. 

\begin{lemma}
Assume the conditions of Theorem \ref{theorem::EpLearnerRate}. On event $A_n$, for any $\theta_1, \theta_2 \in \mathcal{F}$ and $\delta > n^{-1/2}$, we have
    \begin{align*}
        &\hspace{-2cm} \sup_{\theta_1, \theta_2 \in \mathcal{F}: \|\theta_1 - \theta_2\| \leq \delta}\left|  R_{\pi_{n, \diamond} \mu_{n,\diamond}^*}(\theta_2) -  R_{\pi_0, \mu_0}(\theta_2)   - R_{\pi_{n, \diamond} \mu_{n,\diamond}^*}(\theta_1)  +  R_{\pi_0, \mu_0}(\theta_1) \right|   \\
           & \lesssim \sup_{\theta_1, \theta_2 \in \mathcal{F}: \norm{\theta_1 - \theta_2} \leq \delta} \left\{r_n^* n^{-\gamma/(2\gamma+1)}  +  \mathbb{I}_{g_1, g_2} \cdot  (r_n^*)^2 \right\}.
 \end{align*}
  Similarly, on event $A_n$, we have
  \begin{align*}
        &\hspace{-2cm} \sup_{\theta \in \mathcal{F}: \|\theta | \leq \delta}\left|  R_{\pi_{n, \diamond} \mu_{n,\diamond}^*}(\theta) -  R_{\pi_0, \mu_0}(\theta)  \right|    \lesssim  \sup_{\theta \in \mathcal{F}: \norm{\theta} \leq \delta} \norm{\theta}_{\infty}\left\{r_n^* n^{-\gamma/(2\gamma+1)}  +  \mathbb{I}_{g_1, g_2} \cdot  (r_n^*)^2 \right\}.
 \end{align*}
 \label{lemma::P0boundDR}
\end{lemma}

\begin{proof}[Proof of Lemma \ref{lemma::P0boundDR}]

   We work on the event $A_n$, on which, by definition, events \ref{event::mu_n}-\ref{event::boundedCoef} also occur. As in the proof of Lemma \ref{lemma::empProcDeltastar4},  for $\phi \in \mathcal{H}$ and $m \in \{1,2\}$, we denote
$\overline{L}_{\pi, \mu}^{(m)}(\cdot, \phi)  :=  \phi  \cdot \sum_{a \in \mathcal{A}} c_{a,m} (g_m \circ \mu)(a, \cdot) + \overline{\Delta}^{(m)}_{\pi, \mu}(\cdot, \phi) $. To begin, observe that 
   \begin{align*}
    R_{\pi_{n, \diamond} \mu_{n,\diamond}^*}(\theta) -  R_{\pi_0, \mu_0}(\theta)   &=  J^{-1} \sum_{j \in [J]} P_0 \left\{L_{\pi_{n,j}, \mu_{n,j}^*}(\cdot\,,\theta)-  L_{\pi_0, \mu_0}(\cdot\,,\theta) \right\}\\
      &= J^{-1}\sum_{j \in [J]}  \sum_{m \in \{1,2\}} P_0 \left\{\overline{L}_{\pi_{n,j}, \mu_{n,j}^*}^{(m)}(\cdot\,,h_m \circ \theta)-  \overline{L}_{\mu_0}^{(m)}(\cdot\,,h_m \circ \theta) \right\},
   \end{align*} 
   where, for each $m \in \{1,2\}$, the right-hand side is linear in $h_m \circ \theta$. Moreover, since each $h_m$ is Lipschitz continuous, we have $\|\theta_1 - \theta_2\| \leq \delta$ for $\theta_1 , \theta_2 \in \mathcal{F}$ implies, for some $L >0$, that $\|h_m \circ \theta_1 - h_m \circ \theta_2\| \leq L \delta$. Hence, taking the supremum over the previous display, we find  that
   \begin{align*}
      & \sup_{\theta_1, \theta_2 \in \mathcal{F}: \|\theta_1 - \theta_2\| \leq \delta}\left|  R_{\pi_{n, \diamond} \mu_{n,\diamond}^*} -  R_{\pi_0, \mu_0}   - R_{\pi_{n, \diamond} \mu_{n,\diamond}^*}  +  R_{\pi_0, \mu_0} \right|\left[\theta_1- \theta_2 \right] \\
      & \hspace{1cm} \lesssim  \max_{j \in [J], m \in \{1,2\}}\sup_{\phi_1, \phi_2 \in \mathcal{H}: \|\phi_1 - \phi_2\| \leq L\delta} \left|P_0 \left\{\overline{L}_{\pi_{n,j}, \mu_{n,j}^*}^{(m)}(\cdot\,, \phi_1 - \phi_2)-  \overline{L}_{ \mu_0}^{(m)}(\cdot\,, \phi_1 - \phi_2) \right\} \right|,
   \end{align*}
   It suffices to bound, for a given $j \in [J]$ and $m \in \{1,2\}$,
   $$P_0 \left\{\overline{L}_{\pi_{n,j}, \mu_{n,j}^*}^{(m)}(\cdot\,, \phi_1 - \phi_2)-  \overline{L}_{ \mu_0}^{(m)}(\cdot\,, \phi_1 - \phi_2) \right\}$$
   where $\phi_1, \phi_2 \in \mathcal{H}$ are such that $\| \phi_1 - \phi_2\| \leq L \delta$. 
   
   Plugging in the definition of the loss functions, we find for $\phi := \phi_1 - \phi_2$ that
\begin{align*}
   & \hspace{-0.5cm} P_0 \left\{\overline{L}^{(m)}_{\pi_{n,j}, \mu_{n,j}^*}(\cdot\,,\phi)-  \overline{L}^{(m)}_{\mu_0}(\cdot\,,\phi) \right\} \\
     & = E_0\left[ \overline{L}^{(m)}_{\mu_{n,j}^*}(O,\, \phi) - \overline{L}^{(m)}_{\mu_0}(O\,,\phi) \miid \mathcal{D}_n \right]   \\
     &\quad+  \sum_{a \in \mathcal{A}} E_0 \left[ \frac{1(A=a)}{\pi_{n,j}(A \miid W)}  \phi(W)  c_{a,m}(\dot{g}_m\circ \mu_{n,j}^*)(a,W)     \left\{ Y - \mu_{n,j}^*(A, W)\right\}  \miid \mathcal{D}_n \right]\\
      & = E_0\left[ \overline{L}^{(m)}_{\mu_{n,j}^*}(O\,,\phi) - \overline{L}^{(m)}_{\mu_0}(O\,,\phi) \miid \mathcal{D}_n \right]  \\
      & \quad\quad + \sum_{a \in \mathcal{A}} E_0 \left[ \frac{\pi_0(a \miid W)}{\pi_{n,j}(a \miid W)} \phi(W)    c_{a,m}(\dot{g}_m\circ \mu_{n,j}^*)(a,W)    \left\{ \mu_0(a,W) - \mu_{n,j}^*(a, W)\right\} \miid \mathcal{D}_n  \right]\\
      & = \text{(I)} + \text{(II)},
\end{align*}
where we denote:
\begin{align*}
    \text{(I)} &:= E_0\left[ \overline{L}^{(m)}_{\mu_{n,j}^*}(O,\phi) - \overline{L}^{(m)}_{\mu_0}(O\,,\phi) \miid \mathcal{D}_n\right]   \\
    &\quad\;+ \sum_{a \in \mathcal{A}} E_0 \left[      \phi(W)  c_{a,m}(\dot{g}_m\circ \mu_{n,j}^*)(a,W)   \left\{ \mu_0(a,W) - \mu_{n,j}^*(a, W)\right\} \miid \mathcal{D}_n  \right],\\
     \text{(II)} &:=     \sum_{a \in \mathcal{A}} E_0 \left[\phi(W) c_{a,m}(\dot{g}_m\circ \mu_{n,j}^*)(a,W)  \left\{1 - \frac{\pi_0(a \miid W)}{\pi_{n,j}(a \miid W)} \right\}        \left\{ \mu_0(a,W) - \mu_{n,j}^*(a, W)\right\} \miid \mathcal{D}_n  \right].
\end{align*}
Applying Cauchy-Schwarz, \ref{cond::boundedEStimators}, \ref{cond::positivity}, and \ref{cond::regularityOnActionSpace}, term (II) can be upper bounded in absolute value by 
\begin{align*}
    |\text{(II)}| &\lesssim  \max_{a \in \mathcal{A}}\max_{j \in [J]}  P_0 \left| \phi  \{\mu_0(a,\cdot) - \mu_{n,j}(a, \cdot)\} \{\pi_0(a\miid \cdot) - \pi_{n,j}(a \miid  \cdot)\} \right| \\
    &\lesssim   \max_{j \in [J]}  P_0 \left|\phi  \{\mu_0  - \mu_{n,j} \} \{\pi_0  - \pi_{n,j} \} \right| ,
\end{align*}
where the final inequality follows from Condition \ref{cond::positivity} and finiteness of $\mathcal{A}$. 
To bound term (I), note
\begin{align*}
  \text{(I) } & =    \sum_{a \in \mathcal{A}} E_0 \left[  \phi(W) c_{a,m} \left\{ (g_m \circ \mu_{n,j}^*)(a,W)   - (g_m \circ \mu_0)(a,W)\right\} \miid \mathcal{D}_n\right]\\
   & \quad  -   \sum_{a \in \mathcal{A}} E_0 \left[  \phi(W)        c_{a,m}(\dot{g}_m\circ \mu_{n,j}^*)(a,W)   \left\{\mu_{n,j}^*(a, W)- \mu_0(a,W) \right\} \miid \mathcal{D}_n \right].
\end{align*}
First, note that the right-hand side simplifies to zero in the case where $\mathbb{I}_{g_1, g_2} = 0$, since $g_m \circ \mu = \mu$ and $\dot{g}_m \circ \mu = 1 $. Recall that $g_m$ is twice differentiable with Lipschitz continuous derivative. Hence, when $\mathbb{I}_{g_1, g_2} = 1$, using \ref{cond::regularityOnActionSpace}, \ref{cond::boundedEStimators}, \ref{cond::positivity}, and a bound on the remainder of a first-order Taylor expansion of $g_m$, we can continue the previous display as:
\begin{align*}
    \text{(I)} &=    \sum_{a \in \mathcal{A}} E_0 \Bigg[    \phi(W) c_{a,m} \Big\{ (g_m \circ \mu_{n,j}^*)(a,W)   - (g_m \circ \mu_0)(a,W) \\
    &\hspace{10em}- (\dot{g}_m\circ \mu_{n,j}^*)(a,W)\left(\mu_{n,j}^*(a, W)- \mu_0(a,W) \right)\Big\} \,\Big|\, \mathcal{D}_n\Bigg]\\
     & \lesssim \mathbb{I}_{g_1, g_2}   \max_{j \in [J]} P_0 \left| \phi \{\mu_{n,j}^*(a, \cdot ) - \mu_{0}(a, \cdot )\}^2 \right| \\  & \lesssim \mathbb{I}_{g_1, g_2}  \max_{j \in [J]} P_0 \left| \phi \cdot (\mu_{n,j}^*  - \mu_{0})^2 \right| .
\end{align*}
Combining the above bounds and recalling that $\phi = \phi_1 - \phi_2$, we find
 \begin{align*}
         \left|\text{(I)} + \text{(II)}  \right| & \lesssim \max_{j \in [J]}  P_0 \left| (\phi_1-\phi_2) \cdot (\mu_0  - \mu_{n,j} )(\pi_0  - \pi_{n,j} )\right| +   \mathbb{I}_{g_1, g_2}   \max_{j \in [J]} P_0 \left|(\phi_1-\phi_2) (\mu_{n,j}^*  - \mu_{0} )^2 \right|\\
        &  \lesssim  \| \phi_1-\phi_2\|_{\infty} \left\{\max_{j \in [J]}\|\mu_0  - \mu_{n,j}\| \|\pi_0  - \pi_{n,j}\| +   \mathbb{I}_{g_1, g_2}   \max_{j \in [J]} \|\mu_{n,j}^*  - \mu_{0} \|^2   \right\}\\
        &  \lesssim  \delta_{\infty} \left\{\max_{j \in [J]}\|\mu_0  - \mu_{n,j}\| \|\pi_0  - \pi_{n,j}\| +   \mathbb{I}_{g_1, g_2}   \max_{j \in [J]} \|\mu_{n,j}^*  - \mu_{0} \|^2   \right\}\\
       & \lesssim  \delta_{\infty} \left\{r_n^* n^{-\gamma/(2\gamma+1)}  +  \mathbb{I}_{g_1, g_2}  (r_n^*)^2 \right\}.
 \end{align*}
 where the second inequality follows from Cauchy-Schwarz and Events \ref{event::mu_n}-\ref{event::debias_mu_n_star}. The first bound of the lemma then follows. The second bound follows from an identical proof, taking $\phi_1 := \phi$ and $\phi_2 := 0$ for an arbitrary $\phi \in \mathcal{F}$ with $\| \phi\| \leq \delta$.

\end{proof}

\begin{lemma}
Assume the conditions of Theorem \ref{theorem::EpLearnerRate}. On event $A_n$, for any uniformly bounded function class $\mathcal{G}$, we have
\begin{align*}
 \sup_{\phi \in \mathcal{G}} \left|\overline{P}_0\left\{   \overline{\Delta}^{(m)}_{\pi_{n,\diamond}, \mu_{n,\diamond}^*}(\cdot\,, \phi) -   \overline{\Delta}_{\pi_{n,\diamond}, \mu_{n,\diamond}}^{*(m)}(\cdot\,, \phi)\right\} \right| 
&\lesssim  \sup_{\phi \in \mathcal{G}} \norm{\phi}_{\infty}  (r_n^*)^2.
\end{align*}
 \label{lemma::P0boundSieve1}
\end{lemma}
\begin{proof}[Proof of Lemma \ref{lemma::P0boundSieve1}]
In the following, we work on event $A_n$. By definition and the law of iterated expectations, for $\phi \in \mathcal{G}$, we have
\begin{align*}
    \overline{P}_0 \overline{\Delta}^{*(m)}_{\pi_{n,\diamond}, \mu_{n,\diamond}}(\cdot\,, \phi) &= \frac{1}{J}\sum_{j \in [J]} \int \frac{ 1}{\pi_{n,j}(a\miid w)} \left\{   H_{m,\mu_{n,j}}(a,w) \cdot \phi(w)  \right\}\left[y - \mu_{n,j}^{*}(a,w) \right] dP_0(o)\\
     &= \frac{1}{J}\sum_{j \in [J]} \int \frac{ 1}{\pi_{n,j}(a\miid w)} \left\{   H_{m,\mu_{n,j}}(a,w) \cdot \phi(w)  \right\}\left[\mu_0(a,w) - \mu_{n,j}^{*}(a,w) \right] dP_0(o).
\end{align*} 
Similarly,
\begin{align*}
    \overline{P}_0 \overline{\Delta}^{(m)}_{\pi_{n,\diamond}, \mu_{n,\diamond}^*}(\cdot\,, \phi) &= \frac{1}{J}\sum_{j \in [J]} \int \frac{ 1}{\pi_{n,j}(a\miid w)} \left\{   H_{m,\mu_{n,j}^*}(a,w) \cdot \phi(w)  \right\}\left[\mu_0(a,w) - \mu_{n,j}^{*}(a,w) \right] dP_0(o).
\end{align*} 

Next, for each $j \in [J]$, recall from the proof of Lemma \ref{lemma::empProcDeltastar} that 
\begin{equation}
\left|H_{m,\mu_{n,j}^{*}}(a,w) -  H_{m,\mu_{n,j}}(a,w) \right| \lesssim \left|\mu_{n,j}^{*}(a,w) - \mu_{n,j}(a,w) \right|.
\label{proofeqn::HLipschitz}
\end{equation}
Hence, by Cauchy-Schwarz and \ref{cond::positivity}, the above display implies
\begin{align*}
&\left|\overline{P}_0\left\{   \overline{\Delta}^{(m)}_{\pi_{n,\diamond}, \mu_{n,\diamond}^*}(\cdot\,;  \phi) -   \overline{\Delta}_{\pi_{n,\diamond}, \mu_{n,\diamond}}^{*(m)}(\cdot\,;  \phi)\right\} \right| \\
&\lesssim \max_{j \in [J]} \left\{P_0\big | \phi \left(\mu_{n,j}^* - \mu_{n,j}\right)\left(\mu_{n,j}^* - \mu_0\right) \big | \right\}\\
& \lesssim  \sup_{\phi \in \mathcal{G}} \norm{\phi}_{\infty}  \max_{j \in [J]} \left\{P_0\big |   \left(\mu_{n,j}^* - \mu_{n,j}\right)\left(\mu_{n,j}^* - \mu_0\right) \big | \right\}\\
&\lesssim  \sup_{\phi \in \mathcal{G}} \norm{\phi}_{\infty} \max_{j \in [J]} \left\{ P_0 \left(\mu_{n,j}^* - \mu_{0}\right)^2  + P_0 \left|\left(\mu_{n,j} - \mu_{0}\right)\left(\mu_{n,j}^* - \mu_{0}\right) \right|\right\} \\
&\lesssim \sup_{\phi \in \mathcal{G}} \norm{\phi}_{\infty}  \max_{j \in [J]}\left\{  \norm{\mu_{n,j}^* - \mu_0}^2 + \norm{\mu_{n,j} - \mu_{n,j}}\norm{\mu_{n,j}^* - \mu_{n,j}} \right\}\\
&\lesssim \sup_{\phi \in \mathcal{G}} \norm{\phi}_{\infty}  (r_n^*)^2,
\end{align*}
where the final inequality follows from Events \ref{event::mu_n} and \ref{event::debias_mu_n_star}.
The result then follows after taking the supremum on the left-hand side over $\phi \in \mathcal{G}$.
\end{proof}

\begin{lemma}
Assume the conditions of Theorem \ref{theorem::EpLearnerRate}. On event $A_n$, for any uniformly bounded function class $\mathcal{G}$, we have
    \begin{align*}
 \sup_{\phi \in \mathcal{G}} \left|\overline{P}_0  \left\{ \overline{\Delta}^{(m)}_{\pi_{n,\diamond}, \mu_{n,\diamond}^*}(\cdot\,;  \phi - \Pi_{k(n)}\phi) \right\} \right|  &  \lesssim   \sup_{\phi \in \mathcal{G}} \norm{\phi - \Pi_{k(n)}\phi}  \max_{j \in [J]} \norm{\mu_{n,j}^* - \mu_0}\\
 & \lesssim   \sup_{\phi \in \mathcal{G}} \norm{\phi - \Pi_{k(n)}\phi}  r_n^*.
 \end{align*}

  \label{lemma::P0boundSieve2}
\end{lemma}

 \begin{proof}[Proof of Lemma \ref{lemma::P0boundSieve2}]
 In the following, we work on event $A_n$. Let $\phi \in \mathcal{G}$ be arbitrary. By the law of iterated expectations, \ref{cond::positivity}, \ref{cond::boundedEStimators}, and Cauchy-Schwarz, we have
 \begin{align*}
& \hspace{-1.5cm} \left|\overline{P}_0  \left\{ \overline{\Delta}_{\pi_{n,\diamond}, \mu_{n,\diamond}^*}^{(m)}(\cdot\,;  \phi - \Pi_{k(n)}\phi) \right\} \right|    \\
 &\lesssim \frac{1}{J}\sum_{j=1}^J   \left| E_0 \left[\pi_{n,j}^{-1}(A\miid W)  H_{m, \mu_{n,j}^*}(A,W) \left(\phi(W) - \Pi_{k(n)}\phi(W)\right) \left(Y - \mu_{n,j}^*(A,W) \right) \miid \mathcal{D}_n \right] \right|\\
    &\lesssim \max_{j \in [J]}   \left| E_0 \left[\pi_{n,j}^{-1}(A\miid W)  H_{m, \mu_{n,j}^*}(A,W) \left(\phi(W) - \Pi_{k(n)}\phi(W)\right) \left(\mu_0(A,W) - \mu_{n,j}^*(A,W) \right) \miid \mathcal{D}_n \right]\right|\\
     &\lesssim  \sup_{\phi \in \mathcal{G}} \norm{\phi - \Pi_{k(n)}\phi}  \max_{j \in [J]} \norm{\mu_{n,j}^* - \mu_0} \\ 
     & \lesssim   \sup_{\phi \in \mathcal{G}} \norm{\phi - \Pi_{k(n)}\phi}  r_n^*,
 \end{align*}
  where the final inequality follows from Events \ref{event::mu_n} and \ref{event::debias_mu_n_star}.
\end{proof}

\section{Proofs of main results}

\subsection{Proofs for results of Section \ref{section::setup3} }
\label{appendix::convexloss}
 
We define a loss function $(o,\theta) \mapsto L(o,\theta)$ defined on $\mathcal{O} \times \mathcal{F}$ to be $\gamma$-strongly convex \citep{bertsekas2003convex} for some $\gamma > 0$ if, for all $\theta_1, \theta_2 \in \mathcal{F}$ and $o \in \mathcal{O}$, the following holds:
 $$L(o, \theta_1) - L(o, \theta_2) \geq \frac{d}{d\varepsilon}L(o, \theta_2 + \varepsilon(\theta_1 - \theta_2)) \big |_{\varepsilon = 0} + \frac{\gamma}{2} |\theta_1(w) - \theta_2(w)|^2. $$
In fact, if $\theta(o) \mapsto L(o,\theta)$ is twice-differentiable then the loss $L$ is $\gamma$-strongly convex if and only if the second derivative of the map $\theta(o) \mapsto L(o,\theta)$ is positive and bounded below by $\gamma/2$ (Section 3.3. of \cite{fawzi2017topics}).

\begin{lemma}
Assuming the stated conditions in Section \ref{section::setup}, the minimizer $\theta_0 := \argmin_{\theta \in \overline{\mathcal{F}}} R_{0}(\theta)$ exists and is unique.
\label{lemma::uniquePopMinimizer2}
\end{lemma}
\begin{proof}

$\gamma$-strong convexity of the loss function $L_{\mu_0}$ implies, for any $\theta_1, \theta_2 \in \mathcal{F}$ that
$$L_{\mu_0}(w, \theta_1) - L_{\mu_0}(w, \theta_2) \geq \frac{d}{d\varepsilon}L_{\mu_0}(w, \theta_2 + \varepsilon(\theta_1 - \theta_2)) \big |_{\varepsilon = 0} + \frac{\gamma}{2} |\theta_1(w) - \theta_2(w)|^2. $$
Taking the expectation of both sides and exchanging the order of integration and differentiation gives
$$R_{0}(\theta_1) - R_{0}(\theta_2) \geq \frac{d}{d \varepsilon} R_{0}(\theta_2 + \varepsilon (\theta_1 - \theta_2) )\miid_{\varepsilon = 0} + \frac{\gamma}{2} \norm{\theta_1 -\theta_2}_{P_0}^2.$$
It follows that the functional $\theta \mapsto R_{0}(\theta)$ is strongly convex as a mapping from $L^2(P_{0,W})$ to $\mathbb{R}$. We now show that this functional is also continuous. Note for $i=1,2$ that since $g_i$ is continuous and $\mu_0$ has bounded range, we have $g_i \circ \mu_0$ has uniformly bounded range. Since $h_1,h_2$ are Lipschitz continuous, we have that $\theta(w) \mapsto L_{\mu_0}(w, \theta(w))$ is also Lipschitz continuous. Thus, there exists some constant $L >0$ such that
$$\left|R_{0}(\theta_1) - R_{0}(\theta_2)\right| \leq L E_{0} \left|\theta_1(W) - \theta_2(W) \right| \leq L \norm{\theta_1 - \theta_2},$$
by Cauchy-Schwarz. It follows that $\theta \mapsto R_{\mu}(\theta)$ is a Lipschitz continuous functional defined on $L^2(P_0)$. By Theorem 5.5 of \cite{alexanderian2019optimization}, a strongly convex and continuous functional defined on a closed and bounded convex subset of a Hilbert space admits a unique minimizing solution. Since $\overline{\mathcal{F}}$ is closed, bounded, and convex, we conclude that $\argmin_{\theta \in \overline{\mathcal{F}}} R_{0}(\theta)$ is nonempty and contains a unique solution.

\end{proof}

\begin{lemma}
Suppose there exists a constant $\gamma>0$ such that it is $P_0$-almost surely true that $\sum_{a \in \mathcal{A}} c_{a,1} \cdot g_1(\mu_0(a,W))\not=0$, $\ddot{h}_2(\theta(W))\not=0$, and
$$\frac{\ddot{h}_1(\theta(W))}{\ddot{h}_2(\theta(W))} > - \frac{ \sum_{a \in \mathcal{A}} c_{a,2} \cdot g_2(\mu_0(a,W))    }{ \sum_{a \in \mathcal{A}} c_{a,1} \cdot g_1(\mu_0(a,W))  } +   \frac{\gamma}{\ddot{h}_2(\theta(W)) \sum_{a \in \mathcal{A}} c_{a,1} \cdot g_1(\mu_0(a,W)) }.$$
Then, the loss $L_{\mu_0}$ of \eqref{eqn::GeneralClassRisk} is strongly convex.
\label{lemma::uniquePopMinimizer1}
\end{lemma}
\begin{proof}
The stated condition implies that the second derivative of the loss with respect to $\theta(w)$ holding $w \in  \mathcal{W}$ fixed satisfies:
$$\frac{d^2}{d^2\theta(w)}L_{\mu_0}(w,\theta) = \ddot{h}_1(\theta(w)) \cdot \sum_{a \in \mathcal{A}} c_{a,1} (g_1 \circ \mu_0)(a,w) + \ddot{h}_2(\theta(w)) \cdot \sum_{a \in \mathcal{A}} c_{a,2} (g_2 \circ \mu_0)(a,w) > \gamma.$$
The mean value theorem combined with an exact second order Taylor expansion of $\varepsilon \mapsto L_{\mu_0}(w,\theta_2 + \varepsilon (\theta_1 - \theta_2))$ at $\varepsilon =0$ implies, for some $\widetilde \varepsilon_w \in [0,1]$ and $\widetilde{\theta}(w) := \theta_2(w) + \widetilde \varepsilon_w (\theta_1(w) - \theta_2(w))$, that
\begin{align*}
    L_{\mu_0}(w, \theta_1) - L_{\mu_0}(w, \theta_2) &= \frac{d}{d\varepsilon}L_{\mu_0}(w, \theta_2 + \varepsilon(\theta_1 - \theta_2)) \big |_{\varepsilon = 0} + \frac{1}{2}\frac{d^2L_{\mu_0}(w,\theta) }{d^2\theta(w)} \big|_{\theta(w) = \widetilde{\theta}(w)} |\theta_1(w) - \theta_2(w)|^2\\
    & \geq \frac{d}{d\varepsilon}L_{\mu_0}(w, \theta_2 + \varepsilon(\theta_1 - \theta_2)) \big |_{\varepsilon = 0} + \frac{\gamma}{2} |\theta_1(w) - \theta_2(w)|^2,
\end{align*}
as desired.
 
\end{proof}

\begin{proof}[Proof of Theorem \ref{theorem::generalEIF}]

We first determine the efficient influence function of parameters of the form $P \mapsto E_P \left\{\varphi(W) (g \circ \mu_P)(a,W) \right\}$ for $\varphi \in \{h_m \circ \theta: \theta \in \mathcal{F}, m \in \{1,2\}\}$ and $g \in \{g_1, g_2\}$. Note that the population risk corresponding to \eqref{eqn::GeneralClassRisk} can be expressed as a linear combination of the parameters of this form.  

To this end, let $(P_{\varepsilon} : \varepsilon \in \mathbb{R}) \subset \mathcal{M}$ be an arbitrary quadratic mean differentiable submodel with $P_{\varepsilon} = P$ at $\varepsilon = 0$ and score $v \in L^2_0(P)$ at $\varepsilon = 0$. Abusing notation, we compute
\begin{align*}
    \frac{d}{d\varepsilon} E_{P_{\varepsilon}} \left\{\varphi(W) (g \circ \mu_{P_{\varepsilon}})(a,W) \right\} \Big |_{ \varepsilon = 0} &=  \frac{d}{d\varepsilon} \left[ E_{P_{\varepsilon}}   \left\{\varphi(W) (g \circ \mu_{P})(a,W) \right\}   \right] \Big |_{ \varepsilon = 0} \\
   & \quad +   E_P\left\{\varphi(W)\frac{d}{d\varepsilon}  (g \circ \mu_{P_{\varepsilon}})(a,W) \Big |_{ \varepsilon = 0} \right\}   \\
     & = E_P\left\{  \varphi(W) (g \circ \mu_{P})(a,W)   v(O) \right\}  \\
     & \quad +  \frac{d}{d\varepsilon} \left[ E_P  \left\{\varphi(W) (\dot{g} \circ \mu_{P})(a,W)  \mu_{P_{\varepsilon}}(a,W) \right\}   \right]\Big |_{ \varepsilon = 0} .
\end{align*} 
Note,
\begin{align*}
   & \hspace{-1cm} \frac{d}{d\varepsilon} \left[ E_P \left\{\varphi(W) (\dot{g} \circ \mu_{P})(a,W)  \mu_{P_{\varepsilon}}(a,W) \right\} \right] \Big |_{ \varepsilon = 0} \\
    & =  \frac{d}{d\varepsilon} \left[  \int \left\{\varphi(w) (\dot{g} \circ \mu_{P})(a,w)  \frac{1(s=a)Y}{\pi_P(a \miid w)} \right\} P_{\varepsilon}(dy \miid a, w) P_{A,W}(s,w) \right] \Big |_{ \varepsilon = 0}  \\
     & =    E_P\left\{\varphi(W) (\dot{g} \circ \mu_{P})(a,W)  \frac{1(A=a)Y}{\pi_P(a \miid W)} v_Y(O) \right\}    \\
      & =    E_P \left\{\varphi(W) (\dot{g} \circ \mu_{P})(a,W)  \frac{1(A=a)(Y-\mu_P(A,W))}{\pi_P(a \miid W)}v(O)  \right\}  dP(O)  ,
\end{align*}  
where $o \mapsto v_Y(o) := v(o) - E_P[v(O) \miid A=a ,W=w] \in L^2_0(P)$ is the projection of $v$ onto the Hilbert subspace consisting of functions that are, conditional on $A,W$, mean-zero functions of $Y$. The efficient influence function of the parameter $P \mapsto E_P \left\{\varphi(W) (g \circ \mu_P)(a,W) \right\}$ can now be read off as
\begin{align*}
    D: o \mapsto &\varphi(w) (g \circ \mu_{P})(s,w) - E_P[\varphi(W) (g \circ \mu_{P})(s,W)] \\
    &\;+ \left\{\varphi(W) (\dot{g} \circ \mu_{P})(s,w)  \frac{1(a=s)}{\pi_P(a \miid w)} (Y-\mu_P(a,w))\right\}
\end{align*}
noting that 
$  \frac{d}{d\varepsilon} E_{P_{\varepsilon}} \left\{\varphi(W) (g \circ \mu_{P_{\varepsilon}})(a,W) \right\} \Big |_{ \varepsilon = 0} = \langle D, v \rangle_{L^2(P)} .$
The result then follows from linearity of the derivative operator, noting, for $\theta \in \mathcal{F}$, that
$$R_P(\theta) =  \sum_{s \in \mathcal{A}} c_{s,2} E_P \left\{\varphi_1(W) (g_1 \circ \mu_P)(s,W) \right\} +   \sum_{s \in \mathcal{A}} c_{s,2} E_P \left\{\varphi_2(W) (g_2 \circ \mu_P)(s,W) \right\}$$
is a linear combination of parameters of the above form, where $\varphi_m := h_m \circ \theta$ for $m \in \{1,2\}$.

\end{proof}

\subsection{Proof of Theorem   \ref{theorem::EPriskEff}}
\begin{proof}

Before proceeding with the proof, we introduce some notation. Denote $\Pi_{k(n)}\mathcal{H}:= \{\Pi_{k(n)} \phi: \phi \in \mathcal{H}\} \subset \mathcal{H}_{k(n)}$, and $\mathcal{H} - \Pi_{k(n)}\mathcal{H} := \{\phi - \Pi_{k(n)}\phi: \phi \in \mathcal{H}\}$. Define the sieve approximation rate $\rho_{n, \infty} := n^{-1/2} + \{\log k(n)\}^{\nu} k(n)^{-\rho}$ where $\nu, \rho>0$ are the exponents of \ref{cond::sieveApproxthrm}. We note, by Lemma \ref{lemma::sieveRateSup}, it holds that $\sup_{\phi \in \mathcal{H}}\| \phi - \Pi_{k(n)}\phi \|_{\infty} \lesssim \rho_{n, \infty}$.  We recall that $\mathbb{I}_{g_1, g_2}$ denotes the indicator that takes the value $0$ if the functions $g_1$ and $g_2$ in \eqref{eqn::GeneralClassRisk} are the identity function and $1$ otherwise. We will assume, without loss of generality, that event $A_n$ and, therefore, events \ref{event::mu_n}-\ref{event::boundedCoef}, occur.

\textbf{Proof strategy.} To establish the result of the theorem, we proceed as follows. First, we establish that the debiasing term is asymptotically negligible in that $$\|\overline{P}_n {\Delta}_{\pi_{n,\diamond}, \mu_{n,\diamond}^*}\|_{\ell^{\infty}(\mathcal{F})} := \sup_{\theta \in \mathcal{F}} |\overline{P}_n {\Delta}_{\pi_{n,\diamond}, \mu_{n,\diamond}^*}(\cdot; \theta) |=\smallO_p(n^{-1/2}).$$
Noting that $\overline{P}_n {\Delta}_{\pi_{n,\diamond}, \mu_{n,\diamond}^*}= R_{n,k(n)} - R_{n,\pi_{n,\diamond}, \mu_{n,\diamond}^*}$, this implies that $\|R_{n,k(n)} - R_{n,\pi_{n,\diamond}, \mu_{n,\diamond}^*}\|_{\ell^{\infty}(\mathcal{F})} =\smallO_p(n^{-1/2})$. Consequently, the EP-learner risk estimator $R_{n,k(n)}$ is asymptotically equivalent to the one-step risk estimator $R_{n,\pi_{n,\diamond}, \mu_{n,\diamond}^*}$. Afterwards, using the previous result, we show that the EP-learner risk estimator $R_{n,k(n)}$ satisfies $\|R_{n,k(n)} - R_{n,0} \|_{\ell^{\infty}(\mathcal{F})} =\smallO_p(n^{-1/2})$ and is, thus, asymptotically equivalent to the oracle-efficient one-step risk estimator $R_{n,0}$. To do so, it suffices, by the triangle inequality, to establish that $\|R_{n, \pi_n, \mu_n^*} - R_{n,0} \|_{\ell^{\infty}(\mathcal{F})} =\smallO_p(n^{-1/2})$. The desired weak convergence result then follows from weak convergence of $R_{n,0}$ to $R_0$ in $\ell^{\infty}(\mathcal{F})$ and a functional form of Slutsky's lemma.

\textbf{Controlling the debiasing term.} We begin by controlling the debiasing remainder term $\|\overline{P}_n {\Delta}_{\pi_{n,\diamond}, \mu_{n,\diamond}^*}\|_{\ell^{\infty}(\mathcal{F})}$. By the definition of $\overline{\Delta}^{(m)}_{\pi_{n, \diamond}, \mu_{n,\diamond}^*}(\cdot, \phi)$ for $\phi \in \mathcal{H}$, we have that
\begin{align*}
\sup_{\theta \in \mathcal{F}} \left|\overline{P}_n {\Delta}_{\pi_{n, \diamond}, \mu_{n, \diamond}^*}(\cdot\,; \theta) \right| &\lesssim  \max_{m \in \{1,2\}} \sup_{\phi \in \mathcal{H}} \left|\overline{P}_n \overline{\Delta}^{(m)}_{\pi_{n, \diamond}, \mu_{n, \diamond}^*}(\cdot\,, \phi) \right| .
\end{align*}
Hence, it suffices to bound, for each $m \in \{1,2\}$, the term $\sup_{\phi \in \mathcal{H}} \left|\overline{P}_n \overline{\Delta}^{(m)}_{\pi_{n, \diamond}, \mu_{n, \diamond}^*}(\cdot\,, \phi) \right|$. Noting that $\phi \mapsto \overline{\Delta}_{\pi_{n, \diamond}, \mu_{n, \diamond}^*}^{(m)}(\cdot\,, \phi)$ is linear as a mapping in $\phi$, the triangle inequality gives   
\begin{align*}
\sup_{\phi \in \mathcal{H}} \left|\overline{P}_n \overline{\Delta}^{(m)}_{\pi_{n,\diamond}, \mu_{n,\diamond}^*}(\cdot\,, \phi) \right| &\leq  \sup_{\phi \in \mathcal{H}} \left|\overline{P}_n \overline{\Delta}^{(m)}_{\pi_{n,\diamond}, \mu_{n,\diamond}^*}(\cdot\,; \Pi_{k(n)}\phi) \right| \\
&\quad+  \sup_{\phi \in \mathcal{H}}  \left|\overline{P}_n  \left\{ \overline{\Delta}^{(m)}_{\pi_{n,\diamond}, \mu_{n,\diamond}^*}(\cdot\,;  \phi) - \overline{\Delta}^{(m)}_{\pi_{n,\diamond}, \mu_{n,\diamond}^*}(\cdot\,; \Pi_{k(n)}\phi)  \right\}\right|.
\end{align*}
We denote the two terms on the right-hand side of the above display as:
\begin{align*}
    \text{(I)} &:=  \sup_{\phi \in \mathcal{H}} \left|\overline{P}_n \overline{\Delta}^{(m)}_{\pi_{n,\diamond}, \mu_{n,\diamond}^*}(\cdot\,; \Pi_{k(n)}\phi) \right|;\\
    \text{(II)} &:=  \sup_{\phi \in \mathcal{H}}  \left|\overline{P}_n  \left\{ \overline{\Delta}^{(m)}_{\pi_{n,\diamond}, \mu_{n,\diamond}^*}(\cdot\,;  \phi) - \overline{\Delta}^{(m)}_{\pi_{n,\diamond}, \mu_{n,\diamond}^*}(\cdot\,; \Pi_{k(n)}\phi)  \right\}\right|.
\end{align*}
 We bound each of the above terms in turn. To bound (I), recall that the first-order equations characterizing the debiased outcome regression estimators $(\mu_{n,j}^*: j \in [J])$ of Algorithm \ref{alg::debiasing} imply that
 $$\sup_{\phi \in \mathcal{H}} \left|\overline{P}_n \overline{\Delta}^{*(m)}_{\pi_{n,\diamond}, \mu_{n,\diamond}}(\cdot\,; \Pi_{k(n)}\phi) \right| = 0.$$
Next, note that $\overline{\Delta}^{*(m)}_{\pi_{n,\diamond}, \mu_{n,\diamond}} = \overline{\Delta}^{(m)}_{\pi_{n,\diamond}, \mu_{n,\diamond}^*}$ if $\mathbb{I}_{g_1, g_2} = 0$. Hence, in view of the previous display, term (I) is zero in the case where $\mathbb{I}_{g_1, g_2} = 0$. If $\mathbb{I}_{g_1, g_2} = 1$, we can further expand term (I) as follows:
\begin{align*}
\text{(I)} & =  \mathbb{I}_{g_1, g_2} \sup_{\phi \in \mathcal{H}}\left|\overline{P}_n \overline{\Delta}_{\pi_{n,\diamond}, \mu_{n,\diamond}}^{*(m)}(\cdot\,; \Pi_{k(n)}\phi)  + \overline{P}_n\left\{   \overline{\Delta}^{(m)}_{\pi_{n,\diamond}, \mu_{n,\diamond}^*}(\cdot\,; \Pi_{k(n)}\phi) -   \overline{\Delta}_{\pi_{n,\diamond}, \mu_{n,\diamond}}^{*(m)}(\cdot\,; \Pi_{k(n)}\phi)\right\} \right| \\
& \leq  \mathbb{I}_{g_1, g_2} \sup_{\phi \in \mathcal{H}}\left|\overline{P}_n \overline{\Delta}_{\pi_{n,\diamond}, \mu_{n,\diamond}}^{*(m)}(\cdot\,; \Pi_{k(n)}\phi)  \right| \\
&\quad+  \mathbb{I}_{g_1, g_2} \sup_{\phi \in \mathcal{H}} \left| \overline{P}_n\left\{   \overline{\Delta}^{(m)}_{\pi_{n,\diamond}, \mu_{n,\diamond}^*}(\cdot\,; \Pi_{k(n)}\phi) -   \overline{\Delta}_{\pi_{n,\diamond}, \mu_{n,\diamond}}^{*(m)}(\cdot\,; \Pi_{k(n)}\phi)\right\} \right| \\
& \leq  \text{(Ia)} + \text{(Ib)} +\text{(Ic)},
\end{align*}
where, using that $\overline{P}_n = \overline{P}_0 + (\overline{P}_n - \overline{P}_0)$, we define:
\begin{align*}
    \text{(Ia)} &:=\mathbb{I}_{g_1, g_2} \sup_{\phi \in \mathcal{H}} \left|\overline{P}_n \overline{\Delta}_{\pi_{n,\diamond}, \mu_{n,\diamond}}^{*(m)}(\cdot\,; \Pi_{k(n)}\phi)\right|; \\
    \text{(Ib)} &:= \mathbb{I}_{g_1, g_2} \sup_{\phi \in \mathcal{H}}\left|\overline{P}_0\left\{   \overline{\Delta}^{(m)}_{\pi_{n,\diamond}, \mu_{n,\diamond}^*}(\cdot\,; \Pi_{k(n)}\phi) -   \overline{\Delta}_{\pi_{n,\diamond}, \mu_{n,\diamond}}^{*(m)}(\cdot\,; \Pi_{k(n)}\phi)\right\}\right|;\\
    \text{(Ic)}  &:= \mathbb{I}_{g_1, g_2}\sup_{\phi \in \mathcal{H}} \left| (\overline{P}_n - \overline{P}_0)\left\{   \overline{\Delta}^{(m)}_{\pi_{n,\diamond}, \mu_{n,\diamond}^*}(\cdot\,; \Pi_{k(n)}\phi) -   \overline{\Delta}_{\pi_{n,\diamond}, \mu_{n,\diamond}}^{*(m)}(\cdot\,; \Pi_{k(n)}\phi)\right\} \right|.
\end{align*}
The first order derivative equations solved by $(\mu_{n,j}^*: j \in [J])$ imply that term (Ia) is zero. We will bound term (Ib) using the Cauchy-Schwarz inequality and term (Ic) using empirical process techniques. Note, by \ref{cond::sieveApproxthrm} and \ref{cond::regularityOnActionSpace}, we have $\sup_{\phi \in \mathcal{H}} \norm{ \Pi_{k(n)}\phi}_{\infty} \leq  \sup_{\phi \in \mathcal{H}} \norm{ \phi}_{\infty} + \sup_{\phi \in \mathcal{H}} \norm{ \phi - \Pi_{k(n)}\phi}_{\infty} = O(1) + \rho_{n,\infty} = O(1)$, since $\rho_{n,\infty} = o(1)$. Hence, we have $\Pi_{k(n)} \mathcal{H}$ is a uniformly bounded function class. By Lemma \ref{lemma::P0boundSieve1}, uniform boundedness of $\Pi_{k(n)}\mathcal{H}$, and \ref{cond::outcomerateDebiased}, we have
\begin{align*}
\text{(Ib)} &\lesssim \mathbb{I}_{g_1, g_2} (r_n^*)^2 .
\end{align*}
Next, applying Lemma \ref{lemma::empProcDeltastar} with the function class $\mathcal{G}:=\Pi_{k(n)}\mathcal{H}$, and the entropy integral bound of Lemma \ref{lemma::metricentropybounds}, we find that 
\begin{align}
    \text{(Ic)} &= \mathbb{I}_{g_1, g_2} \mathcal{O}_p\bigg( n^{-1/2} r_n^* \mathcal{J}_{\infty} \left(  1 / r_n^*  , \mathcal{F}\right)    +  r_n^*  \sqrt{k(n) \log n / n}\bigg) .
 \label{proof::lemma::boundIc}
 \end{align}
Turning to term (II), we introduce the following bound:
\begin{align*}
 \text{(II)} &= \sup_{\phi \in \mathcal{H}} \left|\overline{P}_n  \left\{ \overline{\Delta}^{(m)}_{\pi_{n,\diamond}, \mu_{n,\diamond}^*}(\cdot\,;  \phi - \Pi_{k(n)}\phi) \right\} \right| \\
 &\leq   \sup_{\phi \in \mathcal{H}} \left|\overline{P}_0  \left\{ \overline{\Delta}^{(m)}_{\pi_{n,\diamond}, \mu_{n,\diamond}^*}(\cdot\,;  \phi - \Pi_{k(n)}\phi) \right\} \right| +  \sup_{\phi \in \mathcal{H}} \left|(\overline{P}_n - \overline{P}_0)  \left\{ \overline{\Delta}^{(m)}_{\pi_{n,\diamond}, \mu_{n,\diamond}^*}(\cdot\,;  \phi - \Pi_{k(n)}\phi) \right\} \right|\\
  & \leq  \text{(IIa)} + \text{(IIb)} +\text{(IIc)},
\end{align*}
where we define:
\begin{align*}
 \text{(IIa)} &:=  \sup_{\phi \in \mathcal{H}} \left| \overline{P}_0  \left\{ \overline{\Delta}^{(m)}_{\pi_{n,\diamond}, \mu_{n,\diamond}^*}(\cdot\,;  \phi - \Pi_{k(n)}\phi) \right\} \right|;\\
 \text{(IIb)} &:=    \sup_{\phi \in \mathcal{H}} \left|(\overline{P}_n - \overline{P}_0)  \left\{  \overline{\Delta}^{(m)}_{\pi_{n,\diamond}, \mu_0}(\cdot\,;  \phi - \Pi_{k(n)}\phi) \right\} \right|;\\
 \text{(IIc)} &:=  \sup_{\phi \in \mathcal{H}} \left| (\overline{P}_n - \overline{P}_0)  \left\{ \overline{\Delta}^{(m)}_{\pi_{n,\diamond}, \mu_{n,\diamond}^*}(\cdot\,;  \phi - \Pi_{k(n)}\phi) -   \Delta_{\pi_{n,\diamond}, \mu_0}(\cdot\,;  \phi - \Pi_{k(n)}\phi) \right\} \right|.
\end{align*}
By Lemma \ref{lemma::P0boundSieve2}, term (IIa) satisfies 
 \begin{align*}
   \text{(IIa)} &= \mathcal{O}_p\left( \sup_{\phi \in \mathcal{H}} \norm{\phi - \Pi_{k(n)}\phi}  r_n^* \right) = \mathcal{O}_p\left( \rho_{n,\infty}r_n^* \right).
 \end{align*}
Next, applying  Lemma \ref{lemma::metricentropybounds}, Lemma \ref{lemma::empProcDeltastar3} with $\mathcal{G}:= \mathcal{H}-\Pi_{k(n)}\mathcal{H}$, and Markov's inequality, we find that
$$ \text{(IIb)}= \mathcal{O}_p\left(n^{-1/2}\mathcal{J}_{\infty}(\max\{n^{-1/2},\rho_{n,\infty}\}, \mathcal{F})\right).$$
Similarly, by Lemma \ref{lemma::empProcDeltastar2}, \ref{cond::outcomerateDebiased}, and Markov's inequality, we have that
\begin{align*}
     \text{(IIc)} = \mathcal{O}_p\left( n^{-1/2} r_n^* \mathcal{J}_{\infty} \left( (\rho_{n,\infty} + n^{-1/2}) / r_n^*  , \mathcal{F}\right)    + r_n^* \rho_{n,\infty} \sqrt{k(n) \log n / n} \right).
\end{align*} 
Combining all the bounds for (I) and (II), we obtain 
 \begin{align*}
\text{(I)} &=   \mathbb{I}_{g_1,g_2}  \mathcal{O}_p\left(  n^{-2\beta/(2\beta+1)} + k(n)\log n /n + n^{-1/2} r_n^* \mathcal{J}_{\infty} \left( 1 / r_n^*  , \mathcal{F}\right) \right);   \\
 \text{(II)} & =   \mathcal{O}_p\bigg( \left\{\rho_{n,\infty} + \rho_{n,\infty}\sqrt{k(n) \log n / n}  \right\}\left\{  n^{-\beta/(2\beta+1)}   + \sqrt{k(n)\log n /n} \right\} \\
     & \quad + n^{-1/2} r_n^* \mathcal{J}_{\infty} \left( \max\{\rho_{n,\infty}, n^{-1/2}\} / r_n^*  , \mathcal{F}\right) +  n^{-1/2}\mathcal{J}_{\infty}\left(\max\{\rho_{n,\infty}, n^{-1/2}\}, \mathcal{F}\right) \bigg).
\end{align*}

We will now show that (I) $+$ (II) $=o_p(n^{-1/2})$ and, so, $\|\overline{P}_n {\Delta}_{\pi_{n,\diamond}, \mu_{n,\diamond}^*}\|_{\ell^{\infty}(\mathcal{F})} =\smallO_p(n^{-1/2})$. To do so, we first show that the following entropy integral remainders are negligible:
\begin{align}
    &n^{-1/2} r_n^* \mathcal{J}_{\infty} \left( 1 / r_n^*  , \mathcal{F}\right) + n^{-1/2} r_n^* \mathcal{J}_{\infty} \left( \max\{\rho_{n,\infty}, n^{-1/2}\} / r_n^*  , \mathcal{F}\right) +  n^{-1/2}\mathcal{J}_{\infty}\left(\max\{ \rho_{n,\infty} , n^{-1/2}\}, \mathcal{F} \right) \nonumber \\
    &\quad=\smallO_p(n^{-1/2}).
\label{eqnproof::efficiencyTheorem::A}
\end{align}
Note $r_n^* = o(1)$ by \ref{cond::outcomerateSecond}, and the upper bound on the growth rate $k(n)$. Also, by \ref{cond::sieveApproxthrm} and the lower bound on the sieve growth rate $k(n)$, we have $\rho_{n,\infty} = o(1)$. Thus, by \ref{cond::regularityOnActionSpace}, each term in the above display is $o_p(n^{-1/2})$. 

In view of the above displays, to show that (I) is  $o_p(n^{-1/2})$, it remains to show that 
$$\mathbb{I}_{g_1,g_2}O_p\left( n^{-2\beta/(2\beta+1)} + k(n)\log n /n \right) =\smallO_p(n^{-1/2}).$$ By \ref{cond::outcomerateSecond}, the outcome regression rate exponent satisfies $\beta > 1/2$ when $\mathbb{I}_{g_1,g_2} = 1$ and, thus,  $n^{-2\beta/(2\beta+1)} = o_P(n^{-1/2})$. Moreover,  since the sieve growth rate satisfies $k(n) = o(\sqrt{n}/\log n)$, we have $k(n)\log n /n  =\smallO_p(n^{-1/2})$. Finally, to show that (II) is $o_p(n^{-1/2})$, it remains to show that, under \ref{cond::sieveApproxthrm}, $\rho_{n,\infty}\sqrt{k(n) \log n /n} =\smallO_p(n^{-1/2})$ and $\rho_{n,\infty} n^{-2\beta/(2\beta+1)}  =\smallO_p(n^{-1/2}) $. The former statement holds since $\rho > 1/2$ and, thus, $\rho_{n,\infty}\sqrt{k(n) \log n} = \{\log k(n)\}^{\nu} k(n)^{-\rho+1/2} (\log n) =\smallO_p(1) $ by \ref{cond::sieveApproxthrm}. The latter statement holds since $k(n)$ satisfies:
\begin{align*}
  & \rho_{n,\infty} n^{-\beta/(2\beta+1)}  = \{\log k(n)\}^{\nu} k(n)^{-\rho} (1/n)^{\beta/(2\beta+1)} = o(n^{-1/2});\\
      \iff & \{\log k(n)\}^{2\nu}  k(n)^{-2\rho}  = o(n^{-1}n^{2\beta/(2\beta+1)} ) = o(n^{-1/(2\beta + 1)}) ;\\
       \iff & k(n)  = \{\log k(n)\}^{\nu/\rho}\omega(n^{(1/2\rho)/(2\beta+1)}  ),
\end{align*}  
where the final expression is true by assumption. We conclude that $\|\overline{P}_n {\Delta}_{\pi_{n,\diamond}, \mu_{n,\diamond}^*}\|_{\ell^{\infty}(\mathcal{F})} =\smallO_p(n^{-1/2})$, which establishes the first statement of the theorem.

\textbf{Establishing oracle-efficiency.} Next, we establish asymptotic equivalence with the oracle efficient one-step estimator $(R_{n,0}(\theta): \theta \in \mathcal{F})$ and, consequently, weak convergence of the EP-learner risk estimator. We begin by demonstrating weak convergence of the oracle one-step risk estimator $(R_{n,0}(\theta): \theta \in \mathcal{F})$. First, the class $\mathcal{F}$ is Donsker since, by \ref{cond::regularityOnActionSpace}, it has finite uniform entropy integral (Theorem 2.8.3 of \citealt{vanderVaartWellner}). Moreover, by \ref{cond::positivity} and Theorem \ref{theorem::generalEIF}, the risk functional $P \mapsto R_P(\theta)$ is pathwise differentiable with efficient influence function $D_0(\cdot\,; \theta)$ for each $\theta \in \mathcal{F}$. Note that the map $\theta \mapsto R_{n,0}(\theta) = R_0(\theta) + P_n D_0(\cdot\,; \theta)$ is pointwise asymptotically linear with $\theta$-specific influence function being the efficient influence function $D_0(\cdot\,; \theta)$. By \ref{cond::positivity} and Lipschitz-continuity of $(h_1, h_2)$ in \eqref{eqn::popriskRR}, the class $(D_0(\cdot\,; \theta): \theta \in \mathcal{F})$ is a Lipschitz transformation of the Donsker class $\mathcal{M}$, and is, thus, also Donsker by preservation of the Donsker property (Theorem 2.10.6 of \citealt{vanderVaartWellner}). Hence, by the functional central limit theorem for Donsker classes (Theorem 2.8.2 of \citealt{vanderVaartWellner}), the process $(\sqrt{n}\{R_{n,0}(\theta) - R_0(\theta)\}: \theta \in \mathcal{F})$ converges weakly in $\ell^{\infty}(\mathcal{F})$ to a tight mean-zero Gaussian process with the claimed covariance structure. Next, we establish weak convergence of EP-learner risk estimator $(\sqrt{n}\{R_{n,k(n)}(\theta) - R_0(\theta)\}: \theta \in \mathcal{F})$ to the same limit in $\ell^{\infty}(\mathcal{F})$. By Slutsky's lemma for weak convergence in Banach spaces (Lemma 1.10.2 of \citealt{vanderVaartWellner}), it suffices to show that $\|R_{n,k(n)}  - R_{n,0}\|_{\ell^{\infty}(\mathcal{F})} = \sup_{\theta \in \mathcal{F}} \left| R_{n,k(n)}(\theta) - R_{n,0}(\theta)\right| =\smallO_p(n^{-1/2})$, so that $R_{n,k(n)}$ is asymptotically equivalent to the oracle one-step risk estimator.

By the triangle inequality, this weak convergence result follows so long as 
$$ \sup_{\theta \in \mathcal{F}} \left| R_{n,k(n)}(\theta) - R_{\pi_{n,\diamond}, \mu_{n,\diamond}^*}(\theta)\right| + \sup_{\theta \in \mathcal{F}} \left| R_{\pi_{n,\diamond} , \mu_{n,\diamond}^*}(\theta) - R_{n,0}(\theta)\right| =\smallO_p(n^{-1/2}).$$
Observe that the first term $\sup_{\theta \in \mathcal{F}} \left| R_{n,k(n)}(\theta) - R_{\pi_{n,\diamond}, \mu_{n,\diamond}^*}(\theta)\right|$ equals the uniform debiasing term $\sup_{\theta \in \mathcal{F}} \left|\overline{P}_n\Delta_{\pi_{n,\diamond}, \mu_{n,\diamond}^*}(\cdot\,\theta) \right|$. Hence, by the first part of this proof, this term is $o_p(n^{-1/2})$. We now turn to the second term. Note the oracle loss is unbiased in that $\overline{P}_0  L_{\pi_0, \mu_0}(\cdot\,,\theta) = \overline{P}_0  L_{ \mu_0}(\theta)$. Therefore, we have the expansion:
\begin{align*}
    &\sup_{\theta \in \mathcal{F}} \left|R_{\pi_{n,\diamond}, \mu_{n,\diamond}^*}(\theta) - R_{n,0}(\theta) \right|\\
    &\quad=   \sup_{\theta \in \mathcal{F} }\left|\overline{P}_n\left\{ L_{\pi_{n,\diamond}, \mu_{n,\diamond}^*}(\theta)-  L_{\pi_0, \mu_0}(\cdot\,,\theta) \right\} \right| \\
    &\quad =   \sup_{\theta \in \mathcal{F}} \left|(\overline{P}_n - \overline{P}_0) \left\{L_{\pi_{n,\diamond}, \mu_{n,\diamond}^*}(\theta)-  L_{\pi_0, \mu_0}(\cdot\,,\theta)\right\}+  \overline{P}_0 \left\{L_{\pi_{n,\diamond}, \mu_{n,\diamond}^*}(\theta)-  L_{\pi_0, \mu_0}(\cdot\,,\theta) \right\} \right|\\
     &\quad \leq \text{(I)} + \text{(II)},
\end{align*}  
where we define the terms:
\begin{align*}
  \text{(I)} &:= \sup_{\theta \in \mathcal{F}} \left|(\overline{P}_n - \overline{P}_0) \left\{L_{\pi_{n,\diamond}, \mu_{n,\diamond}^*}(\cdot\,;\theta)-  L_{\pi_0, \mu_0}(\cdot\,,\theta)\right\}\right|;\\
    \text{(II)}  &:= \sup_{\theta \in \mathcal{F}} \left|\overline{P}_0 \left\{L_{\pi_{n,\diamond}, \mu_{n,\diamond}^*}(\theta)-  L_{\mu_0}(\cdot\,,\theta) \right\} \right|.
\end{align*}  
By Lemma \ref{lemma::empProcDeltastar4}, boundedness of $\mathcal{F}$, and Markov's inequality, we have term (I) satisfies
\begin{align*}
  \text{(I)} &= \mathcal{O}_p\left(   n^{-1/2} s_n^* \mathcal{J}_{\infty} \left( 1 / s_n^*  , \mathcal{F}\right)\right) + \mathcal{O}_p \left(r_n^*\sqrt{k(n) \log n / n} \right).
\end{align*}
By Condition \ref{cond::A2Nuisance}, \ref{cond::outcomerateFirst}, and \ref{cond::outcomerateSecond} and the upper bound on the growth rate $k(n)$, we have $s_n^*$ and $r_n^*$ are $o(1)$. Hence, by \ref{cond::regularityOnActionSpace}, the first term on the right-hand side of the above display is $o_p(n^{-1/2})$. Observe that the second term on the right-hand side is $o_p(n^{-1/2})$ so long as both $k(n)\log n / n =\smallO_p(n^{-1/2})$ and $n^{-\beta/(2\beta+1)}\sqrt{k(n) \log n / n} =\smallO_p(n^{-1/2})$. The first rate holds since $k(n) = o(\sqrt{n}/\log n)$. The second rate holds since, by \ref{cond::A2Nuisance},
$$  n^{-\beta/(2\beta+1)}\sqrt{k(n) \log n / n} = n^{-(1/2)/(2\beta+1)}\sqrt{k(n) \log n} = o(n^{-1/2}),$$
where, for the final inequality, we use that $k(n) = o(n^{2\beta/(2\beta+1)}/ \log n$ by assumption. Putting it all together, we conclude that term (I) is $o_p(n^{-1/2})$. To bound term (II), note by Lemma \ref{lemma::P0boundDR}, \ref{cond::outcomerateDebiased}, and Event \ref{event::debias_mu_n_star} that
    \begin{align*}
       |\text{(II)}| \leq   \sup_{\theta \in \mathcal{F}  }  \left|\overline{P}_0 \left\{L_{\pi_{n,\diamond}, \mu_{n,\diamond}^*}(\theta)-  L_{\mu_0}(\cdot\,,\theta) \right\}\right| & \lesssim      r_n^* n^{-\gamma/(2\gamma+1)}+   \mathbb{I}_{g_1, g_2}  (r_n^*)^2.
 \end{align*}
By \ref{cond::A2Nuisance} we have $\beta > 1/(4\gamma)$ and, therefore, $n^{-\beta/(2\beta+1)} n^{-\gamma/(2\gamma+1)}=\smallO_p(n^{-1/2})$. By \ref{cond::A2Nuisance} and the upper bound on the sieve growth rate $k(n) = o(n^{2\gamma/(2\gamma+1)})$, we have $ \sqrt{k(n) \log n /n} \cdot  n^{-\gamma/(2\gamma+1)} =\smallO_p(n^{-1/2})$. Thus, the first term of the previous display is $o_p(n^{-1/2})$.  The second term on the right-hand side of the above display is $o_p(n^{-1/2})$ by \ref{cond::A2Nuisance} and the upper bound on the growth rate $k(n) = o(n^{1/2}/ \log n$. Combining all our bounds, we conclude that $\sup_{\theta \in \mathcal{F}} \left|R_{\pi_{n,\diamond}, \mu_{n,\diamond}^*}(\theta) - R_{n,0}(\theta)\right| =\smallO_p(n^{-1/2})$ from which the weak convergence result follows. This completes the proof.

 \end{proof}

\subsection{Additional technical lemmas for Theorems \ref{theorem::oracleRate} and \ref{theorem::EpLearnerRate}}
We recall the oracle empirical risk minimizer $\theta_{n,0} = \argmin_{\theta \in \mathcal{F}} R_{n,0}(\theta)$ of the oracle efficient one-step risk estimator $R_{n,0}$. To establish oracle-efficency of the ERM-based EP-learner with respect to the oracle learner $\theta_{n,0}$, we require the following additional technical lemmas. Together, these lemmas establish that the excess risk $R_{n,0}(\theta) - R_{n,0}(\theta_{n,0})$ at the oracle minimizer $\theta_{n,0}$ can be lower bounded by the quadratic term $\gamma \|\theta - \theta_{n,0}\|^2$, up to a negligible empirical process remainder. Importantly, this result holds even if $R_{n,0}$ is nonconvex. 

\begin{lemma}[Approximate strong convexity of oracle-efficient risk estimator]
Under the setup of Section \ref{section::riskfunctions}, there exists some constant $\gamma > 0$ such that for any $\theta \in \mathcal{F}$ we have
     \begin{align*}
     R_{n,0}(\theta ) - R_{n,0}(\theta_{n,0})   
     &\geq \gamma \norm{\theta - \theta_{n,0}}^2 + \frac{1}{2}\sum_{m\in\{1,2\}} (\overline{P}_n-\overline{P}_0) \left\{ w_{0,m} h''_m(\widetilde \theta_{n})(\theta - \theta_{n,0})^2 \right\},
 \end{align*} 
 where $\widetilde \theta_{n} \in  \mathcal{F}$ is some random function. 
\label{lemma::quadraticLowerBound}
\end{lemma}
\begin{proof}[Proof of Lemma \ref{lemma::quadraticLowerBound}]
 Since $\mathcal{F}$ is convex, we have,  for any $\theta \in \mathcal{F}$, that $\theta_{n,0} + (\theta-\theta_{n,0})\varepsilon \in \mathcal{F}$ for all $\varepsilon \in [0,1]$. Moreover, 
 $$\frac{d}{d\varepsilon}  R_{n,0}(\theta_{n,0} + (\theta - \theta_{n,0}) \varepsilon) \big |_{\varepsilon = 0} = \lim_{\varepsilon \downarrow 0}  \frac{R_{n,0}(\theta_{n,0} + (\theta - \theta_{n,0}) \varepsilon) - R_{n,0}(\theta_{n,0}) }{\varepsilon}  \geq 0,$$ since $\theta_{n,0}$ minimizes $R_{n,0}$ over $\mathcal{F}$. Recall that the functions $h_1, h_2$ in the definition of $R_{n,0}$ are assumed twice continuously differentiable. By the mean value theorem and a second-order Taylor expansion of $\varepsilon \mapsto R_{n,0}(\theta_{n,0} + (\theta - \theta_{n,0}) \varepsilon)$ around $\varepsilon = 0$, there exists some random $\varepsilon_{n,0} \in [0,1]$ such that
\begin{align*}
    &R_{n,0}(\theta ) - R_{n,0}(\theta_{n,0}) \\
    &\quad = \frac{d}{d\varepsilon}  R_{n,0}(\theta_{n,0} + (\theta - \theta_{n,0}) \varepsilon) \big |_{\varepsilon = 0} +  \frac{1}{2}\frac{d^2}{d\varepsilon^2}  R_{n,0}(\theta_{n,0} + (\theta - \theta_{n,0}) \varepsilon) \big |_{\varepsilon = \varepsilon_{n,0}}  \\
    &\quad\geq  \frac{1}{2}\frac{d^2}{d\varepsilon^2}  R_{n,0}(\theta_{n,0} + (\theta - \theta_{n,0}) \varepsilon) \big |_{\varepsilon = \varepsilon_{n,0}}   .
\end{align*}
We can write the loss $L_{\pi_0,\mu_0}(\theta, \,\cdot\,)$ corresponding to the risk $R_{n,0}(\theta) = \overline{P}_n L_{\pi_0,\mu_0}(\theta, \,\cdot\,)$ as
 $$L_{\pi_0, \mu_0}(\theta,o) =  \sum_{m \in \{1,2\}} w_{0,m}(y,a,w) (h_m \circ \theta)(w),$$
where, for $m \in \{1,2\}$, we define the weight function,
$$o \mapsto w_{0,m}(y,a,w) := \left\{\sum_{s \in \mathcal{A}} c_{s,m} (g_m \circ \mu_0)(s,w)\right\} + \frac{1}{\pi_0(a,w)}H_{m,\mu_0}(a,w)\left\{y - \mu_0(a,w) \right\},$$ 
which is unbiased in that $E[w_{0,m}(Y,A,W) \miid W= w] = \sum_{s \in \mathcal{A}} c_{s,m} (g_m \circ \mu)(s,w)  $. Now, denote $\theta_{\varepsilon}:= \theta_{n,0} + \varepsilon(\theta - \theta_{n,0})$ and observe that 
 $$ \frac{d^2}{d\varepsilon^2}  R_{n,0}(\theta_{n,0} + (\theta - \theta_{n,0}) \varepsilon) \big |_{\varepsilon = \varepsilon_{n,0}}  =   \sum_{m\in\{1,2\}} \overline{P}_n \left\{ w_{0,m} h''_m(\theta_{\varepsilon_{n,0}})(\theta - \theta_{n,0})^2 \right\},$$
 where $\theta_{\varepsilon_{n,0}} \in \mathcal{F}$.
We also have that
\begin{align*}
    \sum_{m\in\{1,2\}} \overline{P}_n \left\{ w_{0,m} h''_m(\theta_{\varepsilon_{n,0}})(\theta - \theta_{n,0})^2 \right\} &= \sum_{m\in\{1,2\}} \overline{P}_0 \left\{ w_{0,m} h''_m(\theta_{\varepsilon_{n,0}})(\theta - \theta_{n,0})^2 \right\} 
    \\
    & \quad  + \sum_{m\in\{1,2\}} (\overline{P}_n-\overline{P}_0) \left\{ w_{0,m} h''_m(\theta_{\varepsilon_{n,0}})(\theta - \theta_{n,0})^2 \right\}.
\end{align*}
We will show that the first term on the right-hand side can be lowered bounded by $\gamma\norm{\theta - \theta_{n,0}}^2$ for some $\gamma >0$. To this end, using that $E_0[w_{0,m}(Y,A,W) \miid W= w] = \sum_{s \in \mathcal{A}} c_{s,m} (g_m \circ \mu)(s,w)  $ for $m \in \{1,2\}$, we find
\begin{align*}
    &\sum_{m\in\{1,2\}} \overline{P}_0 \left\{ w_{0,m} h''_m(\theta_{\varepsilon_{n,0}})(\theta - \theta_{n,0})^2 \right\} \\
    &\quad= \sum_{m\in\{1,2\}} E_0 \left\{h''_m(\theta_{\varepsilon_{n,0}}(W))(\theta(W) - \theta_{n,0}(W))^2 \sum_{s \in \mathcal{A}} c_{s,m} (g_m \circ \mu)(s,W)  \right\} \\
    &\quad = E_0 \left\{    (\theta(W) - \theta_{n,0}(W))^2   \ddot{\ell}_{\theta}(W) \big|_{\theta = \theta_{\varepsilon_{n,0}}}  \right\},
\end{align*}   
where $w \mapsto \ddot{\ell}_{\theta}(w) := \frac{d^2L_{\mu_0}(\theta,w)}{d\theta(w)^2}$. We assumed in Section \ref{section::riskfunctions} that the loss $L_{\mu_0}$ corresponding with the risk $R_0(\theta)$ is $\gamma$-strongly convex (see Appendix \ref{appendix::convexloss}). Hence, we have that there exists some $\gamma > 0$ such that
$ \ddot{\ell}_{\theta}(W) \geq \gamma $ almost surely and, therefore, the previous display implies that
 \begin{align*}
  \sum_{m\in\{1,2\}} \overline{P}_0 \left\{ w_{0,m} h''_m(\theta_{\varepsilon_{n,0}})(\theta - \theta_{n,0})^2 \right\}  \geq \gamma \norm{\theta - \theta_{n,0}}^2.
\end{align*}   
Putting it all together, we find that
 \begin{align*}
     R_{n,0}(\theta ) - R_{n,0}(\theta_{n,0})  &\geq \frac{1}{2}\sum_{m\in\{1,2\}} \overline{P}_n \left\{ w_{0,m} h''_m(\theta_{\varepsilon_{n,0}})(\theta - \theta_{n,0})^2 \right\} \\
     &\geq \frac{\gamma}{2}\norm{\theta - \theta_{n,0}}^2 + \frac{1}{2}\sum_{m\in\{1,2\}} (\overline{P}_n-\overline{P}_0) \left\{ w_{0,m} h''_m(\theta_{\varepsilon_{n,0}})(\theta - \theta_{n,0})^2 \right\},
 \end{align*} 
 where, by convexity of $\mathcal{F}$, we have $\widetilde{\theta}_n:=\theta_{\varepsilon_{n,0}} \in \mathcal{F}$ almost surely.

\end{proof}

\begin{lemma}
    For any random function $\widetilde{\theta}_n \in \mathcal{F}$, $\delta > 0$, and $m \in \{1,2\}$, we have 
    \begin{align*}
       &E_0^n \left\{\sup_{\theta_3, \theta_2, \theta_1 \in \mathcal{F}: \norm{\theta_2 - \theta_1} \leq \delta} \left|(\overline{P}_n - \overline{P}_0)\left\{ w_{0,m} h''_m(\theta_3)(\theta_2 - \theta_1)^2 \right\} \right|\right\} \\
       &\quad\lesssim n^{-1/2} \mathcal{J}_{\infty}(\delta^{[2 - 1/(2\alpha)]} + n^{-1/2}, \mathcal{F}).
     \end{align*}  
  \label{lemma::quadraticLowerBoundEmpiricalProcess}   
\end{lemma}

\begin{proof}[Proof of Lemma   \ref{lemma::quadraticLowerBoundEmpiricalProcess}   ]
     Consider the function class $\mathcal{G} := \left\{ w_{0,m} h''_m(\theta_3)(\theta_2 - \theta_1)^2 : \theta_1, \theta_2, \theta_3 \in \mathcal{F}\right\}$. Since $h''_m$ is continuous and $\mathcal{F}$ is uniformly bounded, we have that $h''_m(\theta_3)$ is uniformly bounded over all $\theta_3 \in \mathcal{F}$. Moreover, $w_{0,m}$ is uniformly bounded by Conditions \ref{cond::positivity} and \ref{cond::boundedEStimators}. It follows that $\mathcal{G}$ is a uniformly bounded function class, and, for each $\theta_1, \theta_2, \theta_3 \in \mathcal{F}$, that $\norm{ w_{0,m} h''_m(\theta_3)(\theta_2 - \theta_1)^2} \leq M\norm{(\theta_2 - \theta_1)^2} $ for some $M>0$. Note $(\theta_1, \theta_2, \theta_3) \mapsto w_{0,m} h''_m(\theta_3)(\theta_2 - \theta_1)^2$ is a Lipschitz-continuous map, since we use that $h''_m$ is Lipschitz continuous (see Section  \ref{section::riskfunctions}). Thus, Lemma \ref{lemma::lipschitzPreservation} implies that $\mathcal{J}_{\infty}(\delta, \mathcal{G}) \lesssim \mathcal{J}_{\infty}(\delta, \mathcal{F})$. By \ref{cond::theorem2::couplings}, we also have that 
     \begin{align*}
         \norm{(\theta_2 - \theta_1)^2} &\leq \norm{\theta_2 - \theta_1}_{\infty}\norm{\theta_2 - \theta_1}\\
         & \leq \norm{\theta_2 - \theta_1}^{[2 - 1/(2\alpha)]}.
     \end{align*} 
     
     Additionally, an immediate application of Lemma \ref{lemma::maximalineq::supremumEntropy::oneclass} gives
     \begin{align*}
       &E_0^n \left\{\sup_{\theta_3, \theta_2, \theta_1 \in \mathcal{F}: \norm{\theta_2 - \theta_1} \leq \delta} \left|(\overline{P}_n - \overline{P}_0)\left\{ w_{0,m} h''_m(\theta_3)(\theta_2 - \theta_1)^2 \right\} \right|\right\} \\
       &\quad\leq  E_0^n \sup_{ g\in \mathcal{G}: \|g\| \leq M\delta^{[2 - 1/(2\alpha)]}} \left|(\overline{P}_n - \overline{P}_0) g \right|\\
       &\quad\lesssim n^{-1/2} \mathcal{J}_{\infty}(\delta^{[2 - 1/(2\alpha)]} + n^{-1/2}, \mathcal{F}).
     \end{align*}   
     
\end{proof}

\subsection{Proofs of Theorems in Section \ref{Section::ERMEPLearner}}
\label{appendix::proofsoracle}
\begin{proof}[Proof of Theorem \ref{theorem::oracleRate}]
Since $\mathcal{F}$ is convex, we have,  for any $\theta \in \mathcal{F}$, that $\theta_{0} + (\theta-\theta_{0})\varepsilon \in \mathcal{F}$ for all $\varepsilon \in [0,1]$. Moreover, since $\theta_{0}$ minimizes $R_{0}$ over $\mathcal{F}$, it holds that
 $$\frac{d}{d\varepsilon}  R_{0}(\theta_{0} + (\theta - \theta_{0}) \varepsilon) \big |_{\varepsilon = 0} = \lim_{\varepsilon \downarrow 0}  \frac{R_{0}(\theta_{0} + (\theta - \theta_{0}) \varepsilon) - R_{0}(\theta_{0}) }{\varepsilon}  \geq 0.$$ 
Using the definition of strong convexity of the risk $R_{0}$ at $\theta_0$ provided in Section \ref{appendix::convexloss}, we have that 
$$ R_0(\theta_{n,0}) - R_0(\theta_0) \geq \frac{d}{d\varepsilon}  R_{0}(\theta_{0} + (\theta_{n,0} - \theta_{0}) \varepsilon) \big |_{\varepsilon = 0} +  \frac{\gamma}{2}\norm{\theta_{n,0} - \theta_0}^2. $$
Thus, we have
$$ \norm{\theta_{n,0} - \theta_0}^2 \lesssim R_0(\theta_{n,0}) - R_0(\theta_0) .$$
Next, using that $\theta_{n,0}$ is the empirical risk minimizer so that $R_{n,0}(\theta_{n,0}) - R_{n,0}(\theta_0) \leq 0$, we obtain the inequality
\begin{align*}
  \norm{\theta_{n,0} - \theta_0}^2  &\lesssim   R_0(\theta_{n,0}) - R_0(\theta_0) \\
  & \leq R_0(\theta_{n,0}) - R_{n,0}(\theta_{n,0})  - \left\{ R_0(\theta_0) -  R_{n,0}(\theta_{n,0})\right\}\\
  & \lesssim (P_n - P_0)\left\{L_{\pi_0, \mu_0}(\theta_{n,0}) -  L_{\pi_0, \mu_0}(\theta_{0}) \right\}.
\end{align*}
 Lemma \ref{lemma::lipschitzPreservation} combined with Lemma \ref{lemma::maximalineq::supremumEntropy::oneclass} implies for any $\delta > n^{-1/2}$ that
$$\phi_n(\delta) := E_0 \sup_{ \theta \in \mathcal{F}: \norm{\theta - \theta_0} \leq \delta:} \left| (P_n - P_0)\left\{L_{\pi_0, \mu_0}(\cdot\,,\theta) -  L_{\pi_0, \mu_0}(\theta_{0})  \right\}\right| \lesssim n^{-1/2}\mathcal{J}_{\infty}(\delta , \mathcal{F}). $$
As in the proof of Theorem \ref{theorem::EpLearnerRate} below, to obtain the rate result, we apply Theorem 3.2.5 of \cite{vanderVaartWellner} where, for $\theta \in \mathcal{F}$, we take $\Theta := \mathcal{F}$, $d(\theta, \theta_0) := \|\theta - \theta_0\|$, $\mathbb{M}(\theta) := - R_0(\theta)$, and $\mathbb{M}_n(\theta) := -R_{n,0}(\theta)$. We then find that $ \norm{\theta_{n,0} - \theta_0} = \mathcal{O}_p(\delta_{n,0})$ for any $\delta_{n,0} > n^{-1/2}$ satisfying $\phi_n(\delta_{n,0}) \leq \delta_{n,0}^2$. In view of the previous display, $\delta_{n,0} = \min \left\{\delta > n^{-1/2}: \mathcal{J}_{\infty}(\delta, \mathcal{F}) \leq \sqrt{n} \delta^2 \right\}$ is one such choice. The result $\delta_{n,0}  = O(n^{-\alpha/(2\alpha+1)})$ then follows from \ref{cond::regularityOnActionSpace}.
\end{proof}

\begin{proof}[Proof of Theorem \ref{theorem::EpLearnerRate}]

For ease of notation, we denote the EP-learner empirical risk minimizer by $\theta_n := \theta_{n,k(n)}^*$. We also denote the quantity we wish to obtain a rate for by $\delta_n := \|\theta_n - \theta_{n,0}\|$. Since $\theta_n$ minimizes the EP-learner risk $R_{n,k(n)}(\theta)$ over $\theta \in \mathcal{F}$, we have $R_{n,k(n)}(\theta_n) - R_{n,k(n)}(\theta_{n,0}) \leq 0$. Using this, we obtain the following upper bound for the excess risk:
    \begin{align*}
     R_{n,0}(\theta_n) - R_{n,0}(\theta_{n,0}) &=   \left\{R_{n,0}(\theta_n) - R_{n,0}(\theta_{n,0}) \right\} - \left\{R_{n,k(n)}(\theta_n) - R_{n,k(n)}(\theta_{n,0}) \right\} \\
     &\quad+ \left\{R_{n,k(n)}(\theta_n) - R_{n,k(n)}(\theta_{n,0}) \right\} \\
     & \leq R_{n,0}(\theta_n) - R_{n,k(n)}(\theta_n) - R_{n,0}(\theta_{n,0}) + R_{n,k(n)}(\theta_{n,0}) \\
     & \leq \phi_{n,1}(\delta_n),
    \end{align*} 
    where, for $\delta >0$, we define
    \begin{align*}
      \phi_{n,1}(\delta) :=  \sup_{\theta_1, \theta_2 \in \mathcal{F}: \norm{\theta_1 - \theta_2} \leq \delta} \left|  R_{n,0}(\theta_2) - R_{n,k(n)}(\theta_2) - R_{n,0}(\theta_{1}) + R_{n,k(n)}(\theta_{1})\right|
    \end{align*}
In addition, by Lemma \ref{lemma::quadraticLowerBound}, we have the following lower bound for the excess risk:
\begin{align*}
    \delta_n^2 &\lesssim R_{n,0}(\theta_n) - R_{n,0}(\theta_{n,0}) +  \phi_{n,2}(\delta_n),
\end{align*}
 where, for $\delta >0$, we define
\begin{align*}
     \phi_{n,2}(\delta):=  \sum_{m \in \{1,2\}} \sup_{\theta_1, \theta_2, \theta_3 \in \mathcal{F}: \norm{\theta_2 - \theta_1} \leq \delta} \left|(\overline{P}_n - \overline{P}_0)\left\{ w_{0,m} h''_m(\theta_3)(\theta_2 - \theta_1)^2 \right\} \right|.
\end{align*}
Finally, combining the lower and upper bounds for the excess risk and letting $\phi_n(\delta) := \phi_{n,1}(\delta) + \phi_{n,2}(\delta)$, we obtain the following inequality for $\delta_n$:
\begin{align*}
 \delta_n^2 \lesssim   \phi_n(\delta_n).
\end{align*}

We now use the above inequality to obtain a rate of convergence for $\delta_n$. Recall that $A_n$ is the event on which \ref{event::mu_n}-\ref{event::boundedCoef} hold, where the constant $M > 0$ is chosen so that $A_n$ occurs, for all $n$ large enough, with probability at least $1 - \varepsilon$, where $\varepsilon > 0$ is arbitrary. We first bound $E_0^n\{\mathbb{I}_{A_n}\phi_n(\delta)\} = E_0^n\{\mathbb{I}_{A_n}\phi_{n,1}(\delta)\} + E_0^n\{\mathbb{I}_{A_n}\phi_{n,2}(\delta)\}$ for a fixed deterministic $\delta >0$. It follows directly from Lemma \ref{lemma::quadraticLowerBoundEmpiricalProcess} that
$$ E_0^n\{\mathbb{I}_{A_n}\phi_{n,2}(\delta)\} \lesssim n^{-1/2} \mathcal{J}_{\infty}(\delta^{[2 - 1/(2\alpha)]} + n^{-1/2}, \mathcal{F}).$$ To bound $E_0^n\{\mathbb{I}_{A_n}\phi_{n,1}(\delta)\}$, let $\alpha > 0$, $\rho > 1/2$, and $\nu \geq 0$ satisfy \ref{cond::sieveApproxthrm} and \ref{cond::theorem2::couplings} and let $\rho_{n,\infty} := \{\log k(n)\}^{\nu} k(n)^{-\rho}$. By Lemma \ref{lemma::sieveRateSup},  we have $\sup_{\phi \in \mathcal{H}} \norm{\phi - \Pi_{k(n)}\phi}_{\infty} = O(\rho_{n,\infty})$. Lemma \ref{lemma::modulusOfCont} given at the end of this proof establishes, for any $\delta \geq n^{-1/2}$, that $ E_0^n \{\mathbb{I}_{A_n}\phi_{n,1}(\delta)\}$ satisfies the following bound:
\begin{align*}
    E_0^n \{\mathbb{I}_{A_n} \phi_{n,1}(\delta)\} &\lesssim  \delta^{1-1/(2\alpha)}   \left\{ r_n^*  n^{-\gamma/(2\gamma+1)} + r_n^*  \sqrt{k(n) \log n / n} + o(n^{-1/2})  \right\}\\
    &  + \quad r_n^*  \min\{\delta, \{\log k(n)\}^{\nu}k(n)^{-\rho}\} + n^{-1/2} \left[ \{\log k(n)\}^{2\nu}\{k(n)^{-\rho}\}\right]^{1-1/(2\alpha)} .
\end{align*}

For the moment, we assume the following bound holds and proceed with the rate of convergence proof. We begin by further bounding the expectation of $E_0^n \{\mathbb{I}_{A_n}\phi_{n,1}(\delta)\}$. First, by \ref{cond::regularityOnActionSpace}, note 
$$ n^{-1/2} \left[ \{\log k(n)\}^{2\nu}\{k(n)^{-\rho}\}\right]^{1-1/(2\alpha)} \lesssim   \{\log k(n)\}^{2\nu} n^{-1/2} k(n)^{-\rho[1- 1/(2\alpha)]}.$$  Next, observe that $k(n) \log n / n  = o(n^{-1/2})$ since $k(n) = o(\sqrt{n} / \log n)$. Moreover, since $k(n) = o(n^{2\beta/(2\beta+1)}/ \log n)$, we have $n^{-\beta/(2\beta+1)} \sqrt{k(n) \log n / n} = o(n^{-1/2})$ by \ref{cond::A2Nuisance}. Hence, it holds that 
$$r_n^* \sqrt{k(n) \log n /n} \leq n^{-\beta/(2\beta+1)}  \sqrt{k(n) \log n / n} + k(n) \log n /n = o(n^{-1/2}).$$
Similarly, since $k(n) = o(n^{2\gamma/(2\gamma+1)}/ \log n)$, we have $ n^{-\gamma/(2\gamma+1)}\sqrt{k(n) \log n / n} = o(n^{-1/2})$ by \ref{cond::A2Nuisance}. Again, by \ref{cond::A2Nuisance}, it holds that $\beta > 1/(4\gamma)$ and, thus, $n^{-\beta/(2\beta+1)} n^{-\gamma/(2\gamma+1)}  =   o(n^{-1/2}) $. Thus, it holds that $r_n^* n^{-\gamma/(2\gamma+1)} = o(n^{-1/2})$. Combining these bounds, we find
\begin{align*}
        E_0^n\{\mathbb{I}_{A_n} \phi_{n,1}(\delta)\} &\lesssim    o(n^{-1/2}) \delta^{1-1/(2\alpha)}  +  r_n^*  \min\{\delta, \{\log k(n)\}^{\nu}k(n)^{-\rho}\}  \\
        &\quad+  \{\log k(n)\}^{2\nu} \cdot n^{-1/2} k(n)^{-\rho[1- 1/(2\alpha)]}.
\end{align*}
Using this bound for $ E_0^n\{\mathbb{I}_{A_n}\phi_{n,1}(\delta)\}$ and recalling that $ E_0^n\{\mathbb{I}_{A_n}\phi_{n,2}(\delta) \}\lesssim n^{-1/2} \mathcal{J}_{\infty}(\delta^{[2 - 1/(2\alpha)]} + n^{-1/2}, \mathcal{F})  $, it follows, for any $\delta > n^{-1/2}$, that
\begin{align}
    E_0^n\{ \mathbb{I}_{A_n}\phi_{n}(\delta) \}
    & \lesssim  o(n^{-1/2}) \delta^{1-1/(2\alpha)}  +  r_n^*  \min\{\delta, \{\log k(n)\}^{\nu}k(n)^{-\rho}\} \nonumber \\
    & \quad  +  \{\log k(n)\}^{2\nu} \cdot n^{-1/2} k(n)^{-\rho[1- 1/(2\alpha)]} + n^{-1/2} \mathcal{J}_{\infty}(\delta^{[2 - 1/(2\alpha)]} + n^{-1/2}, \mathcal{F}).  \label{eqn::upperboundLocalMod}
\end{align}
To obtain the rate result, we apply the proof of Theorem 3.2.5 of \cite{vanderVaartWellner} on the event $A_n$ where, for $\theta \in \mathcal{F}$, we take $\Theta := \mathcal{F}$, $d(\theta, \theta_0) := \|\theta - \theta_0\|$, $\mathbb{M}(\theta) := - R_0(\theta)$, and $\mathbb{M}_n(\theta) := -R_{n,k(n)}(\theta)$. Specifically, for $\widetilde{\delta}_n := \inf\{\delta \geq n^{-1/2}:  E_0^n\{\mathbb{I}_{A_n}\phi_{n}(\delta)\} \lesssim \delta^2 \}$, the proof of Theorem 3.2.5 applied on the event $A_n$ establishes that, for every $\varepsilon' > 0$, we can find a fixed constant $C > 0$ such that
$$P(A_n \cap \{\norm{\theta_{n} - \theta_{n,0}} \geq  C \widetilde{\delta}_n\} ) \leq \varepsilon'.$$
Hence, 
\begin{align*}
    P(\norm{\theta_{n} - \theta_{n,0}} \geq  C \widetilde{\delta}_n ) & \leq P(A_n \cap \{ \norm{\theta_{n} - \theta_{n,0}} \geq  C \widetilde{\delta}_n\} ) + P(A_n^c \cap \{\norm{\theta_{n} - \theta_{n,0}} \geq  C \widetilde{\delta}_n \})\\
    & \leq \varepsilon' +  P(A_n^c) \leq \varepsilon' + \varepsilon,
\end{align*} 
for all $n$ large enough.
Since $\varepsilon > 0$ is arbitrary, we then conclude that the EP-learner satisfies $\norm{\theta_{n} - \theta_{n,0}} = \mathcal{O}_p(\widetilde{\delta}_n)$.


Next, we determine an upper bound for $\widetilde{\delta}_n$. By the definition of $\widetilde{\delta}_n$, it holds that $\widetilde{\delta}_n \lesssim \delta $ for any $\delta > n^{-1/2}$ such that $E_0^n\{\mathbb{I}_{A_n}\phi_{n}(\delta)\} \lesssim \delta^2$. In view of \eqref{eqn::upperboundLocalMod} and the entropy bound of \ref{cond::regularityOnActionSpace}, the inequality $E_0^n\{\mathbb{I}_{A_n}\phi_{n}(\delta)\} \lesssim \delta^2$ is satisfied, up to a constant, by any $\delta > n^{-1/2}$ such that
\begin{align*}
    \delta^2 &\gtrapprox o(n^{-1/2}) \delta^{1-1/(2\alpha)}  +  r_n^*  \min\{\delta, \{\log k(n)\}^{\nu}k(n)^{-\rho}\}  +  \{\log k(n)\}^{\nu} \cdot n^{-1/2} k(n)^{-\rho[1- 1/(2\alpha)]} \\
& \quad + n^{-1/2} \left\{\delta^{\{2 - 1/(2\alpha)\}\{1 - 1/(2\alpha)\}} \right\}  . 
\end{align*}
We claim the above inequality is satisfied, up to a constant, by any $\delta > n^{-1/2}$ such that all of the following hold simultaneously: (1) $o(n^{-1/2}) \delta^{1-1/(2\alpha)} \lesssim \delta^2$; (2) $r_n^*  \min\{\delta, \{\log k(n)\}^{\nu}k(n)^{-\rho}\}  \lesssim \delta^2$; (3) $\{\log k(n)\}^{2\nu} \cdot n^{-1/2} k(n)^{-\rho[1- 1/(2\alpha)]}  \lesssim \delta^2$; and (4) $n^{-1/2}\left\{\delta^{\{2 - 1/(2\alpha)\}\{1 - 1/(2\alpha)\}} \right\} \lesssim \delta^2$. Note, by monotonicity of $\frac{\delta^{1-1/(2\alpha)} }{\delta^2}$ that $ \delta' > 0$ satisfying (1) implies that any $\delta > \delta'$ also satisfies (1). A similar argument establishes the same property for solutions of inequalities (2) and (3). For (4), the same property can be established upon noting that $\delta \mapsto \frac{\delta^{\{2 - 1/(2\alpha)\}\{1 - 1/(2\alpha)\}}}{\delta^2} $ is also monotone as $[2 - 1/(2\alpha)][1 - 1/(2\alpha)] < 2$. Hence, if we can find individual solutions $\delta_{n,1}, \delta_{n,2}, \delta_{n,3}, \delta_{n,4} > n^{-1/2}$ that respectively satisfy (1), (2), (3), and (4), then their maximum $\max\{\delta_{n,1}, \delta_{n,2}, \delta_{n,3}, \delta_{n,4}\}$ will simultaneously satisfy (1)-(4). 
Note the first and fourth bound are satisfied by $\delta_{n,1} = \delta_{n,4} = o(n^{-\alpha/(2\alpha+1)}) $, since $\alpha > 1/2$ implies $2 - 1/(2\alpha) > 1$. Next, noting $\{\log k(n)\}^{2\nu} \lesssim \{\log n\}^{2\nu}$, the third constraint is satisfied by $\delta_{n,3} > 0$ such that $\delta_{n,3}^2 = O(  \{\log n\}^{2\nu} n^{-1/2} k(n)^{-\rho[1- 1/(2\alpha)]})  $. Lastly, note $r_n^*  \min\{\delta_{n,2}, \{\log k(n)\}^\nu k(n)^{-\rho}\}  \lesssim \delta_{n,2}^2$ is satisfied by some $\delta_{n,2} = O(\min\{ r_n^* , \sqrt{r_n^* \{\log n\}^\nu k(n)^{-\rho}\}}\})$. It follows that $\max\{\delta_{n,1}, \delta_{n,2}, \delta_{n,3}, \delta_{n,4}\}$ and, hence, also $\delta_{n,1}+ \delta_{n,2}+ \delta_{n,3} + \delta_{n,4}$ satisfies (1)-(4), up to a constant. Taking our upper bound for $\widetilde{\delta}_n$ as $\delta_{n,1}+ \delta_{n,2}+ \delta_{n,3} + \delta_{n,4}$, we conclude
\begin{align}
    \widetilde{\delta}_n^2 &\lesssim o(n^{-2\alpha/(2\alpha+1)} ) +  O\left( \{\log n\}^{2\nu} n^{-1/2} k(n)^{-\rho[1- 1/(2\alpha)]}  +  r_n^*\min\{ r_n^* ,  \{\log k(n)\}^\nu k(n)^{-\rho}\} \right).
    \label{eqn::ratewithmin}
\end{align}
Recalling that $r_n^* = n^{-\beta/(2\beta+1)} + \sqrt{k(n) \log n /n}$ and taking the second argument in the minimum operation in \eqref{eqn::ratewithmin}, we obtain the bound:
\begin{align*}
   \widetilde{\delta}_n^2 &\lesssim o(n^{-2\alpha/(2\alpha+1)} ) + O\left(  \{\log n\}^{2\nu} n^{-1/2} k(n)^{-\rho[1- 1/(2\alpha)]}  +  r_n^*  \{\log k(n)\}^\nu k(n)^{-\rho }\right)\\
    & \lesssim  o(n^{-2\alpha/(2\alpha+1)}) + O\left( \{\log n\}^{2\nu}  \left\{n^{-\beta/(2\beta+1)} + \sqrt{k(n) \log n /n} + n^{-1/2} k(n)^{\rho/(2\alpha))}\right\} k(n)^{-\rho} \right)\\
  & \lesssim o(n^{-2\alpha/(2\alpha+1)}) +  O\left(  \{\log n\}^{2\nu} \left\{n^{-\beta/(2\beta+1)} + \sqrt{k(n) \log n /n} + \sqrt{k(n)^{(\rho/\alpha)} / n}\right\} k(n)^{-\rho}\right).
\end{align*}
Taking the first argument in the minimum operation in  \eqref{eqn::ratewithmin}, we also have the following bound:
\begin{align*}
    \widetilde{\delta}_n^2 &\lesssim o(n^{-2\alpha/(2\alpha+1)} ) +  O\left(n^{-2\beta/(2\beta+1)} +  k(n) \log n /n  + \{\log n\}^{2\nu} n^{-1/2} k(n)^{-\rho[1- 1/(2\alpha)]}   \right)  .
\end{align*}
Taking the square root of both sides of the above bounds gives the desired rate of the Theorem. To complete the proof, it remains to prove Lemma \ref{lemma::modulusOfCont}, which we do next.
 \end{proof}




Recall from the proof of Theorem \ref{theorem::EpLearnerRate}, for $\delta >0$, that
$$ \phi_{n,1}(\delta) :=  \sup_{\theta_1, \theta_2 \in \mathcal{F}: \norm{\theta_1 - \theta_2} \leq \delta} \left|  R_{n,0}(\theta_2) - R_{n,k(n)}(\theta_2) - R_{n,0}(\theta_{1}) + R_{n,k(n)}(\theta_{1})\right|.$$
Further recall that $r_n^* := n^{-\beta/(2\beta+1)} + \sqrt{k(n) \log n / n}$, $s_n^* := r_n^* + n^{-\gamma/(2\gamma+1)}$ and, moreover, $\rho_{n,\infty} := \{\log k(n)\}^{\nu} k(n)^{-\rho}$.
The follow lemma concerns the expected value of this quantity.
\begin{lemma}[Proof of bound for $  E_0^n\{ \mathbb{I}_{A_n}\phi_{n,1}(\delta)\}$]
\label{lemma::modulusOfCont}
For any $\delta>n^{-1/2}$, it holds that
\begin{align*}
    E_0^n\{ \mathbb{I}_{A_n} \phi_{n,1}(\delta) \}&\lesssim  \delta^{1-1/(2\alpha)}   \left\{ r_n^*  n^{-\gamma/(2\gamma+1)}  + r_n^*  \sqrt{k(n) \log n / n} + o(n^{-1/2})  \right\}\\
    & \quad + r_n^*  \min\{\delta, \{\log k(n)\}^{\nu} k(n)^{-\rho}\} +  n^{-1/2} \left[ \{\log k(n)\}^{2\nu}\{k(n)^{-\rho}\}\right]^{1-1/(2\alpha)}  .
\end{align*}
\end{lemma}
\begin{proof}
As shorthand, we denote $\left[F - \widetilde{F} \right]  \left[ \theta_1 - \theta_2\right] := \left[F(\theta_1) - \widetilde{F}(\theta_1)\right] - \left[F(\theta_2) - \widetilde{F}(\theta_2)\right]$ for any two functionals $F, \widetilde{F} \in \ell^{\infty}(\mathcal{F})$.  

To bound $ E_0^n\{ \mathbb{I}_{A_n}\phi_{n,1}(\delta)\}$, we proceed along lines similar to the proof of Theorem  \ref{theorem::EPriskEff}. Adding and subtracting, we have, by the triangle inequality, that
\begin{align}
    E_0^n\{ \mathbb{I}_{A_n}\phi_{n,1}(\delta) \} &\leq E_0^n \left[ \mathbb{I}_{A_n}\sup_{\theta_1 , \theta_2 \in \mathcal{F}: \norm{\theta_1 - \theta_2} \leq \delta } \left| \left[R_{n,0} - R_{\pi_{n,\diamond}, \mu_{n,\diamond}^*}\right]\left[\theta_1 - \theta_2\right] \right| \right] \nonumber \\
    & \quad + E_0^n \left[ \mathbb{I}_{A_n} \sup_{\theta_1 , \theta_2 \in \mathcal{F}: \norm{\theta_1 - \theta_2} \leq \delta } \left| \left[R_{\pi_{n,\diamond}, \mu_{n,\diamond}^*} - R_{n, k(n)}\right]\left[\theta_1 - \theta_2\right] \right| \right] .
  \label{eqn::theorem2::AB}
\end{align}
Turning to the second term in the above bound, note that
\begin{align*}
    & \hspace{-1cm} E_0^n \left[ \mathbb{I}_{A_n} \sup_{\theta_1 , \theta_2 \in \mathcal{F}: \norm{\theta_1 - \theta_2} \leq \delta } \left| \left[R_{\pi_{n,\diamond}, \mu_{n,\diamond}^*} - R_{n, k(n)}\right]\left[\theta_1 - \theta_2\right] \right| \right]  \\
    & = E_0^n\left[ \mathbb{I}_{A_n}\sup_{\theta_1 , \theta_2 \in \mathcal{F}: \norm{\theta_1 - \theta_2} \leq \delta } \left|\overline{P}_n \left\{{\Delta}_{\pi_{n,\diamond}, \mu_{n,\diamond}^*}(\cdot,\theta_1) -   {\Delta}_{\pi_{n,\diamond}, \mu_{n,\diamond}^*}(\cdot,\theta_2)  \right\} \right|\right]\\
    & \leq \sum_{m \in \{1,2\}} E_0^n\left[ \mathbb{I}_{A_n}\sup_{\theta_1 , \theta_2 \in \mathcal{F}: \norm{\theta_1 - \theta_2} \leq \delta } \left|\overline{P}_n  \overline{\Delta}_{\pi_{n,\diamond}, \mu_{n,\diamond}^*}^{(m)}(\cdot, h_m \circ \theta_1 - h_m \circ \theta_2)   \right|\right],
\end{align*}
where we recall, by definition, that the map $\phi \mapsto \overline{P}_n\overline{\Delta}^{(m)}_{\pi_{n,\diamond}, \mu_{n,\diamond}^*}(\cdot\,,\phi)$ is linear for $m \in \{1,2\}$.
Hence, we have
\begin{align*}
    E_0^n\{ \mathbb{I}_{A_n}\phi_{n,1}(\delta) \} &\leq E_0^n \left[ \mathbb{I}_{A_n}\sup_{\theta_1 , \theta_2 \in \mathcal{F}: \norm{\theta_1 - \theta_2} \leq \delta } \left| \left[R_{n,0} - R_{\pi_{n,\diamond}, \mu_{n,\diamond}^*}\right]\left[\theta_1 - \theta_2\right] \right| \right] \leq \text{(I)} + \text{(II)},
\end{align*}
where we denote:
\begin{align*}
   & \text{(I)} := E_0^n\left[ \mathbb{I}_{A_n}\sup_{\theta_1 , \theta_2 \in \mathcal{F}: \norm{\theta_1 - \theta_2} \leq \delta } \left|\left[R_{n,0} - R_{\pi_{n,\diamond}, \mu_{n,\diamond}^*}\right]\left[\theta_1 - \theta_2\right] \right|\right];\\
   & \text{(II)}:=  \sum_{m \in \{1,2\}} E_0^n\left[ \mathbb{I}_{A_n}\sup_{\theta_1 , \theta_2 \in \mathcal{F}: \norm{\theta_1 - \theta_2} \leq \delta } \left|\overline{P}_n  \overline{\Delta}_{\pi_{n,\diamond}, \mu_{n,\diamond}^*}^{(m)}(\cdot,h_m \circ \theta_1 - h_m \circ \theta_2) \right| \right] \\
   &\phantom{\text{(II)}:}=  \sum_{m \in \{1,2\}} E_0^n\left[ \mathbb{I}_{A_n}\sup_{\phi\in \mathcal{H}^{\star}_{\delta} } \left|\overline{P}_n  \overline{\Delta}_{\pi_{n,\diamond}, \mu_{n,\diamond}^*}^{(m)}(\cdot,\phi) \right| \right],
\end{align*}
where we let $\mathcal{H}^{\star}_{\delta} := \left\{h_m\circ \theta_1 - h_m\circ \theta_2 : \theta_1,\theta_2\in\mathcal{F},\|\theta_1-\theta_2\|\le\delta,m\in\{1,2\} \right\}$.

We bound each of the above terms in turn. To bound term (I), consider the decomposition:
\begin{align*}
  &\left[R_{n,0} - R_{\pi_{n,\diamond}, \mu_{n,\diamond}^*}\right]\left[\theta_1 - \theta_2\right]  \\
  &= \left[\overline{P}_n \left( L_{\pi_0, \mu_0} - L_{\pi_{n,\diamond}, \mu_{n,\diamond}^*}\right)\right]\left[\theta_1 - \theta_2 \right]\\
    & = \left[\overline{P}_0\left(L_{\pi_0, \mu_0} - L_{\pi_{n,\diamond}, \mu_{n,\diamond}^*}\right)\right]\left[\theta_1 - \theta_2\right]  + \left[\left(\overline{P}_n - \overline{P}_0 \right)\left(L_{\pi_0, \mu_0} - L_{\pi_{n,\diamond}, \mu_{n,\diamond}^*}\right) \right]\left[\theta_1 - \theta_2\right]. 
\end{align*} 
By \ref{cond::theorem2::couplings}, it holds that $ \sup_{\theta_1 , \theta_2 \in \mathcal{F}: \norm{\theta_1 - \theta_2} \leq \delta } \|\theta_1 - \theta_2\|_{\infty} \leq \delta^{1- 1/(2\alpha)}$. Thus, applying Lemma \ref{lemma::P0boundDR}, it holds that
\begin{align*}
     & \hspace{-1cm} \mathbb{I}_{A_n} \sup_{\theta_1 , \theta_2 \in \mathcal{F}: \norm{\theta_1 - \theta_2} \leq \delta } \left|\left[\overline{P}_0\left(L_{\pi_0, \mu_0} - L_{\pi_{n,\diamond}, \mu_{n,\diamond}^*}\right)\right]\left[\theta_1 - \theta_2\right] \right|  \\
 &       \hspace{2cm} \lesssim \sup_{\theta_1 , \theta_2 \in \mathcal{F}: \norm{\theta_1 - \theta_2} \leq \delta } \|\theta_1 - \theta_2\|_{\infty} \left\{ r_n^*  n^{-\gamma/(2\gamma+1)}  +  \mathbb{I}_{g_1, g_2}  (r_n^*)^2 + r_n^* \sqrt{k(n) \log n / n} \right\}\\
 &  \hspace{2cm} \lesssim \delta^{1- 1/(2\alpha)}\left\{ r_n^*  n^{-\gamma/(2\gamma+1)}  +  \mathbb{I}_{g_1, g_2}  (r_n^*)^2 + r_n^* \sqrt{k(n) \log n / n} \right\},
\end{align*}  
where the right-hand side is deterministic. Note that $ \mathbb{I}_{g_1, g_2} (r_n^*)^2 = o(n^{-1/2})$ by \ref{cond::A2Nuisance} and the constraint that $\mathbb{I}_{g_1, g_2} k(n) = \mathbb{I}_{g_1, g_2} o(n^{2\beta/(2\beta+1)}/ \log n)$.  Hence, the term $ \delta^{1-1/(2\alpha)} \mathbb{I}_{g_1, g_2}   (r_n^*)^2 =  \delta^{1-1/(2\alpha)} o(n^{-1/2})$. Next, the entropy integral bound of \ref{cond::regularityOnActionSpace} implies that $  s_n^* \mathcal{J}_{\infty} \left(   \delta  / s_n^*  , \mathcal{F}\right) \lesssim \delta^{1-1/(2\alpha)}    (s_n^*)^{1/(2\alpha)} $. Hence, by Lemma \ref{lemma::empProcDeltastar4}, it holds that
\begin{align*}
    & \hspace{-2cm} E_0^n \left[ \mathbb{I}_{A_n}  \sup_{\theta_1 , \theta_2 \in \mathcal{F}: \norm{\theta_1 - \theta_2} \leq \delta } \left|\left[\left(\overline{P}_n - \overline{P}_0 \right)\left(L_{\pi_0, \mu_0} - L_{\pi_{n,\diamond}, \mu_{n,\diamond}^*}\right) \right]\left[\theta_1 - \theta_2\right] \right| \right]\\
 &       \hspace{1cm}\lesssim  n^{-1/2} s_n^* \mathcal{J}_{\infty} \left(  \delta  / s_n^*  , \mathcal{F}\right) \\
  &       \hspace{1cm} \lesssim \delta^{1-1/(2\alpha)} \left\{n^{-1/2}   (s_n^*)^{1/(2\alpha)}\right\}.
\end{align*}  
 Consequently, applying both of the above bounds and applying the triangle inequality, we find
\begin{align*}
     \text{(I)}  
     & \lesssim  \delta^{1-1/(2\alpha)} \left[ r_n^* n^{-\gamma/(2\gamma+1)} + o(n^{-1/2}) + r_n^*  \sqrt{k(n) \log n / n} +  n^{-1/2}  (s_n^*)^{1/(2\alpha)} \right].
\end{align*}

Since $r_n^* =o(1)$, we have $ n^{-1/2}r_n^* \mathcal{J}_{\infty} \left( \delta   / r_n^*  , \mathcal{F}\right)\lesssim \delta^{1-1/(2\alpha)}o(n^{-1/2})$. Moreover, the fact that $  \sqrt{k(n) \log n / n}   \leq r_n^*$ implies that
$$  \mathbb{I}_{g_1, g_2}\delta^{1-1/(2\alpha)}   r_n^*\sqrt{k(n) \log n / n} \leq   \mathbb{I}_{g_1, g_2}\delta^{1-1/(2\alpha)}   (r_n^*)^2 = \delta^{1-1/(2\alpha)}   o(n^{-1/2}).$$

We now turn to term (II). Note, by Lipschitz continuity of the fixed functions $h_1$ and $h_2$ in the definition of $R_0$, there exists some Lipschitz constant $L > 0$ such that $ \norm{h_m \circ \theta_1 - h_m \circ \theta_2}_{\infty} \leq L \| \theta_1 - \theta_2\|_{\infty}$ and $ \norm{h_m \circ \theta_1 - h_m \circ \theta_2} \leq L \| \theta_1 - \theta_2\|$ for all $\theta_1, \theta_2 \in \mathcal{F}$ and $m\in\{1,2\}$. In view of this, we define the function classes:
\begin{align*}
    \mathcal{H}_{\delta} & := \left\{\phi_1 - \phi_2: \phi_1, \phi_2 \in \mathcal{H}; \, \norm{\phi_1 - \phi_2} \leq L\delta;\, \norm{\phi_1 - \phi_2}_{\infty} \leq L\delta^{1-1/(2\alpha)} \right\};\\
    \mathcal{H}_{1,\delta}^{k(n)} &:= \left\{ \phi - \Pi_{k(n)}\phi: \phi \in \mathcal{H}_\delta \right\};\;\mathcal{H}_{2,\delta}^{k(n)}:= \left\{\Pi_{k(n)}\phi: \phi \in \mathcal{H}_\delta \right\}.
\end{align*}
 Importantly, for $m\in\{1,2\}$, using the norm coupling of \ref{cond::theorem2::couplings}, it follows that $h_m \circ \theta_1 - h_m \circ \theta_2 \in \mathcal{H}_{\delta}$ for all $\theta_1, \theta_2 \in \mathcal{F}$ with $\|\theta_1 - \theta_2\| \leq \delta$. Hence, 
$ \mathcal{H}^{\star}_{\delta}  \subseteq \mathcal{H}_{\delta}$, which we will use to bound the supremum over $\mathcal{H}^{\star}_{\delta}$ appearing in term (II) by a supremum over $\mathcal{H}_{\delta}$. We will then use the function classes $  \mathcal{H}_{1,\delta}^{k(n)} $ and $  \mathcal{H}_{2,\delta}^{k(n)} $ to further bound the supremum over $\mathcal{H}_{\delta}$.

Before proceeding with the proof, we derive norm and entropy integral bounds for $\mathcal{H}_{\delta}$, $\mathcal{H}_{1,\delta}^{k(n)} $, and $\mathcal{H}_{2,\delta}^{k(n)}$. To this end, by \ref{cond::regularityOnActionSpace}, \ref{cond::sieveApproxthrm0}, and Lemma \ref{lemma::metricentropybounds}, we have $\mathcal{J}_{\infty}(\delta , \mathcal{H}_{\delta}) \lesssim \mathcal{J}_{\infty}(\delta , \mathcal{F})$ and $\mathcal{J}_{\infty}(\delta , \mathcal{H}_{1,\delta}^{k(n)}) + \mathcal{J}_{\infty}(\delta , \mathcal{H}_{2,\delta}^{k(n)}) \lesssim  \mathcal{J}_{\infty}(\{\log k(n)\}^{\nu} \delta , \mathcal{F})$ where $\nu \geq 0$. Moreover, since $\Pi_{k(n)}$ is an orthogonal projection, we have both $ \|\Pi_{k(n)} \phi\|  \leq \|\phi\| \lesssim \delta$ and $ \|\phi - \Pi_{k(n)} \phi \|   \leq \|\phi\| \lesssim \delta$ for any $\phi \in \mathcal{H}_{\delta}$. Also, since $\|\cdot\|\le\|\cdot\|_\infty$ and by Lemma \ref{lemma::sieveRateSup}, $\sup_{\phi \in  \mathcal{H}_{1,\delta}^{k(n)}} \|\phi - \Pi_{k(n)}\phi\|\le \sup_{\phi \in  \mathcal{H}_{1,\delta}^{k(n)}} \|\phi - \Pi_{k(n)}\phi\|_{\infty} \lesssim \rho_{n, \infty}$. Combining the preceding observations, $\sup_{\phi \in  \mathcal{H}_{1,\delta}^{k(n)}}\|\phi\|\lesssim \min\{\rho_{n,\infty}, \delta\}$ and $ \sup_{\phi \in  \mathcal{H}_{2,\delta}^{k(n)}}\|\phi\| \lesssim \delta$. In addition, by the triangle inequality, we have
$$\sup_{\phi \in  \mathcal{H}_{2,\delta}^{k(n)}}\|\phi\|_{\infty} = \sup_{\phi \in  \mathcal{H}_{\delta} }\|\Pi_{k(n)}\phi\|_{\infty} \leq \sup_{\phi \in  \mathcal{H}_{\delta} } \|\phi - \Pi_{k(n)}\phi\|_{\infty} + \sup_{\phi \in  \mathcal{H}_{\delta} }\|\phi\|_{\infty}\lesssim \rho_{n,\infty} + \delta^{1-1/(2\alpha)} ,$$ 
where we used that $ \sup_{\phi \in  \mathcal{H}_{\delta} }\|\phi\|_{\infty} \lesssim \delta^{1-1/(2\alpha)}$ by \ref{cond::theorem2::couplings}.

Proceeding with the proof, observe, by the triangle inequality, that
\begin{align*}
 \text{(II)}  & \lesssim \max_{m \in \{1,2\}} E_0^n \left[ \mathbb{I}_{A_n} \sup_{\phi \in \mathcal{H}_{\delta}}   \left|\overline{P}_n \overline{\Delta}^{(m)}_{\pi_{n,\diamond}, \mu_{n,\diamond}^*}(\cdot\,,\phi) \right|\right]\\
  & \lesssim \text{(IIA)} + \text{(IIB)},
\end{align*}
where
\begin{align*}
   & \text{(IIA)} := \max_{m \in \{1,2\}} E_0^n \left[  \mathbb{I}_{A_n}\sup_{\phi \in  \mathcal{H}_{1,\delta}^{k(n)}} \left|\overline{P}_n \overline{\Delta}^{(m)}_{\pi_{n,\diamond}, \mu_{n,\diamond}^*}(\cdot\,,\phi)\right|\right]; \\
   & \text{(IIB)} :=  \max_{m \in \{1,2\}} E_0^n\left[  \mathbb{I}_{A_n} \sup_{\phi \in  \mathcal{H}_{2,\delta}^{k(n)}} \left|\overline{P}_n \overline{\Delta}^{(m)}_{\pi_{n,\diamond}, \mu_{n,\diamond}^*}(\cdot\,,\phi) \right|\right].
\end{align*} 
To bound (II), it suffices to bound $\text{(IIA)}$ and $\text{(IIB)}$.
To this end, term (IIA) can be further bounded as $\text{(IIA)}  \leq  \text{(IIA1)} + \text{(IIA2)}+ \text{(IIA3)}$,
where
\begin{align*}
   \text{(IIA1)} &:=  \max_{m \in \{1,2\}} E_0^n \left[  \mathbb{I}_{A_n} \sup_{\phi \in  \mathcal{H}_{1,\delta}^{k(n)}} \left|\overline{P}_0 \overline{\Delta}^{(m)}_{\pi_{n,\diamond}, \mu_{n,\diamond}^*}(\cdot\,,\phi) \right| \right];\\
    \text{(IIA2)} &:=  \max_{m \in \{1,2\}} E_0^n \left[  \mathbb{I}_{A_n} \sup_{\phi \in  \mathcal{H}_{1,\delta}^{k(n)}} \left|\left( \overline{P}_n- \overline{P}_0\right) \left(  \overline{\Delta}_{\pi_{n,\diamond}, \mu_0}^{(m)}(\cdot\,,\phi) \right) 
    \right| \right];\\
     \text{(IIA3)} &:=   \max_{m \in \{1,2\}} E_0^n\left[  \mathbb{I}_{A_n} \sup_{\phi \in  \mathcal{H}_{1,\delta}^{k(n)}} \left|\left( \overline{P}_n- \overline{P}_0\right) \left(\overline{\Delta}^{(m)}_{\pi_{n,\diamond}, \mu_{n,\diamond}^*}(\cdot\,,\phi) - \overline{\Delta}_{\pi_{n,\diamond}, \mu_0}^{(m)}(\cdot\,,\phi) \right) 
    \right|\right].
\end{align*}
To bound (IIA1), we apply Lemma \ref{lemma::P0boundSieve2} with $\mathcal{G}:=\mathcal{H}_{1,\delta}^{k(n)}$, the sieve approximation rate of \ref{cond::sieveApproxthrm}, and Event \ref{event::debias_mu_n_star} to obtain the bound:
\begin{align*}
 \text{(IIA1)} & \lesssim    \sup_{\phi \in \mathcal{H}_{1,\delta}^{k(n)}} \norm{\phi - \Pi_{k(n)}\phi}  r_n^* \lesssim \min\{\delta,  \rho_{n,\infty} \} \cdot r_n^*.
 \end{align*}
To bound (IIA2), we similarly apply Lemma \ref{lemma::empProcDeltastar3} with $\delta := \rho_{n,\infty} +  n^{-1/2}$ and the entropy bound of Lemma \ref{lemma::metricentropybounds} to obtain:
\begin{align*}
\text{(IIA2)}& \lesssim n^{-1/2} \mathcal{J}_{\infty}\left(\{\log k(n)\}^{\nu} \{\rho_{n,\infty} +  n^{-1/2}\}, \mathcal{F}\right) \\
& \lesssim n^{-1/2} \left[ \{\log k(n)\}^{\nu}\{\rho_{n,\infty}+  n^{-1/2}\}\right]^{1-1/(2\alpha)}.
\end{align*}
To bound (IIA3), we apply Lemma \ref{lemma::empProcDeltastar2} and the entropy bound of Lemma \ref{lemma::metricentropybounds} to obtain:
\begin{align*}
  \text{(IIA3)} & \lesssim  n^{-1/2} r_n^* \mathcal{J}_{\infty} \left( (\{\log k(n)\}^{\nu}\{ \rho_{n,\infty}  + n^{-1/2}\} ) / r_n^*  ,\mathcal{F} \right) + \min\{\delta^{1 - 1/(2\alpha)}, \rho_{n,\infty}\}  r_n^*  \sqrt{k(n) \log n / n}\\
  & \lesssim  o\left(n^{-1/2} \left[ \{\log k(n)\}^{\nu}\{\rho_{n,\infty}+  n^{-1/2}\}\right]^{1-1/(2\alpha)} \right) +    \min\{\delta, \rho_{n,\infty } \} r_n^*  \sqrt{k(n) \log n / n}   ,
\end{align*}
where the final inequality follows from \ref{cond::regularityOnActionSpace} and $ r_n^* = o(1)$.
Finally, combining the bounds for (IIA1), (IIA2), and (IIA3), we obtain the bound:
\begin{align*}
 \text{(IIA)} & \lesssim   \min\{\delta, \rho_{n,\infty}\}  r_n^*  + n^{-1/2} \left[ \{\log k(n)\}^{\nu}\{\rho_{n,\infty}+  n^{-1/2}\}\right]^{1-1/(2\alpha)} + \min\{\delta, \rho_{n,\infty } \} r_n^*  \sqrt{k(n) \log n / n}   .
\end{align*}
We now turn to (IIB). Firstly, observe that term (IIB) is zero if $\mathbb{I}_{g_1, g_2} = 0$; that is, if $g_1$ and $g_2$ in the risk definition of \eqref{eqn::popriskRR} equal the identity function $(x \mapsto x)$. Therefore, (IIB) can be bounded as:
 \begin{align}
   \text{(IIB)} &=  \mathbb{I}_{g_1, g_2}  \max_{m \in \{1,2\}}  E_0^n \left[  \mathbb{I}_{A_n}\sup_{\phi \in \mathcal{H}_{2, \delta}^{k(n)}}  \left|\overline{P}_n \overline{\Delta}^{(m)}_{\pi_{n,\diamond}, \mu_{n,\diamond}^*}(\cdot\,, \phi) \right| \right]\nonumber\\
   & \leq \mathbb{I}_{g_1, g_2}\max_{m \in \{1,2\}}  E_0^n \left[  \mathbb{I}_{A_n}\sup_{\phi \in \mathcal{H}_{2, \delta}^{k(n)}} \left| \overline{P}_n \overline{\Delta}_{\pi_{n,\diamond}, \mu_{n,\diamond}}^{*(m)}(\cdot\,, \phi) \right| \right] \nonumber \\ 
   & \quad +\mathbb{I}_{g_1, g_2} \max_{m \in \{1,2\}}  E_0^n\left[  \mathbb{I}_{A_n}\sup_{\phi \in \mathcal{H}_{2, \delta}^{k(n)}} \left| (\overline{P}_n-  \overline{P}_0) \left( \overline{\Delta}^{(m)}_{\pi_{n,\diamond}, \mu_{n,\diamond}^*}( \phi) - \overline{\Delta}_{\pi_{n,\diamond}, \mu_{n,\diamond}}^{*(m)}(\cdot\,, \phi)\right) \right| \right]\label{eqn::boundSplitSieve} \\
    & \quad +\mathbb{I}_{g_1, g_2} \max_{m \in \{1,2\}}  E_0^n\left[  \mathbb{I}_{A_n}\sup_{\phi \in \mathcal{H}_{2, \delta}^{k(n)}} \left| \overline{P}_0\left( \overline{\Delta}^{(m)}_{\pi_{n,\diamond}, \mu_{n,\diamond}^*}(\cdot\,, \phi) - \overline{\Delta}_{\pi_{n,\diamond}, \mu_{n,\diamond}}^{*(m)}(\cdot\,, \phi)\right) \right| \right]. 
 \label{eqn::theorem2::justprojection}
\end{align}
The first term on the right is zero since $\mathcal{H}_{2, \delta}^{k(n)} \subset \mathcal{H}_{k(n)}$ and the sieve-MLEs $\{\mu_{n,j}^*: j \in [J]\}$ satisfy $ \overline{P}_n \overline{\Delta}_{\pi_{n,\diamond}, \mu_{n,\diamond}}^{*(m)}(\cdot\,, \phi) = 0$ for all $\phi \in\mathcal{H}_{k(n)}$. By the triangle inequality and using that $\mathcal{H}_{2,\delta}^{k(n)}\subseteq \{\phi_1+\phi_2 : \phi_1\in \mathcal{H}_{\delta},\phi_2\in \mathcal{H}_{1,\delta}^{k(n)}\}$, we can further upper bound \eqref{eqn::boundSplitSieve} as
\begin{align*}
     \text{(IIB)}  \leq  \text{(IIB1)} +  \text{(IIB2)}  +  \text{(IIB3)},
\end{align*}
where we define:
\begin{align*}
    \text{(IIB1)} &:= \mathbb{I}_{g_1, g_2} \max_{m \in \{1,2\}}  E_0^n \left[  \mathbb{I}_{A_n}\sup_{\phi \in \mathcal{H}_{\delta} } \left| (\overline{P}_n-  \overline{P}_0) \left( \overline{\Delta}^{(m)}_{\pi_{n,\diamond}, \mu_{n,\diamond}^*}( \phi) - \overline{\Delta}_{\pi_{n,\diamond}, \mu_{n,\diamond}}^{*(m)}(\cdot\,, \phi)\right) \right|\right]\nonumber \\
      \text{(IIB2)}  & := \mathbb{I}_{g_1, g_2} \max_{m \in \{1,2\}}  E_0^n\left[  \mathbb{I}_{A_n}\sup_{\phi \in \mathcal{H}_{1, \delta}^{k(n)}} \left| (\overline{P}_n-  \overline{P}_0) \left( \overline{\Delta}^{(m)}_{\pi_{n,\diamond}, \mu_{n,\diamond}^*}( \phi) - \overline{\Delta}_{\pi_{n,\diamond}, \mu_{n,\diamond}}^{*(m)}(\cdot\,, \phi)\right) \right|\right[ \nonumber\\
       \text{(IIB3)}  & := \mathbb{I}_{g_1, g_2} \max_{m \in \{1,2\}}  E_0^n\left[  \mathbb{I}_{A_n}\sup_{\phi \in \mathcal{H}_{2, \delta}^{k(n)}} \left| \overline{P}_0\left( \overline{\Delta}^{(m)}_{\pi_{n,\diamond}, \mu_{n,\diamond}^*}(\cdot\,, \phi) - \overline{\Delta}_{\pi_{n,\diamond}, \mu_{n,\diamond}}^{*(m)}(\cdot\,, \phi)\right) \right| \right].
\end{align*}
 We now bound terms (IIB1), (IIB2), and (IIB3). Notice that term (IIB2) is identical to term (IIA3), which we bounded earlier. Hence, our earlier bound implies that
$$ \text{(IIB2)} \lesssim  \mathbb{I}_{g_1, g_2}  o\left(n^{-1/2} \left[ \{\log k(n)\}^{\nu}\{\rho_{n,\infty}+  n^{-1/2}\}\right]^{1-1/(2\alpha)} \right) +  \mathbb{I}_{g_1, g_2} \rho_{n,\infty }  r_n^*  \sqrt{k(n) \log n / n}   .$$
We claim that (IIB1) satisfies the bound:
\begin{align*}
    \text{(IIB1)} &\lesssim \, \mathbb{I}_{g_1, g_2} \left\{n^{-1/2} r_n^* \mathcal{J}_{\infty} \left( \delta  / r_n^*  , \mathcal{F}\right) + \sup_{\phi \in \mathcal{H}_{\delta}} \|\phi\|_{\infty} r_n^*\sqrt{k(n) \log n / n} \right\} \\
    &\lesssim \, \mathbb{I}_{g_1, g_2} \left\{n^{-1/2} r_n^* \mathcal{J}_{\infty} \left( \delta  / r_n^*  , \mathcal{F}\right) +   \delta^{1-1/(2\alpha)}r_n^*\sqrt{k(n) \log n / n} \right\}.
\end{align*}
In the above, the first inequality follows from Lemma \ref{lemma::empProcDeltastar} with $\mathcal{G} :=  \mathcal{H}_{\delta}$ and the entropy bound $\mathcal{J}_{\infty}(\delta , \mathcal{H}_{\delta}) \lesssim \mathcal{J}_{\infty}(\delta,  \mathcal{F})$ that we derived earlier. For the final inequality, we used that $\sup_{\phi \in \mathcal{H}_{\delta}} \|\phi\|_{\infty}  \lesssim \delta^{1-1/(2\alpha)}$ by \ref{cond::theorem2::couplings}. Next, to bound (IIB3), note, by Lemma \ref{lemma::P0boundSieve1}, \ref{cond::sieveApproxthrm}, Lemma \ref{lemma::metricentropybounds}, and Event \ref{event::debias_mu_n_star}, that:
\begin{align*}
\text{(IIB3)} &\lesssim  \mathbb{I}_{g_1, g_2}\sup_{\phi \in \mathcal{H}_{2, \delta}^{k(n)}} \norm{ \phi}_{\infty}  \left\{r_n^*\right\}^2\leq   \mathbb{I}_{g_1, g_2}\{\delta^{1-1/(2\alpha)} + \rho_{n,\infty} \}\left\{r_n^*\right\}^2.
\end{align*}
Combining the bounds for (IIB1), (IIB2), and (IIB3), we find:
\begin{align*}
       \text{(IIB)} &\lesssim   \mathbb{I}_{g_1, g_2} n^{-1/2} \left[ \{\log k(n)\}^{\nu}\{\rho_{n,\infty}+  n^{-1/2}\}\right]^{1-1/(2\alpha)} + \mathbb{I}_{g_1, g_2} n^{-1/2} r_n^* \mathcal{J}_{\infty} \left( \delta   / r_n^*  , \mathcal{F}\right) \\ 
          & \quad   + \mathbb{I}_{g_1, g_2}   \delta^{1-1/(2\alpha)}   r_n^*\sqrt{k(n) \log n / n}   
     +  \mathbb{I}_{g_1, g_2}\{\delta^{1-1/(2\alpha)} + \rho_{n,\infty} \}\left\{r_n^*\right\}^2.
\end{align*}
Now, note that $ \mathbb{I}_{g_1, g_2} (r_n^*)^2 = o(n^{-1/2})$ by \ref{cond::A2Nuisance} and the constraint that $k(n) = o(n^{2\beta/(2\beta+1)}/ \log n)$. Since $r_n^* =o(1)$, we have $ n^{-1/2}r_n^* \mathcal{J}_{\infty} \left( \delta   / r_n^*  , \mathcal{F}\right)\lesssim \delta^{1-1/(2\alpha)}o(n^{-1/2})$. Moreover, since $  \sqrt{k(n) \log n / n}   \leq r_n^*$, we have 
$$  \mathbb{I}_{g_1, g_2}\delta^{1-1/(2\alpha)}   r_n^*\sqrt{k(n) \log n / n} \leq   \mathbb{I}_{g_1, g_2}\delta^{1-1/(2\alpha)}   (r_n^*)^2 = \delta^{1-1/(2\alpha)}   o(n^{-1/2}).$$ 
Combining these bounds, we find 
\begin{align*}
       \text{(IIB)} &\lesssim   \mathbb{I}_{g_1, g_2}  n^{-1/2} \left[ \{\log k(n)\}^{\nu}\{\rho_{n,\infty}+  n^{-1/2}\}\right]^{1-1/(2\alpha)}  +    \mathbb{I}_{g_1, g_2}o(n^{-1/2})\delta^{1-1/(2\alpha)}   \\
       &\quad+    \mathbb{I}_{g_1, g_2} \rho_{n,\infty} \left(r_n^*\right)^2 .
\end{align*}

Thus, combining our bounds for (IIA) and (IIB), we obtain:
\begin{align*}
  \text{(II)} & \lesssim   r_n^*  \min\{\delta, \rho_{n,\infty}\}    + n^{-1/2} \left[ \{\log k(n)\}^{\nu}\{\rho_{n,\infty}+  n^{-1/2}\}\right]^{1-1/(2\alpha)} + \min\{\delta, \rho_{n,\infty } \} r_n^*  \sqrt{k(n) \log n / n}  \\ 
  & \quad +   \mathbb{I}_{g_1, g_2}  n^{-1/2} \left[ \{\log k(n)\}^{\nu}\{\rho_{n,\infty}+  n^{-1/2}\}\right]^{1-1/(2\alpha)}  +    \mathbb{I}_{g_1, g_2}o(n^{-1/2})\delta^{1-1/(2\alpha)}      \\
  &\quad +    \mathbb{I}_{g_1, g_2} \rho_{n,\infty} \left(r_n^*\right)^2 .
\end{align*}
To bound the final term above, note, by \ref{cond::A2Nuisance} and the condition that $\mathbb{I}_{g_1, g_2} k(n) = \mathbb{I}_{g_1, g_2} o(n^{2\beta/(2\beta+1)}/ \log n)$, that 
\begin{align*}
    \mathbb{I}_{g_1, g_2} \rho_{n,\infty} \left(r_n^*\right)^2 &= \rho_{n,\infty}  o(n^{-1/2}) \\
    & = \left[ \{\log k(n)\}^{\nu}\{\rho_{n,\infty}+  n^{-1/2}\}\right]^{1-1/(2\alpha)}O\left(  n^{-1/2}\right),
\end{align*} 
where, for the final equality, we use that $\log k(n) = O(\log n)$ and $1 - \frac{1}{2\alpha} < 1$. Furthermore, for the term $\min\{\delta, \rho_{n,\infty } \} r_n^*  \sqrt{k(n) \log n / n} $, we have the trivial bound $\min\{\delta, \rho_{n,\infty } \} r_n^*  \sqrt{k(n) \log n / n}  \leq \min\{\delta, \rho_{n,\infty } \} r_n^*  $, since $ \sqrt{k(n) \log n / n}  = o(1)$.
Hence,
\begin{align*}
  \text{(II)} 
     & \lesssim   r_n^*  \min\{\delta, \rho_{n,\infty}\} + n^{-1/2} \left[ \{\log k(n)\}^{\nu}\{\rho_{n,\infty}+  n^{-1/2}\}\right]^{1-1/(2\alpha)}  + 
     \mathbb{I}_{g_1, g_2} o(n^{-1/2}) \delta^{1-1/(2\alpha)}.
\end{align*}
Here, we combined alike terms and used that $\mathbb{I}_{g_1, g_2}(r_n^*)^2 =    o(n^{-1/2})$ by \ref{cond::A2Nuisance}.  We also used that $ \rho_{n,\infty} o(n^{-1/2})  = \left[ \{\log k(n)\}^{\nu}\{\rho_{n,\infty}+  n^{-1/2}\}\right]^{1-1/(2\alpha)}O\left(  n^{-1/2}\right)$, since $\log k(n) = O(\log n)$ and $1 - \frac{1}{2\alpha} < 1$. 

Finally, combining our bounds for terms (I) and (II) and combining terms, we obtain the following bound:
\begin{align*}
    E_0^n\{  \mathbb{I}_{A_n} \phi_{n,1}(\delta)\} &\leq  \text{(I)}+  \text{(II)} \\
    & \lesssim    \delta^{1-1/(2\alpha)} \left[ r_n^* n^{-\gamma/(2\gamma+1)} + o(n^{-1/2}) + r_n^*  \sqrt{k(n) \log n / n} +  n^{-1/2}  (s_n^*)^{1/(2\alpha)} \right]\\
    & \quad + r_n^*  \min\{\delta, \rho_{n,\infty}\} + n^{-1/2} \left[ \{\log k(n)\}^{\nu}\{\rho_{n,\infty}+  n^{-1/2}\}\right]^{1-1/(2\alpha)}  \\
   + &  \quad 
     \mathbb{I}_{g_1, g_2} o(n^{-1/2}) \delta^{1-1/(2\alpha)}  \\
      & \lesssim   \delta^{1-1/(2\alpha)}   \left[ r_n^*   n^{-\gamma/(2\gamma+1)}  + r_n^*  \sqrt{k(n) \log n / n} +  n^{-1/2}  (s_n^*)^{1/(2\alpha)} + o(n^{-1/2})  \right]\\
    & \quad  + r_n^*  \min\{\delta, \rho_{n,\infty}\} + n^{-1/2} \left[ \{\log k(n)\}^{\nu}\{\rho_{n,\infty}+  n^{-1/2}\}\right]^{1-1/(2\alpha)} .
\end{align*}
Recall that $ \rho_{n,\infty} := \{\log k(n)\}^{\nu} k(n)^{-\rho} \gtrapprox n^{-1/2}$ with $\rho > 1/2$ and note, by \ref{cond::A2Nuisance} and \ref{cond::theorem2::couplings}, that $(s_n^*)^{1/(2\alpha)} = o(1)$. Using this, we finally obtain the desired bound:
\begin{align*}
    E_0^n\{ \mathbb{I}_{A_n} \phi_{n,1}(\delta) \}&\lesssim  \delta^{1-1/(2\alpha)}   \left\{ r_n^*  n^{-\gamma/(2\gamma+1)}  + r_n^*  \sqrt{k(n) \log n / n} + o(n^{-1/2})  \right\}\\
    & \quad + r_n^*  \min\{\delta, \{\log k(n)\}^{\nu} k(n)^{-\rho}\} +  n^{-1/2} \left[ \{\log k(n)\}^{2\nu}\{k(n)^{-\rho}\}\right]^{1-1/(2\alpha)}  .
\end{align*}

 \end{proof}

 \subsection{Proof of Theorem \ref{theorem::EpLearnerOracleEff}}

\begin{proof}[Proof of Theorem \ref{theorem::EpLearnerOracleEff}]
 
   It suffices to show $\norm{\theta_{n,k(n)}^* - \theta_{n,0}} =\smallO_p(n^{-\alpha/(2\alpha+1)})$ as this implies, by \ref{cond::theorem2::oracleRateTight}, that
    $$\|\theta_{n, k(n)}^* - \theta_{n,0}\|/E_0^n\|\theta_{n,0} - \theta_0\|=\smallO_p(n^{-\alpha/(2\alpha+1)})/ E_0^n\|\theta_{n,0} - \theta_0\|=\smallO_p(1),$$
    as desired.

    To this end, note that $n^{1/2} = o(n^{2c(\beta,\gamma)/(2c(\beta,\gamma) + 1)})$ since $\beta, \gamma > 1/2$. Therefore, since $k(n) \leq  o(n^{1/2}/\log n)$, we have $k(n) \leq o(n^{2c(\beta,\gamma)/(2c(\beta,\gamma) + 1)}/ \log n )$ satisfies the growth rate bound of Theorem \ref{theorem::EpLearnerRate}. Thus, by Theorem \ref{theorem::EpLearnerRate}, we have $\norm{\theta_{n,k(n)}^* - \theta_{n,0}} =\smallO_p(n^{-\alpha/(2\alpha+1)}) + \mathcal{O}_p(\varepsilon_n)$, where $\varepsilon_n$ is defined above Theorem \ref{theorem::EpLearnerRate}. Hence, to show $\norm{\theta_{n,k(n)}^* - \theta_{n,0}} =\smallO_p(n^{-\alpha/(2\alpha+1)})$, it suffices to show that $\varepsilon_n =\smallO_p(n^{-\alpha/(2\alpha+1)})$. In view of the minimum in the definition of $\varepsilon_n$ and recalling that we assume $\rho = \alpha$, it further suffices to show that 
    $$a_n = \{\log n\}^{2\nu}\left\{n^{-\beta/(2\beta+1)}  + \sqrt{k(n) \log n /n} + \sqrt{k(n)/ n}\right\} k(n)^{-\alpha}= o(n^{-2\alpha/(2\alpha+1)}).$$
    This holds by assumption on our growth condition on $k(n)$.
For the choices of $\beta$, $\gamma$, and $\alpha$ required by our conditions, Lemma \ref{lem:exist-k} guarantees the existence of a sequence $k(n)$ for which this holds satisfying our conditions. We now turn to proving Lemma \ref{lem:exist-k}.

 \end{proof}

\begin{lemma}\label{lem:exist-k}
Suppose $\alpha>1/2$ and $0<\min\{\tfrac{1}{2}, \beta\}\leq\gamma$. If $\beta > \min\{\tfrac{1}{2}, 2\alpha/(4\alpha^2+1)\}$, then there exists a sequence $k(n)$ such that 
\[
\frac{(\log n)\,k(n)}{n^{2c(\beta,\gamma)/(2c(\beta,\gamma)+1)}}\to 0
\quad\text{and}\quad
a_n\, n^{\tfrac{2\alpha}{2\alpha+1}} \to 0,
\]
where $c(\beta,\gamma) = \min\{\gamma, \beta, 1/2\}$ and
\[
a_n = \{\log n\}^{2\nu}\Bigl\{n^{-\beta/(2\beta+1)} + \sqrt{k(n)\log n /n} + \sqrt{k(n)/n}\Bigr\}\, k(n)^{-\alpha}.
\]
\end{lemma}
\begin{proof}[Proof of Lemma \ref{lem:exist-k} (short version)]
Let $c:=c(\beta,\gamma)=\min\{\gamma,\beta,1/2\}$. Since $\gamma\ge \min\{\beta,1/2\}$ by assumption, we have
$c=\min\{\beta,1/2\}$. Take $k(n)=n^{t}/(\log n)^{s}$ with $t\in(0,1)$ and $s>1$.

\emph{(1) First constraint.} 
\[
\frac{(\log n)k(n)}{n^{2c/(2c+1)}}=n^{\,t-\frac{2c}{2c+1}}(\log n)^{1-s}\to 0
\quad\Longleftrightarrow\quad
t<U(c):=\frac{2c}{2c+1}.
\]

\emph{(2) Second constraint.}
\[
a_n=\{\log n\}^{2\nu}\Bigl\{n^{-\beta/(2\beta+1)}+\sqrt{k(n)\log n/n}+\sqrt{k(n)/n}\Bigr\}k(n)^{-\alpha}.
\]
Since $(\log n)^r=o(n^\delta)$ for any fixed $r\in\mathbb{R}$ and $\delta>0$, logarithmic factors can be absorbed once the polynomial exponents in $n$ are strictly negative. In particular, for the $\sqrt{k\log n/n}$ term, the extra $(\log n)^{1/2}$ factor is negligible compared to $n^\delta$ for some $\delta>0$. It is therefore enough that the $n$-exponents are negative for
\[
n^{-\beta/(2\beta+1)}k(n)^{-\alpha}n^{2\alpha/(2\alpha+1)} 
\quad\text{and}\quad
\Bigl(\frac{k(n)}{n}\Bigr)^{1/2}k(n)^{-\alpha}n^{2\alpha/(2\alpha+1)}.
\]
Writing $k(n)^{-\alpha}=n^{-\alpha t}(\log n)^{\alpha s}$ gives the two inequalities
\[
-\frac{\beta}{2\beta+1}-\alpha t+\frac{2\alpha}{2\alpha+1}<0
\quad\text{and}\quad
\frac{t-1}{2}-\alpha t+\frac{2\alpha}{2\alpha+1}<0,
\]
i.e.
\[
t>\frac{1}{\alpha}\!\left(\frac{2\alpha}{2\alpha+1}-\frac{\beta}{2\beta+1}\right)
\quad\text{and}\quad
t>\frac{1}{2\alpha+1}.
\]
Hence the second constraint is equivalent to $t>L(\alpha,\beta)$ where
\[
L(\alpha,\beta):=\max\!\left\{\frac{1}{2\alpha+1},\ \frac{1}{\alpha}\!\left(\frac{2\alpha}{2\alpha+1}-\frac{\beta}{2\beta+1}\right)\right\}.
\]

\emph{(3) Nonemptiness of }$(L(\alpha,\beta),U(c))$.
Let $m:=\min\{\beta,1/2\}$ so $U(c)=U(m)=\frac{2m}{2m+1}$. It remains to show
\[
\frac{1}{\alpha}\!\left(\frac{2\alpha}{2\alpha+1}-\frac{\beta}{2\beta+1}\right)\;<\;\frac{2m}{2m+1}.
\]
If $\beta\le 1/2$ then $m=\beta$ and the above is equivalent to $\beta>\frac{2\alpha}{4\alpha^2+1}$.  
If $\beta\ge 1/2$, note that $g(\beta):=\frac{1}{\alpha}\!\left(\frac{2\alpha}{2\alpha+1}-\frac{\beta}{2\beta+1}\right)$ is decreasing in $\beta$ (since $\beta\mapsto \beta/(2\beta+1)$ is increasing). Thus
\[
g(\beta)\le g(1/2)=\frac{6\alpha-1}{4\alpha(2\alpha+1)}<\frac{1}{2}
\quad\text{since }8\alpha^2-8\alpha+2=8(\alpha-\tfrac12)^2>0,
\]
and $2m/(2m+1)=1/2$ when $m=1/2$. Hence in all cases $L(\alpha,\beta)<U(c)$ provided
$\beta>\min\{1/2,\,2\alpha/(4\alpha^2+1)\}$.

Pick any $t\in\bigl(L(\alpha,\beta),U(c)\bigr)$ and, say, $s=2$. Then both constraints hold, completing the proof.
\end{proof}

Recall that $$a_n = \{\log n\}^{2\nu}\left\{n^{-\beta/(2\beta+1)}  + \sqrt{k(n) \log n /n} + \sqrt{k(n)/ n}\right\} k(n)^{-\alpha}$$ under the conditions of the theorem

\begin{lemma}\label{lem:exist-k}
Suppose $\alpha>1/2$ and $0<\min\{\tfrac{1}{2}, \beta\}\leq\gamma$. If $\beta > \min\{\tfrac{1}{2}, 2\alpha/(4\alpha^2+1)\}$, then there exists a sequence $k(n)$ such hat $\frac{(\log n)\,k(n)}{n^{2c(\beta,\gamma)/(2c(\beta,\gamma)+1)}}\to 0$   and $a_n n^{\tfrac{2\alpha}{2\alpha+1}} \rightarrow 0 $, where $c(\beta,\gamma) = \min\{\gamma, \beta, 1/2\}$
 
\end{lemma}
\begin{proof}[Proof of Lemma \ref{lem:exist-k}]
For $k(n)=n^t/(\log n)^s$, the requirement $a_n=o(n^{-2\alpha/(2\alpha+1)})$ together with
\[
\frac{(\log n)k(n)}{n^{2\beta/(2\beta+1)}}\to 0,
\qquad 
n^{-1/2}(\log n)k(n)\to 0
\]
is equivalent to
\[
L(\alpha,\beta) < t < U(\beta), \qquad s>1,
\]
where 
\[
L(\alpha,\beta):=\max\!\Bigl\{\tfrac{1}{2\alpha+1},\,\tfrac{1}{\alpha}\Bigl(\tfrac{2\alpha}{2\alpha+1}-\tfrac{\beta}{2\beta+1}\Bigr)\Bigr\},
\qquad
U(\beta):=\min\!\Bigl\{\tfrac12,\,\tfrac{2\beta}{2\beta+1}\Bigr\}.
\]
We claim this interval is nonempty iff $\beta>2\alpha/(4\alpha^2+1)$. Fix $\epsilon>0$ small and choose $\sigma$ with $L(\alpha,\beta)+\epsilon<\sigma<U(\beta)$; then
\[
k(n):=\frac{n^\sigma}{(\log n)^2}
\]
satisfies
\[
\omega\!\Bigl(n^{\epsilon+L(\alpha,\beta)}\Bigr)\le k(n)\le o\!\bigl(n^{U(\beta)}/\log n\bigr),
\]
since $\sigma<U(\beta)$ implies $n^\sigma=o(n^{U(\beta)}/\log n)$.

It remains to prove the claim. Let
\[
A:=\frac{1}{2\alpha+1},\qquad 
B:=\frac{1}{\alpha}\Bigl(\frac{2\alpha}{2\alpha+1}-\frac{\beta}{2\beta+1}\Bigr),\qquad
C:=\tfrac12,\qquad
D:=\frac{2\beta}{2\beta+1}.
\]
Then $L(\alpha,\beta)=\max\{A,B\}$ and $U(\beta)=\min\{C,D\}$. Since $A<C$ for all $\alpha>0$, the condition $L<U$ is equivalent to $B<U$.

If $\beta\le\tfrac12$, then $U=D$ and
\[
B<D 
\;\Longleftrightarrow\;
\frac{1}{\alpha}\Bigl(\frac{2\alpha}{2\alpha+1}-\frac{\beta}{2\beta+1}\Bigr)
< \frac{2\beta}{2\beta+1}
\;\Longleftrightarrow\;
\beta>\frac{2\alpha}{4\alpha^2+1}.
\]
If $\beta\ge\tfrac12$, then $U=C=\tfrac12$ and, since $A<\tfrac12$, we need $B<\tfrac12$. One checks that $\beta>\tfrac{2\alpha}{4\alpha^2+1}$ implies $B<\tfrac12$, whereas if $\beta\le\tfrac{2\alpha}{4\alpha^2+1}$ then $B\ge D\ge U$, so $L\ge U$.

Hence $L(\alpha,\beta)<U(\beta)$ iff $\beta>\tfrac{2\alpha}{4\alpha^2+1}$, proving the claim and the lemma.
\end{proof}

\section{Rate for debiased outcome regression nuisance}
\label{appendix::debiasedrate}

In this section, we verify that Condition~\ref{cond::outcomerateDebiased} holds when using the squared error loss in Algorithm~\ref{alg::debiasing} (Method~2). Although we focus on the squared error loss here, the argument relies primarily on the strong convexity of the loss and thus extends, with minor modifications, to more general loss functions.

The following lemma shows, under reasonable conditions, that 
$$\sum_{j=1}^J \bigl\|\mu_{n,j}^* - \mu_0\bigr\|_{\overline P_0}
\;\lesssim\;
\sum_{j=1}^J  \bigl\|\mu_{n,j} - \mu_0\bigr\|_{\overline P_0}
\;+\;
O_p\!\left(\sqrt{\tfrac{k(n)\log n}{n}}\right).$$
Assuming this bound, Condition~\ref{cond::outcomerateDebiased} follows directly from the nuisance rate condition (Condition~\ref{cond::outcomerateSecond}) on the initial nuisance estimators $(\mu_{n,j}: j \in [J])$.

Let $k = k(n)$. In the following lemma, define the population analogue of $\beta_n$ in Algorithm~\ref{alg::debiasing} as
\[
\beta_0
= \argmin_{\beta \in \mathbb{R}^{2k}}
\sum_{j=1}^K
E_0\!\left[
\frac{1}{\pi_{n,j}(A \mid W)}
\left\{Y - \mu_{n,j}(A, W) - \beta^\top \widehat{\varphi}_{k,j}(A, W)\right\}^2
\;\middle|\;
\mathcal{D}_n \setminus \mathcal{D}_n^j
\right].
\]

\begin{lemma}\label{lem:proj_rate}
Let $\beta_n$ and $\mu_{n,\diamond}^*$ be obtained from Algorithm~\ref{alg::debiasing} using Method~2.  
Assume Condition~\ref{cond::boundPos} and suppose that, for all $n$, 
$\max\{\|\beta_n\|_{\infty},\,\|\beta_0\|_{\infty}\}<M$ almost surely.  
Let $k=k(n)$. Then
\[
\bigl\|\mu_{n,\diamond}^* - \mu_0\bigr\|_{\overline P_0}
\;\le\;
\bigl\|\mu_{n,\diamond} - \mu_0\bigr\|_{\overline P_0}
\;+\;
O_p\!\left(\sqrt{\tfrac{k\log n}{n}}\right).
\]
\end{lemma}
\begin{proof}
   Let $k = k(n)$. The estimated coefficient $\beta_n$ obtained from Method 2 of Algorithm~\ref{alg::debiasing} can be written as:
    $$\beta_n = \argmin_{\beta \in \mathbb{R}^{2k}} \sum_{i=1}^n \frac{1}{\pi_{n,j(i)}(A_i \mid W_i)} \left\{Y_i - \mu_{n,j(i)}(A_i, W_i) - \beta^\top \widehat{\varphi}_{k, j(i)}(A_i, W_i) \right\}^2.$$
    Denote the corresponding population risk minimizer by 
    $$\beta_0 = \argmin_{\beta \in \mathbb{R}^{2k}} \sum_{j=1}^K E_0\left[\frac{1}{\pi_{n,j}(A \mid W)} \left\{Y - \mu_{n,j}(A, W) - \beta^\top \widehat{\varphi}_{k, j}(A, W) \right\}^2 \mid  \mathcal{D}_n \backslash \mathcal{D}_n^j\right].$$
    The first-order optimality conditions for $\beta_n$ imply, for all $\beta \in \mathbb{R}^{2k}$ that
    \begin{align*}
        \sum_{i=1}^n \frac{1}{\pi_{n,j(i)}(A_i \mid W_i)}\{\mu_{n, j}(\beta)-\mu_{n,j}\}(A_i,W_i)\left\{Y_i - \mu_{n,j(i)}^*(A_i, W_i) \right\} = 0,
    \end{align*}
    where $\mu_{n, j}(\beta) := \mu_{n,j} + \beta^\top \widehat{\varphi}_{k, j}$. Using our notation, the above implies that 
    \[
    \frac{1}{J}\sum_{j=1}^J P_{n,j}\!\left[\tfrac{1}{\pi_{n,j}}\,\{\mu_{n, j}(\beta)-\mu_{n,j}\}\{\mathcal{I}_Y - \mu_{n,j}^*\}\right]  =  0,
    \]
    where $\mathcal{I}_Y : o \mapsto y$ denotes the $y$ coordinate projection map. Adding and subtracting terms, we find, for all $\beta \in \mathbb{R}^{2k}$, that
    \begin{align*}
  \frac{1}{J}\sum_{j=1}^J  P_0\!\left[\tfrac{1}{\pi_{n,j}}\,\{\mu_{n, j}(\beta)-\mu_{n,j}\}\{\mathcal{I}_Y - \mu_{n,j}^*\}\right]   =  \frac{1}{J}\sum_{j=1}^J  (P_0 - P_{n,j})\!\left[\tfrac{1}{\pi_{n,j}}\,\{\mu_{n, j}(\beta)-\mu_{n,j}\}\{\mathcal{I}_Y - \mu_{n,j}^*\}\right].
    \end{align*}
    Moreover, the first-order optimality conditions of $\beta_0$ imply, for all $\beta \in \mathbb{R}^{2k}$, that
    $$  \frac{1}{J}\sum_{j=1}^J  P_0\!\left[\tfrac{1}{\pi_{n,j}}\,\{\mu_{n, j}(\beta)-\mu_{n,j}\}\{\mathcal{I}_Y - \mu_{n,j}(\beta_0) \}\right]  = 0. $$
Hence, combining the previous two displays, we obtain
\[
\frac{1}{J}\sum_{j=1}^J P_0\!\left[\tfrac{1}{\pi_{n,j}}\,\{\mu_{n, j}(\beta)-\mu_{n,j}\}\{\mu_{n,j}(\beta_0)-\mu_{n,j}^*\}\right]
= \frac{1}{J}\sum_{j=1}^J (P_0 - P_{n,j})\!\left[\tfrac{1}{\pi_{n,j}}\,\{\mu_{n, j}(\beta)-\mu_{n,j}\}\{\mathcal{I}_Y - \mu_{n,j}^*\}\right].
\]
Evaluating at $\beta=\beta_0$ and $\beta=\beta_n$ and subtracting the two identities yields
\[
\frac{1}{J}\sum_{j=1}^J P_0\!\left[\tfrac{1}{\pi_{n,j}}\,\{\mu_{n,j}(\beta_0)-\mu_{n,j}(\beta_n)\}\{\mu_{n,j}(\beta_0)-\mu_{n,j}^*\}\right]
= \frac{1}{J}\sum_{j=1}^J (P_0 - P_{n,j})\!\left[\tfrac{1}{\pi_{n,j}}\,\{\mu_{n,j}(\beta_0)-\mu_{n,j}(\beta_n)\}\{\mathcal{I}_Y - \mu_{n,j}^*\}\right].
\]
Since, by definition, $\mu_{n,j}(\beta_n)=\mu_{n,j}^*$, this simplifies to
\[
\frac{1}{J}\sum_{j=1}^J P_0\!\left[\tfrac{1}{\pi_{n,j}}\,\{\mu_{n,j}(\beta_0)-\mu_{n,j}^*\}^2\right]
= \frac{1}{J}\sum_{j=1}^J (P_0 - P_{n,j})\!\left[\tfrac{1}{\pi_{n,j}}\,\{\mu_{n,j}(\beta_0)-\mu_{n,j}^*\}\{\mathcal{I}_Y - \mu_{n,j}^*\}\right].
\]
Hence, using our notation and that $\tfrac{1}{\pi_{n,j}} < \eta^{-1}$ by \ref{cond::positivity}, we obtain the regret inequality:
$$\|\mu_{n,\diamond}^*  - \mu_{n,\diamond}(\beta_0) \|_{\overline{P}_0}^2 \lesssim  \frac{1}{J}\sum_{j=1}^J (P_0 - P_{n,j})\!\left[\tfrac{1}{\pi_{n,j}}\,\{\mu_{n,j}(\beta_0)-\mu_{n,j}^*\}\{\mathcal{I}_Y - \mu_{n,j}^*\}\right].$$
Since $\|\beta_n\|_{\infty} < M$ almost surely, both $\mu_{n,j}^*$ and $\mu_{n,j}(\beta_0)$ belong to the bounded function class:
\[
\mathcal{H}_{n,j}^k := \bigl\{\, \mu_{n,j} + \beta^\top \widehat{\varphi}_{k,j} : \beta \in \mathbb{R}^{2k}, \|\beta\|_{\infty} < M \,\bigr\}.
\]
Note that, since by Condition~\ref{cond::boundPos} both $\mu_{n,j}$ and $\widehat{\varphi}_{k,j}$ are almost surely uniformly bounded in $k$ and $j$, it follows that every element of $\mathcal{H}_{n,j}^k$ is uniformly bounded by a fixed constant. 

Define the random quantity $\hat\delta_n := \bigl\|\mu_{n,\diamond}^* - \mu_{n,\diamond}(\beta_0)\bigr\|_{\overline P_0},$
which we aim to bound. To obtain a rate bound, we use that the above regret inequality implies:
\[
\hat\delta_n^2
\;\le\;
\frac{1}{J}\sum_{j=1}^J  \sup_{\,g_j \in \mathcal{G}_{n,j}^k:\; \|g_j\|_{P_0} \le J C\,\hat\delta_n}
\; \left|(P_0 - P_{n,j})\, g_j\right|,
\]
where $\mathcal{G}_{n,j}^k
:= \bigl\{\, \tfrac{1}{\pi_{n,j}}(\mu_{n,j}(\beta_0)-h)\,(\mathcal{I}_Y - h) : h \in \mathcal{H}_{n,j}^k \,\bigr\},$
and \(C>0\) is a constant such that $\bigl\|\tfrac{1}{\pi_{n,j}}(\mu_{n,\diamond}(\beta_0)-\mu_{n,\diamond}^*)\,(\mathcal{I}_Y - \mu_{n,\diamond}^*)\bigr\|_{\overline P_0}
\;\le\; C\,\bigl\|\mu_{n,\diamond}(\beta_0)-\mu_{n,\diamond}^*\bigr\|_{\overline P_0},$ which is finite by Condition~\ref{cond::boundPos}.

To obtain an in-probability bound on $\hat \delta_n$, we follow a standard ERM analysis that leverages maximal inequalities for empirical processes based on uniform entropy integrals. Conditional on the training data $\mathcal{D}_n \setminus \mathcal{D}_{n,j}$, the function class $\mathcal{G}_{n,j}^k$ has a uniform entropy integral satisfying $\mathcal{J}(\delta,\mathcal{G}_{n,j}^k) \lesssim \mathcal{J}(\delta,\mathcal{H}_{n,j}^k)$ by preservation of uniform entropy integrals under Lipschitz continuous transformations \citep{vanderVaartWellner}. Moreover, $\mathcal{H}_{n,j}^k$ is a bounded subset of a fixed affine space of (at most) dimension $2k$ and therefore has VC–subgraph dimension bounded by $O(k)$ \citep{vanderVaartWellner}. By standard bounds on uniform entropy integrals for classes with finite VC–subgraph dimension \citep{vanderVaartWellner}, we have $\mathcal{J}(\delta,\mathcal{H}_{n,j}^k) \;\lesssim\; \delta \sqrt{k \log\!\bigl(1/\delta\bigr)},$
and, thus, $\mathcal{J}(\delta,\mathcal{G}_{n,j}^k) \;\lesssim\; \delta \sqrt{k \log\!\bigl(1/\delta\bigr)}.$
Applying the local maximal inequality of \cite{vanderVaart2011local} (Theorem 2.1), we  have, for any fixed $\delta > \sqrt{k \log n / n}$, that
$$E_n^0\left[\frac{1}{J}\sum_{j=1}^J  \sup_{\,g_j \in \mathcal{G}_{n,j}^k:\; \|g_j\|_{P_0} \le J C\,\hat\delta}
\; \left|(P_0 - P_{n,j})\, g_j\right| \right] \lesssim \delta \sqrt{k \log n /n}. $$

Informally, if the above bound held with the random radius $\hat\delta_n$, then by Markov's inequality our regret inequality would yield
\[
\hat\delta_n^2 \;\lesssim\; \mathcal{O}_p\!\left(\hat\delta_n \sqrt{\frac{k \log n}{n}}\right),
\]
from which it follows that $\hat\delta_n = \mathcal{O}_p\!\left(\sqrt{\frac{k \log n}{n}}\right)$. While the bound does not hold directly for the random $\hat\delta_n$, a standard peeling argument---see the proof of Theorem~\ref{theorem::EpLearnerRate} and Section 3.2 in \cite{vanderVaartWellner}---establishes that this rate is indeed correct.

Hence, we conclude that
$$\bigl\|\mu_{n,\diamond}^* - \mu_{n,\diamond}(\beta_0)\bigr\|_{\overline P_0} =  \mathcal{O}_p\!\left(\sqrt{\frac{k \log n}{n}}\right).$$
Moreover, by the triangle inequality, we have that
\begin{align*}
    \bigl\|\mu_{n,\diamond}^* -\mu_0\bigr\|_{\overline P_0} &\leq \bigl\|\mu_{n,\diamond}(\beta_0) - \mu_0\bigr\|_{\overline P_0} + \bigl\|\mu_{n,\diamond}^* - \mu_{n,\diamond}(\beta_0)\bigr\|_{\overline P_0} \\
    &\leq   \bigl\|\mu_{n,\diamond}(\beta_0) - \mu_0\bigr\|_{\overline P_0} + \mathcal{O}_p\!\left(  \sqrt{\frac{k \log n}{n}}\right).
\end{align*}
Moreover, since $\mu_{n,\diamond}(\beta_0)$ minimizes the distance $\|\mu_{\diamond} - \mu_0\|_{\overline{P}_0}$ over $\mu_{\diamond} \in \mathcal{H}_{n, \diamond}^k$ and the initial estimate $\mu_{n,\diamond}$ is an element of $\mathcal{H}_{n, \diamond}^k$ (take $\beta = 0$), we have that
$\bigl\|\mu_{n,\diamond}(\beta_0) - \mu_0\bigr\|_{\overline P_0}  \leq \|\mu_{n,\diamond} - \mu_0\|_{\overline{P}_0}.$
Hence
$$\bigl\|\mu_{n,\diamond}^* -\mu_0\bigr\|_{\overline P_0} \leq  \|\mu_{n,\diamond} - \mu_0\|_{\overline{P}_0} + \mathcal{O}_p\!\left(  \sqrt{\frac{k \log n}{n}}\right),$$
where $k$ may grow with $n$.

\end{proof}

\end{document}